\documentclass[twoside,11pt]{article}

\usepackage{jmlr2e_arxiv}

\usepackage{chngcntr} 

\newtheorem{observation}{Observation}

\newcounter{axiomCounter}
\renewcommand{\theaxiomCounter}{\Roman{axiomCounter}}

\newenvironment{axiomCustom}[1]
  {\refstepcounter{axiomCounter} 
  \par\addvspace{\topsep}%
  \noindent\textbf{Axiom~\theaxiomCounter{ } (#1)} 
  \itshape\ignorespaces\noindent}
  {\par\addvspace{\topsep}%
  \ignorespacesafterend}

\usepackage{graphicx}
\usepackage{caption}
\usepackage{subcaption}

\usepackage{wrapfig} 
\usepackage{float}	  
\usepackage{calc} %
\usepackage{afterpage} 

\usepackage{fontawesome} 
\usepackage[nolist,smaller]{acronym}

\usepackage{algorithm}
\usepackage[noend]{algpseudocode}

\usepackage{cancel} 
\usepackage{bbm, bm} 
\usepackage{IEEEtrantools}

\usepackage[dvipsnames]{xcolor}
\usepackage{enumitem}
\usepackage{marvosym} 

\definecolor{navy}{RGB}{0, 0, 128}
\definecolor{pythonPurple}{RGB}{148,0,211}
\definecolor{PiColor}{HTML}{5D3A9B } 
\definecolor{darkgreen}{rgb}{0,0.3,0}
\definecolor{pythonOrange}{RGB}{255,165,0}

\definecolor{CBred}{RGB}{213, 94, 0}
\definecolor{CBblue}{RGB}{0,114,178}

\definecolor{VIColor}{HTML}{D55E00}
\definecolor{GVIColor}{HTML}{0072B2}
\definecolor{GVIColor1}{HTML}{0072B2}
\definecolor{GVIColor2}{HTML}{56B4E9}
\definecolor{FVIColor}{HTML}{009E73}
\definecolor{DVIColor}{HTML}{009E73}

\newcommand{\DVIColor}[1]{\textcolor{DVIColor}{#1}}%
\newcommand{\GVIColor}[1]{\textcolor{GVIColor1}{#1}}%
\newcommand{\VIColor}[1]{\textcolor{VIColor}{#1}}%

\usepackage{tikz} 
\usetikzlibrary{arrows,shapes}
\usepackage{bigints} 

\usepackage[utf8]{inputenc} 
\usepackage[T1]{fontenc}    
\usepackage{hyperref}       
\usepackage{url}            
\usepackage{booktabs}       
\usepackage{nicefrac}       
\usepackage{microtype}      
\usepackage{xspace}

\usepackage{xcolor} 
\usepackage[utf8]{inputenc}
\usepackage[english]{babel}
\usepackage{multicol}
\usepackage{fontawesome} 

\newcommand{\acro}[1]{\textsc{#1}\xspace}

\newcommand{\GP}{{\acro{\smaller GP}}}
\newcommand{\DGP}{\acro{\smaller DGP}}
\newcommand{\DGPs}{\acro{\smaller DGPs}}
\newcommand{\K}{\acro{\smaller K}}

\newcommand{\GVI}{\acro{\smaller GVI}}
\newcommand{\ELBO}{\acro{\smaller ELBO}}
\newcommand{\KLD}{\acro{\smaller KLD}}

\newcommand{\DVI}{\acro{\smaller DVI}}

\newcommand{\RoT}{\acro{\smaller RoT}}
\newcommand{\BMM}{\acro{\smaller BMM}}

\newcommand{\Prior}{\acro{\smaller P}}
\newcommand{\Likelihood}{\acro{\smaller L}}
\newcommand{\Computation}{\acro{\smaller C}}

\newcommand{\PriorM}{\Lightning\acro{\smaller P}}
\newcommand{\LikelihoodM}{\Lightning\acro{\smaller L}}
\newcommand{\ComputationM}{\Lightning\acro{\smaller C}}
\newcommand{\PriorS}{\faLightbulbO\acro{\smaller P}}
\newcommand{\LikelihoodS}{\faLightbulbO\acro{\smaller L}}
\newcommand{\ComputationS}{\faLightbulbO\acro{\smaller C}}

\newcommand{\renyi}{R\'enyi }

\newcommand{\AD}{\acro{\smaller $D_{A}^{(\alpha)}$}}
\newcommand{\BD}{\acro{\smaller $D_{B}^{(\beta)}$}}
\newcommand{\GD}{\acro{\smaller $D_{G}^{(\gamma)}$}}
\newcommand{\ABGD}{\acro{\smaller $D_{G}^{(\alpha,\beta,r)}$}}
\newcommand{\ABGDReparam}{\acro{\smaller $\tilde{D}_{G}^{(\alpha,\beta)}$}}
\newcommand{\RAD}{\acro{\smaller $D_{AR}^{(\alpha)}$}}
\newcommand{\wKLD}{\acro{\smaller$\frac{1}{w}$KLD}}

\newcommand{\ADqp}{\acro{{\smaller $D_{A}^{(\alpha)}$}$(q||\pi)$}}
\newcommand{\BDqp}{\acro{{\smaller $D_{B}^{(\beta)}$}$(q||\pi)$}}
\newcommand{\GDqp}{\acro{{\smaller $D_{G}^{(\gamma)}$}$(q||\pi)$}}
\newcommand{\RADqp}{\acro{{\smaller $D_{AR}^{(\alpha)}$}$(q||\pi)$}}

\newcommand{\VI}{\acro{\smaller VI}}

\newcommand{\EP}{\acro{\smaller EP}}
\newcommand{\BNN}{\acro{\smaller BNN}}
\newcommand{\BNNs}{\acro{\smaller BNNs}}
\newcommand{\RMSE}{\acro{\smaller RMSE}}
\newcommand{\NLL}{\acro{\smaller NLL}}
\newcommand{\MVN}{\acro{\smaller MVN}}
\newcommand{\SPD}{\acro{\smaller SPD}}

\newcommand{\VAE}{\acro{\smaller VAE}}

\newcommand{\BBGVI}{\acro{\smaller BBGVI}}

\newcommand{\BOCPD}{\acro{\smaller BOCPD}}

\newtheorem{condition}{Condition}

\newcommand{\Lb}{\acro{\smaller $\mathcal{L}_{p}^{\beta}$}}
\newcommand{\Lg}{\acro{\smaller $\mathcal{L}_{p}^{\gamma}$}}

\newcommand{\bv}{Best viewed in color. }

\DeclareMathOperator*{\argmin}{arg\,min}

\newcommand{\EPSRC}{\acro{\smaller EPSRC}}
\newcommand{\OxWaSP}{\acro{\smaller OxWaSP}}

\def\*#1{\bm{#1}} 
\makeatletter
\newcommand{\myitem}[1]{%
\item[#1]\protected@edef\@currentlabel{#1}%
}
\makeatother

\jmlrheading{?}{2019}{??-??}{??/??}{??/??}{Jeremias Knoblauch, Jack Jewson and Theo Damoulas}

\ShortHeadings{Generalized Variational Inference}{Knoblauch, Jewson and Damoulas}
\firstpageno{1}

\begin{document}

\title{Generalized Variational Inference: \\ Three arguments for deriving new Posteriors}

\author{\name Jeremias Knoblauch \email j.knoblauch@warwick.ac.uk \\
       \addr   The Alan Turing Institute\\
        Dept. of Statistics\\
        University of Warwick\\
        Coventry, CV4 7AL, UK \\
       \AND
       \name Jack Jewson \email j.e.jewson@warwick.ac.uk \\
       \addr   The Alan Turing Institute\\
        Dept. of Statistics\\
        University of Warwick\\
        Coventry, CV4 7AL, UK \\
        \AND
        \name   Theodoros Damoulas \email t.damoulas@warwick.ac.uk \\
        \addr The Alan Turing Institute\\
        Depts. of Computer Science \& Statistics\\
        University of Warwick\\
        Coventry, CV4 7AL, UK
        }

\editor{Leslie Pack Kaelbling}

\maketitle

\begin{abstract}
 %
 
 In this paper we advocate 
 an optimization-centric view on Bayesian statistics and introduce a novel generalization of Bayesian inference.
 %
 On both counts, our  inspiration is the representation of Bayes' rule as an infinite-dimensional optimization problem as shown independently by \citet{Csiszar, DonskerVaradhan, Zellner}.
 First, we use this representation to prove a surprising
 optimality result of standard Variational Inference (\VI) methods: Under the proposed view, the standard Evidence Lower Bound (\ELBO) maximizing \VI posterior is always preferable to alternative approximations of the Bayesian posterior. 
 %
 %
 Next, we argue 
 for an optimization-centric generalization of standard Bayesian inference.
 The need for this generalization arises in situations of severe misalignment
 between reality and three assumptions underlying the standard Bayesian posterior: (1) Well-specified priors, (2) well-specified likelihood models and (3) the availability of infinite computing power.
 %
 In response to this observation, our generalization is defined by three arguments and named the Rule of Three (\RoT).
 Each of its three arguments relaxes one of the assumptions underlying standard Bayesian inference. 
 We axiomatically derive the \RoT and recover existing methods as special cases, including the Bayesian posterior and its approximation by standard Variational Inference (\VI). 
 In contrast, alternative approximations to the Bayesian posterior maximizing other \ELBO-like objectives violate these axioms.
 %
 %
 Finally, we introduce a special case of the \RoT that we call
 Generalized Variational Inference (\GVI).
 \GVI posteriors are a large and tractable family of belief distributions specified by three arguments:
 %
 A loss, a divergence  and a variational family.
 \GVI posteriors possess appealing theoretical properties, including consistency and an interpretation as an approximate \ELBO. 
 The last part of the paper explores some attractive applications of \GVI in popular machine learning models, including robustness and more appropriate marginals.
 After deriving black box inference schemes for \GVI posteriors, their predictive performance is investigated on Bayesian Neural Networks and Deep Gaussian Processes, where \GVI can comprehensively improve upon existing methods.
\end{abstract}


\section{Introduction}

Though famously first discovered in \citet{BayesThm}, the version of Bayes' Theorem that a modern audience would be familiar with is much closer to the one in \citet{LaplaceBayes}.
%
%
%
Bayes' Theorem (or Bayes' rule) is one of the most fundamental results in probability theory and states that for a probability measure $\mathbb{P}$ and two events $A$, $B$, it holds that
\begin{IEEEeqnarray}{rCl}
    \mathbb{P}\left(A|B\right) & = & \dfrac{\mathbb{P}\left(B|A \right)\mathbb{P}\left(A\right)}{\mathbb{P}\left(B\right)}.
    \nonumber
\end{IEEEeqnarray}
As usual, $\mathbb{P}\left(A|B\right)$ denotes the conditional probability of event $A$ given that event $B$ occured.
It would take nearly two more centuries for this mathematical result to be used as the basis for an entire school of statistical inference \citep{WhenDidBayesBecomeBayes}. 
More precisely, \citet{WhereBayesOccurs} makes the first mention of the term \textit{Bayesian} in our modern understanding \citep{FirstOccurenceBayes}.

Bayesian statistics uses Bayes' Theorem to conduct inference on an unknown and unobservable event $A$. 
Specifically, suppose that one can compute for an observable event $B$ the probability $\mathbb{P}(B|A)$ and has a prior belief $\mathbb{P}(A)$ about the event $A$ before observing $B$. 
In this situation, Bayes' rule tells us that we should be able to draw probabilistic inferences on $A|B$ by computing the probability $\mathbb{P}(A|B)$. 
In practice, the events $A$ quantify the uncertainty about a parameter of interest $\*\theta \in \*\Theta$ and so are of the form $A \subset \*\Theta$. 
The prior beliefs about events $A$ are usually specified by some probability {density} $\pi:\*\Theta \to \mathbb{R}_+$ inducing the probability measure $\mathbb{P}(A) = \int_{A}d\pi(\*\theta)$.
This leaves us with the need to specify a probability distribution $\mathbb{P}\left(B|A \right)$ that relates the (unobserved) parameter $\*\theta$ to the (observable) event $B$. 
In practice, one typically sets $B = x_{1:n}$ to correspond to $n$ observations $x_{1:n}$. 
The next step is to define a distribution of $B|A$. This amounts to positing a likelihood function $p_n(x_{1:n}|\*\theta)$ and setting $\mathbb{P}(B|A) = p_n(x_{1:n}|\*\theta)$. 
Put together, this yields the standard \textit{Bayesian posterior} that we denote as $q^{\ast}_{\text{B}}(\*\theta)$ throughout the paper and which is given by
\begin{IEEEeqnarray}{rCl}
    q^{\ast}_{\text{B}}(\*\theta) & = & 
    \dfrac{p_n(x_{1:n}|\*\theta)\pi(\*\theta)}{Z}. 
    \nonumber
\end{IEEEeqnarray}
Here, $Z = \int_{\*\Theta}p_n(x_{1:n}|\*\theta)d\pi(\*\theta)$ is the normalizing constant---also known as partition function---whose computation generally makes the Bayesian posterior intractable.
Bayesian inference is appealing both conceptually and practically: Unlike frequentist inference, Bayesian methods allow inferences to be informed by domain expertise in form of a carefully specified prior belief $\pi(\*\theta)$.
Further, Bayesian inference produces belief distributions (rather than point estimates) over the parameter of interest $\*\theta \in \*\Theta$ that best fits the observed data $x_{1:n}$ while taking into account a prior belief $\pi(\*\theta)$ about appropriate values of $\*\theta$. 
%
As a consequence, Bayesian inferences automatically quantify uncertainty about $\*\theta$. This is practically useful in many situations, but especially if one uses $\*\theta$ predictively: Integrating over $q^{\ast}_{\text{B}}(\*\theta)$ avoids being over-confident about the best value of $\*\theta$, substantially improving predictive performance \citep[see e.g.][]{PredictionBayes}.
Amongst other benefits, it is this enhanced predictive performance that has cast Bayesian inference as one of the predominant paradigms in contemporary large-scale statistical inference and machine learning.

While Bayesian methods automatically quantify the uncertainty about their inferences, this comes at a cost: 
In the translation of Bayes' rule into the Bayesian posterior $q^{\ast}_{\text{B}}(\*\theta)$, we have made three implicit but crucial assumptions.
Firstly, we have assumed that the modeller has a prior belief $\pi(\*\theta)$ which is worth being taken into account and which the modeller is capable of writing out mathematically. 
Secondly, we specified the likelihood function $p_n(x_{1:n}|\*\theta)$ as a conditional probability. In other words, we have assumed that the model is correctly specified, which is to say that $p_n(x_{1:n}|\*\theta^{\ast}) = d\mathbb{P}(x_{1:n})$ for some unknown value of $\*\theta^{\ast} \in \*\Theta$.
Thirdly, we have assumed the availability of enough computational power to compute and perform exact inference based upon the generally intractable posterior $q^{\ast}_{\text{B}}(\*\theta)$.  
In many situations, these three assumptions built into $q^{\ast}_{\text{B}}(\*\theta)$ are harmless.
For modern large-scale statistical machine learning tasks however, they are frequently violated. 




\begin{figure}[b!]
        \vskip 0.1in
        \begin{center}
            \centerline{\includegraphics[trim= {0cm 0.0cm 0cm 0.0cm}, clip, 
            width=1.00\columnwidth]{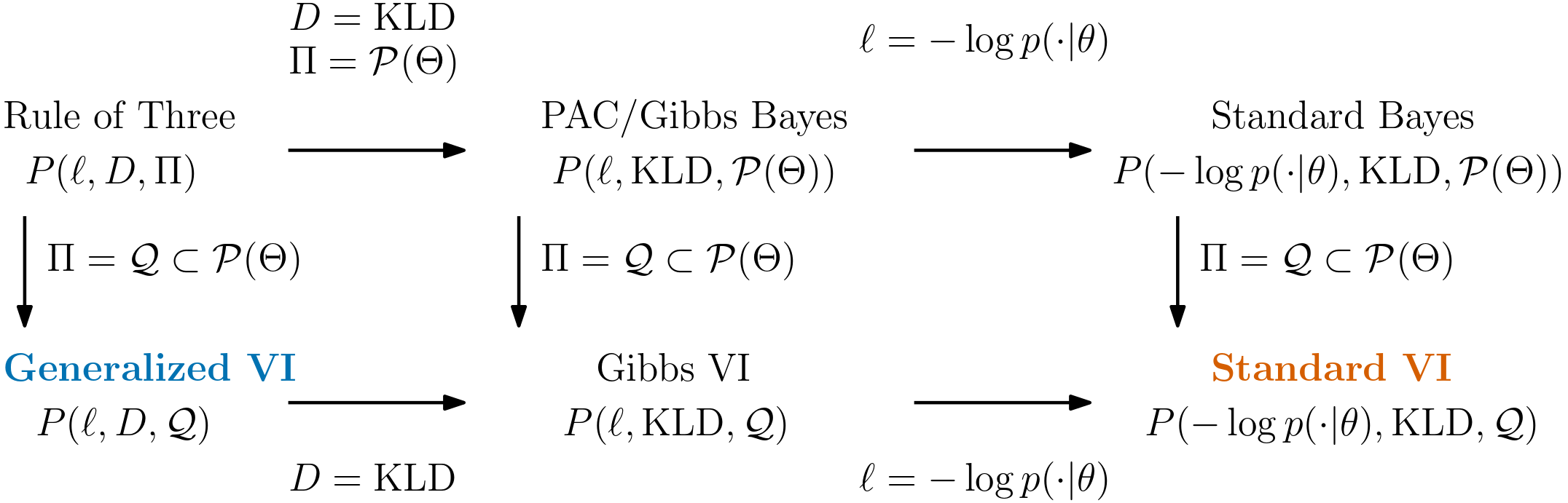}}
            \caption{
            A taxonomy of some important belief distributions as special cases of the \RoT. 
            }
            \label{Fig:taxonomy}
        \end{center}
        \vskip -0.2in
\end{figure}

To address this, the current paper takes a step back from Bayes' Theorem and the standard Bayesian posterior $q^{\ast}_{\text{B}}(\*\theta)$ to define a generalized class of posterior belief distributions.
Throughout, we motivate this with the tension between the three main assumptions underlying standard Bayesian inference on the one hand and the requirement of many contemporary statistical applications on the other hand.
To resolve this tension, we define a generalization of Bayesian inference that we call the Rule of Three (\RoT). 
The \RoT is specified by an optimization problem over the space of probability measures $\mathcal{P}(\*\Theta)$ on $\*\Theta$ with three arguments. 
These arguments are a loss function $\ell$, a divergence $D$ measuring the deviation of the posterior from the prior and a space $\Pi \subseteq \mathcal{P}(\*\Theta)$ of feasible solutions. 
Together, these three ingredients define posterior beliefs of the form 
\begin{IEEEeqnarray}{rCl}
    q^{\ast}(\*\theta) 
    & = & 
    \argmin_{q \in \Pi}\left\{
        \mathbb{E}_{q(\*\theta)}\left[ \sum_{i=1}^n \ell(\*\theta, x_i) \right]
        +
        D(q\|\pi)
    \right\}
    \overset{\text{def}}{=}
    P(\ell, D, \Pi).
    \label{eq:intro:RoT}
\end{IEEEeqnarray}
This recovers previous generalizations of Bayesian inference, including those inspired by Gibbs posteriors \citep[e.g.][]{GoshBasuPseudoPosterior, Bissiri, Jewson, RobustBayesGamma, MMDBayes}, tempered posteriors \citep[e.g.][]{SafeLearning,SafeBayesian,holmes2017assigning,InconsistencyBayesInference,DunsonCoarsening}, as well as PAC-Bayesian approaches \citep[for a recent overview, see][]{PACPrimer}.
we illustrate this taxonomy in Figure \ref{Fig:taxonomy}.
Unlike any of these previous generalizations however, posteriors taking the form $P(\ell, D, \Pi)$ may be non-multiplicative.
One of the most important implications of this is that in contrast to previous generalizations, the \RoT can recover standard 
Variational Inference (\VI) posteriors based on minimizing the Kullback-Leibler Divergence (\KLD) to $q^{\ast}_{\text{B}}(\*\theta)$. 
%
%
Notably, this is true even though standard \VI is derived as an approximation to the Bayesian posterior $q^{\ast}_{\text{B}}(\*\theta)$. 
Even more remarkably, variational approximations to the Bayesian posterior that are constructed by minimizing  divergences other than the \KLD are \textit{not} recovered by the \RoT.
%
%
%
%
This inspires us to define and investigate Generalized Variational Inference (\GVI), the tractable special case for the \RoT in which $\Pi = \mathcal{Q} = \{q(\*\theta|\*\kappa): \*\kappa \in \*K\} \subset \mathcal{P}(\*\Theta)$ is chosen to be a variational family. 
Various theoretical and empirical findings lead us to conclude that \GVI posteriors are well-suited to real world inference problems and are an exciting first step on the way to derive generalized and tractable posterior belief distributions.
%
The paper draws these conclusions in five steps.
\begin{itemize}
    \myitem{}\textbf{Section \ref{sec:VI_and_Bayes_posteriors}}: 
        We recapitulate the standard approach to Bayesian inference and various variational approximation schemes for $q^{\ast}_{\text{B}}(\*\theta)$.
        Unconventionally, we do so through the lens of infinite-dimensional optimization. 
        This view provides a number of interesting insights: 
        For example, it enables a natural breakdown of variational approximation methods. 
        Further, it reveals that relative to the infinite-dimensional optimization problem whose solution is the Bayesian posterior, the standard Variational Inference (\VI) posterior is the optimal solution in its finite-dimensional variational family.
        Perhaps surprisingly, this also implies that for any fixed variational family, alternative approximations in the same variational family are sub-optimal.
        %
        %
    \myitem{}\textbf{Section \ref{sec:Examining_Bayesian_paradigms}}:
        We explain why a generalized view on Bayesian inference is necessary. 
        To this end, we first give a brief overview over the three assumptions that justify Bayesian inference:
        %
        The availability of both an appropriately specified prior belief and likelihood as well as sufficient computational power to address the intractability of $q^{\ast}_{\text{B}}(\*\theta)$.
        We then proceed to contrast these three assumptions with the realities of modern day large-scale statistical inference and use three examples to explain the real world problems arising from this misalignment between assumptions and reality.
    \myitem{}\textbf{Section \ref{sec:RoT}}:
        We derive a generalized representation of Bayesian inference that we call the Rule of Three (\RoT) based on three simple axioms.
        The \RoT is inspired by our optimality finding regarding standard \VI in Section \ref{sec:VI_and_Bayes_posteriors}. Thus, unlike previous generalizations it defines an \textit{optimization-centric} outlook on Bayesian inference.
        We discuss the \RoT and explain how it can address the adverse effects of violating the assumptions underlying standard Bayesian inference.
        Further, we connect the \RoT to existing Bayesian methods, the information bottleneck method and PAC-Bayesian approaches.
    \myitem{}\textbf{Section \ref{sec:GVI}}:
        Translating the conceptual contribution of the \RoT into a methodological one, we introduce Generalized Variational Inference (\GVI).
        We explain how to use \GVI for robust inference and more appropriate marginal variances. We also point to some theoretical findings, including frequentist consistency and an interpretation of \GVI as approximate evidence lower bound.
        Lastly, we discuss computation of \GVI posteriors. While special cases permit closed form objectives, one generally needs to rely on stochastic Black Box \GVI (\BBGVI).
    \myitem{}\textbf{Section \ref{sec:experiments}}:
        We reinforce the conceptual and methodological appeal of \GVI with two large-scale inference applications: Bayesian Neural Networks (\BNNs) and Deep Gaussian Processes (\DGPs). 
        In different ways, both model classes are representative for the different ways in which contemporary large-scale inference is often misaligned with the assumptions underlying the standard Bayesian posterior. 
        We show that appropriately addressing this misalignment dramatically improves performance.
\end{itemize}
Throughout, we radically simplify the presentation for improved readability: 
For example, we do not incorporate latent variables into our notation in spite of demonstrating \GVI on a Deep Gaussian Process (\DGP) latent variable model in Section \ref{sec:experiments_GPs}. 
Further, we assume that losses are additive, homogeneous and such that the $i$-th loss term $\ell(\*\theta, x_i)$ only depends on $x_i$. None of these assumptions are necessary, and one could replace $\ell(\*\theta, x_i)$ by $\ell_i(\*\theta, x_{1:i})$ or $\ell_i(\*\theta, x_{\text{nbh}(i)})$ for some neighbourhood $\text{nbh}(i) \subset \{1, 2, \dots n\}$ throughout the paper without violating the principles of the \RoT or \GVI\footnote{
    For the interested reader, we note that the appropriate notational extension of \GVI to latent variables and non-homogeneous losses is formalized in \citet{GVIConsistency}, see e.g. Assumption 1 and Remarks 1,2 and 3 therein.
}.

\section{An optimization-centric view on Bayesian inference}
\label{sec:VI_and_Bayes_posteriors}

Before presenting our main findings, we set the stage by introducing an optimization-centric view on (generalized) Bayesian inference.
Specifically, we draw attention to an isomorphism between the Bayesian posterior and an infinite-dimensional optimization problem and discuss three implications of this relationship.
%
%
\begin{itemize}
    \item[]
    \textbf{Section \ref{sec:objective_commitment}:}
    Committing to any exact Bayesian posterior is equivalent to committing to a particular optimization problem over the space of probability measures
    \item[]
    \textbf{Section \ref{sec:VI_opt}:}
    Taking an optimization-centric view of Bayesian inference and
    holding the variational family fixed, standard Variational Inference (\VI) produces \textit{optimal} approximations of the exact Bayesian posterior. 
    \item[]
    \textbf{Section \ref{sec:reconciling_suboptimality}:}
    non-standard \VI methods based on alternative divergences are \textit{suboptimal} approximations of the exact Bayesian posterior.
\end{itemize}
%

\subsection{Preliminaries}

%
Given a prior belief $\pi(\*\theta)$ about the parameter and observations $x_{1:n}$ linked to $\*\theta$ via a likelihood function $p(x_i|\*\theta)$, the \textbf{standard Bayesian posterior} belief $q_{\text{B}}^{\ast}(\*\theta)$ is computed through a multiplicative updating rule  with $\ell(\*\theta, x_i) = -\log p(x_i|\*\theta)$ as
\begin{IEEEeqnarray}{rCl}
    q_{\text{B}}^{\ast}(\*\theta) & \propto & \pi(\*\theta)\prod_{i=1}^n\exp\{-\ell(\*\theta, x_i)\}.
    \label{eq:generalized_bayes_posterior}
\end{IEEEeqnarray}
While this way of writing Bayes rule might seem cumbersome, it reveals that the multiplicative structure is in principle applicable to {any} loss function. 
This leads to the development of a \textbf{generalized Bayesian posterior}  by replacing the negative log likelihood with any loss $\ell: \*\Theta \times \mathcal{X} \to \mathbb{R}$.
If the normalizer of eq. \eqref{eq:generalized_bayes_posterior} exists, such treatment provides a coherent and principled way to update beliefs about an arbitrary parameter $\*\theta$ \citep{Bissiri}.

To make this generalization tangible, imagine that $\*\theta$ denotes the median for the data generating mechanism that produced $x_{1:n}$ and that one wishes to update beliefs about it in a Bayesian manner.
A loss-based Bayesian treatment of this problem would combine a prior belief $\pi$ about the median with the loss $\ell(\*\theta, x_i) = |\*\theta - x_i|_1$. 
Together, these two ingredients yield a generalized Bayesian posterior belief about the median as given above. For some applications with $\ell(\*\theta, x_i) \neq -\log p(x_i|\*\theta)$, see \citet{GoshBasuPseudoPosterior, VBConsistencyAlquier2, InconsistencyBayesInference, Jewson, RBOCPD, MMDBayes} or \citet{RobustBayesGamma}.

Throughout this paper, we do not notationally distinguish standard and generalized Bayesian posteriors.
%
%
Unless we make the distinction explicit, we subsume both types of belief distributions under the name \textbf{Bayesian posterior}
and denote any posterior belief computed as in eq. \eqref{eq:generalized_bayes_posterior} by $q_{\text{B}}^{\ast}(\*\theta)$.
%
%
%
%
The asterisk superscript in $q_{\text{B}}^{\ast}(\*\theta)$ emphasizes an observation we make next and that will be a recurrent theme throughout the paper: \textit{Any} posterior belief distribution is the result of an appropriately specified optimization problem.

\subsection{Bayesian inference as infinite-dimensional optimization}
\label{sec:objective_commitment}

The traditional perspective on Bayesian posteriors derives from the basic laws of probability and in particular Bayes' Theorem:
If $p(x_i|\*\theta)$ denotes the conditional density of $x_i$ given $\*\theta$ and $\pi(\*\theta)$ denotes the  density of $\*\theta$, the conditional probability of $\*\theta$ given $x_{1:n}$ is given by $q^{\ast}_{\text{B}}(\*\theta)$ in eq. \eqref{eq:generalized_bayes_posterior} with $\ell(\*\theta, x_i) = -\log p(x_i|\*\theta)$. 
%
%
This multiplicative update rule is motivated slightly differently for generalized Bayesian posteriors, but the inherent logic largely remains the same: By imposing coherence, one forces the priors and losses into an exponentially additive relationship (see also Section \ref{sec:coherence}). One interpretation of this is that one  treats the loss terms $\exp\{-\ell(\*\theta, x_i \}$ as quasi-likelihoods, rendering the resulting posteriors at very strongly inspired by conditionalization and the fundamental rules of probability that underlie Bayes' rule\footnote{Note that the literature on PAC-Bayesian procedures can provide different justifications for multiplicative update rules \citep[see e.g.][]{PACmeetsBayesianInference, PACPrimer}.}.

While Bayes' rule and eq. \eqref{eq:generalized_bayes_posterior} are well-known, there is a conceptually rather different path for arriving at $q_{\text{B}}^{\ast}(\*\theta)$:
Dating back at least to \citet{Csiszar} and \citet{DonskerVaradhan}, it was shown that Bayesian inference can  be recast as the solution to an infinite-dimensional optimization problem. This result was rediscovered in statistics by \citet{Zellner} and states that for $\mathcal{P}(\*\Theta)$ denoting the space of all probability measures on $\*\Theta$, the Bayesian posterior is given by
%
%
%
%
\begin{IEEEeqnarray}{rCl}
    q^{\ast}_{\text{B}}(\*\theta) & = & 
    \argmin_{q \in \mathcal{P}(\*\Theta)}\left\{
        \mathbb{E}_{q(\*\theta)}\left[
        -\sum_{i=1}^n
            \log(p(x_i|\*\theta))
        \right]
        +
        \KLD\left(q || \pi \right)
    \right\},   \nonumber \label{eq:Zellner_standard}
\end{IEEEeqnarray}
where \KLD is the Kullback-Leibler divergence \citep{kullback1951information} given by
\begin{IEEEeqnarray}{rCl}
    \KLD(q||\pi) & = &  \mathbb{E}_{q(\*\theta)}\left[\log\left(\dfrac{q(\*\theta)}{\pi(\*\theta)}\right)\right] = \mathbb{E}_{q(\*\theta)}\left[\log q(\*\theta)\right] - \mathbb{E}_{q(\*\theta)}\left[\log \pi(\*\theta)\right].
    \nonumber
\end{IEEEeqnarray}
Similarly, the generalized Bayesian posteriors of \citet{Bissiri} 
solve the optimization problem given by 
\begin{IEEEeqnarray}{rCl}
    q^{\ast}_{\text{B}}(\*\theta) & = & 
    \argmin_{q \in \mathcal{P}(\*\Theta)}\left\{
        \mathbb{E}_{q(\*\theta)}\left[
        \sum_{i=1}^n
            \ell(\*\theta, x_i)
        \right]
        +
        \KLD\left(q || \pi \right)
    \right\}.   \label{eq:Zellner}
\end{IEEEeqnarray}
This objective allows a re-interpretation of Bayesian inference as regularized optimization: As in maximum likelihood inference or other empirical risk minimization tasks, one wishes to minimize some loss function over the data. 
Unlike with frequentist methods however, one wishes to quantify uncertainty and obtain a belief distribution rather than a point estimate. Consequently, one adds the \KLD regularization term. 
Note that if this \KLD term were absent from eq. \eqref{eq:Zellner}, the solution of the optimization problem would simply be a Dirac mass $\delta_{\widehat{\*\theta}_n}(\*\theta)$ at the empirical risk minimizer $\widehat{\*\theta}_n$, since $\delta_{\widehat{\*\theta}_n}(\*\theta) \in \mathcal{P}(\*\Theta)$. 
%
%

%
For completeness' sake, we now provide a self-contained proof of eq. \eqref{eq:Zellner} which is based on the one in the supplementary material of \citet{Bissiri}. 
This encompasses the original result of \citep{Csiszar} and \citep{DonskerVaradhan} for $\ell(\*\theta, x_i) = -\log p(x_i|\*\theta)$.
\begin{theorem}
    If $Z = \int_{\*\Theta} \exp\left\{-\sum_{i=1}^n\ell(\*\theta, x_i)\right\} \pi(\*\theta)d\*\theta<\infty$, then the solution of eq. \eqref{eq:Zellner} exists and is equivalent to the generalized Bayesian posterior $q^{\ast}_{\text{B}}(\*\theta)$ as given in eq. \eqref{eq:generalized_bayes_posterior}.
    \label{thm:Zellner}
\end{theorem}
\begin{proof}
One may rewrite the objective in eq. \eqref{eq:Zellner} as 
\begin{IEEEeqnarray}{rCl}
    q^{\ast}(\*\theta) & = &
    \argmin_{q \in \mathcal{P}(\*\Theta)}\left\{
        \int_{\*\Theta} 
            \left[\log\left( \exp\left\{ 
                \sum_{i=1}^n\ell(\*\theta, x_i)
            \right\} \right) + 
            \log\left(\dfrac{q(\*\theta)}{\pi(\*\theta)}\right)
            \right]q(\*\theta)d\*\theta
    \right\} \nonumber \\
    & = &
    \argmin_{q \in \mathcal{P}(\*\Theta)}\left\{
        \int_{\*\Theta} 
            \log\left( \dfrac{q(\*\theta)}{
            \pi(\*\theta)\exp\left\{ 
                -\sum_{i=1}^n\ell(\*\theta, x_i)
            \right\}}\right)q(\*\theta)
            d\*\theta
    \right\}.
    \nonumber
\end{IEEEeqnarray}
As one only cares about the minimizer $q^{\ast}(\*\theta)$ (and not the objective value), it also holds that for any constant $Z>0$, the above is equal to
\begin{IEEEeqnarray}{rCl}
    q^{\ast}(\*\theta) & = &
    \argmin_{q \in \mathcal{P}(\*\Theta)}\left\{
        \int_{\*\Theta} 
            \log\left( \dfrac{q(\*\theta)}{
            \pi(\*\theta)\exp\left\{ 
                -\sum_{i=1}^n\ell(\*\theta, x_i)
            \right\}Z^{-1}}\right)q(\*\theta)d\*\theta - \log Z
    \right\} \nonumber \\
    &=& 
    \argmin_{q \in \mathcal{P}(\*\Theta)}\left\{
        \KLD\left(q(\*\theta)\Big\|\pi(\*\theta)\exp\left\{ 
                -\sum_{i=1}^n\ell(\*\theta, x_i)
            \right\}Z^{-1}\right)
    \right\}.
    \nonumber
\end{IEEEeqnarray}
Lastly, one sets $Z = \int_{\*\theta} \exp\left\{-\sum_{i=1}^n\ell(\*\theta, x_i)\right\} \pi(\*\theta)d\*\theta$ and notes that as the $\KLD$ is a statistical divergence, it is minimized uniquely if its two arguments are the same, so $q^{\ast}(\*\theta) = q^{\ast}_B(\*\theta)$.
\end{proof}
The result in Theorem \ref{thm:Zellner} implies an important isomorphism that drives much of the current paper's development: Any commitment to a (standard or generalized) Bayesian posterior is always a commitment to an optimization objective.
In other words, the Bayesian posterior for a given inference problem is adequate if and only if the objective in eq. \eqref{eq:Zellner} is.
%

\begin{observation}[Isomorphism]
    Suppose an agent A wishes to conduct inference based upon the Bayesian posterior $q^{\ast}_{\text{B}}(\*\theta)$ in eq. \eqref{eq:generalized_bayes_posterior}.
    Then agent A conducts inference by solving an  optimization problem, namely the one in eq. \eqref{eq:Zellner}.
    \label{observation:isomorphism}
\end{observation}
The above observation implies that computing the Bayesian posterior $q^{\ast}_{\text{B}}(\*\theta)$ is equivalent to assuming that the objective of eq. \eqref{eq:Zellner} is appropriate for producing belief distributions.
%
In Section \ref{sec:Examining_Bayesian_paradigms}, we will explain why the usefulness of the standard Bayesian posterior---and thus of the objective in eq. \eqref{eq:Zellner}---is at least doubtful for many contemporary statistical machine learning problems.
%
%


\begin{remark}
    The sceptical reader may notice that given the Bayesian posterior,   eq. \eqref{eq:Zellner} is not the unique solution-inducing problem.
    Specifically, suppose that $D$ is a statistical divergence. Then $D(q\|p) \geq 0$ and $D(q\|p) = 0$ if and only if $q(\*\theta)=p(\*\theta)$ almost everywhere. 
    Hence, one could object that in fact any optimization problem of the form
    \begin{IEEEeqnarray}{rCl}
        q^{\ast}(\*\theta) & = & \argmin_{q\in\mathcal{P}(\*\Theta)}D(q\|q_{\text{B}}^{\ast}) \label{eq:generic_FVI}
    \end{IEEEeqnarray}
    will be solved setting $q^{\ast}(\*\theta) = q_{\text{B}}^{\ast}(\*\theta)$.
    While true, this is tautological: In particular, such optimization problems shed no light on how $q_{\text{B}}^{\ast}(\*\theta)$ was arrived at. 
    By their very construction, problems of this form pre-suppose that $q_{\text{B}}^{\ast}(\*\theta)$ is the posterior belief one wishes to obtain.

    Thus, while there exist infinitely many optimization problems whose solution is $q^{\ast}_{\text{B}}(\*\theta)$, some are more meaningful than others. 
    Specifically, whenever one seeks to solve an objective of the form given in eq. \eqref{eq:generic_FVI}, the Bayesian posterior appears deus ex machina.
    This does not allow us any interpretation about what $q^{\ast}(\*\theta)$ itself stands for and why it is a desirable belief distribution to target. 
    In contrast, eq. \eqref{eq:Zellner} shows that the Bayesian posterior arises as the solution of a clearly interpretable optimization problem.
    %
    %
\end{remark}




\subsection{Optimality of standard Variational Inference}
\label{sec:VI_opt}

While Bayesian posteriors of the form given in eq. \eqref{eq:generalized_bayes_posterior} are analytically available up to a normalizing constant, this is not immediately useful.
Specifically, (asymptotically) exact computations of expectations and integrals with respect to these posteriors are in general only possible through sampling methods and incur a large computational burden.
To alleviate this problem, numerous approximations to the exact Bayesian posterior have been proposed.
The principal idea of any such approximation is to force the posterior belief into some parametric form.
%
Specifically, one seeks to approximate $q^{\ast}_{\text{B}}(\*\theta) \approx q^{\ast}_{\text{A}}(\*\theta)$, where $q^{\ast}_{\text{A}}(\*\theta) \in \mathcal{Q}$ and 
\begin{IEEEeqnarray}{rCl}
    \mathcal{Q} & = & \left\{ q(\*\theta|\*\kappa): \*\kappa \in \*K \right\}
    \label{eq:variational_family}
\end{IEEEeqnarray}
denotes a family of distributions parameterized by $\*\kappa$. 
It is obvious that this significantly reduces the computational burden, as it transforms the optimization from an infinite-dimensional into a finite-dimensional one. 
%

The literature on approximations of this sort is large and has diverse origins. 
Their development arguably started with Laplace Approximations \citep[see e.g. the seminal papers of][]{Laplace1, Laplace3, Laplace2}, which have recently been refined substantially into Integrated Nested Laplace Approximations \citep{INLA}.  
A second family of approximation methods known as Expectation Propagation \citep{EP2,EP} was motivated through factor graphs and message passing \citep{MinkaDivergences}.
The third and arguably most successful approach originated by connecting the Expectation-Maximization algorithm \citep{EM} and the variational free energy from physics \citep{EMFreeEnergy},  culminating in Variational Inference (\VI) \citep{JordanVI, BealThesis}.
For these methods, $\mathcal{Q}$ in eq. \eqref{eq:variational_family} is called the variational family.

Traditionally, two main interpretations of \VI prevail. Firstly, one can derive its objective function as an Evidence Lower Bound (\ELBO). Secondly, one can show that the same objective function minimizes the \KLD between $\mathcal{Q}$ and $q^{\ast}_{\text{B}}(\*\theta)$.
In this paper, we introduce and advocate a third--to the best of our knowledge novel---view on \VI: Relative to the objective in eq. \eqref{eq:Zellner}, \VI corresponds to the best $\mathcal{Q}$-constrained solution.

\begin{figure}	
	\centering
	\begin{subfigure}[t]{0.42\columnwidth}
		\centering
		\includegraphics[width=1\columnwidth]{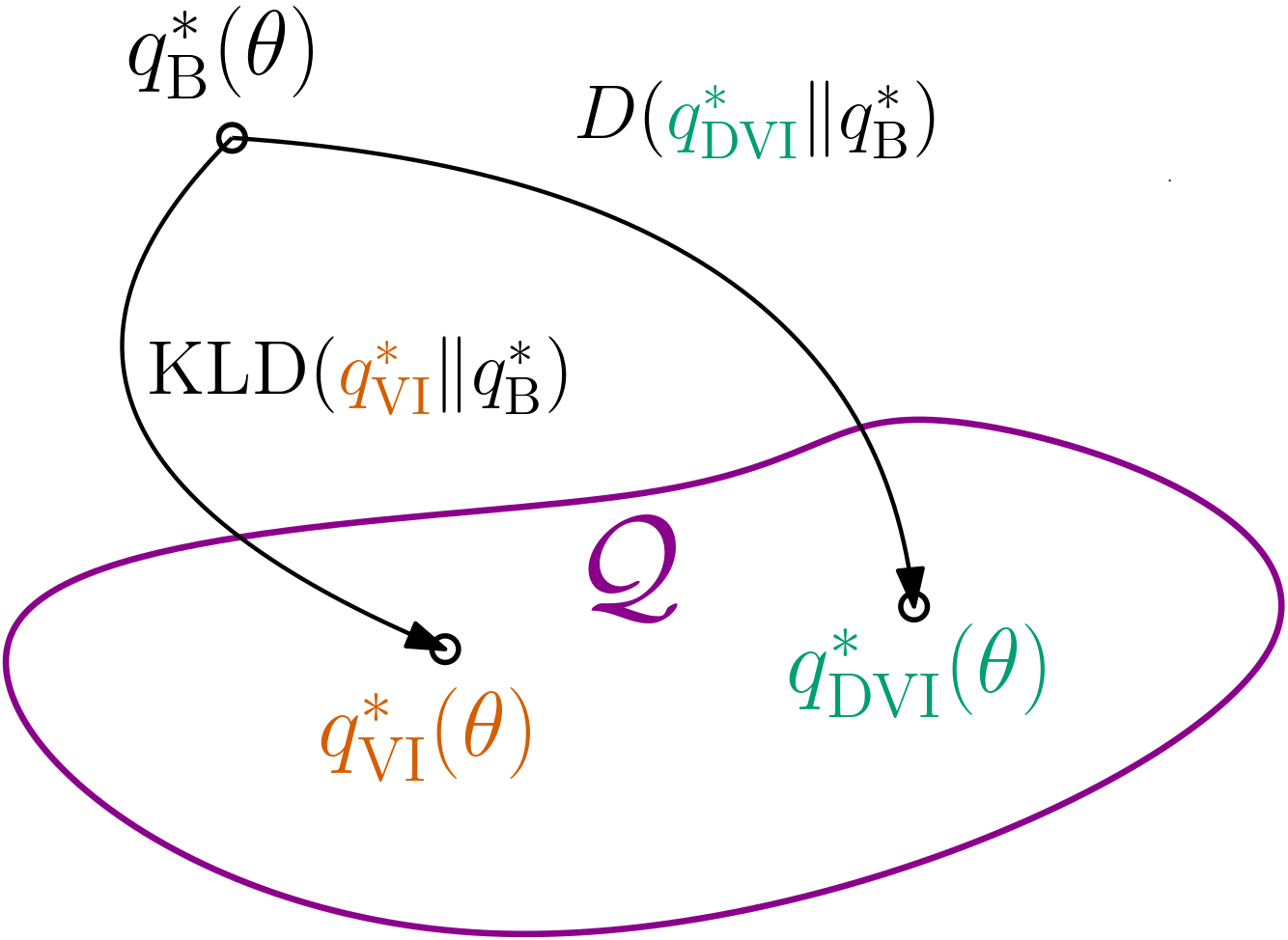}
		\hspace*{0.05cm}
		\caption{\DVIColor{\textbf{\DVI}} interpretation of \VIColor{\textbf{\VI}}}
		\label{SubFig:DVI_projection}		
	\end{subfigure}
	\quad
	\begin{subfigure}[t]{0.53\columnwidth}
		\centering
		\includegraphics[width=1\columnwidth]{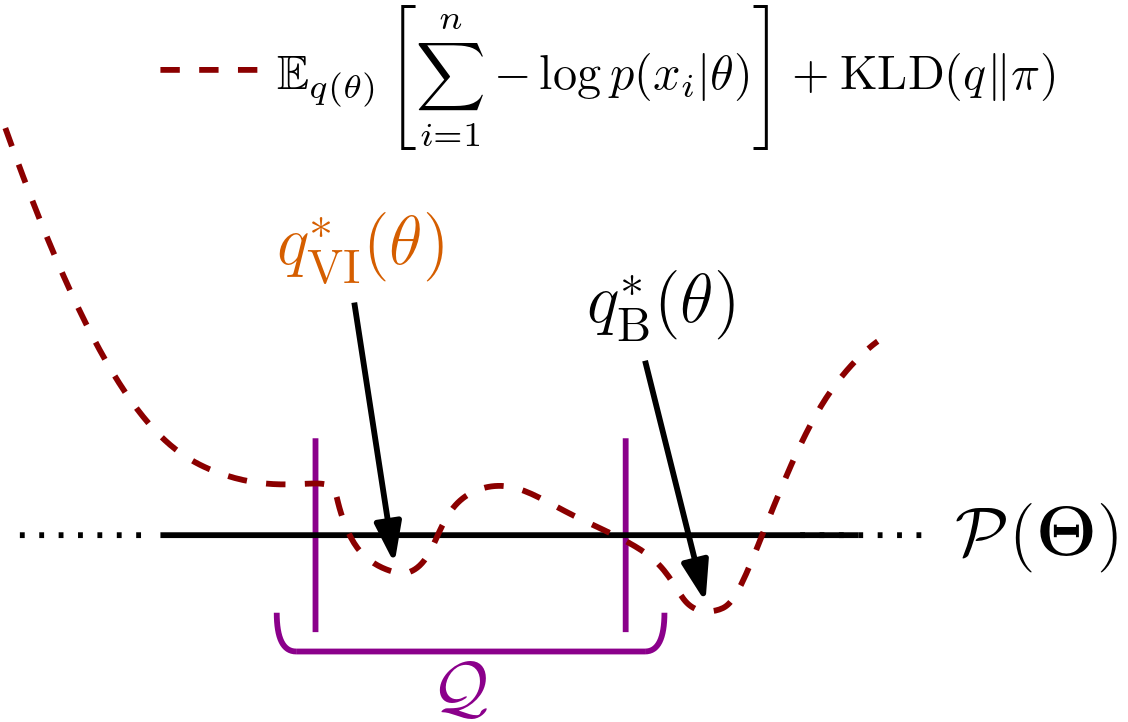}
		\hspace*{0.05cm}
		\caption{Interpretation of \VIColor{\textbf{\VI}} as in Theorem \ref{thm:VI_optimality}}
		\label{SubFig:GVI_projection}		
	\end{subfigure}
	\caption{
	   \bv
	    Depicted is a schematic to clarify the conceptual distinction between two interpretations of \VIColor{\textbf{\VI}}.
	    \DVIColor{\textbf{\DVI}} methods interpret \VIColor{\textbf{\VI}} as the \KLD-projection of $q^{\ast}_{\text{B}}(\*\theta)$ into the variational family $\color{pythonPurple}{\mathcal{Q}}$. New methods are then derived by replacing the \KLD with alternative projection operators.
	    In contrast, Theorem \ref{thm:VI_optimality} interprets \VIColor{\textbf{\VI}} posteriors as best solutions to a constrained optimization problem. 
	    Rather than finding the global optimum $q^{\ast}_{\text{B}}(\*\theta)$ of the optimization problem in eq. \eqref{eq:Zellner_standard}, \VIColor{\textbf{\VI}} finds the best solution in the subset ${\color{pythonPurple}{\mathcal{Q}}}\subset \mathcal{P}(\*\Theta)$.
	    This optimization-centric view on variational methods is also the logic underlying {\textbf{\GVI}} posteriors.
	}
	\label{Fig:DVI_vs_GVI_interpretation}
\end{figure}
    


\subsubsection{\VI as log evidence bound}
\label{sec:VI_as_ELBO}

A standard way of deriving \VI is the following: Since one wishes to obtain the posterior maximizing the evidence in the data, a reasonable objective is to pick the element in the approximating family $\mathcal{Q}$ that maximizes the evidence.
This logic is operationalized by observing that for $Z$ the normalizing constant (or partition function) and any $q(\*\theta) \in \mathcal{Q}$,
\begin{IEEEeqnarray}{rCl}
    \log p(x_{1:n}) 
    & = & 
    \log \left( \int_{\*\Theta}Z^{-1}\exp\{-\sum_{i=1}^n\ell(\*\theta, x_i)\}\pi(\*\theta)d\*\theta \right)
    \nonumber \\
    &= &
    \log \left( \int_{\*\Theta}\exp\{-\sum_{i=1}^n\ell(\*\theta, x_i)\}\dfrac{\pi(\*\theta)}{q(\*\theta)}q(\*\theta)d\*\theta \right) - \log Z
    \nonumber \\
    &\overset{JIE}{\geq} &
    \int_{\*\Theta}\log \left( \exp\left\{-\sum_{i=1}^n\ell(\*\theta, x_i)\right\}\dfrac{\pi(\*\theta)}{q(\*\theta)}\right)q(\*\theta)d\*\theta 
    - \log Z \label{Equ:EvidenceLowerBound}
\end{IEEEeqnarray}
where we have applied Jensen's Inequality in the last step.
The right hand side of eq. \eqref{Equ:EvidenceLowerBound} is called the Evidence Lower Bound (\ELBO), and
maximizing it over $\mathcal{Q}$ is independent of $Z$. %
With this, one finally obtains the standard \VI posterior as solution to an optimization problem given by
\begin{IEEEeqnarray}{rCl}
	q_{\VI}^{\ast}(\*\theta) & = & \argmin_{q \in \mathcal{Q}}
	\left\{ 
		\mathbb{E}_{q(\*\theta)}\left[
        \sum_{i=1}^n
            \ell(\*\theta, x_i)
        \right]
        +
        \KLD\left(q || \pi \right) 
	\right\}.
	\label{eq:standard_VI_ELBO}
\end{IEEEeqnarray} 
Here, $q_{\VI}^{\ast}(\*\theta) = q(\*\theta|\*\kappa^{\ast})$ for some optimal parameter $\*\kappa^{\ast} \in \*K$.

%
%
Taking inspiration from the Evidence Lower Bound interpretation of the \VI objective, alternative approximation methods produce \VI posteriors maximizing generalized Evidence Lower Bounds
\citep[e.g.][]{GLBO, IWVI, IWVAE}. For some bound $\log p(x_{1:n}) \leq \text{G-}\ELBO(q)$ on the evidence, such methods produce posteriors as %
%
%
\begin{IEEEeqnarray}{rCl}
    q^{\ast}_{\text{G}-\ELBO}(\*\theta) & = &
    \argmin_{q\in\mathcal{Q}}\left\{
        \text{G-}\ELBO(q)
    \right\}.
    \nonumber
\end{IEEEeqnarray}
%
%

%
%


\subsubsection{\VI as \KLD-minimization \& Discrepancy \VI (\DVIColor{\textbf{\DVI}})}
\label{sec:VI_as_KLD_min_and_DVI}

A second well-known perspective on standard \VI posteriors is motivated by rewriting the objective in eq. \eqref{eq:standard_VI_ELBO} in terms of the Kullback-Leibler Divergence (\KLD) as follows:
%
%
\begin{IEEEeqnarray}{rCl}
     q_{\VI}^{\ast}(\*\theta) & = & 
	\argmin_{q \in \mathcal{Q}}
	\left\{ 
	    \KLD\left(q(\*\theta)\Big\|q^{\ast}_{\text{B}}(\*\theta)\right)
    \right\}
    \nonumber
\end{IEEEeqnarray}
The relevant algebraic arguments are similar to the ones used in the proof of Theorem \ref{thm:Zellner}. In effect, the resulting re-arrangement of terms above shows that standard \VI finds the element $q_{\VI}^{\ast}(\*\theta) \in \mathcal{Q}$ closest to the Bayesian posterior belief in the \KLD-sense.

This insight has produced a growing body of literature seeking to minimize (local or global) discrepancies $D$ between $\mathcal{Q}$ and $q^{\ast}_{\text{B}}(\*\theta)$ that are different from the \KLD \citep[e.g.][]{EP, EP2, RenyiDiv, GeometricRenyiAlphaVI, ChiDiv, AlphaDiv, FisherFVI, ABCdiv, OperatorVI, fDivVI}. 
For a disrepancy measure $D: \mathcal{P}(\*\Theta) \times \mathcal{P}(\*\Theta) \to \mathbb{R}$, methods of this kind compute approximations to $q^{\ast}_{\text{B}}(\*\theta)$ based on objectives of the form
\begin{IEEEeqnarray}{rCl}
    q^{\ast}_{\DVI}(\*\theta) & = &
    \argmin_{q \in \mathcal{Q}}
	\left\{ 
	    D\left(q(\*\theta)\Big\|q^{\ast}_{\text{B}}(\*\theta)\right)
    \right\}.
    \nonumber
\end{IEEEeqnarray}
Throughout the remainder of this paper, we call procedures of this nature with $D\neq \KLD$ \DVIColor{\textbf{Discrepancy Variational Inference (\DVI)}} methods. Their logic and intuitive appeal is summarized in Figure \ref{SubFig:DVI_projection}.

\subsubsection{\VI as constrained optimization}
\label{sec:VI_constrained_opt}


While the interpretations of \VI as optimizing over an evidence lower bound and as minimizing a discrepancy are well-known, this paper presents a third interpretation: \VI posteriors are also the $\mathcal{Q}$-optimal solutions to the optimization problem in eq. \eqref{eq:Zellner} characterizing Bayesian inference.
To spot this, simply compare eq. \eqref{eq:Zellner} to eq. \eqref{eq:standard_VI_ELBO} and notice that the optimization problem only differs through the space over which optimization is performed: As opposed to the set of all probability measures $\mathcal{P}(\*\Theta)$ as in eq. \eqref{eq:Zellner}, eq. \eqref{eq:standard_VI_ELBO} only considers the parameterized subset $\mathcal{Q}$. This observation is rather significant and bears important implications summarized in the following Theorem and Figure \ref{SubFig:GVI_projection}.
%
%
%
\begin{theorem}[Optimality of standard \VI]
    Relative to the infinite-dimensional optimization problem over $\mathcal{P}(\*\Theta)$ characterizing Bayesian inference and a fixed finite-dimensional variational family $\mathcal{Q}$, standard \VI produces the optimal posterior belief in $\mathcal{Q}$.
    \label{thm:VI_optimality}
\end{theorem}
\begin{proof}
    %
    First, notice that the \VI posterior belief distribution $q^{\ast}_{\VI}(\*\theta)$ and
    the Bayesian posterior belief distribution $q^{\ast}_{\text{B}}(\*\theta)$ both seek to minimize 
    \begin{IEEEeqnarray}{rCl}
        \mathbb{E}_{q(\*\theta)}\left[\sum_{i=1}^n\ell(\*\theta, x_i)\right] 
        +
        \KLD(q\|\pi) \nonumber
    \end{IEEEeqnarray}
    over $q(\*\theta)$.
    Second, notice that $q^{\ast}_{\VI}(\*\theta)$ is the minimizer of this objective relative to $\mathcal{Q}$ while $q^{\ast}_{\text{B}}(\*\theta)$ is the minimizer relative to $\mathcal{P}(\*\Theta)$.
    Third, note that $\mathcal{Q} \subset \mathcal{P}(\*\Theta)$.
    %
\end{proof}

\begin{remark}
    The result of Theorem \ref{thm:VI_optimality} is important: As Observation \ref{observation:isomorphism} explained, committing to a Bayesian posterior $q^{\ast}_{\text{B}}(\*\theta)$ is equivalent to committing to the objective function in eq. \eqref{eq:Zellner}. 
    In other words, if we judge the posterior belief $q^{\ast}_{\text{B}}(\*\theta)$ to be desirable, we are also saying that the objective function in eq. \eqref{eq:Zellner} encodes properties that we want our posterior belief to adhere to. 
    Accordingly, once we restrict our posterior beliefs to be in a subset $\mathcal{Q} \subset \mathcal{P}(\*\Theta)$, we should want to compute the best possible solution to the \textit{same} objective in $\mathcal{Q}$.
    As Theorem \ref{thm:VI_optimality} shows, this is exactly what \VI does.
    Importantly, this has two implications: Firstly, it gives another meaningful interpretation to standard \VI approximations to $q^{\ast}_{\text{B}}(\*\theta)$. Secondly, it provides important insights into the suboptimality of alternative approximation methods summarized in the following Corollary.
\end{remark}

\begin{corollary}[Suboptimality of alternative methods]
    Relative to the infinite-dimensional optimization problem over $\mathcal{P}(\*\Theta)$ characterizing Bayesian inference and a fixed finite-dimensional variational family $\mathcal{Q}$, methods different from standard \VI  produce sub-optimal posterior beliefs.
    \label{corollary:suboptimality}
\end{corollary}
\begin{proof}
    We prove this by contradiction: Suppose the posterior belief $q^{\ast}_{\text{A}}(\*\theta)$ was produced by some alternative method $A$ that is not equivalent to standard \VI. Suppose also that $q^{\ast}_{\text{A}}(\*\theta)$ is the $\mathcal{Q}$-optimal posterior relative to  eq. \eqref{eq:Zellner}. 
    By definition of standard \VI, it then holds that that for \textit{any} sequence of observations $x_{1:n}$ and for all $n$,
    \begin{IEEEeqnarray}{rCl}
        	\mathbb{E}_{q^{\ast}_{\VI}(\*\theta)}\left[
            \sum_{i=1}^n
                \ell(\*\theta, x_i)
            \right]
            +
            \KLD\left(q^{\ast}_{\VI} || \pi \right) 
            & \leq &
            \mathbb{E}_{q^{\ast}_{\text{A}}(\*\theta)}\left[
            \sum_{i=1}^n
                \ell(\*\theta, x_i)
            \right]
            +
            \KLD\left(q^{\ast}_{\text{A}} || \pi \right). 
            \nonumber
    \end{IEEEeqnarray}
    Suppose the inequality is strict. This immediately yields a contradiction with the supposition that $q^{\ast}_{\text{A}}(\*\theta)$ is the $\mathcal{Q}$-optimal posterior relative to  eq. \eqref{eq:Zellner}. Alternatively, suppose the inequality is an equality for {any} sequence of observations $x_{1:n}$ and for all $n$. This too immediately yields a contradiction: In particular, it violates the supposition that the method producing $q^{\ast}_{\text{A}}(\*\theta)$ is not equivalent to standard \VI. 
    Thus, the desired contradiction is obtained.
\end{proof}

\begin{remark}
    Corollary \ref{corollary:suboptimality}  implies that once the variational family $\mathcal{Q}$ is fixed, producing an approximation to $q^{\ast}_{\text{B}}(\*\theta)$ which does not correspond to standard \VI posteriors is sub-optimal under an optimization-centric view on Bayesian inference. 
    %
    %
    This notion of optimality is important because Theorem \ref{thm:Zellner} and Observation \ref{observation:isomorphism} expose the isomorphic relationship between the optimization problem in eq. \eqref{eq:Zellner} and the exact Bayesian posterior.
    An immediate consequence of Corollary \ref{corollary:suboptimality} is thus that  methods built around generalized evidence lower bound formulations, alternative Discrepancy Variational Inference (\DVI) methods or Expectation Propagation (\EP) approaches \citep[e.g.][]{EP, EP2} are sub-optimal relative to standard \VI: 
    If one wishes to minimize the objective that defines the Bayesian posterior $q^{\ast}_{\text{B}}(\*\theta)$, it is irrational to pick any $q \in \mathcal{Q}$ not produced by standard \VI.
    \label{remark:suboptimality}
\end{remark}

\subsection{Reconciling (sub)optimality with empirical evidence}
\label{sec:reconciling_suboptimality}


At first glance, the conclusions from Corollary \ref{corollary:suboptimality} and Remark \ref{corollary:suboptimality} seem to contradict numerous landmark findings in the area of approximate Bayesian inference:
Firstly, there are various issues with standard \VI that are well-known and hinder its effectiveness in certain situations \citep[see e.g.][]{TurnerSahani11}.
For this reason, various alternative approximations have proven successful in practice \citep[e.g.][]{EP, INLA} and often produce more desirable posterior inferences.
All this seems to contradict the (sub)optimality results in Theorem \ref{thm:VI_optimality} and Corollary \ref{corollary:suboptimality}. 

%
This contradiction resolves itself upon closer examination.
Specifically, the practical relevance of any optimality result hinges on two crucial assumptions that are typically violated in practice: Firstly, one needs to assume that the original objective in eq. \eqref{eq:Zellner} is appropriately specified. 
Secondly, one needs the variational family $\mathcal{Q}$ to be rich enough so that the statement $q^{\ast}_{\VI}(\*\theta) \approx q^{\ast}_{\text{B}}(\*\theta)$ is not completely vacuous.
%
%
Conversely, this means that $q_{\text{A}}^{\ast}(\*\theta)$ can produce more desirable approximations to $q^{\ast}_{\text{B}}(\*\theta)$ than $q^{\ast}_{\VI}(\*\theta)$ whenever one of the following holds:
\begin{itemize}
    \myitem{(i)} the original objective in eq. \eqref{eq:Zellner} is misspecified and does not reflect the belief distribution we wish to compute;
    \label{item:VI-failure:misspec}
    \myitem{(ii)} The approximating family $\mathcal{Q}$ is inappropriately specified so that the statement $q(\*\theta) \approx q^{\ast}_{\text{B}}(\*\theta)$ is vacuous for any $q \in \mathcal{Q}$.
    \label{item:VI-failure:Q}
\end{itemize} 
%
%
Under \ref{item:VI-failure:misspec}, $q_{\text{A}}^{\ast}(\*\theta)$ will outperform $q^{\ast}_{\VI}(\*\theta)$ whenever its objective {implicitly} encodes desirable properties for the posterior belief distribution that are not part of the objective in eq. \eqref{eq:Zellner}.
Similarly, if $q_{\text{A}}^{\ast}(\*\theta)$ is designed to accommodate specific choices of $\mathcal{Q}$ that $q^{\ast}_{\VI}(\*\theta)$ struggles with, it will perform well under \ref{item:VI-failure:Q}.

For example, virtually all posteriors produced within the \DVI family \citep[e.g.][]{RenyiDiv, AlphaDiv, ChiDiv, SABDiv} are designed to address (ii): In particular, these methods prevent unimodal approximations from focusing too strongly around the empirical risk minimizer of $\*\theta$. 
For standard \VI, this phenomenon is common whenever $\mathcal{Q}$  is the mean field variational family, which explains why \DVI often empirically outperforms standard \VI for this popular choice of $\mathcal{Q}$.
%
%
%
%
Taking the optimization-centric view on posterior beliefs, this implies that in spite of being sub-optimal relative to eq. \eqref{eq:Zellner}, \DVI methods pose objectives that are often better-suited to produce belief distributions in $\mathcal{Q}$.
This raises an interesting question: 
Rather than thinking of inference in a subset $\mathcal{Q}\subset \mathcal{P}(\*\Theta)$ as approximate, can we adapt a radical optimization-centric view and \textit{directly} design appropriately specified objectives to generate posterior beliefs with desirable properties?
The remainder of this paper gives an affirmative answer to this question in the form of Generalized Variational Inference (\GVI).

Taking inspiration from \ref{item:VI-failure:misspec} and \ref{item:VI-failure:Q}, the next section first takes a step back and explores the conditions under which such alternative \GVI posteriors could be desirable.
As we shall see, the isomorphic relationship between eq. \eqref{eq:Zellner} and the Bayesian posterior provides a comprehensive answer to this question:
%
%
Specifically, we explain the ways in which the assumptions underpinning the traditional Bayesian paradigm giving rise to the Bayesian posterior $q^{\ast}_{\text{B}}(\*\theta)$ and eq. \eqref{eq:Zellner} are often misaligned with the realities of contemporary statistical machine learning.
%
This misalignment problem has three important dimensions: The information contained in the prior belief \ref{assumption:prior}, the role of the likelihood model \ref{assumption:lklhood}, and the availability of computational resources \ref{assumption:computation}.
%

%
%
%





\section{A reality check: Re-examining the traditional Bayesian paradigm}
\label{sec:Examining_Bayesian_paradigms}
%
%
In the following section, we illuminate 
the misalignment between the assumptions underlying the traditional Bayesian paradigm and the way in which modern statistical machine learning uses (approximate) Bayesian posteriors to conduct inference. 
\begin{itemize}
    \item[] First, \textbf{Section \ref{sec:traditional_bayesian_paradigm}} recalls and elaborates on the three crucial assumptions underlying the standard Bayesian posterior: An appropriate prior \ref{assumption:prior} and likelihood \ref{assumption:lklhood} and an infinite computational budget \ref{assumption:computation}.
    \item[] Next, \textbf{Section \ref{sec:challenging_traditional_paradigm}} exposes the misalignment of these three assumptions with inferential practices in contemporary statistical machine learning and large-scale inference.
    \item[] Lastly, \textbf{Sections \ref{sec:violation_P}--\ref{sec:violation_C}} 
    illustrates the adverse real-world consequences arising from violating these assumptions.
\end{itemize}

%

\subsection{The traditional Bayesian paradigm}
\label{sec:traditional_bayesian_paradigm}

Due to their direct correspondence with the fundamental rules of probability, Bayesian posteriors $q^{\ast}_{\text{B}}(\*\theta)$ are desirable objects to be basing inference on.
%
%
To see why, suppose the following three conditions hold true. 
\begin{itemize}
\myitem{\textbf{(\Prior)}}
    The \textbf{\Prior}rior $\pi(\*\theta)$ is correctly specified: 
    It encodes the best available judgement about $\*\theta$ based on \textit{all} information available to the modeller. Crucially, the distribution $\*\pi(\*\theta)$ 
    is assumed to reflect this prior belief \textit{exactly}.
    This implies that $\pi(\*\theta)$ should \textit{completely} 
    reflect all information available to the modeller such as previously observed observations $x_{-m:0}$ of the same phenomenon or domain expertise relating to the problem domain and the statistical model.
    \label{assumption:prior}
    \myitem{\textbf{(\Likelihood)}}
        \label{assumption:lklhood}
    There exists an (unknown but fixed) $\*\theta^{\ast}$ making the \textbf{\Likelihood}ikelihood model equivalent to the data generating mechanism of $x_{i}$. This is to say that $x_i \sim p(x_i|\*\theta^{\ast})$.\footnote{
        We note here that to keep the presentation simpler, we are giving conditions that are stricter than what is required for Bayesian analysis. 
        In particular, \ref{assumption:lklhood} corresponds to an objectivist treatment of the likelihood and can be weakened under the subjectivist paradigm for Bayesian analysis. 
        In this paradigm, the treatment of the likelihood mirrors that of the prior: It now simply corresponds to the modeller's belief about the process that generated the data. 
        While this first sounds like a weaker requirement, it ends up producing the same misspecification problems as \ref{assumption:lklhood}. 
        Specifically, a subjectivist treatment of the likelihood requires the modeller to express her beliefs about the likelihood function \textit{exactly}. This forces her to make more probability statements than she realistically has time or introspection for \citep[see e.g.][]{goldstein1990influence,o2004probability, goldstein2006subjective}.
        The result is that the likelihood function supplied by the modeller is  \textit{at best} going to be an approximate description of the modeller's beliefs.
        This provides the subjectivist interpretation of misspecification. 
        Notice that it directly mirrors the objectivist interpretation of misspecification in \ref{assumption:lklhood}: The likelihood function supplied is \textit{at best} going to be an approximate description of the true data generating mechanism. 
    }
    \myitem{\textbf{(\Computation)}}
    The budget for \textbf{\Computation}omputation is infinite, so the  complexity of computing the belief $q_{\text{B}}^{\ast}(\*\theta)$ 
    can be ignored.
    \label{assumption:computation}
\end{itemize}
If \ref{assumption:lklhood}, \ref{assumption:prior} and  \ref{assumption:computation} are satisfied, it immediately follows that the best belief for the event $\{\*\theta^{\ast} = \*\theta\}|\{\*x_{1:n} = x_{1:n}\}$ is simply given by the analytically available posterior
\begin{IEEEeqnarray}{rCl}
    d\mathbb{P}\left(\*\theta|x_{1:n}\right) & \propto & 
    d\mathbb{P}\left( \*\theta \right)
    \prod_{i=1}^n d\mathbb{P}\left( x_i|\*\theta \right) = 
    \pi(\*\theta)\prod_{i=1}^n p(x_i|\*\theta) = q^{\ast}_{\text{B}}(\*\theta). 
    \label{eq:standard_bayes_rule}
\end{IEEEeqnarray}
Note that \ref{assumption:prior} and \ref{assumption:lklhood} lend a meaningful interpretation to Bayes' rule in form of conditional probability updates. 
Complementing this, \ref{assumption:computation} ensures that it is feasible to compute the generally intractable solution $q^{\ast}_{\text{B}}(\*\theta)$ of eq. \eqref{eq:Zellner_standard}.
%
%
Accordingly, \ref{assumption:computation} generally is interpreted to mean that a Markov Chain Monte Carlo algorithm can be run for long enough to accuratley represent? $q^{\ast}_{\text{B}}(\*\theta)$.
%
%
%
In summary, if \ref{assumption:prior}, \ref{assumption:lklhood} and \ref{assumption:computation} hold, $q^{\ast}_{\text{B}}(\*\theta)$ is the only desirable posterior belief distribution. 
%

%
But how well does reality align with \ref{assumption:prior}, \ref{assumption:lklhood} and \ref{assumption:computation}?
Turning attention to \ref{assumption:computation} first, most traditional scientific disciplines have little need to worry about computational complexity and will resort to sampling schemes for two reasons: Firstly, 
the models are often relatively simple and thus straightforward to infer. 
Secondly, even for more complicated models the experimental setup and data collection typically outweighs the cost of computation by orders of magnitude.
%
%
%
As for \ref{assumption:prior} and \ref{assumption:lklhood}, neither prior nor likelihood are ever perfect reflections of one's full prior beliefs \citep[see e.g.][]{goldstein1990influence,o2004probability, goldstein2006subjective} or the data generating mechanism \citep[see e.g.][]{Bernardo}.
In other words, \ref{assumption:prior} and \ref{assumption:lklhood} are invariably violated when interpreted literally.
However and as enshrined in Box's aphorism that \textit{all models are wrong, but some are useful}, this is not a problem so long as these violations are sufficiently small.
In traditional statistics, ensuring that these violations are small has typically been enforced through a simple recursion \cite[e.g.][]{box1980sampling,robustBayes}. Specifically, until you are confident that both \ref{assumption:prior} and \ref{assumption:lklhood} are close enough to the truth, repeat the following: Check if \ref{assumption:lklhood} or \ref{assumption:prior} are violated severely. If they are, choose a more appropriate likelihood and prior.
%
In order to operationalize this iterative logic, batteries of descriptive statistics, tests and model selection criteria have been developed over the years. 
%

In summary then, ignoring the computational overhead and 
iteratively refining likelihoods and priors is rightfully the predominant inferential strategy for traditional scientific endeavours. 
Not only is domain expertise relevant for designing priors and likelihood, but the process of finding an appropriate model often provides valuable insights in itself. Further, the expensive part of the analysis is typically data \textit{collection}. Consequently, it is typically not prohibitive to perform inference even with the most computationally expensive of sampling schemes.
In line with this, most methodological contributions in statistical sciences rely to a substantial degree on \ref{assumption:prior}, \ref{assumption:lklhood} and \ref{assumption:computation}.

\subsection{Machine Learning: Challenging the traditional paradigm}
\label{sec:challenging_traditional_paradigm}

Contemporary large-scale inference applications have frequently turned the traditional schematic of statistical model design upside down:
Rather than carefully designing an appropriate likelihood model $p(x_i|\*\theta)$ for a specific data domain, statistical machine learning research 
is typically characterized by the search of a flexible algorithm that can fit \textit{any} data set $x_{1:n}$ well enough to produce useful inferences.
%
%
The resulting likelihood models are typically not attempting to describe any data generating processes in the sense of \ref{assumption:lklhood}. Rather, they are highly over-parameterized functions of $\*\theta$ and typically un-identifiable, meaning that $\*\theta^{\ast}$ is neither interpretable nor unique.
%
Such statistical machine learning models have three major issues under the traditional paradigm of Bayesian inference that are readily identified:
\begin{itemize}
    \myitem{\textbf{(\PriorM)}} Invariably, the \textbf{\Prior{}}rior is misspecified. 
    Two factors compound this issue: Firstly,  the large number of parameters over-parameterizing the likelihoods of many statistical machine learning models  are no longer interpretable. 
    This often prohibits domain experts to carry out carefully guided prior elicitation. 
    Secondly, priors are typically selected at least in part for their computational feasibility. 
    This fundamentally alters the interpretation of the prior: Rather than the result of an attempt to capture the modeller's knowledge before observing the data, the prior takes the role of a reference measure or regularizer.
    To make matters worse, the number of parameters is often large relative to $n$. 
    In turn, this implies that the priors have a disproportional effect on inference via $q^{\ast}_{\text{B}}(\*\theta)$, a problem we will discuss in Example \ref{example:violation_P} in the context of Bayesian Neural Networks.
    \label{violation:prior}
    \myitem{\textbf{(\LikelihoodM)}} Clearly, the \textbf{\Likelihood{}}ikelihood is misspecified.
    This often has adverse side effects: While
    using an off-the-shelf and often over-parameterized likelihood function can provide a good fit for the typical behaviour of the data, it often causes severe problems with heterogeneous or untypical data points. 
    We will demonstrate this phenomenon on a changepoint problem in Example \ref{example:violation_L}.
    \label{violation:likelihood}
    \myitem{\textbf{(\ComputationM)}}
    %
    With increasingly complex statistical models, \ref{assumption:computation} has proven an increasingly infeasible description of reality.
    %
    Accordingly, this problem has inspired numerous directions of research, including variational methods and Laplace approximations.
    Example \ref{example:violation_C} illustrates this for the case of Gaussian Processes.
    %
    %
    \label{violation:computation}
\end{itemize}
Under the challenges outlined in \ref{violation:prior}, \ref{violation:likelihood} and \ref{violation:computation}, standard Bayesian posteriors often do not provide appropriate belief distributions.
In the remainder, we will explain how and why this is the case for many parts of modern large-scale inference.

\begin{figure}[h!]
    \begin{center}
    \includegraphics[trim= {1.15cm 3.50cm 2.50cm 0.0cm}, clip,  
    width=1\columnwidth]{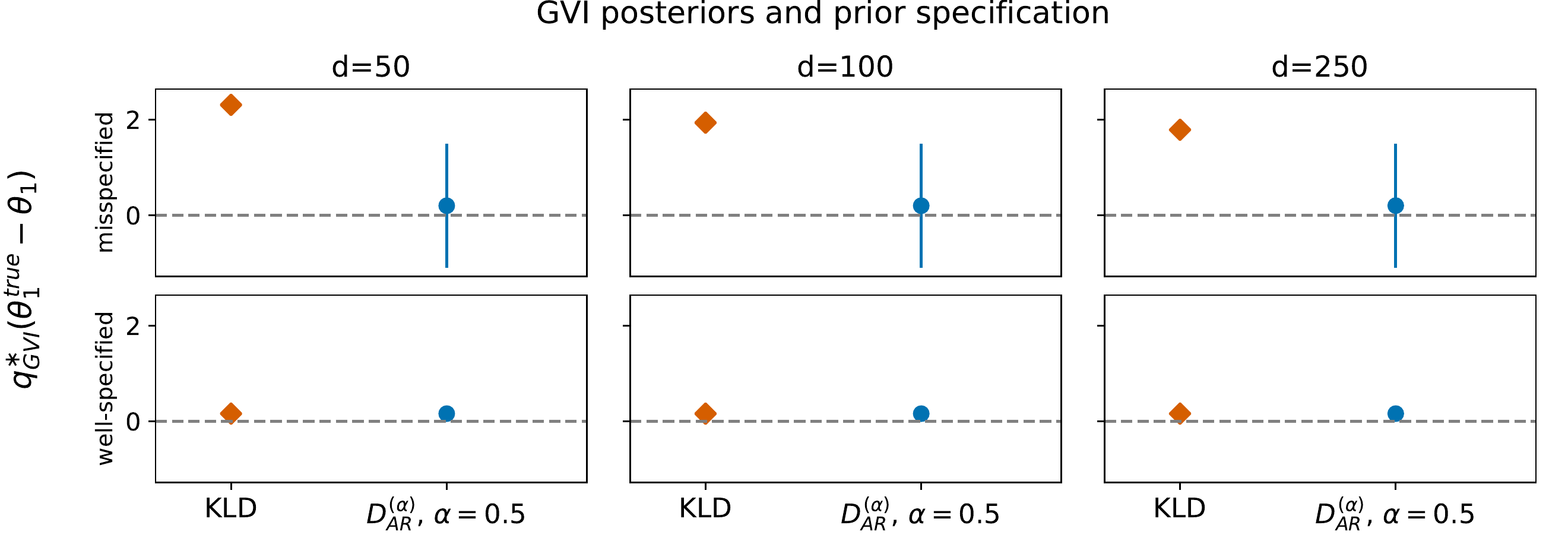}
    \caption{
        \bv
        Taken from \citet{GVIConsistency}, the plot shows the impact of different prior beliefs on inference in a Bayesian normal mixture model with $n=50$ observations and mixture components in $\mathbb{R}^d$ for different choices of $d$.
        Specifically, the plot compares inference outcomes under a misspecified prior (\textbf{Top}) against those under a well-specified prior (\textbf{Bottom}).
        It does so by depicting the average absolute difference between the true parameter values and their MAP estimate on the $y$-axis. 
        Here, the solid whiskers' length corresponds to one standard deviation of the underlying posterior. 
        %
        %
        %
        The plot shows that using the \VIColor{\textbf{\KLD}} as uncertainty quantifier as in \VIColor{\textbf{ standard Variational Inference (\VI)}} will produce undesirable uncertainty quantification under misspecified prior beliefs.  
        %
        %
        In contrast, \GVIColor{\textbf{Generalized Variational Inference (\GVI)}} with  \GVIColor{\textbf{R\'enyi's} $\*\alpha$-\textbf{divergence}} as uncertainty quantifier produces desirable uncertainty quantification in both settings.
        %
    }
    \label{Fig:prior_misspec_BMM}
    \vskip -0.25in
    \end{center}
\end{figure}

\subsection{Prior misspecification}
\label{sec:violation_P}

For most finite-dimensional parameters, even severely misspecified priors can often be harmless.
%
%
Sometimes, this expresses itself in theoretical results: For example, prior misspecification is typically no problem in the asymptotic sense.  Specifically, so long as \ref{assumption:lklhood} holds, it suffices that  $\pi(\*\theta^{\ast})>0$ for the standard Bayesian posterior to contract around $\*\theta^{\ast}$ at rate $O(n^{-1/2})$ \citep[see e.g.][and references therein]{BayesConsistencyReview, BayesianConsistencyClassic1, BayesianConsistencyClassic2, BayesianConsistencyWalker}. 

Oftentimes, these results are used as an apology to neglect the role of prior specification.
While  it is reassuring that the sequence of standard Bayesian posteriors shrinks to the population-optimum as $n \to \infty$, this does not describe the real world:
In particular, $n$ is usually fixed and only a single posterior is computed. 
Note that whenever $n$ is fixed, it is possible to specify an arbitrarily bad prior belief. 
This means that once one departs from assuming that \ref{assumption:prior} is at least approximately correct, the standard Bayesian posterior belief about $\*\theta^{\ast}$ can be made arbitrarily inappropriate---even if \ref{assumption:lklhood} still holds.
Figure \ref{Fig:prior_misspec_BMM} illustrates this on a Bayesian Mixture Model and also shows how Generalized Variational Inference (\GVI) can solve this problem.
%
As the Figure shows, finite data can make prior misspecification a more serious issue, even more so if (i) the parameter space is large relative to $n$ or (ii) it is impossible to specify priors in a principled way.
As we discuss in the next example, a model invariably affected by both problems is the Bayesian Neural Network (\BNN{}).
%

\begin{example}[Deep Bayesian models as violations of \ref{assumption:prior}]

    Bayesian Neural Networks (\BNN{}s) \citep[][]{BNNs2, BNNs} seek to combine Deep Learning models with Bayesian uncertainty quantification.
    %
    %
    For the parameter vector $\*\theta$ of weights, let $F(\*\theta)$ be the non-linear composition of activation functions specified by a Neural Network. 
    A conceptually appealing way of thinking about \BNN{}s is as an arbitrarily flexible likelihood function with a large number of parameters $d = |\*\Theta|$.
    This is to say that one believes that (at least approximately), $x_i \sim p(x_i|F(\*\theta^{\ast}))$ for some $\*\theta^{\ast} \in \*\Theta$.
    For a prior $\pi(\*\theta)$ about $\*\theta$, this means that \BNN{}s seek to do inference on the posterior given by
    \begin{IEEEeqnarray}{rCl}
        q^{\ast}(\*\theta) & \propto & \pi(\*\theta)\prod_{i=1}^np(x_i|F(\*\theta)). \nonumber
    \end{IEEEeqnarray}
    At first, this approach seems conceptually appealing: Not only does one circumvent most issues with \ref{assumption:lklhood} by making the likelihood function almost arbitrarily flexible, but one also  quantifies uncertainty in the usual Bayesian manner.
    While both observations are correct, they mask a potentially severe issue with this approach: Namely, specifying the prior $\pi(\*\theta)$ in a principled way and in (approximate) accordance with \ref{assumption:prior} is impossible in practice:
    Firstly, because the vector $\*\theta$ indexes a black box model, its entries do not correspond to interpretable quantities. Accordingly, building prior beliefs based on domain expertise about the data $x_{1:n}$ is not feasible.
    Secondly, since computational aspects are a major concern for \BNN{}s, one is typically constrained to choosing priors that factorize over $\*\theta$.
    %
    %
    As a consequence, practitioners often resort to choosing default priors which are not motivated as prior beliefs in the original sense or by an attempt to approximately satisfy \ref{assumption:prior}. 
    Specifically, one typically just picks $\pi(\*\theta) = \prod_{j=1}^d\pi_j(\*\theta_j)$, where $\pi_j(\*\theta_j)$ is a standard normal distribution for all $j$.
    Choosing priors in this ad-hoc fashion violates the principles underlying classical Bayesian modelling (see also Section \ref{sec:GVI_prior_misspecification}).
    %
    %
    This is 
    especially problematic whenever $n$ is small relative to $d$: In these situations, prior influence serves as a strong source of information about $\*\theta$.
    Thus, if the prior is misspecified and $n$ is small relative to $d$, the (incorrect) information contained in the prior often overshadows the information in the data.
    At the same time, reliable uncertainty quantification is most important whenever $n$ is small relative to $d$.
    Indeed, this is a well-known issue and is addressed in various contributions by up-weighting the likelihood (down-weighting the \KLD term in the \ELBO), see \citet{ BNNUQ2,BNNUQ5,BNNUQ1,BNNUQ3,BNNUQ4}. 
    \label{example:violation_P}
\end{example}

For completeness, we note that the current paper does not discuss uninformative and so-called objective priors \cite[see, e.g.][]{jeffreys1961theory, zellner1977maximal,bernardo1979reference,berger1992development,jaynes2003probability, berger2006case}. 
Priors of this kind are constructed to be as uninformative as possible and thus in some ways objective. 
In many ways, they are a principled and natural response to the critique of ill-informed priors.  
Generally however, their construction results in so-called improper priors--densities that do not correspond to a finite measure and thus do not integrate to one.
While this is not generally prohibitive, it would severely complicate the developments of Section \ref{sec:RoT} because most divergences are not well-defined for improper priors\footnote{
    The \KLD is the exception to this rule: As it depends on the log normalizer of $\pi(\*\theta)$ in an additive fashion, improper priors can still be admissible so long as eq. \eqref{eq:Zellner} yields a solution for the unnormalized version of the \KLD as given in \citet{ABCdiv}.
}. 

\subsection{Likelihood Misspecification}
\label{sec:violation_L}

While prior misspecification affects inference adversely,
the issue for inferential practice is even more serious if \ref{assumption:lklhood} is violated thoroughly: Whenever the likelihood model for $x_i$ is severely misspecified, inference outcomes suffer dramatically.
Moreover, not even the asymptotic regime offers a remedy and the adverse effects of misspecification persist as $n\to\infty$.
The traditional approach to addressing this issue is straightforward: If the likelihood model $p(x_i|\*\theta)$ is misspecified, simply investigate why exactly it fits the data poorly.
After residual analysis, intense study of descriptive statistics and consultation with domain experts, redesign it to arrive at a likelihood model $p'(x_i|\*\theta')$, which hopefully provides a better fit  to the data and (approximately) satisfies \ref{assumption:lklhood}.
In other words, the traditional view is that any problem with misspecification is really a problem with careless modelling.

As outlined in Section \ref{sec:challenging_traditional_paradigm}, this strategy is neither practiced nor feasible with contemporary large-scale models.
%
%
%
The naive interpretation of likelihood functions as corresponding to an appropriately good description of the true data generating process in the sense of \ref{assumption:lklhood} is thus wholly inappropriate.
This is especially important as many large-scale models are mainly interested in capturing the \textit{typical} behaviour of the data---rather than \textit{fully} modelling every aspect of a population. 
While this may appear to be a minor point at first glance, it has serious consequences for inferential practice. To see why, suppose a population contains a small number of outlying observations, local heterogeneities or spiky noise.
The naive interpretation of the likelihood as in \ref{assumption:lklhood} \textit{assumes} that these untypical aspects are encoded in the likelihood function. 
Hence, if $x_i$ is an outlier so that $p(x_i|\*\theta)$  is very close to zero for some value of $\*\theta$ constructed to fit the rest of the data,
the inference machinery of traditional statistics interprets this as a strong signal: After all, if the likelihood model is an approximately correct description of the data generating mechanism, the most informative observations are those that do \textit{not} fit the model fitted to the rest of the data.
%
It follows that aberrant parts of the data will have a disproportional
impact on inference outcomes---leading standard inference methods to break down \citep[see also][]{Jewson}. 
%
%

While it is in general hard to visualize this issue, influence functions provide a concise way of showcasing the problem. 
Roughly speaking, influence functions in a Bayesian context quantify the impact the $(n+1)$-th observation $x_{n+1}$ has on the posterior distribution $q^{\ast}_{\text{B}}(\*\theta)$ constructed using the first $n$ observations \citep{peng1995bayesian}. 
This discrepancy is measured by computing a divergence between the posteriors based on $x_{1:n}$ and on $x_{1:(n+1)}$. 
%
%
Using the Fisher-Rao divergence \citep[for its geometric properties as explained in][]{InfFct}, Figure \ref{Fig:Influence_fct_pic} compares the influence of a standard Bayesian posterior with that of a posterior belief computed using Generalized Variational Inference (\GVI).
%
The left side of the Figure formalizes the intuition we have just developed: In the standard Bayesian case, the influence of $x_{n+1}$ on the posterior belief grows stronger and stronger the more untypical
it is relative to previously observed data.
Similarly, the right side shows the adverse effect this has on the posterior predictive. 
%
To make the implications of influence functions for inferential practice more tangible, we additionally demonstrate the outlier problem in Example \ref{example:violation_L}.
\begin{figure}[h!]
    \begin{center}
    \includegraphics[
    width=0.49\columnwidth]{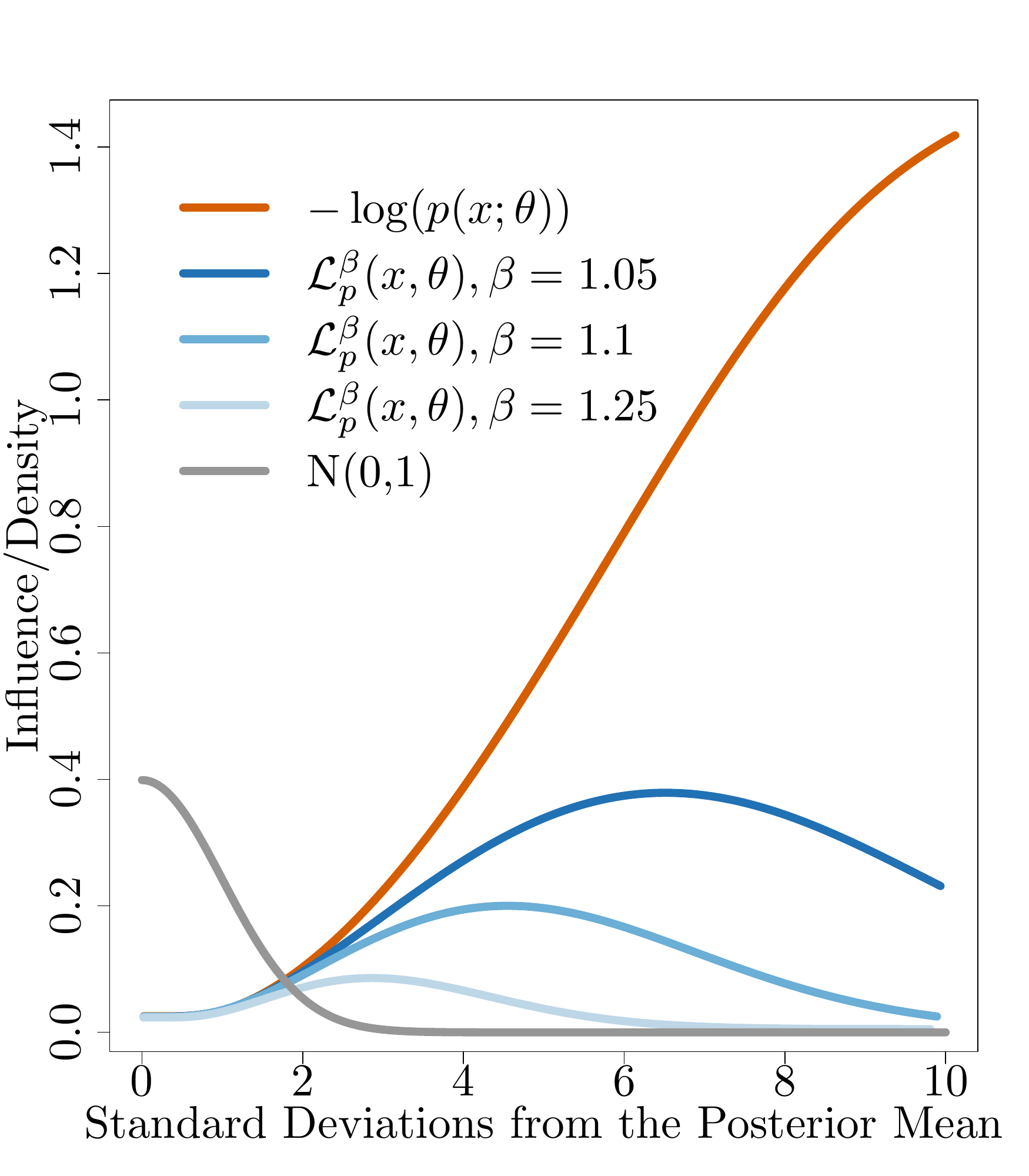}
    \includegraphics[
    width=0.49\columnwidth]{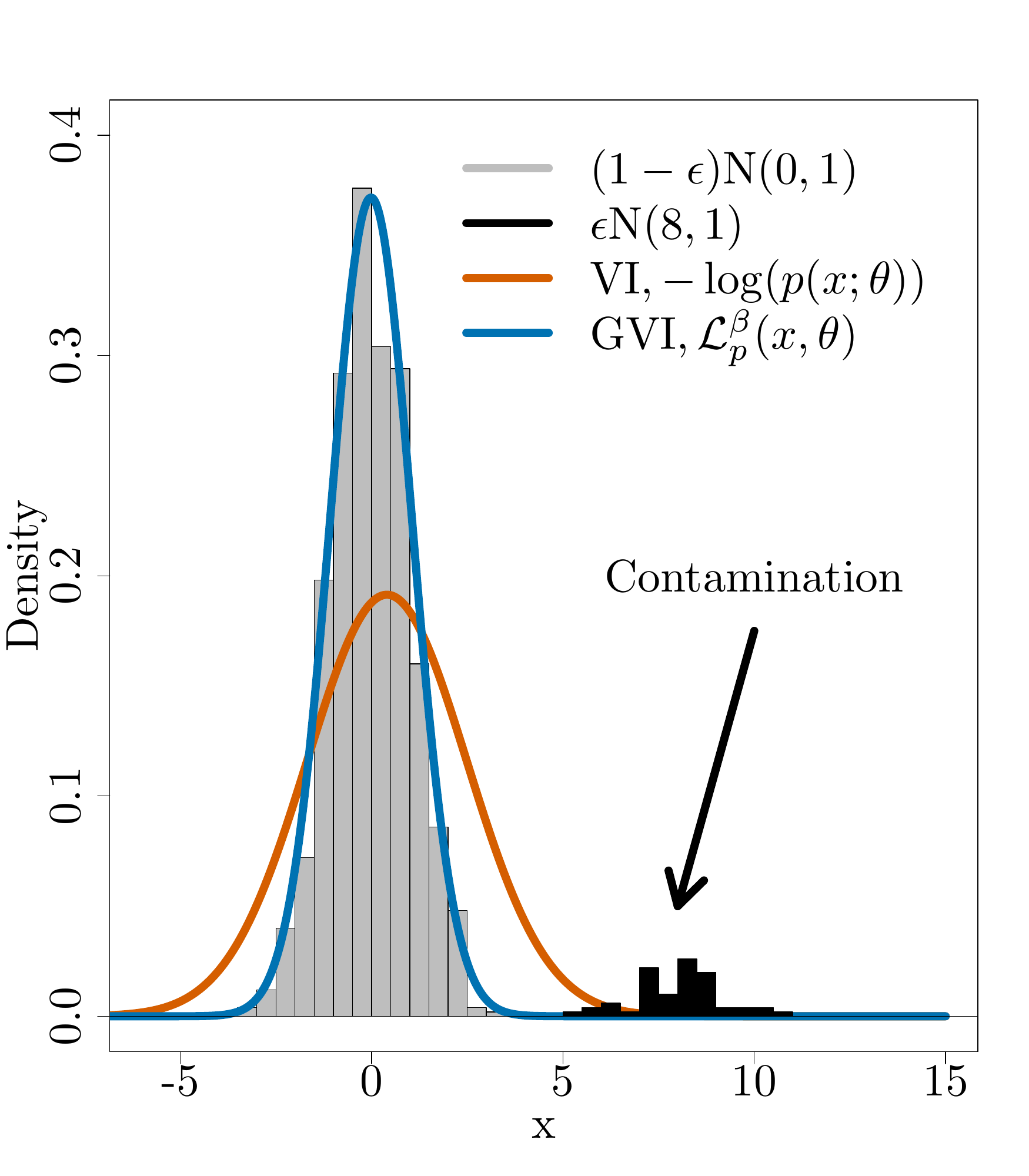}
    \caption{
        \bv
        The plots compare influence functions (\textbf{Left}) and predictive posteriors (\textbf{Right}) of a \textcolor{VIColor}{\textbf{standard}} Bayesian  against a \textcolor{GVIColor}{\textbf{\GVI}} posterior.
        \textbf{Left:} The influence functions of scoring the normal likelihood with a \textcolor{VIColor}{\textbf{standard}} negative log likelihood against a \textcolor{GVIColor}{\textbf{robust}} scoring rule derived from $\beta$-divergences.
        \textbf{Right:} A univariate normal is fitted using all the data depicted, including the outlying contamination. The posterior predictive corresponding to the \textcolor{GVIColor}{\textbf{robust}} scoring rule is able to ignore these outliers. This stands in contrast to the posterior predictive based on \textcolor{VIColor}{\textbf{standard}}, which assigns increasingly large influence to outlying observations.
    }
    \label{Fig:Influence_fct_pic}
    \vskip -0.25in
    \end{center}
\end{figure}

\begin{example}[Outliers as violations of \ref{assumption:lklhood}]
    While there exist many formalizations of the notion of an outlier, the conceptually most useful one is probably the $\varepsilon$-contamination model.
    In the $\varepsilon$-contamination model, the data points $x_i$ are generated according to a contaminated density composed additively as
    \begin{IEEEeqnarray}{rCl}
        p_{\text{true}}(x_i) & = & (1-\varepsilon) \cdot p(x_i|\*\theta^{\ast}) + \varepsilon \cdot o(x_i), \nonumber
    \end{IEEEeqnarray}
    for some small and fixed $0<\varepsilon<1$, a fixed parameter value $\*\theta^{\ast}$ of interest and a contaminating outlier-generating density $o$.
    %
    An obvious violation of \ref{assumption:lklhood} for this case would be fitting the data only to the non-contaminated component $p(x_i|\*\theta)$ in order to infer $\*\theta^{\ast}$.

    Considering this type of model misspecification is especially poignant in Bayesian On-line Changepoint Detection (\BOCPD), a well-studied family of models that yield computationally efficient algorithms \citep[see e.g.][]{BOCD, FearnheadOnlineBCD, HazardLearningBOCD, GPBOCD, CaronDoucet, TurnerVB, BOCPDMS, RBOCPD}.
    \BOCPD aims to segment a data stream in real time and achieves this via an efficient recursion updating the Bayesian posterior with each newly arriving observation.
    A canonical application example of \BOCPD is the well-log data set first discussed by \citet{FirstWellLog}. Its observations describe the abruptly changing nuclear responses of rock stratification during the course of drilling a well. 
    Generally, the different rock strata are clearly distinguishable from one another. However, rock formation processes are noisy and sometimes interrupted by extraordinary events such as tsunamis, earth quakes or eruptions. 
    Accordingly, the data points generated are surprisingly close to an $\varepsilon$-contaminated normal distribution within each of the clearly distinguishable  rock strata.
    Figure \ref{Fig:RBOCPD_example} is taken from \citet{RBOCPD} and shows how this phenomenon renders vanilla \BOCPD an unreliable algorithm.
    It also shows that this issue can be remedied by constructing alternative posterior belief distributions via a Generalized Variational Inference (\GVI) procedure relying on a robust loss function derived from the $\beta$-divergence.
    %
    %
    \label{example:violation_L}
\end{example}

\begin{figure}[h!]
        \vskip 0.1in
        \begin{center}
            \centerline{\includegraphics[trim= {2.25cm 0.25cm 3.5cm 0.9cm}, clip, 
            width=1.00\columnwidth]{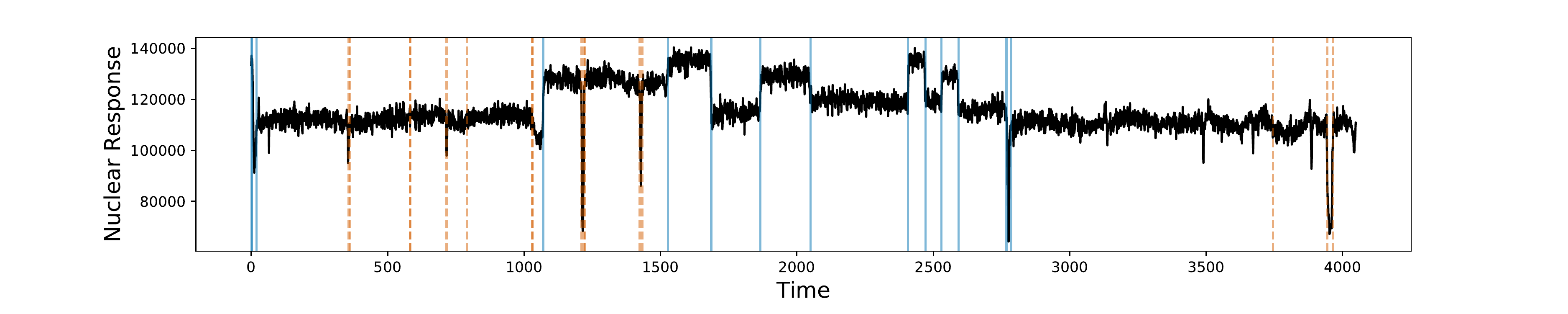}}
            \caption{
            \bv
            Inference outcomes of \BOCPD on the well log data set using the \VIColor{\textbf{standard Bayesian posterior}} and a \GVIColor{\textbf{\GVI posterior}}  constructed with robust losses based on the $\beta$-divergence.
            Solid vertical lines correspond to Maximum A Posteriori (MAP) segmentation of \GVIColor{\textbf{\GVI posterior}}, dashed vertical lines mark incorrect changepoints \textit{additionally} detected under \VIColor{\textbf{standard Bayesian inference}}.
            }
            \label{Fig:RBOCPD_example}
        \end{center}
        \vskip -0.2in
    \end{figure}

\subsection{Computation mismatch}
\label{sec:violation_C}

As Theorem \ref{thm:Zellner} shows, the Bayesian posterior $q^{\ast}_{\text{B}}(\*\theta)$ is the result of an optimisation problem over the infinite-dimensional space $\mathcal{P}(\*\Theta)$. Generally therefore, the posterior itself also does not have a closed form expressed in a finite-dimensional parameter.
In fact, the only case in which $q^{\ast}_{\text{B}}(\*\theta)$ can be represented through a finite-dimensional parameter is when prior and likelihood are conjugate to one another.
%
%
Accordingly, performing inference with $q^{\ast}_{\text{B}}(\*\theta)$ is in general a hard problem, which manifests itself through the need to compute the generally intractable normalizing constant.
To circumvent this problem, it is common to leverage Markov Chain Monte Carlo algorithms that produce an exact representation of $q^{\ast}_{\text{B}}(\*\theta)$ if the chain runs indefinitely and collects infinitely many (correlated) samples.
In practice, collecting a finite number of samples from the chain will often yield a reasonable approximation to the posterior so long as $d = |\*\Theta|$ is not too large. 
For large values of $d$ however, the number of samples required to make the approximation useful is often too large to make sampling a computationally viable strategy: 
For example, in the \textit{best} case scenario, Random Walk Metropolis Hastings scales like $\mathcal{O}(d^2)$ \citep{roberts1997weak}, the Metropolis-adjusted Langevin algorithm like $\mathcal{O}(d^{4/3})$ \citep{roberts1998optimal} and Hamiltonian Monte Carlo like $\mathcal{O}(d^{5/4})$ \citep{beskos2013optimal}. Note that these results assume independence and Gaussiantiy---so in practice scaling rates  are typically much worse. 

An alternative way of avoiding the computation of a normalizing constant are various approximation strategies seeking to project $q^{\ast}_{\text{B}}(\*\theta)$ into some parameterized subset $\mathcal{Q} \subset \mathcal{P}(\*\Theta)$. 
This strategy will produce approximations $q^{\ast}_{\text{A}}(\*\theta)$ of high quality only if the set $\mathcal{Q}$ is chosen to be  sufficiently rich so that the statement $q^{\ast}_{\text{A}}(\*\theta) \approx q^{\ast}_{\text{B}}(\*\theta)$ is not completely vacuous.
Importantly however, most posterior belief distributions $q^{\ast}_{\text{A}}(\*\theta)$ that are regularly computed this way barely deserve to be called approximations to $q^{\ast}_{\text{B}}(\*\theta)$. 
For example, consider the mean field normal variational family given by
\begin{IEEEeqnarray}{rCl}
    \mathcal{Q}_{\text{MFN}} &= & 
    \left\{ 
        \prod_{j=1}^d \mathcal{N}(\*\theta_j|\mu_j, \sigma_j^2): \mu_j \in \mathbb{R}, \sigma^2_j \in \mathbb{R}_{>0} \text{ for all } j
    \right\}.
    \label{eq:Q_MFN}
\end{IEEEeqnarray}
For most interesting non-trivial posterior distributions $q^{\ast}_{\text{B}}(\*\theta)$, there will not exist any element in $\mathcal{Q}_{\text{MFN}}$ that could be considered an approximation to $q^{\ast}_{\text{B}}(\*\theta)$ in any meaningful way: After all, this variational family directly assumes $\mathcal{O}(d^2)$ independence relationships in the approximate posterior belief for $\*\theta$.
Worse still: As approximations are particularly attractive when $|\*\Theta|=d$ is large, in practice we will resort to these insufficiently expressive ``approximations'' to $q^{\ast}_{\text{B}}(\*\theta)$ \textit{precisely} when the elements in $\mathcal{Q}_{\text{MFN}}$ are structurally most dissimilar from $q^{\ast}_{\text{B}}(\*\theta)$.
As we proceed to explain in the remainder of the paper, we think it is often unhelpful to think of posterior beliefs $q^{\ast}_{\text{A}}(\*\theta)$ computed in this way as approximations to $q^{\ast}_{\text{B}}(\*\theta)$.
Rather, we think of them as defining a new and distinct posterior belief distribution in their own right.

To make the preceding discussion more tangible and highlight the importance that the frequent violation of \ref{assumption:computation} has played in research on statistical machine learning, Example \ref{example:violation_C} illuminates the importance of computational considerations for Gaussian Processes.

\begin{example}[large-scale Gaussian processes as violations of \ref{assumption:computation}]
    Many Bayesian machine learning models prohibit exact computation. One particularly interesting case are Gaussian Processes (\GP{}s): Even in the special cases where they admit closed form posteriors, it may well be impossible to compute them exactly for sufficiently large inference problems.
    The reason is that for $n$ observations, direct computation of the associated \GP{} posterior takes $\mathcal{O}(n^3)$ time.
    As a consequence, an entire literature is dedicated to bringing down this prohibitive computational complexity \citep[see for instance][]{GPApprox1,GPApprox3, GPApprox2, GPApprox4} and developing software or computer-architecture specific methods geared towards inference with \GP{}s \citep[e.g.][]{gpflow, GPytorch, BayesOptPyTorch, GPexample1}.
    Furthermore, with deep (i.e., hierarchical) approaches to \GP{}s introduced in \citet{DGPs} and extended in various directions \citep[e.g.][]{VAEDGPs, diffGP}, this challenge has only become more important \citep[see e.g.][]{DGPEP, randomFeaturesDGP, DeepGPsVI}.
    \label{example:violation_C}
\end{example}

\section{The Rule of Three: A new Bayesian paradigm}
\label{sec:RoT}

As the last sections have shown, the assumptions which form the core of traditional Bayesian inference are often not a good basis for modern large-scale statistical inference.
%
In response to this observation, the following section seeks to generalize the Bayesian paradigm. As we shall see, the way in which we do so is strongly inspired by the preceding development in two important ways: 
Firstly, as we did in Section \ref{sec:VI_and_Bayes_posteriors}, we take inspiration from an optimization-centric view on Bayesian inference.
Secondly, since we saw in Section \ref{sec:Examining_Bayesian_paradigms} that there are three potentially problematic assumptions underlying Bayesian inference, we construct our generalization in order to address each of these concerns directly and modularly.
%
This development proceeds in three steps.
\begin{itemize}
    \item[] \textbf{Section \ref{sec:axioms}} sets out axioms that  are a minimal requirement for any posterior belief distribution. In accordance with these axioms, we derive the Rule of Three (\RoT).
    \item[] \textbf{Sections \ref{sec:deriving_RoT} \& \ref{sec:RoT-addresses-traditional-issues}}  discuss the \RoT as a recipe for producing posterior belief distributions  and elaborate on its three interpretable ingredients.
    We also show how the \RoT can \textbf{directly} address the concerns associated with imposing \ref{assumption:prior}, \ref{assumption:lklhood} and \ref{assumption:computation}.
    \item[]  \textbf{Section \ref{sec:insight_RoT}} demonstrates that the axiomatic development is both helpful and useful by comparing the \RoT with existing methods that generate belief distributions. 
\end{itemize}
%
%

\subsection{An axiomatic foundation for Bayesian inference}
\label{sec:axioms}

In this section, we set out to produce a novel axiomatic foundation for Bayesian inference that is flexible enough to deal with contemporary real-world challenges.
Before doing so, we define two core concepts required in their development.
\begin{definition}[Loss Function]
    Losses are functions $\ell:\*\Theta \times \mathcal{X} \to \mathbb{R}$ which for any observation sequence $x_{1:n} \in \mathcal{X}^n$
    have empirical risk minimizers 
    \begin{IEEEeqnarray}{rCl}
        \widehat{\*\theta}_n & = &
        \argmin_{\*\theta \in \*\Theta}\left\{
            \sum_{i=1}^n\ell(\*\theta, x_i)
        \right\}.
        \nonumber
    \end{IEEEeqnarray} 
\end{definition}
\begin{definition}[Statistical Divergence]
    Statistical divergences are functions $D:\mathcal{P}(\*\Theta) \times \mathcal{P}(\*\Theta) \to \mathbb{R}_{\geq 0}$ which satisfy that $D(q\|\pi) \geq 0$ and $D(q\|\pi) = 0$ if and only if $q(\*\theta) = \pi(\*\theta)$ almost everywhere.
\end{definition}
In the axiomatic development to follow, we will avoid introducing measure-theoretic notation. 
Thus, for simplicity we will assume that all densities are defined with respect to the Lebesgue measure on $\mathbb{R}^d$.
To the same end, we will also slightly abuse notation in two ways.
Firstly, we will write that $q^{\ast} \in \mathcal{P}(\*\Theta)$ for probability densities $q^{\ast}(\*\theta)$ on $\*\Theta$, even though probability densities are not in $\mathcal{P}(\*\Theta)$. 
However, $q^{\ast}(\*\theta)$ induces a measure $\mu_{q^{\ast}} \in \mathcal{P}(\*\Theta)$ as $\mu_{q^{\ast}}(A) = \int_{A}dq^{\ast}(\*\theta)$ for any measurable set $A \subset \*\Theta$. 
Thus, whenever we write $q^{\ast} \in \mathcal{P}(\*\Theta)$, what we mean is that $\mu_{q^{\ast}} \in \mathcal{P}(\*\Theta)$.
Similarly, we will sometimes write $q_1^{\ast} \neq q_2^{\ast}$ to mean that there exists a measurable set $A \in \*\Theta$ such that $\mu_{q_1^{\ast}}(A) \neq \mu_{q_2^{\ast}}(A)$.
%
%
We are now finally ready to state the axiomatic foundations. To avoid confusion, note that Axioms will build on each other in the order in which they are stated. 


\begin{axiomCustom}{Representation}
   The posterior $q^{\ast} \in \mathcal{P}(\*\Theta)$ is constructed by solving an optimization problem over some space $\Pi \subseteq \mathcal{P}(\*\Theta)$.
    The optimization seeks to minimize exactly two criteria that do not interact:
    \begin{itemize}
        %
        \myitem{(i)} 
        The in-sample loss $\sum_{i=1}^n\ell(\*\theta, x_i)$ to be expected under $q^{\ast}(\*\theta)$.
        \label{axiom:repr:loss}
        \myitem{(ii)}
        The deviation from the prior $\pi(\*\theta)$ as measured by the magnitude of some statistical divergence $D$.
        %
        %
        \label{axiom:repr:divergence}
    \end{itemize}
    \label{Axiom:representation}
\end{axiomCustom}
\begin{remark}
    By inspecting eq. \eqref{eq:Zellner}, it becomes clear that the above axiom is a generalization of traditional Bayesian inference:
    Firstly, eq. \eqref{eq:Zellner} reveals that standard Bayesian inference solves an optimization problem over $\Pi = \mathcal{P}(\*\Theta)$. 
    Secondly, eq. \eqref{eq:Zellner} also shows that \ref{axiom:repr:loss} of the above Axiom also holds for standard Bayesian posteriors. In traditional Bayesian inference, the loss to be minimized is a negative log likelihood, while more recent iterations have allowed for broader classes of losses \citep[e.g.][]{Bissiri, Jewson}. 
    Thirdly, eq. \eqref{eq:Zellner} demonstrates that \ref{axiom:repr:divergence} is also satisfied for standard Bayesian inference: The Kullback-Leibler Divergence (\KLD) penalizes large deviations of $q^{\ast}_{\text{B}}(\*\theta)$ from the prior $\pi(\*\theta)$.
\end{remark}
%
%
Reiterating the essence of Observation \ref{observation:isomorphism}, the previous  axiom formalizes our understanding of Bayesian inference as an optimization problem over a potentially infinite-dimensional function space.
In fact, it already tells us that for a fixed prior, posterior beliefs derived under our new axiomatic approach have three distinct ingredients: The loss $\ell$, the divergence $D(\cdot\|\pi)$ and the space $\Pi$.
Making this insight more precise immediately yields the following representation Theorem.
\begin{theorem}[Form 1]
    Under Axiom \ref{Axiom:representation}, posterior belief distributions can be written as  
    \begin{IEEEeqnarray}{rCl}
        q^{\ast}(\*\theta) = \argmin_{q\in\Pi}
    \left\{
            f\left(\mathbb{E}_{q(\*\theta)}\left[
                \sum_{i=1}^n \ell(\*\theta, x_i)
            \right], D(q||\pi)\right)
    \right\},
    \nonumber
    \end{IEEEeqnarray} 
    where $f:\mathbb{R}^2\to\mathbb{R}$ is some function which is non-decreasing in both its arguments.
    \label{Thm:Form}
\end{theorem}
\begin{proof}
    This follows directly from Axiom \ref{Axiom:representation}: 
    Firstly, any posterior belief distribution $q^{\ast}(\*\theta)$ is the solution to an optimization problem over $\Pi$. 
    Thus, for an appropriately structured objective $\text{Obj}$, one can write 
    \begin{IEEEeqnarray}{rCl}
        q^{\ast}(\*\theta) = \argmin_{q\in\Pi}
    \left\{
            \text{Obj}(q)
    \right\}.
    \nonumber
    \end{IEEEeqnarray}
    Hence, the question becomes what the objective looks like. The answer is provided by parts \ref{axiom:repr:loss} and \ref{axiom:repr:divergence} of Axiom \ref{Axiom:representation}:
    By \ref{axiom:repr:loss}, the optimization's objective depends on $\mathbb{E}_{q(\*\theta)}[\ell(\*\theta, x_i)]$. Further, by \ref{axiom:repr:divergence} it also depends on the divergence $D$ between prior and $q$, i.e. on $D(q||\pi)$. 
    Clearly then, for some function $f:\mathbb{R}^2\to\mathbb{R}$, 
    \begin{IEEEeqnarray}{rCl}
        \text{Obj}(q) = f\left(\mathbb{E}_{q(\*\theta)}\left[
                \sum_{i=1}^n \ell(\*\theta, x_i)
            \right], D(q||\pi)\right).
        \nonumber
    \end{IEEEeqnarray} 
    Noting that both arguments are not allowed to interact and are to be minimized, it is clear that $f$ is non-decreasing in both arguments, which completes the proof.
\end{proof}
This result is a first and helpful step, but in itself does not suffice to yield objectives that are useful in practice.
Specifically, we need to get a handle on the function $f$. It is clear that under Axiom \ref{Axiom:representation} alone, very little can be said about $f$.
The next two axioms seek to address this very issue.
The Axiom \ref{Axiom:translationInvariance} is imposed for a simple purpose: We want posteriors to be invariant to uninformative components of the loss $\ell$ as well as the divergence term $D$.
In other words, adding a constant $C$ to the loss or the divergence from the prior should not change our inferences about $\*\theta$.
%
%
\begin{axiomCustom}{Information Equivalence}
     Take any two constants $C, M \in \mathbb{R}_{>0}$ and 
     let the posteriors $q^{\ast}_1, q^{\ast}_{2} \in \mathcal{P}(\*\Theta)$ be computed based on the same optimization problem, but with two different loss functions $\ell^{(1)}$ and $\ell^{(2)}$ as well as two different divergences $D^{(1)}, D^{(2)}$ so that $D^{(1)} =  D^{(2)} + M$ and so that $\ell^{(1)} = \ell^{(2)} + C$. 
     Then, 
     $q^{\ast}_1(\*\theta) = q^{\ast}_{2}(\*\theta)$ (almost everywhere).     \label{Axiom:translationInvariance}
\end{axiomCustom}
\begin{remark}
    At this point, one may pause and wonder if one would not like the above axiom to be stronger. Specifically, would one want different posteriors if $\ell^{(1)} = w\cdot \ell^{(2)}$ for some $w \neq 1$?
    Since we want our methods to recover existing Bayesian inference techniques as special cases, this question is readily answered.
    In particular, notice that many Bayesian inference techniques \textbf{do} attach information to pre-multiplying losses with a constant.
    %
    %
    For example, in the Power Bayesian framework \cite[e.g.][]{SafeLearning,SafeBayesian,holmes2017assigning,InconsistencyBayesInference,DunsonCoarsening} one wishes to down-weight the information in the likelihood terms by considering the pseudo-likelihood terms $p(x_i|\*\theta)^w$ for some $w \in (0,1) $ instead of $p(x_i|\*\theta)$. 
    %
    %
    This procedure is attractive since it produces alternative posterior beliefs that contract to a point mass at a slower rate than the standard posteriors. 
    Re-examining eq. \eqref{eq:Zellner}, it becomes clear that relative to standard Bayes rule, power likelihoods are in fact a simple re-weighting scheme. Specifically,  one replaces $\ell(\*\theta, x_i) = -\log p(x_i|\*\theta)$ by the weighted version $\ell(\*\theta, x_i) = -w\log p(x_i|\*\theta)$.
    \label{remark:power_bayes_axioms}
\end{remark}
We are now ready to state the last axiom. Observe that Axiom \ref{Axiom:informationAdjustment} determines conditions under which the posterior must not be affected by changing inputs. Complementing this, we still require sufficient conditions under which the posteriors are guaranteed to differ.
\begin{axiomCustom}{Generalized Likelihood Principle} 
     Suppose the posteriors $q^{\ast}_n, q^{\ast}_{n+m} \in \Pi$ are computed based on an optimization problem satisfying Axiom \ref{Axiom:representation}. Assume that the same optimization problem is used for both posteriors, except for a difference in the samples $x_{1:n}$ and $x_{1:n+m}$  and potentially different loss functions $\ell^{(1)}$ and $\ell^{(2)}$ that are used, respectively.
     \begin{itemize}
         \myitem{(i)}
         Provided that there is an information difference between $x_{1:n}$ and $x_{1:n+m}$ for $m>0$, the posteriors are different.
         In other words, $\sum_{i=1}^n\ell^{(1)}\left(\*\theta,x_i\right)\neq \sum_{i=1}^{n+m}\ell^{(2)}\left(\*\theta,x_i\right)$ implies that 
         $q^{\ast}_n \neq q^{\ast}_{n+m}$, even if $\ell^{(1)} = \ell^{(2)}$.
         \myitem{(ii)}   
         Provided that there is a difference in the measure of information between $\ell^{(1)}$ and $\ell^{(2)}$, the posteriors are different.
         In other words, $\ell^{(1)} \neq \ell^{(2)}$ implies that 
         $q^{\ast}_n \neq q^{\ast}_{n+m}$, even if $m = 0$ so that $x_{1:n} = x_{1:n+m}$.
     \end{itemize}
    \label{Axiom:informationAdjustment}
\end{axiomCustom}
\begin{remark}
    This axiom has a particularly interesting interpretation as a generalized version of the well-known likelihood principle associated with standard Bayesian inference. 
    Roughly speaking, the likelihood principle says that all information in $x_{1:n}$ relevant to conducting inference on the model parameters is contained in the likelihood functions evaluated at $x_{1:n}$.
    Similarly, Axiom \ref{Axiom:informationAdjustment} says that all information the sample $x_{1:n}$ contains about $\*\theta$ is contained in the loss function evaluated on the relevant data sample.
    %
\end{remark}
\begin{remark}
    Notice that part (ii) of Axiom \ref{Axiom:informationAdjustment} requires that $\ell^{(1)}$ and $\ell^{(2)}$ measure information differently, which precludes that $\ell^{(1)} = \ell^{(2)} + C$ for some $C \in \mathbb{R}_{>0}$ by Axiom \ref{Axiom:translationInvariance}.
\end{remark}
Recalling the discussion of Remark \ref{remark:power_bayes_axioms}, we next state a result showing that our axiomatic approach respects the logic of Power Bayes and related procedures. Specifically, re-weighting the likelihood (or the losses more generally) will yield different posteriors.
\begin{corollary}
    If Axiom \ref{Axiom:informationAdjustment} holds, and if the
    posteriors $q^{\ast}_1 \in \Pi$ and $q^{\ast}_2 \in \Pi$ are based on the same observations and losses $\ell$ and $w \cdot \ell$ for $w \in \mathbb{R}_{>0}\setminus \{1 \}$, respectively, then $q^{\ast}_1 \neq q^{\ast}_2$.
    \label{Thm:multiplication_axiom}
\end{corollary}
\begin{proof}
    This follows from part (ii) of Axiom \ref{Axiom:informationAdjustment}.
\end{proof}
Finally, we investigate how the axiomatic developments set out above simplify the structure of objectives producing posterior belief distributions.
To achieve this, it is clear that we need to re-investigate $f:\mathbb{R}^2 \to \mathbb{R}$ as in Theorem \ref{Thm:Form}. 
Since we want $f$ to recover the standard Bayesian posterior of eq. \eqref{eq:Zellner} as well as the standard \VI posterior in eq. \eqref{sec:VI_as_ELBO}, it is natural to restrict attention to elementary operations. 
Doing so produces the following representation result that is both stricter and more useful than Theorem \ref{Thm:Form}.
\begin{theorem}
    Suppose the posterior belief $q^{\ast} \in \mathcal{P}(\*\Theta)$ satisfies Axiom \ref{Axiom:representation}.
    Recall from Theorem \ref{Thm:Form} that this is equivalent to saying that for some $\Pi \subseteq \mathcal{P}(\*\Theta)$, some loss function $\ell:\*\Theta \times \mathcal{X} \to \mathbb{R}$ and some divergence $D(\cdot\|\pi):\mathcal{P}(\*\Theta) \to \mathbb{R}_{\geq 0}$, the posterior can be written as
    \begin{IEEEeqnarray}{rCl}
        q^{\ast}(\*\theta) = \argmin_{q\in\Pi}
    \left\{
            f\left(\mathbb{E}_{q(\*\theta)}\left[
                \sum_{i=1}^n \ell(\*\theta, x_i)
            \right], D(q\|\pi)\right)
    \right\},
    \nonumber
    \end{IEEEeqnarray} 
    where $f$ is some function $f:\mathbb{R}^2\to\mathbb{R}$.
    If $f(x,y) = x \circ y$ is an elementary  operation on $\mathbb{R}$ and $q^{\ast} \in \mathcal{P}(\*\Theta)$ satisfies Axiom \ref{Axiom:translationInvariance}, this objective is uniquely identified as
    \begin{IEEEeqnarray}{rCl}
        q^{\ast}(\*\theta) = \argmin_{q\in\Pi}
        \left\{
            \mathbb{E}_{q(\*\theta)}\left[
                \sum_{i=1}^n \ell(\*\theta, x_i)
            \right] + D(q\|\pi)
        \right\}.
        \label{eq:RoT_thm}
    \end{IEEEeqnarray}
    \label{Thm:additivityD}
\end{theorem}
\begin{proof}
    %
    The elementary operations are addition, subtraction, multiplication and division.
    Consider the losses $\ell^{(1)}$ and $\ell^{(2)} = \ell^{(1)} + C$ for $C \in \mathbb{R}_{>0}$ a constant.
    It is straightforward to see that Axiom \ref{Axiom:translationInvariance} is violated if $\circ$ is multiplication:
    \begin{IEEEeqnarray}{rCl}
    &&\argmin_{q \in \Pi} \left\{
        \mathbb{E}_{q(\*\theta)}\left[
                \sum_{i=1}^n \ell^{(2)}(\*\theta, x_i)
            \right] \cdot D(q\|\pi)
    \right\} \nonumber \\
 & = & \argmin_{q \in \Pi}\left\{ 
        \mathbb{E}_{q(\*\theta)}\left[
                \sum_{i=1}^n \ell^{(1)}(\*\theta, x_i)
            \right] \cdot D(q\|\pi) + C\cdot D(q\|\pi)
      \right\} \nonumber \\
 & \neq & 
     \argmin_{q \in \Pi}\left\{ 
        \mathbb{E}_{q(\*\theta)}\left[
                \sum_{i=1}^n \ell^{(1)}(\*\theta, x_i)
            \right] \cdot D(q\|\pi)
      \right\} \nonumber 
    \end{IEEEeqnarray}
    and similarly if $\circ$ is division.
    However, it  is straightforward to see that the optimization problem is invariant to adding constants to the loss or the divergence if $\circ$ is addition or subtraction.
    Finally,  subtracting the prior regularizer is a direct and obvious violation of  part \ref{axiom:repr:divergence} in Axiom \ref{Axiom:representation}, it follows that addition is the only elementary operation on $\mathbb{R}$ satisfying both Axioms and the result follows.
\end{proof}
%
%
The last result is of crucial importance for the further development of the paper: Specifically, eq. \eqref{eq:RoT_thm} provides a generic and flexible recipe for the design of novel posterior distributions. 
In fact, it is this equation that we discuss next.

\subsection{The Rule of Three}
\label{sec:deriving_RoT}

%
Following from the axiomatic developments of the last section that culminated in Theorem \ref{Thm:additivityD}, we next discuss the interpretations and theoretical properties of posterior belief distributions generated from objectives as in eq. \eqref{eq:RoT_thm}.
To simplify the representation throughout the remainder, we first define notation for posteriors of this form.
\begin{definition}[Rule of Three (\RoT)]
    Take observations $x_{1:n}$, a prior $\pi(\*\theta)$, a space $\Pi \subseteq \mathcal{P}(\*\Theta)$, a loss function $\ell:\*\Theta \times \mathcal{X} \to \mathbb{R}$ and a divergence $D(\cdot\|\pi): \Pi \to \mathbb{R}_{\geq 0}$.
    With this in hand, we say that posteriors have been constructed via the \textbf{Rule of Three (\RoT)} if they can be written as
    \begin{IEEEeqnarray}{rCl}
        q^{\ast}(\*\theta) = \argmin_{q\in\Pi}
        \left\{
            \mathbb{E}_{q(\*\theta)}\left[
                \sum_{i=1}^n \ell(\*\theta, x_i)
            \right] + D(q\|\pi)
        \right\} = P(\ell, D, \Pi).
        \nonumber
    \end{IEEEeqnarray}
    Here, $P(\ell, D, \Pi)$ is a short-hand notation for the \RoT suppressing dependence on $n$ and $\pi$. 
    \label{def:RoT}
\end{definition}
\begin{remark}
    Before moving on, note that in the \RoT, $D$ determines the shape of the uncertainty, i.e. how exactly uncertainty is quantified.
    To see that this is true, consider eq. \eqref{eq:Zellner} but leave out the \KLD from the optimization problem. This yields the (non-\RoT) problem
    \begin{IEEEeqnarray}{rCl}
        \widehat{q}(\*\theta)
        & = &
        \argmin_{q \in \mathcal{P}(\*\Theta)}\left\{ 
            \mathbb{E}_{q}\left[ \sum_{i=1}^n \ell(\*\theta, x_i) \right]
        \right\}.
        \label{eq:D_quantifies_uncertainty}
    \end{IEEEeqnarray}
    Denoting $\widehat{\*\theta}_n = \arg\min_{\*\theta \in \*\Theta}\left\{ \sum_{i=1}^n \ell(\*\theta, x_i)\right\}$ as the empirical risk minimizer (maximum likelihood estimate if the loss is a negative log likelihood) and $\delta_{y}(x)$ as the Dirac measure at $y$, it is immediately clear that
    $\widehat{q}(\*\theta) = \delta_{\widehat{\*\theta}_n}(\*\theta)$, which holds as $\delta_{\widehat{\*\theta}_n} \in \mathcal{P}(\*\Theta)$. 
    %
    \label{remark:axiom:representation}
\end{remark}
The \RoT is a corner stone of our contribution: Based on the axiomatic development, we argue that posteriors should take the form $P(\ell, D, \Pi)$. 
The most practically feasible versions of these problems correspond to the case where $\Pi$ is a $\*\kappa$-parameterized family of distributions $\mathcal{Q} = \{q(\*\theta|\*\kappa): \*\kappa \in \*K\}$. Beliefs of this form are Generalized Variational Inference (\GVI) posteriors. They are a special case of the \RoT that we explore in Section \ref{sec:GVI}.

On top of the axiomatic foundation from which the \RoT directly originates,  
the next two subsections provide additional reasons for designing posterior beliefs according to the form $P(\ell, D, \Pi)$.
First, we give an interpretation of the three components of the \RoT. In particular, we show that they directly address the three issues \ref{violation:prior}, \ref{violation:likelihood} and \ref{violation:computation} via an intuitive modularity result.
Second, we recover some existing Bayesian methods as special cases of the \RoT. We discuss the meaning of a Bayesian method (not) being representable via $P(\ell, D, \Pi)$ and use this as a springboard to motivate \GVI.

\subsection{Modularity of the Rule of Three}
\label{sec:RoT-addresses-traditional-issues}

Taking another look at the constituent parts 
of $P(\ell, D, \Pi)$, it becomes clear that each component of the optimization problem serves a specific and separate purpose. In particular, posteriors generated by the \RoT have three ingredients.
\begin{itemize}
    \myitem{\textbf{(\LikelihoodS})}
    A \textbf{loss} $\ell:\*\Theta\times \mathcal{X}\to\mathbb{R}$. The loss \textit{defines} the parameter of interest $\*\theta$ by linking it to the observations $x_{1:n}$. 
    Throughout, we make a number of assumptions on the loss. None of these assumptions are required, but they significantly simplify the presentation: 
    Firstly, we assume that the losses are additive and identical over all observations\footnote{
        The losses are in fact not required to be identical and we can easily replace $\ell(\*\theta, x_i)$ by $\ell_i(\*\theta, x_i)$. For example, one can set $\ell_i(\*\theta, x_i) =\ell(\*\theta, x_i|x_{1:i-1})$. Here, the $x_i$-th observation is conditioned on the first $i-1$ observations as is common in time series models. More generally, any conditional dependence structure is easily incorporated into the \RoT at the expense of complicating notation, see also \citet{GVIConsistency}.
    }. 
    Secondly, we assume that the loss depends on a parameter $\*\theta$ rather than a latent variable\footnote{
        Except in the experiments on Deep Gaussian Processes in Section \ref{sec:experiments}. Here, the losses are in fact directly defined relative to latent variables.
    }.
    Thirdly, we assume that the losses are deterministic and do not depend on unknown (local or global) latent variables\footnote{
        While latent variable models are not the focus of the current paper, the \RoT and \GVI are easily extended to the latent variable case, see \citet{GVIConsistency}.
    }.
    \myitem{\textbf{(\PriorS})}
    A divergence ${D}:\mathcal{P}(\*\Theta)\times \mathcal{P}(\*\Theta) \to \mathbb{R}_+$ that imposes a cost for the posterior to deviate too much from the prior $\pi(\*\theta)$.
    Recalling Remark \ref{remark:axiom:representation} and specifically eq. \eqref{eq:D_quantifies_uncertainty}, it is clear that
     $D$ determines how uncertainty about $\*\theta$ is quantified with the posterior. Accordingly, we also call $D$ the \textbf{uncertainty quantifier}. 
    \myitem{\textbf{(\ComputationS})}
    A set of \textbf{feasible posteriors $\Pi \subseteq \mathcal{P}(\*\Theta)$} the objective specified by the \RoT is minimized over. 
    The word ``feasible'' here is used in the optimization sense: As the form of $P(\ell, D, \Pi)$ reveals, any $q(\*\theta) \in \Pi$ is a feasible solution (i.e. posterior). 
    %
\end{itemize} 
From this, it is clear that $P(\ell, D, \Pi)$ has a modular interpretation and decomposes into three parts with distinct functions.
Moreover, taking another glance at the problems \ref{violation:prior}, \ref{violation:likelihood} and \ref{violation:computation}, observe that each of the arguments of $P(\ell, D, \Pi)$ addresses one of the concerns raised in Section \ref{sec:Examining_Bayesian_paradigms}:
Firstly, as the loss $\ell$ determines the parameter, it can be used to tackle model misspecification and other violations of \ref{assumption:lklhood}. 
Secondly---and assuming one has already specified the best possible prior belief that is available and/or computationally feasible---the uncertainty quantifier $D$ can be deployed to change the way in which priors influence inference, directly tackling issues surrounding \ref{assumption:prior}. Accordingly, changing $D$ shifts the uncertainty quantification about $\*\theta$.
Thirdly, the space of potential posterior beliefs $\Pi$ can be chosen in such a way as to address the problems with \ref{assumption:computation}: The more computational power is available, the more complex $\Pi$ can become.
In fact, making this modularity conceptually more precise, we arrive at the following result:
\begin{theorem}[\RoT modularity]
    {
        %
        Hold $n, \pi(\*\theta)$ and  $\Pi$ fixed and
        take $q_1^{\ast}(\*\theta) \in \Pi$ as a posterior computed via $P(\ell, D,  \Pi)$.
        If one wishes to derive an alternative posterior $q_2^{\ast}(\*\theta) \in \Pi$ through the \RoT
        \begin{itemize}
            \item[(1)]
             which avoids or is robust to model misspecification, this amounts to changing $\ell$.
            \item[(2)]
            which is robust to prior misspecification without changing the parameter of interest, this amounts to changing $D$.
            \item[(3)]
            which affects uncertainty quantification without changing the parameter of interest, this amounts to changing $D$.
        \end{itemize}
    } \label{Thm:GVI_modularity}
\end{theorem}
\begin{remark}
    The proof can be found in the Appendix \ref{Appendix:modularity_proof}. While it is easy to prove, it requires carefully defining robustness to model misspecification as in \citet{Tukeey1960} and thus is somewhat laborious.
\end{remark}

\begin{table}[th!]
\begin{center}
\begin{small}
\begin{tabular}{ p{4.45cm}p{4.25cm}p{2.5cm}p{1.5cm}}
    \multicolumn{1}{l}{Method} & 
    \multicolumn{1}{l}{$\ell(\*\theta, x_i)$} & \multicolumn{1}{l}{$D$} & 
    \multicolumn{1}{l}{$\Pi$}  \\[0.1cm]
    \hline\hline &&&\\[-0.25cm]
    Standard Bayes & 
        $-\log p(x_i|\*\theta)$ & 
        \KLD & 
        $\mathcal{P}(\*\Theta)$ \\[0.1cm]
    Power Likelihood Bayes$^{1}$ & 
        $-\log p(x_i|\*\theta)$ &
        $\frac{1}{w}$\KLD, $w<1$ & 
        $\mathcal{P}(\*\Theta)$ \\[0.1cm]
    Composite Likelihood Bayes$^{2}$ & 
        $-w_i\log p(x_i|\*\theta)$ &
        \KLD & 
        $\mathcal{P}(\*\Theta)$ \\[0.1cm]
    Divergence-based Bayes$^{3}$ & 
        divergence-based $\ell$ & 
        \KLD & 
        $\mathcal{P}(\*\Theta)$ \\[0.1cm]
    PAC/Gibbs Bayes$^{4}$  & 
        any $\ell$ & 
        \KLD & 
        $\mathcal{P}(\*\Theta)$ \\[0.1cm]
    \VAE{}$^{5,\dagger}$ & 
        $-\log p_{\*\zeta}(x_i|\*\theta) $ &
        $\KLD$  & 
        $\mathcal{Q}$ \\[0.1cm]
    $\beta$-\VAE{}$^{6,\dagger}$ & 
        $-\log p_{\*\zeta}(x_i|\*\theta)$ &
        $\beta\cdot\KLD$, $\beta > 1$  & 
        $\mathcal{Q}$ \\[0.1cm]
    Bernoulli-\VAE{}$^{7,\dagger}$ & 
        continuous Bernoulli &
        $\KLD$  & 
        $\mathcal{Q}$ \\[0.1cm]
    \textcolor{VIColor}{\textbf{Standard \VI}} & 
        $-\log p(x_i|\*\theta)$ & 
        \KLD & 
        $\mathcal{Q}$ \\[0.1cm]
    Power \VI{}$^{8}$ & 
        $-\log p(x_i|\*\theta)$ & 
        $\frac{1}{w}$\KLD, $w<1$ & 
        $\mathcal{Q}$ \\[0.1cm]
    Utility \VI{}$^{9}$ & 
        $-\log p(x_i|\*\theta) + \log u(h, x_i)$ &
        \KLD & 
        $\mathcal{Q}$ \\[0.1cm]
    %
    Regularized Bayes$^{10}$ & 
        $-\log p(x_i|\*\theta) + \phi(\*\theta, x_i)$ &
        \KLD & 
        $\mathcal{Q}$ \\[0.1cm]
    Gibbs \VI{}$^{11}$ & 
        any $\ell$ & 
        \KLD & 
        $\mathcal{Q}$ \\[0.1cm]
    %
    \textcolor{GVIColor1}{\textbf{Generalized \VI}} & 
    any $\ell$ & 
    any $D$ & 
    $\mathcal{Q}$ \\[0.1cm]
    \hline \hline  \\
\end{tabular}
\end{small}
{\renewcommand{\arraystretch}{1.2}
\caption{Relationship of $P(\ell, D, Q)$ to a selection of existing methods. 
$^{1}$\citep[e.g.][]{SafeLearning,SafeBayesian,holmes2017assigning,InconsistencyBayesInference,DunsonCoarsening},  
$^{2}$\citep[e.g.][]{compositeLikelihoodOverview, compositeLikelihood1, compositeLikelihood2, OllieDGP},
$^{3}$\citep[e.g.][]{MinDisparities, GoshBasuPseudoPosterior, AISTATSBetaDiv, Jewson, MMDBayes},
$^{4}$\citep{Bissiri,PACmeetsBayesianInference, PACPrimer}, 
$^{5}$\citep[][]{VAE}, 
$^{6}$\citep[][]{bVAE}, 
$^{7}$\citep{bernoulliVAE}
$^{8}$\citep[e.g.][]{alpha-VI, AnnealedVI}
$^{9}$\citep[e.g.][]{VIUtility, VIUtilityOld}
$^{10}$(\citet{RegBayes1}, but only if the regularizer can be written as $\mathbb{E}_{q(\*\theta)}\left[\phi(\*\theta, \*x)\right]$ as in \citet{RegBayes2}), 
$^{11}$\citep[e.g.][]{VBConsistencyAlquier2}
$^{\dagger}$For the \VAE entries in the table, we abuse notation by denoting the local latent variable for $x_i$ as $\*\theta$. Further, $\*\zeta$ denote the generative parameters.
}
\label{table:inference_methods}
}
\end{center}
\vskip -0.1in
\end{table}

\subsection{Connecting the Rule of Three to existing methods}
\label{sec:insight_RoT}

Beyond axiomatic foundations and the interpretable modularity result of the last section, the form $P(\ell, D, \Pi)$ also sheds light on existing methods. As we will see, most existing Bayesian methods are special cases of the \RoT. 
However, certain posterior beliefs derived as approximations to $q^{\ast}_{\text{B}}(\*\theta)$ are \textit{not}.
This section explains this distinction, how it is relevant and why it provides a theoretical argument for the construction of
posterior belief distributions through Generalized Variational Inference (\GVI).

As Table \ref{table:inference_methods} shows, an impressive array of Bayesian methods are recovered by the \RoT. This includes a wide range of approxiate Bayesian methods and in particular \textcolor{VIColor}{\textbf{standard \VI}}. In contrast, alternative approximation methods such as Laplace Approximations, Discrepancy Variational Inference (\DVI), Expectation Propagation (\EP) or  \VI posteriors derived from Generalized Evidence Lower Bounds do \textit{not} satisfy these  axiomatic desiderata.

%


\subsubsection{Methods as special cases of the \RoT}

Unsurprisingly, as the \RoT was constructed as a strict generalization of standard Bayesian inference methods, most existing Bayesian methods are recovered as special cases. 
Further, by virtue of incorporating the space $\Pi$ into the objective, $P(\ell, D, \Pi)$ also recovers many posterior beliefs that are constructed as approximations to the Bayesian posterior.
%
Table \ref{table:inference_methods} gives a non-exhaustive overview over {some} of belief distributions the \RoT recovers as special cases.

One of the key findings of the table is that \textcolor{VIColor}{\textbf{standard \VI}} satisfies the axiomatic foundations underlying the \RoT.
In other words, the \RoT does \textit{not} judge the  belief distribution of eq. \eqref{eq:generalized_bayes_posterior} derived from Bayes rule to be preferable to certain classes of approximations \textit{by default}. 
The reason for this is simple: Unlike the traditional Bayesian paradigm, the \RoT can explicitly encode the availability of finite or infinite computational resources through the argument $\Pi$.
Hence, the moment the computational resources are scarce and posterior beliefs can only be computed over a parameterized subset $\mathcal{Q} \subset \mathcal{P}(\*\Theta)$, \textcolor{VIColor}{\textbf{standard \VI}} produces the best computationally feasible 
posteriors  relative to the objective in eq. \ref{eq:Zellner}.
In this sense, the \RoT respects the optimality result of \textcolor{VIColor}{\textbf{standard \VI}} presented in Theorem \ref{thm:VI_optimality}.

\subsubsection{Coherence and the \RoT}
\label{sec:coherence}

Another insightful observation is that unlike previous generalizations such as the Generalized Bayesian update in eq. \eqref{eq:generalized_bayes_posterior}, posterior belief distributions generated under the \RoT are allowed to break a property referred to as \textit{coherence} or \textit{Bayesian additivity} \citep[e.g.][]{Bissiri, Edwin}.
In a nutshell, coherence says that posterior beliefs have to be generated according to some function $\psi:\mathbb{R}^2 \to \mathbb{R}$ which for the prior $\pi(\*\theta)$ and loss terms $\ell(\*\theta, x_1), \ell(\*\theta, x_2)$ behaves as
\begin{IEEEeqnarray}{rCl}
    \psi\left( \ell(\*\theta, x_2), \psi\left(\ell(\*\theta, x_1), \pi(\*\theta) \right) \right)
    & = &
    \psi\left(\ell(\*\theta, x_1), \ell(\*\theta, x_2), \pi(\*\theta) \right) 
    \nonumber.
\end{IEEEeqnarray}
Effectively, this property ends up enforcing a multiplicative update rule and exponential additivity as in eq. \eqref{eq:generalized_bayes_posterior}.
Recalling the construction of the standard Bayesian posterior in Section \ref{sec:traditional_bayesian_paradigm}, it should be clear that the desirability of coherence is a \textit{direct} result of assuming \ref{assumption:prior} and \ref{assumption:computation}. 
To see that this intuition is accurate, note that treating the prior belief according to \ref{assumption:prior} and assuming infinite computational power via \ref{assumption:computation} is exactly equivalent to setting $D = \KLD$ and $\Pi = \mathcal{P}(\*\Theta)$. Next, solving eq. \eqref{eq:Zellner} with these specifications as in the proof of Theorem \ref{thm:Zellner}, one obtains the coherent exponentially additive update rule in eq. \eqref{eq:generalized_bayes_posterior}. 
In other words, generating posteriors that can violate coherence is identical to generating posteriors that do not have to rely on \ref{assumption:prior} and \ref{assumption:computation}.
Since this is precisely what we set out to do in the first place, this is in fact a desirable property for posteriors generated by the \RoT.

\subsubsection{Links to information theory and PAC-Bayes}

It will be worthwhile exploring the interesting links between the \RoT and alternative methods for the construction of belief distributions in the future. Here, we simply note in passing that special cases of the form $P(\ell, D, \Pi)$ were arrived at with strong theoretical arguments starting from at least two completely different paradigmatic bases.

\subparagraph{PAC-Bayes:}
While PAC-Bayesian results often have intimate links with Bayesian inference \citep[see e.g.][]{PACmeetsBayesianInference, InconsistencyBayesInference}, their motivations and origins are distinct \citep[see e.g.][]{ShaweTaylor, PACPrimer}:
Unlike Bayesian inference, PAC-Bayesian results are not constructed based on \ref{assumption:prior} and \ref{assumption:lklhood}.
Rather, their aim is the derivation of generalization bounds for belief distributions $q(\*\theta)\in \mathcal{P}(\*\Theta)$ defined over some hypothesis space (corresponding to the parameter space $\*\Theta$) relative to a loss function (corresponding to $\ell$).
For example, under a prior belief $\pi(\*\theta)$, a loss $\ell$ and a data generating mechanism for $x_{1:n}$ satisfying appropriate regularity conditions and for any $q(\*\theta)\in \mathcal{P}(\*\Theta)$ as well as for any fixed value of $\varepsilon > 0$, a rescaled version of McAllester's seminal bound \citep{McAllester1, McAllester2} says that with probability at least $1-\varepsilon$, 
\begin{IEEEeqnarray}{rCl}
            \mathbb{E}_{q(\*\theta)}\left[ 
                n\cdot \mathbb{E}_{\*x_{1:n}}\left[ 
                    \frac{1}{n}\sum_{i=1}^n\ell(\*\theta, \*x_i)
                \right]
            \right]
            \leq
            \mathbb{E}_{q(\*\theta)}\left[ 
                        \sum_{i=1}^n\ell(\*\theta, x_i)
                \right] 
            + 
            \sqrt{
                \dfrac{
                    \KLD(q, \pi) + \log\frac{2\sqrt{n}}{\varepsilon}
                }{
                    2n^{-0.5}
                }
            }.
    \nonumber
\end{IEEEeqnarray}
Minimizing the right hand side of this bound with respect to $q(\*\theta)$ over some set $\Pi \subseteq \mathcal{P}(\*\Theta)$ immediately recovers a special case for the \RoT given by $P(\ell, D_{\text{McAllester}}, \Pi)$.  Here, $D_{\text{McAllester}}$ is just the \KLD term of the original bound with a subtracted constant: 
\begin{IEEEeqnarray}{rCl}
    D_{\text{McAllester}}(q\|\pi) & = &
    \sqrt{
        \dfrac{
            \KLD(q, \pi) + \log\frac{2\sqrt{n}}{\varepsilon}
        }{
            2n^{-0.5}
        }
    }
    -
    \sqrt{
        \dfrac{
            \log\frac{2\sqrt{n}}{\varepsilon}
        }{
            2n^{-0.5}
        }
    }
    \nonumber
\end{IEEEeqnarray}
This is done to ensure that $D_{\text{McAllester}}(q\|\pi) = 0$ if and only if $\pi(\*\theta) = q(\*\theta)$ (almost everywhere). Note that it does not effect the minimizer of the right hand side. 
A similar logic can be applied to virtually all PAC-Bayesian bounds, including bounds based on alternative divergences \citep{PACRenyiDiv, PACfDiv}. 
This suggests that there are insightful connections between PAC-Bayes and the \RoT. 
Further, \GVI{}---the tractable case when $\Pi = \mathcal{Q}$ is a parameterized subset of $\mathcal{P}(\*\Theta)$---is a promising way forward to scale and operationalize PAC-Bayesian learning.
In fact, \citet{Dichotomize} constitutes the first step in this direction and we will explore these connections in future work.

\subparagraph{Information Theory:}
Another strong link exists between the \RoT and results derived from Information Theory and specifically the predictive Information Bottleneck \citep[see][]{infoBottleneckMethod, infoBottleneckMethod2, AlemiBottleneck}.
The name arises from the fact that this problem seeks to compress all information of the (possibly infinitely many) observations in $\*x_P$ into the finite-dimensional quantity $\*\theta$ which maximizes the amount of information about (possibly infinitely many) future observations denoted as $\*x_F$.
More concisely, the predictive Information Bottleneck is given by the infinite-dimensional optimization problem
\begin{IEEEeqnarray}{rCl}
    \max_{p(\*\theta|\*x_P)} I(\*\theta; \*x_F) \quad \text{ s.t. } \quad I(\*\theta; \*x_P) = I_0.
    \nonumber
\end{IEEEeqnarray}
Here, $I$ denotes predictive information of some kind---typically mutual information---so that $I(\*\theta; \*x_F)$ quantifies the amount of information the quantity $\*\theta$ can tell us about the future while $I(\*\theta; \*x_P)$ is a measure of its complexity.
As shown in \citet{AlemiBottleneck}, a straightforward application of Lagrangian optimization together with a natural variational bound readily transforms this problem into an optimization problem of the form $P(\ell, D, \Pi)$.

\section{Generalized Variational Inference (\GVI)}
\label{sec:GVI}

The next section finally introduces a version of the \RoT that is feasible for real-world inference and that we call Generalized Variational Inference (\GVI). As the name suggests, \GVI is any posterior distribution generated from $P(\ell, D, \mathcal{Q})$, where $\mathcal{Q}$ is a parameterized subset of $\mathcal{P}(\*\Theta)$ as given in eq. \eqref{eq:variational_family}. 
The remainder proceeds as follows:
\begin{itemize}
    \item[] \textbf{Section \ref{sec:GVI_advantages}} 
        motivates why \GVI procedures generate conceptually appealing and coherent posterior belief distributions.
    \item[] \textbf{Section \ref{sec:GVI_applications}} 
        focuses on the main applications of \GVI in the current paper. In particular, we explain how \GVI can be used for  inferences that (i) are robust to model misspecification, (ii) are robust to prior misspecification and (iii) produce more appropriate marginal distributions.
    \item[] \textbf{Section \ref{sec:GVI_theory}} 
        discusses two strong theoretical guarantees for \GVI: Frequentist consistency and an interpretation as an approximate lower bound on the evidence of a (generalized) Bayesian posterior.
    \item[] \textbf{Section \ref{sec:GVI_inference}}
        focuses on strategies for inference. First, we derive  a closed form objective for a particular set of \GVI posteriors  that are robust to model misspecification. Second, we introduce black box inference for \GVI and prove more results on closed form expressions.
\end{itemize}

\subsection{Advantages of \GVI over alternative \VI methods}
\label{sec:GVI_advantages}

As Table \ref{table:inference_methods} shows, posteriors derived via $P(\ell, D, \Pi)$ recover many Bayesian methods motivated as approximations to the Bayesian posterior.
Crucially however, a large collection of approximation strategies for the standard Bayesian posterior are \textit{not} special cases of the \RoT. 
Unsurprisingly, these are the exact same methods that are suboptimal relative to the objective defining Bayesian inference (see Corollary \ref{corollary:suboptimality}). 
They encompass a variety of strategies we elaborated on in Section \ref{sec:VI_as_ELBO} and \ref{sec:VI_as_KLD_min_and_DVI}. The most prominent of these include \VI based on generalized \ELBO formulations, Laplace approximations, Expectation Propagation (\EP) as well as Discrepancy Variational Inference (\DVI).
%

We stress that we do \textit{not} claim that these alternative posterior approximations are always performing worse than \VI in practice---this is \textit{not} what Corollary \ref{corollary:suboptimality} says. 
What we \textit{do} claim and prove is that such alternative approximations do \textit{not} provide the optimal posterior relative to eq. \eqref{eq:Zellner}---an equation that is isomorphic with the Bayesian posterior.
This is an important distinction because it means that the empirical success of alternative approximations has a theoretically appealing interpretation: 
As explained in some detail in Section \ref{sec:reconciling_suboptimality}, if alternative approximations $q_{\text{A}}^{\ast}(\*\theta)$ to the Bayesian posterior $q_{\text{B}}^{\ast}(\*\theta)$ perform better than the standard variational approximation $q^{\ast}_{\VI}(\*\theta)$, the objective underlying $q_{\text{A}}^{\ast}(\*\theta)$ must be targeting a more appropriate posterior belief.
Thus, undesirable inference outcomes with standard \VI are synonymous with an inappropriately designed objective. 
Following this line of reasoning, the most transparent way to improve poor performance of standard \VI posteriors is a \textit{direct} adjustment of the objective that generated them.
%
%
Conveniently, the axiomatic development of Section \ref{sec:RoT} provides a solution of precisely  this kind: The \RoT, which defines a versatile and modular family of objectives useful for interpretably adapting the standard \VI objective, a technique we call Generalized Variational Inference (\GVI). 

\begin{definition}[Generalized Variational Inference (\GVI)]
    Solving any \RoT of form $P(\ell, D,  \mathcal{Q})$ for $\mathcal{Q} = \{q(\*\theta|\*\kappa): \*\kappa \in \*K\}$ a parameterized subset of $\mathcal{P}(\*\Theta)$ (also called a variational family) constitutes a procedure we call Generalized Variational Inference (\GVI).
    \label{Def:GVI}
\end{definition}
\begin{remark}
    Notice that \GVI posteriors satisfy Axioms \ref{Axiom:representation}, \ref{Axiom:translationInvariance} and \ref{Axiom:informationAdjustment}  \textbf{by definition}. It immediately follows that they also inherit the modularity result in Theorem \ref{Thm:GVI_modularity}.
    By the same token, approximations $q^{\ast}_{\text{A}}(\*\theta)$ to the Bayesian posterior that cannot be rewritten as a \GVI procedure will \textbf{violate} the Axioms set out in Section \ref{sec:axioms}. 
    Clearly, this need not be a problem if $\mathcal{Q}$ is a sufficiently rich set of approximating distributions: In this case, the approximation $q^{\ast}_{\text{A}}(\*\theta) \approx q^{\ast}_{\text{B}}(\*\theta)$ is very reliable. 
    As $q^{\ast}_{\text{B}}(\*\theta)$ itself satisfies the Axioms in Section \ref{sec:axioms}, $q^{\ast}_{\text{A}}(\*\theta)$ will not violate the Axioms in a practically meaningful way. 
    
    In practice however, $\mathcal{Q}$ is typically hopelessly restrictive. Consequently, it does not contain any qualitatively good approximations to $q^{\ast}_{\text{B}}(\*\theta)$. 
    In this case, violation of the Axioms is often a serious problem. 
    Example \ref{example:label_switching} and Figure \ref{Fig:MultiModality} illustrate this on a Bayesian Mixture Model (\BMM{}).
    We also revisit this issue with our experiments in Section \ref{sec:experiments_BNN}, where we observe its real world consequence on Bayesian Neural Networks.
    %
    \label{remark:approximations_vs_'approximations'}
\end{remark}

\begin{example}[Label switching and multi-modality]
A recurrent theme in the research on variational approximations $q^{\ast}_{\text{A}}(\*\theta)$ to $q^{\ast}_{\text{B}}(\*\theta)$ is the observation that standard \VI with a mean field normal family will center closely around the maximum likelihood estimate \citep[e.g.][]{TurnerSahani11}.
This phenomenon is often referred to as the \textbf{zero-forcing} behaviour of the \KLD \citep{MinkaDivergences}. Its effect are  undesirably overconfident variational posteriors $q^{\ast}_{\VI}(\*\theta)$. Moreover, this problem is especially pronounced when  the approximated posterior beliefs $q^{\ast}_{\text{B}}(\*\theta)$ are multi-modal.
%
Popular approaches to address this issue are Expectation Propagation (\EP) \citep{EP, EP2} and Divergence Variational Inference (\DVI) methods as introduced in Section \ref{sec:VI_as_KLD_min_and_DVI} \citep[e.g.][]{AlphaDiv, RenyiDiv, ChiDiv}.
All of these approaches seek to (locally or globally) minimize an alternative \textbf{zero-avoiding} divergence $D$ between $q^{\ast}_{\text{A}}(\*\theta)$ and $q^{\ast}_{\text{B}}(\*\theta)$.
Unfortunately, none of these approaches satisfy the Axioms set out in \ref{sec:axioms}. 
An immediate consequence is that the modularity encoded by Theorem \ref{Thm:GVI_modularity} does \textit{not} apply to \EP or \DVI.
This may seem inconsequential, but it really is not: Unlike with \GVI, \textbf{changing the divergence in the \DVI-sense no longer  affects uncertainty quantification alone}.
In other words, we may accidentally interfere with the loss and warp the way the goodness of a parameter value $\*\theta$ is assessed in undesirable ways.

Using Bayesian mixture models (\BMM{}s), we show that problems associated with violating the Axioms or Section \ref{sec:axioms} indeed occur in practice.
\BMM{}s produce multi-modal posteriors because the likelihood function is invariant to switching parameter labels. 
In other words, \BMM{}s have multiple parameter values that constitute equally good fits to the data.
With this in mind, we simulate $n=100$ observations from the model 
\begin{IEEEeqnarray}{rCl}
    p\left(x|\*\theta = (\mu_1, \mu_2)\right) &=&
    0.5\cdot\mathcal{N}(x|\mu_1,0.65^2)+0.5\cdot\mathcal{N}(x|\mu_2,0.65^2)
    \nonumber
\end{IEEEeqnarray}
for two different settings of $\*\theta= (\mu_1, \mu_2)$.
For inference, we use the well-specified prior belief $\mu_j \sim \mathcal{N}(0,2^2)$, $j=1,2$.  
Using the correctly specified likelihood function $\ell(\*\theta, x_i) = -\log p\left(x_i|\*\theta = (\mu_1, \mu_2)\right)$, we then compare the standard Bayesian posterior $q^{\ast}_{\text{B}}(\*\theta)$, the standard \VI posterior, the \DVI posterior based on R\'enyi's $\alpha$-divergence (\RAD) \citep{RenyiDiv} and the \GVI posterior using $D=\RAD$ as uncertainty quantifier.
For all posteriors produced by variational methods, we use the mean field normal family for $\mathcal{Q}$.

Figure \ref{Fig:MultiModality} shows the results. 
Clearly, $q^{\ast}_{\text{B}}(\*\theta)$ is multi-modal because there are two equally good parameter values describing the data by virtue of the fact that $p\left(x|\*\theta = (\mu_1, \mu_2)\right) = p\left(x|\*\theta = (\mu_2, \mu_1)\right)$.
Because we enforce uni-modality through $\mathcal{Q}$, the variational posteriors have a clear interpretation: Firstly, their means have the interpretation of (one of the two) best parameter values of $\*\theta = (\mu_1, \mu_2)$. Secondly, their variances quantify the uncertainty about this best value.
For both settings of the true value for $\*\theta$, \DVI produces a posterior that reflects a highly undesirable belief: In particular, the mode of the \DVI posterior is located at a (locally) worst value of $\*\theta$. 
Unsurprisingly and as the bottom right plot shows, this  has adverse consequences for predictive performance.
This behaviour is entirely attributable to the fact that unlike \GVI posteriors, \DVI posteriors violate the axiomatic foundation set out in Section \ref{sec:axioms} and thus do not inherit the modularity result of Theorem \ref{Thm:GVI_modularity}.
%
%
In other words: In the \GVI framework, changing the \KLD to another divergence only changes uncertainty quantification and \textbf{does not} affect the way the best parameter is found. In sharp contrast, the \DVI framework comes with no such guarantee. Accordingly, posteriors produced with \DVI may conflate uncertainty quantification and the way the best parameter is found.
%
In this context, Figure \ref{Fig:MultiModality} serves as a morality tale: The modularity of standard \VI and \GVI ensures transparency. In contrast, \DVI methods have no such guarantees and can easily end up conflating the search for the best parameter with uncertainty quantification\footnote{Clearly, this problem is not guaranteed to occur: As
Remark \ref{remark:approximations_vs_'approximations'} explains, 
the conflation of loss and uncertainty quantification will not be a problem so long as $\mathcal{Q}$ is sufficiently rich to produce approximations that endow the statement $q^{\ast}_{\text{A}}(\*\theta) \approx q^{\ast}_{\text{B}}(\*\theta)$ with meaning.}.

 \begin{figure}[t!]
 \begin{center}
 \includegraphics[trim= {0.0cm 0.5cm 0.5cm 0.5cm}, clip,   width=0.49\columnwidth]{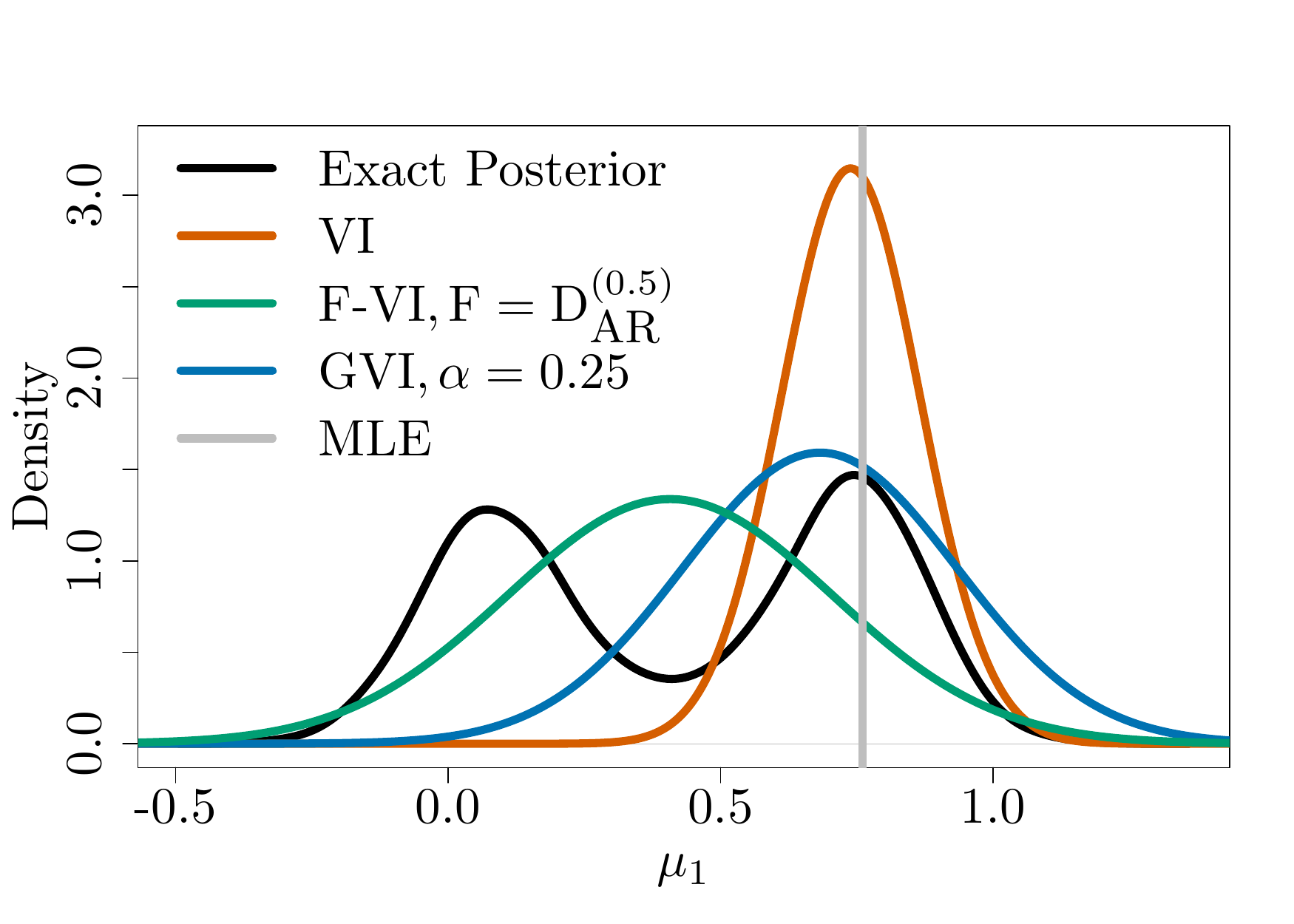}
 \includegraphics[trim= {0.0cm 0.5cm 0.5cm 0.5cm}, clip,   width=0.49\columnwidth]{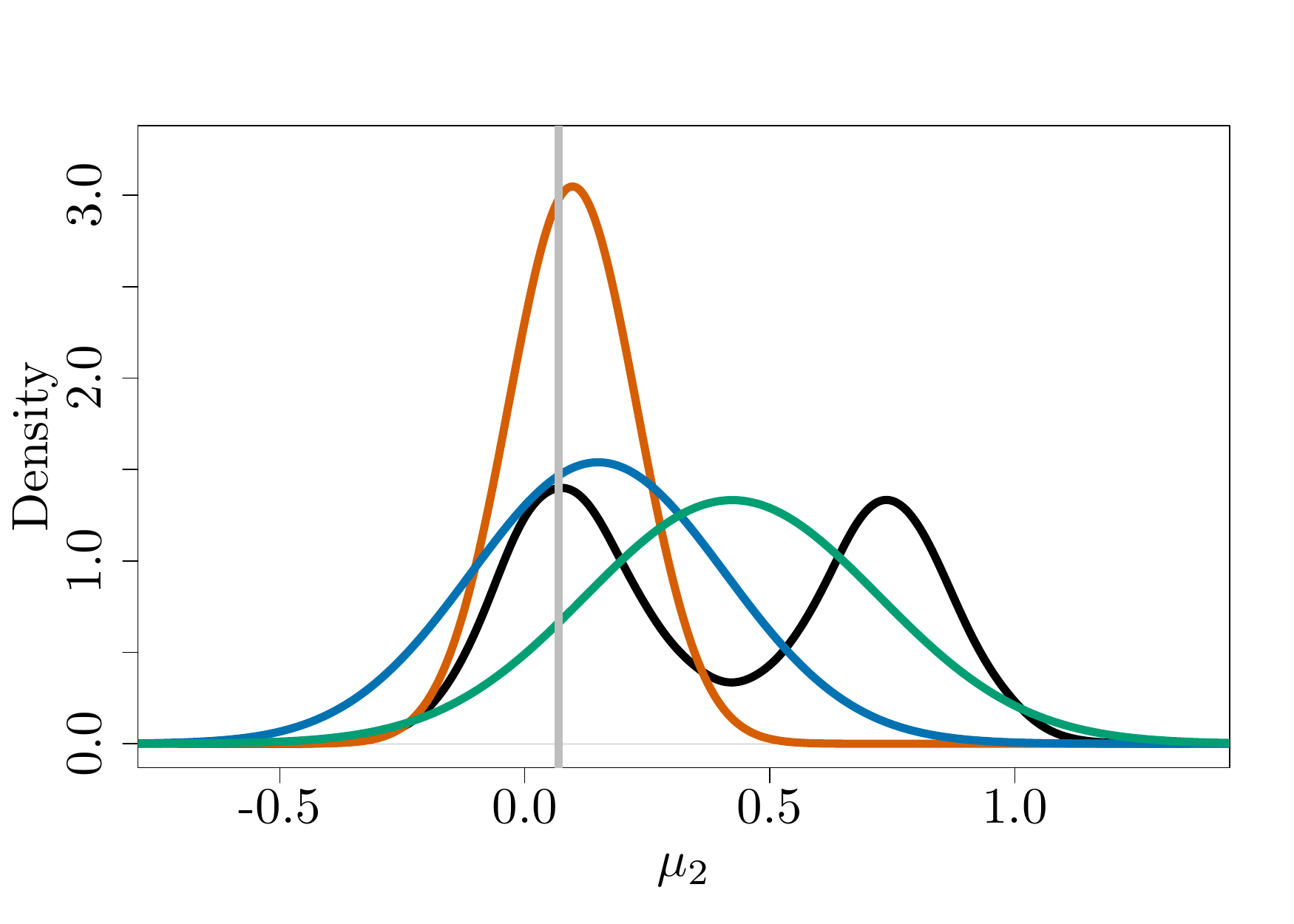}
 \includegraphics[trim= {0.0cm 0.5cm 0.5cm 0.5cm}, clip,   width=0.49\columnwidth]{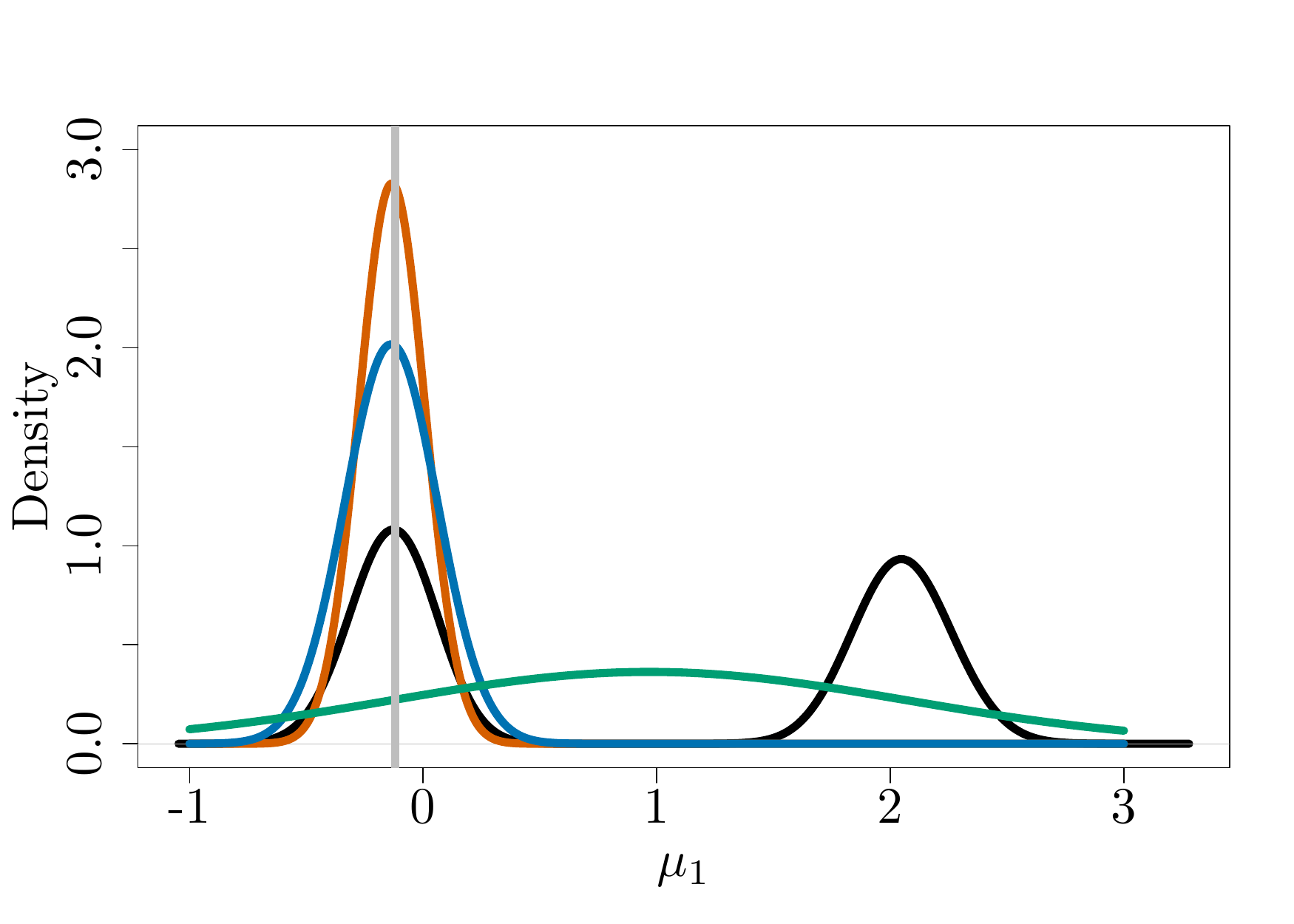}
 \includegraphics[trim= {0.0cm 0.5cm 0.5cm 0.5cm}, clip,   width=0.49\columnwidth]{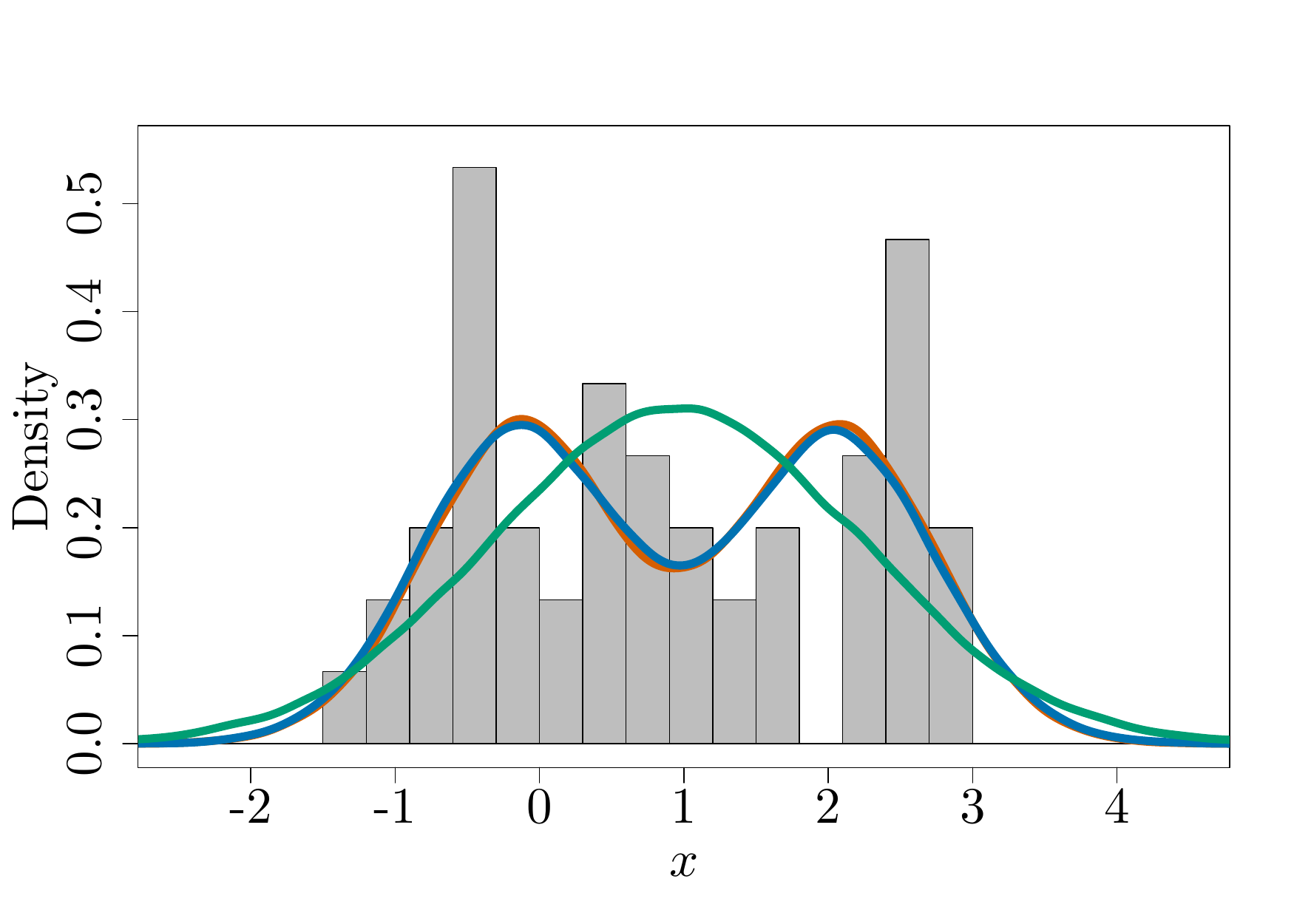}
 \caption{
 \bv
 Depicted are the outcomes of inference in a \BMM{} model, namely the (multimodal) \textbf{exact Bayesian} inference posterior, \VIColor{\textbf{standard \VI{}}}, \DVIColor{\textbf{\DVI}} with \DVIColor{\RAD} \citep{RenyiDiv} and \GVIColor{\textbf{\GVI}} with \GVIColor{\RAD} as uncertainty quantifier. 
 \textbf{Top:} Posterior marginals for $\mu_1 = 0, \mu_2 = 0.75$. The Maximum A Posteriori estimate of the \DVIColor{\textbf{\DVI}} posterior yields a locally worst value for $\*\theta$ relative to the \textbf{exact Bayesian} posterior. In contrast, \VIColor{\textbf{standard \VI}} and \GVIColor{\textbf{\GVI}} respect the loss: They produce a posterior belief centered around one (of the two) values of $\*\theta$ minimizing the loss. 
 \textbf{Bottom left:} Posterior marginal for $\mu_1= 0, \mu_2 = 2$. The effects of the top row become even stronger as the modes move further apart. 
 \textbf{Bottom right:}
 Posterior predictive for $\mu_1= 0, \mu_2 = 2$ against the histogram depicting the actual data. 
 \VIColor{\textbf{\VI}}, \GVIColor{\textbf{\GVI}} and \textbf{exact Bayesian} inference perform well and almost identically.
 \DVIColor{\textbf{\DVI}} performs poorly and captures none of the two mixture components of the \BMM{}. 
 %
 }
 \label{Fig:MultiModality}
 \end{center}
 \end{figure}

\label{example:label_switching}
\end{example}

\subsection{\GVI use cases: Robust inference \& better marginals}
\label{sec:GVI_applications}

Summarizing the findings of the last paragraphs, saying that the standard \VI posterior produces undesirable inferences or inappropriate uncertainty quantification amounts to saying that the objective in eq. \eqref{eq:standard_VI_ELBO} is inappropriate for the inference task at hand.
By virtue of the modularity result in Theorem \ref{Thm:GVI_modularity}, it is also clear that \GVI is uniquely positioned to define alternative objectives to address this in a direct and targeted fashion. 
Inspired by this, we next focus on three situations that cause problems for standard \VI, but are easily solved by \GVI: Prior misspecificaton \ref{violation:prior}, model misspecification \ref{violation:likelihood}, and marginal variances.

\subsubsection{\GVI and prior misspecification}
\label{sec:GVI_prior_misspecification}

Section \ref{sec:violation_P} outlines the problems associated with violating \ref{assumption:prior}: If the prior does not even approximately reflect the best available judgement, inference outcomes are adversely affected.
This phenomenon is particularly pronounced whenever the prior is specified according to some default setting. For example, in the case of Bayesian Neural Networks (\BNN{}s), a typical choice of prior is a standard normal distribution that factorizes over the network weights.
While this may seem harmless or even uninformative, a supposedly uninformative prior specification of this kind actually encompasses a large degree of information, e.g.
\begin{itemize}
    \myitem{\textbf{(U)}} The prior belief is \textit{unimodal}. In other words, we believe that there exists a \textit{uniquely most likely} parameterization of the network before observing any data. 
    \label{BNN-prior-assumption:unimodal}
    \myitem{\textbf{(I)}} The prior belief is that all network weights of a \BNN are {uncorrelated}. In fact, we even believe that all network weights of a \BNN are both \textit{pairwise and mutually independent}.\footnote{For joint normal distributions, variables are uncorrelated if and only if they are independent.} 
    \label{BNN-prior-assumption:independence}
\end{itemize}
The above implications are in direct and strong contradiction to our best possible judgements about \BNN{}s and thus violate \ref{assumption:prior}:
\begin{itemize}
    \myitem{\textbf{(\Lightning{}U)}} 
    Neural Networks are well-understood to have multiple parameter settings that are equally good \citep[e.g.][]{NNoptima}. The unimodality assumption outlined in \ref{BNN-prior-assumption:unimodal} is thus clearly not a reflection of the best judgement available: A prior belief in accordance with \ref{assumption:prior} would encode multimodality.
    \myitem{\textbf{(\Lightning{}I)}}  
    By construction, Neural Networks encode a significant degree of dependence in their parameters: The best values for parameters in the $l$-th layer will strongly depend on the best values for parameters in the $(l-1)$-th layer (and vice versa). 
    Hence, assuming uncorrelatedness (much less so independence!) directly contradicts our best judgement.
\end{itemize}

From this, it is obvious that a fully factorized normal distribution is hardly an appropriate default prior for \BNN{}s in the sense of \ref{assumption:prior} in Section \ref{sec:traditional_bayesian_paradigm}.
At the same time, it is often prohibitive or computationally infeasible to construct alternative prior beliefs that reflect our best judgements more accurately.
In other words, we are stuck with a sub-optimal prior.
Under the standard Bayesian paradigm, this is not an acceptable position.
In contrast, the new paradigm outlined in Section \ref{sec:axioms} does not require the prior to be flawless.
Using \GVI, we can thus use our very imperfect prior to design more appropriate posterior beliefs: Simply adapt the argument $D$ which regularizes the posterior belief against the prior.
In particular, we want to adapt $D$ such that the resulting posteriors satisfy two criteria: Firstly, they should be more robust to priors which strongly contradict the observed data. Secondly, they should still provide reliable uncertainty quantification.

As the toy example in Figure \ref{Fig:prior_misspec_BMM} shows, this can be achieved by picking a robust replacement $D_{\text{robust}}$ for the \KLD. A recent overview of robust divergences was provided by \cite{ABCdiv}.
Such divergences are constructed with a hyperparameter that regulates the degree of robustness and which recovers the \KLD as a limiting case. In the current paper, all robust divergences are parameterized such that one recovers the \KLD as the hyperparameter approaches unity. 
For example, R\'enyi's $\alpha$-divergence---henceforth denoted \RAD and introduced by \citet{renyi1961measures}---is such a robust alternative to the \KLD. Using the parameterization in \citet{ABCdiv}, it is given by
\begin{IEEEeqnarray}{rCl}
    \RAD(q\|\pi) & = & \frac{1}{\alpha(\alpha-1)}\log\left( 
            \mathbb{E}_{q(\*\theta)}\left[ \left(\frac{\pi(\*\theta)}{q(\*\theta)}\right)^{1-\alpha} \right]
        \right).
    \label{eq:RAD}
\end{IEEEeqnarray}
Originally, R\'enyi's $\alpha$-divergence was motivated as the \textit{geometric mean} information to discriminate between the two hypotheses $\*\theta \sim \pi(\*\theta)$ and $\*\theta \sim q(\*\theta)$ of order $\alpha$, for some $\alpha \in (0,1)$.
Similarly, the original motivations for the \KLD was its interpretation as the \textit{arithmetic mean} information to discriminate between $\*\theta \sim \pi(\*\theta)$ and $\*\theta \sim q(\*\theta)$ \citep{kullback1951information}.

\begin{figure}[b!]
\begin{center}
\includegraphics[trim= {1.75cm 0.0cm 2.2cm 0.5cm}, clip,  
 width=1\columnwidth]{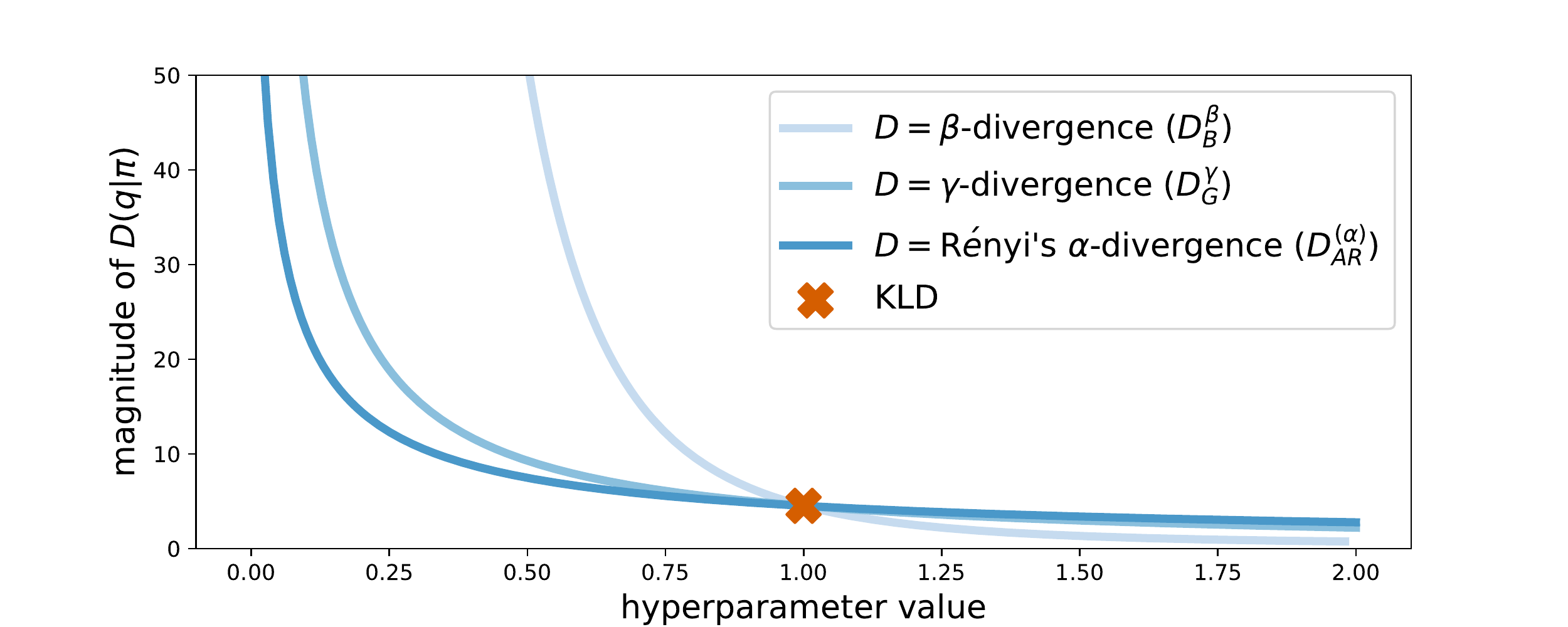}
\caption{
Depicted is the magnitude $D(q\|\pi)$ for different \GVIColor{\textbf{robust divergences $\*D$}} and the \VIColor{\textbf{\KLD}} for two Normal Inverse Gamma distributions given by $q(\*\theta) = \mathcal{NI}^{-1}(\*\theta; \*\mu_q, \*V_q, a_q, b_q)$ and $\pi(\*\theta) = \mathcal{NI}^{-1}(\*\theta; \*\mu_{\pi}, \*V_{\pi}, a_{\pi}, b_{\pi})$ with $\*\mu_{\pi} = (0, 0)^T$, $\*V_{\pi} = 25 \cdot I_2$, $a_{\pi} = 500$, $b_{\pi} = 500$ and $\*\mu_q = (2.5, 2.5)^T$, $\*V_q = 0.3 \cdot I_2$, $a_q = 512$, $b_q = 543$.
}
\label{Fig:div_magnitudes}
\end{center}
\end{figure}

Since geometric means are more robust measures of central tendency than arithmetic means, the \RAD will generally be a more robust measure of discrepancy between $\pi(\*\theta)$ and $q(\*\theta)$ so long as $\alpha \in (0,1)$.
Clearly, the degree of this robustness will depend on $\alpha$. Generally, picking lower values of $\alpha\in (0,1)$ will produce more prior-robust measures of discrepancy. Indeed, $\RAD(q\|\pi)$ recovers $\KLD(q\|\pi)$ as $\alpha \to 1$ and $\KLD(\pi\|q)$ as $\alpha \to 0$.
A host of other divergences behave similarly, including $\alpha$-, $\beta$- and $\gamma$-divergences as well as their generalizations \citep{ABCdiv}. 
%
We plot their magnitude for two Normal Inverse Gamma distributions with different divergence hyperparameter values in Figure \ref{Fig:div_magnitudes}. The plot illustrates that (i) hyperparameter values below (above) unity impose larger (smaller) penalties than the \KLD and that (ii) all robust divergences recover the \KLD as their hyperparameters approach one. 
Importantly, a larger regularization term does {not} necessarily imply that a misspecified prior dominates inference:
In fact, Figure \ref{Fig:prior_misspec_BMM} shows that the contrary holds using the example of \RAD 
with $\alpha = 0.5$.
As Appendix \ref{Appendix:uncertainty_quantification} demonstrates in great detail, R\'enyi's $\alpha$-divergence is representative of the behaviour displayed by various  robust divergences. 
To keep things simple, we thus focus on  \RAD{} 
in the remainder of the paper. 


\subsubsection{\GVI and model misspecification}\label{sec:GVI_model_misspecification}

Section \ref{sec:violation_L} explains how and why \ref{violation:likelihood} can severely impede the usefulness of the Bayesian posterior: Assuming that the likelihood is an accurate reflection of the data generating mechanism makes inferences susceptible to outliers, heterogeneity and other adversarial aspects of the data.
Further recalling the isomorphism between eq. \eqref{eq:Zellner_standard} and eq. \eqref{eq:standard_bayes_rule}, it also is clear that treating the likelihood model as (approximately) correct is synonymous to using the log score $\ell(\*\theta, x_i) = -\log p(x_i|\*\theta)$ to assess how well the model fits the data.

Based on this observation, a promising line of work has sought to produce more robust scoring rules for probability models \citep[see e.g.][and references therein]{dawid2016minimum}.
While there are other ways to derive proper scoring rules, the conceptually most appealing constructions are arguably based on statistical divergences.
The intuition behind this approach consists in two observations: First, if $n$ is large enough, one can invoke the law of large numbers and rewrite the parameter value $\widehat{\*\theta}_n$ minimizing the log score as 
%
\begin{IEEEeqnarray}{rCCCCl}
 \widehat{\*\theta}_n & = &
 \min_{\*\theta \in \*\Theta} \left\{ \frac{1}{n}\sum_{i=1}^n-\log p(x_i|\*\theta) \right\}
 & \overset{n\to\infty}{\approx} &
 \min_{\*\theta \in \*\Theta} \left\{
 \mathbb{E}_{\mathbb{P}_{\*x}}\left[ - \log p(\*x|\*\theta) \right]
 \right\}
 \nonumber \\
 &= & 
 \min_{\*\theta \in \*\Theta} \left\{
 \mathbb{E}_{\mathbb{P}_{\*x}}\left[ - \log \left(\frac{p(\*x|\*\theta)}{d\mathbb{P}_{\*x}(\*x)}\right)\right]
 \right\}
 & = &
 \min_{\*\theta \in \*\Theta}\KLD
    \left(d\mathbb{P}_{\*x}\|p(\*x|\*\theta)
 \right).
 \nonumber
\end{IEEEeqnarray}
In other words, the log score targets the parameter value $\*\theta$ which is best in the \KLD-sense.
%
Second, as it is well-known that the \KLD is not robust \citep[e.g.][]{ABCdiv}, it is often advantageous to find a robust alternative divergence $D_{\text{robust}}$ to reverse-engineer the above derivation and arrive at some robust scoring rule $\mathcal{L}^{\text{robust}}(\*\theta, x_i)$.

The first developments in this direction started with the family of $\beta$-divergences by \citet{BasuDPD, mihoko2002robust}, but the idea has since been extended to various other discrepancies. This includes the Fisher divergence \citep{scoreMatching}, the family of $\gamma$-divergences \citep[e.g.][]{GammaDivNotSummable, GammaDivSummable} as well as Minimum Stein Discrepancies \citep{MinSteinDiscrepancy}.
The Generalized Bayesian posterior associated with this procedure then arises from replacing the negative log likelihood in eqs. \eqref{eq:generalized_bayes_posterior} and \eqref{eq:Zellner_standard} by $\mathcal{L}^{\text{robust}}(\*\theta, x_i)$.
A growing literature has focused on using the scoring rules derived from these divergences to perform Generalized Bayesian inference of precisely this sort \citep[see e.g.][]{MinDisparities, GoshBasuPseudoPosterior, MMDBayes}. 
\citet{Jewson} provides a recent survey of this field and connects it to the idea of \citet{Bissiri}. 
As Figures \ref{Fig:Influence_fct_pic} and \ref{Fig:RBOCPD_example} demonstrate for the scoring rule $\mathcal{L}^{\beta}(\*\theta, x_i)$ derived from the $\beta$-divergence, this produces  reliably robust posterior inferences.
We state the analytic form of $\mathcal{L}^{\beta}(\*\theta, x_i)$ and another robust scoring rule $\mathcal{L}^{\gamma}(\*\theta, x_i)$ derived from the $\gamma$-divergence \citep[see][]{GammaDivSummable} in Section \ref{sec:experiments_GPs}, where we show how they can be used to robustify Deep Gaussian Processes. 
%
%

%

\subsubsection{\GVI and marginal variances}
The standard \VI objective can be inappropriate even in situations where 
assuming appropriately specified priors \ref{assumption:prior} and likelihood functions \ref{assumption:lklhood} underlying the traditional Bayesian paradigm are a useful working assumption. 
%
For instance, the uncertainty quantification of standard \VI is often inappropriate when $\mathcal{Q}$ is a mean field variational family factorizing dimension-wide over $\*\theta$ and the individual entries of $\*\theta$ exhibit strong dependence \citep[e.g.][]{TurnerSahani11}. 
Oftentimes, this phenomenon is referred to as \textit{mode seeking behaviour} \citep[e.g.][]{MinkaDivergences}. The name itself also reveals that this problem is intimately linked to unimodal---and thus in practice mean field normal---variational families $\mathcal{Q}$.
Again, the modularity result of Theorem \ref{Thm:GVI_modularity} can be helpful: Provided that one has no flexibility about the choice of $\mathcal{Q}$ and has fixed the prior $\pi(\*\theta)$, one can adapt the \GVI posterior's uncertainty quantification properties by changing $D = \KLD$ to an alternative divergence.
Figure \ref{Fig:marginal_variances} illustrates this  flexibility of \GVI by comparing the uncertainty quantification properties of three different robust divergences to standard \VI.


\begin{figure}[t!]
\begin{center}
\includegraphics[trim= {3.2cm 0.5cm 3.7cm 0.35cm}, clip,  
 width=1.00\columnwidth]{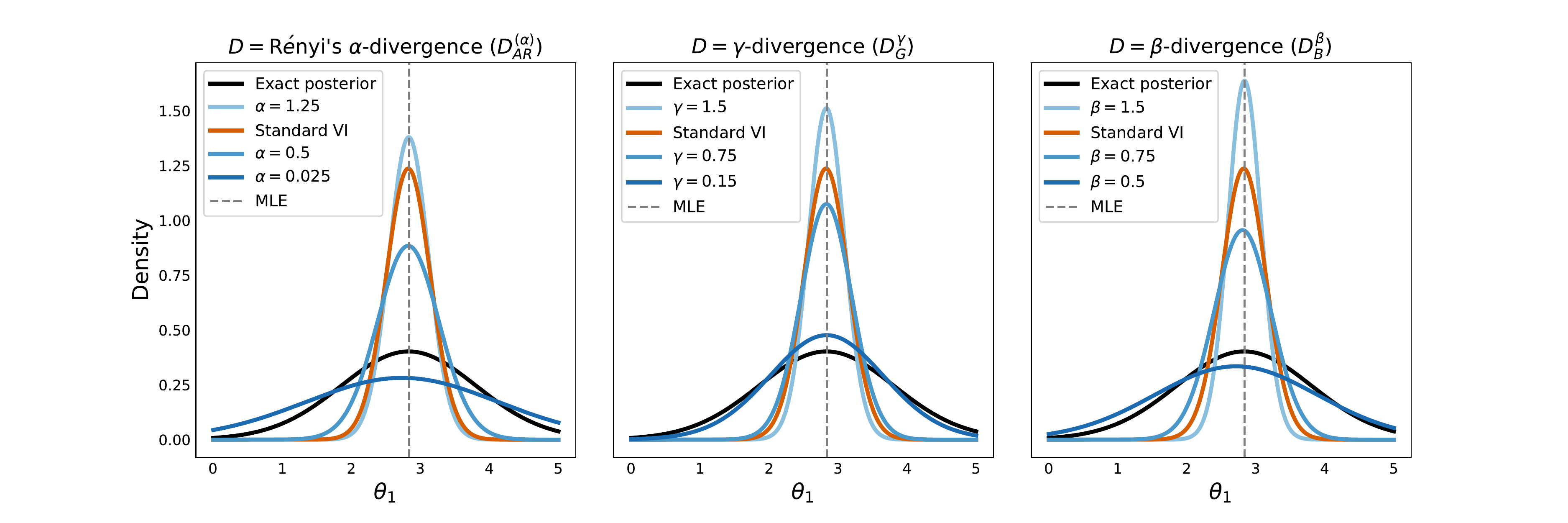}
\caption{
\bv
Marginal \VIColor{\textbf{\VI}} compared to different \GVIColor{\textbf{\GVI}} posteriors for the coefficient $\theta_1$ of data simulated from a Bayesian linear model (see Appendix \ref{Appendix:uncertainty_quantification} for details). 
For all posteriors, the loss $\ell$ is the correctly specified negative log likelihood of the true data generating mechanism. Further, for all variational posteriors the belief is constrained to lie inside a mean field normal family $\mathcal{Q}$.
Due to high correlation between the coefficients for the \textbf{exact posterior}, \VIColor{\textbf{standard \VI}} produces undesirably over-concentrated belief distributions. 
In contrast, appropriately choosing the hyperparameters of alternative robust divergences $D \neq \KLD$ provides more desirable uncertainty quantification.
}
\label{Fig:marginal_variances}
\end{center}
\end{figure}




\subsection{Theoretical properties of \GVI}
\label{sec:GVI_theory}

Clearly, the principal appeal of \GVI lies in its modularity and the associated subjective choices of $\ell$, $D$ and $\mathcal{Q}$.
Because of this increased need for the statistical modeller to provide sensible problem specification, one may worry about producing inferences that are non-sensical. 
The following section studies two theoretical findings that impose meaningful limits to the damage a badly specified \GVI posterior can do: 
Firstly, we point to novel results showing that \GVI posteriors are consistent in the frequentist sense: As more and more data are observed, the \GVI posterior will collapse to the population optimum regardless of $D$.
Secondly, we show that \GVI posteriors with uncertainty quantifiers based on robust divergences can be interpreted as approximations to Bayesian posteriors with a power likelihood.

\begin{figure}[t!]
\begin{center}
\includegraphics[trim= {2.95cm 0.25cm 3.95cm 0.9cm}, clip,  
 width=1.00\columnwidth]{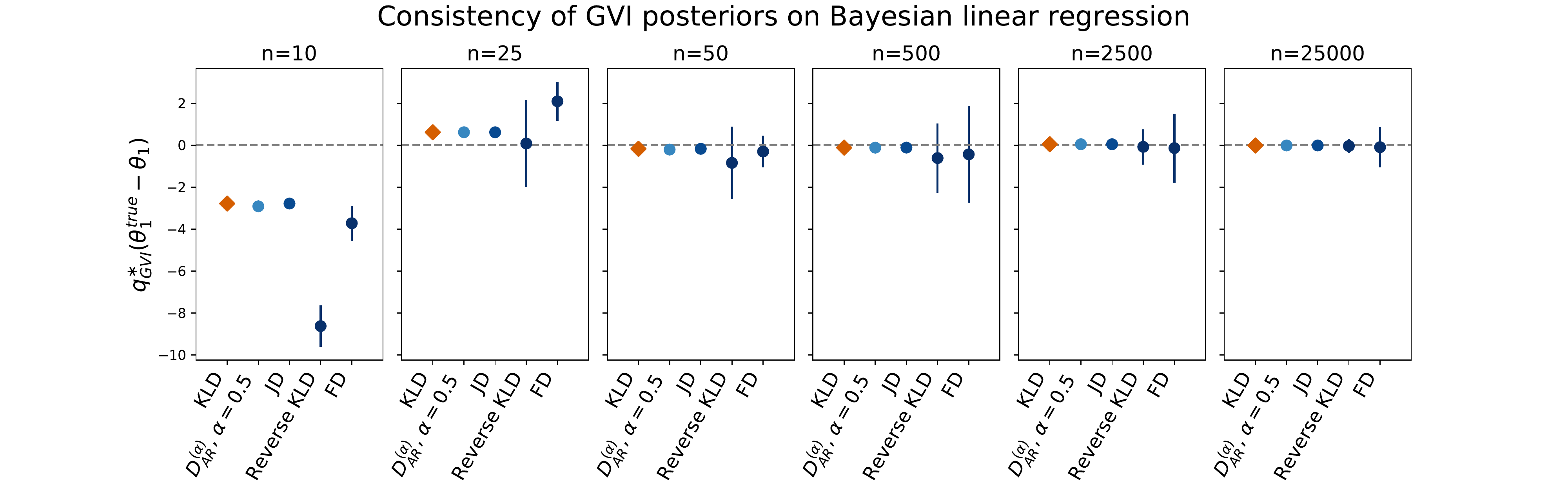}
\caption{
\bv
Marginal \VIColor{\textbf{\VI}} and different \GVIColor{\textbf{\GVI}} posteriors for the first coefficient of a simulated $20$-dimensional Bayesian Linear Model based on $n$ observations.
The loss $\ell$ is the correctly specified negative log likelihood of the true data generating mechanism and the uncertainty quantifier is varied along the $x$-axis. Depicted uncertainty quantifiers are the forward and reverse \KLD, R\'enyi's $\alpha$-divergence (\RAD), 
Jeffrey's Divergence (JD) as well as the Fisher Divergence (FD).
All posteriors are members of the mean field normal family $\mathcal{Q}$.
Because all inferred posterior beliefs are normals, dots are used to mark out the posterior mean and whiskers to denote the posterior standard deviation.
All posteriors are re-centered around the true value of the coefficient, so that the $y$-axis shows how far the posterior belief is from the truth.
}
\label{Fig:convergence_plot}
\end{center}
\end{figure}


\subsubsection{Frequentist consistency}

\citet{GVIConsistency} shows that \GVI posteriors are consistent in the Frequentist sense. 
In other words, they collapse to a point mass at the population-optimal value as the number of observations tends to infinity.
This holds under a wide range of extremely mild regularity conditions on the arguments $\ell, D$ and $\mathcal{Q}$.
Here, we state a simple version of the result for independent data with the mean field normal variational family.
\begin{theorem}[Frequentist consistency of \GVI]
    Suppose that Assumption 1 in \citet{GVIConsistency} holds.
    Choosing $\mathcal{Q}$ to be the mean field normal family and letting $D$ be lower-semi-continuous in its first argument, suppose that $D(q\|\pi) < \infty$ for all $q \in \mathcal{Q}$.
    Further, let $\mathbb{P}_{\*x}$ be the true probability measure of some random variable $\*x$ and suppose that the observations $x_{1:n}$ are independent and identically distributed draws from $\*x$. 
    If the prior is not infinitely bad for the population of $\*x$ (which is to say that $\mathbb{E}_{\pi}\left[\mathbb{E}_{\mathbb{P}_{\*x}} \left[ \ell(\*\theta, \*x) \right]\right]< \infty$), then
    \begin{IEEEeqnarray}{rCl}
        q^{\ast}_{\GVI, n}(\*\theta)
        & \overset{D}{\longrightarrow} &
        \delta_{\*\theta^{\ast}}(\*\theta),
        \nonumber
    \end{IEEEeqnarray}
    where $q^{\ast}_{\GVI, n}(\*\theta)$  is the \GVI posterior corresponding to the problem $P(\ell, D, \mathcal{Q})$ based on $n$ observations and $\*\theta^{\ast} = \argmin_{\*\theta}\left[\mathbb{E}_{\mathbb{P}_{\*x}} \left[ \ell(\*\theta, \*x) \right]\right]$ is the population-optimal parameter value.
\end{theorem}
\begin{remark}
    This Theorem is a straightforward invocation of Corollary 1 in \citet{GVIConsistency}. Assumption 1 guarantees a number of conditions that are required to make  \GVI  a well-defined optimization problem. For example, it ensures that the sum of the losses has minimizers for any finite $n$ and in the large data limit and that the loss expected under $\mathbb{P}_{\*x}$ is finite.
\end{remark}
This finding is illustrated in Figure \ref{Fig:convergence_plot}, which is taken from \citet{GVIConsistency}.
The plot shows that as the theory suggests, the posteriors collapse to a point mass under mild regularity conditions on the uncertainty quanitfier $D$. 
Unsurprisingly, speed and nature of the convergence depends on the choice of $D$.

%
%
%
%


\subsubsection{\GVI as a posterior approximation}

Although the axiomatic development in Section \ref{sec:axioms} shows that \GVI produces a posterior belief distribution that is valid in its own right, 
one can also interpret certain \GVI posteriors as approximations to (generalized) Bayesian posteriors as in eq. \eqref{eq:generalized_bayes_posterior}.
%
In particular, we show that for robust divergences $D_{\text{robust}}^{(\rho)}$ parameterized by some hyperparameter $\rho$ so that $\lim_{\rho \to 1} D_{\text{robust}}^{(\rho)} = \KLD$, the \GVI objective constitutes a lower bound on the evidence of a (generalized) Bayesian posterior.
Results of this form can be shown to hold for R\'enyi's $\alpha$-divergence (\RAD), the $\gamma$-divergence (\GD) as well as the $\beta$-divergence (\BD). 
As they are structurally similar, we only state the bound corresponding to $D=\RAD$ and defer the results for \GD and \BD as well as all proofs to Appendix \ref{Appendix:lower_bound_proofs}.

\begin{theorem}[\GVI as approximate Evidence Lower bound with $D=\RAD$]
%
The objective of a \GVI posterior based on $P(\ell, \RAD, \mathcal{Q})$ has an interpretation as lower bound on the $c(\alpha)$-scaled (generalized) evidence lower bound of $P(w(\alpha)\cdot \ell, \KLD, \mathcal{P}(\*\Theta))$:

\begin{IEEEeqnarray}{rCl}
    \mathbb{E}_{q(\*\theta)}\left[\sum_{i=1}^n\ell(\*\theta,\*x_i)\right]
    +
    \RAD(q||\pi)
    & \geq & 
    -c(\alpha) \cdot \ELBO^{w(\alpha)\ell}(q) + S_1(\alpha, q, \pi)
    \label{Equ:LowerBoundMarginalLossLikelihoodRAD}
\end{IEEEeqnarray}
where $\ELBO^{w(\alpha)\ell}$ denotes the Evidence Lower Bound associated with standard \VI relative to the generalized Bayesian posterior given by
\begin{IEEEeqnarray}{rCl}
    q^{w(\alpha)\ell}_B(\*\theta)& \propto &
    \pi(\*\theta)\exp\left(-w(\alpha)\sum_{i=1}^n\ell(\*\theta,\*x_i)\right),
    \nonumber
\end{IEEEeqnarray}
where $c(\alpha) = \min\{1, \alpha^{-1}\}$, $w(\alpha) = \max\{1, \alpha\}$ and $S_1(\alpha, q, \pi)$ is an interpretable slack term.
%
\label{Thm:LowerBoundMarginalLossLikelihoodRAD}
\end{theorem}
\begin{remark}
     The take-away from the bound in eq. \eqref{Equ:LowerBoundMarginalLossLikelihoodRAD} is that the slack term $S_1(\alpha, q, \pi)$ 
     introduces the main difference between $P(\ell, \RAD, \mathcal{Q})$ and $P(w(\alpha)\cdot\ell, \KLD, \mathcal{Q})$.
    As Appendix \ref{Appendix:lower_bound_proofs} shows, is possible but rather tedious to make analytically more concise statements about $S_1(\alpha, q, \pi)$.
    Specifically, the main function of the slack term  is the introduction of more conservative uncertainty quantification.
    %
    Studying this empirically reveals that for R\'enyi's $\alpha$-divergence, this has the effect of enabling prior robust inference.
    This point is demonstrated in Figure \ref{Fig:prior_misspec_BMM} and elaborated upon in Section \ref{sec:GVI_prior_misspecification}, but is perhaps best summarized in Figure \ref{Fig:prior_robustness_thm}: Since $c(\alpha) = 1$ for $\alpha \in (0,1)$, the left hand side of the plot corresponds to $P(\ell, \RAD, \mathcal{Q})$ and the right-hand side to $P(w(\alpha)\cdot\ell, \KLD, \mathcal{Q})$.
    %
    The plot shows the difference between minimizing the \GVI and the \ELBO objectives in eq. \eqref{Equ:LowerBoundMarginalLossLikelihoodRAD}: relative to the \ELBO objective $P(w(\alpha)\cdot\ell, \KLD, \mathcal{Q})$, the belief distributions $P(\ell, \RAD, \mathcal{Q})$ based on the \RAD are much less sensitive to badly specified priors.
    In fact, the \GVI posteriors implicitly determine if it is worth taking the prior into account:
    On the one hand, the uncertainty quantification with $D=\RAD$ is \textbf{less affected} than standard \VI for badly specified priors. On the other hand, \GVI posteriors with $D=\RAD$  are still \textbf{more conservative} than standard \VI for well specified priors. 
\end{remark}

\begin{figure}[t!]
\begin{center}
\includegraphics[trim= {3.2cm 0.5cm 3.7cm 0.35cm}, clip,  
 width=1.00\columnwidth]{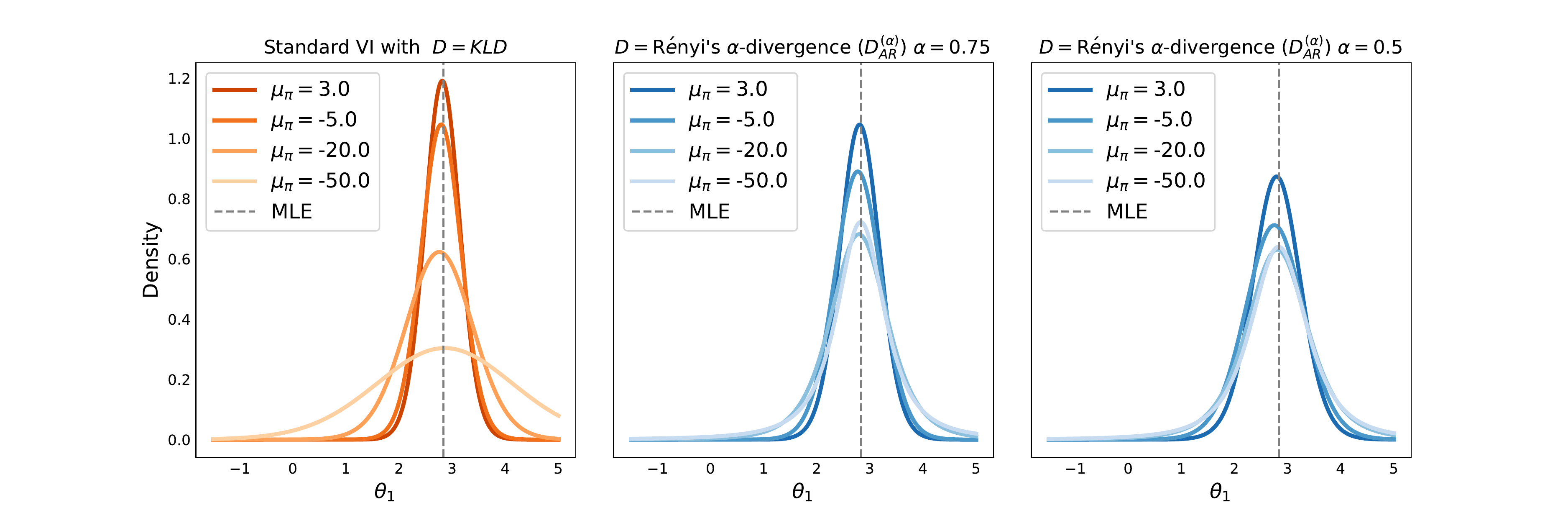}
\caption{
\bv
Marginal \VIColor{\textbf{\VI}} compared to different \GVIColor{\textbf{\GVI}} posteriors for the coefficient $\theta_1$ of data simulated from a $d$-dimensional Bayesian linear model with different priors (see Appendix \ref{Appendix:uncertainty_quantification} for details).
The prior for the coefficients is a Normal Inverse Gamma distribution given by $\*\mu \sim \mathcal{NI}^{-1}(\mu_{\pi}\cdot 1_d, v_{\pi}\cdot I_d, a_{\pi}, b_{\pi})$ with $v_{\pi} = 4 \cdot I_d$, $a_{\pi} = 3$, $b_{\pi} = 5$ and various values for $\mu_{\pi}$.
For all posteriors, the loss $\ell$ is the correctly specified negative log likelihood of the true data generating mechanism. 
Further, all variational posteriors are constrained to lie inside a mean field normal family $\mathcal{Q}$.
Notice that the \VIColor{\textbf{standard \VI}} posterior corresponds to the \ELBO component on the right hand side of the bound in eq. 
\eqref{Equ:LowerBoundMarginalLossLikelihoodRAD}. In contrast, the \GVIColor{\textbf{\GVI}} posteriors are obtained by maximizing the left hand side of the same bound.
}
\label{Fig:prior_robustness_thm}
\end{center}
\end{figure}

\subsection{Inference with Generalized Variational Inference (\GVI)}
\label{sec:GVI_inference}

This section outlines two powerful inference strategies for \GVI: Quasi-conjugate and fully black box inference.
Built on earlier findings in \citet{RBOCPD}, we show that a class of \GVI posteriors based on robust likelihood scoring rules admits closed form objectives. Because this closed form objective emerges when the non-robust counterpart of the likelihood is conjugate to the prior, we call the resulting inference procedure quasi-conjugate.
For more complicated models, closed form objectives will not be available. To address this, we also introduce a black box inference procedure for arbitrary choices of $\ell$ and $D$.

\subsubsection{Quasi-conjugate inference}
\label{sec:quasi-conjugate-inference}




This paper so far has given little attention to specifying the constraining family for a \GVI problem.
One interesting interdependence between loss function and variational family was studied in \citet{RBOCPD}: 
When applying the robust scoring rule $\mathcal{L}^{\beta}$  \citep{BasuDPD} derived from the $\beta$-divergence (\BD) to a likelihood associated with a conjgate prior $\pi(\theta|\kappa_0)$, there is a considerable advantage in taking $\mathcal{Q}$ to be the family of the conjugate prior. 
Since $\mathcal{L}^{\beta}(\*\theta, x_i) \to -\log p(x_i|\*\theta)$ as $\beta \to 1$, the (generalized) Bayesian posterior 
\begin{IEEEeqnarray}{rCl}
    q_{\text{B}}^{\beta}(\*\theta) & \propto & \pi(\*\theta)\prod_{i=1}^n\exp\left\{ - \mathcal{L}^{\beta}(\*\theta, x_i) \right\}
    \nonumber
\end{IEEEeqnarray}
is contained in $\mathcal{Q}$  as $\beta \to 1$.
The results in \citet{RBOCPD} and intuition thus suggest that so long as $|\beta - 1| < \varepsilon$ for some sufficiently small value $\varepsilon > 0$ and $D= \KLD$, constraining the posterior to be in $\mathcal{Q}$ still produces excellent approximations to $q_{\text{B}}^{\beta}(\*\theta)$. 
%
%
%
Beyond the approximation quality, choosing the quasi-conjugate variational family offers tangible computational advantages: As Theorem 2 in \citet{RBOCPD} shows, under these circumstances the objective defining $P(\mathcal{L}^{\beta}, \KLD, \mathcal{Q})$ is available in closed form.
Consequently, no sampling or approximation is required and the optimum is usually found within a very small number of iterations. 

Complementing this finding, Proposition \ref{Thm:QuasiConjugacyGD} extends quasi-conjugacy to the robust scoring rule $\mathcal{L}^{\gamma}$ derived from the $\gamma$-divergence (\GD) as provided for in \citet{GammaDivSummable}.
Similarly to  $\mathcal{L}^{\beta}$,  $\mathcal{L}^{\gamma}(\*\theta, x_i) \to -\log p(x_i|\*\theta)$ as $\gamma \to 1$, so that the same intuition that applied to $\mathcal{L}^{\beta}$ also applies here.
Note that the conditions for $\mathcal{L}^{\gamma}$ in Proposition \ref{Thm:QuasiConjugacyGD} are slightly more restrictive than those derived for $\mathcal{L}^{\beta}$.
This is due to the fact that while the same integral term appears in both, it is additive for  $\mathcal{L}^{\beta}$ but multiplicative for  $\mathcal{L}^{\gamma}$.
While the proof is conceptually straightforward, it is notationally cumbersome and thus deferred to Appendix \ref{Appendix:quasi-conjugacy-GD}.

\begin{proposition}[Closed form \GVI objectives with $\mathcal{L}^{\gamma}$] 
\label{Thm:QuasiConjugacyGD}
Let $\mathcal{L}^{\gamma}(\*\theta, \cdot)$ be the $\gamma$-divergence based scoring rule for likelihood $p(\cdot|\*\theta)$. Suppose $p(\cdot|\*\theta)$ admits conjugacy relative to the exponential distributions given by $\mathcal{Q}$ and let the conjugate prior $\pi(\*\theta|\*\kappa_0) \in \mathcal{Q}$. 
Writing 
\begin{IEEEeqnarray}{rCl}
    p(x|\*\theta) & = & 
    h(x)\exp\left\{ g(x)^TT(\*\theta) - B(x) \right\},
    \nonumber \\
    q(\*\theta|\*\kappa) & = &
    h(\*\theta)\exp\left\{ \eta(\*\kappa)^T T(\*\theta) - A(\eta(\*\kappa)) \right\},
    \nonumber \\
    \mathcal{N} & = & \left\lbrace \*\kappa : 
    \exp\{A(\eta(\*\kappa))\}
    <\infty\right\rbrace,
    \nonumber
\end{IEEEeqnarray}
the objective of $P(\mathcal{L}^{\gamma}, \KLD, \mathcal{Q})$ has closed form if for observations $x_{1:n}$ and all $q \in \mathcal{Q}$
\begin{IEEEeqnarray}{rCCCCCCCl}
    I^{(\gamma)}(\*\theta) & = & \int_{\mathcal{X}}p(x|\*\theta)^{\gamma}dx,   
    \quad F_1(\*\kappa)& = &\int_{\*\Theta} T(\*\theta) q(\*\theta|\*\kappa) d\*\theta,   
    \quad F_2(\*\kappa)& = &\int_{\*\Theta} I^{(\gamma)}(\*\theta)^{\frac{1-\gamma}{\gamma}} q(\*\theta|\*\kappa) d\*\theta \nonumber
\end{IEEEeqnarray}
are closed form functions of $\*\theta$ and $\*\kappa$ and if for all $x_i$, 
$\left(\eta(\*\kappa)+(\gamma -1)g(x_i)\right)\in\mathcal{N}$.
\end{proposition}

\subsubsection{Additional details on Black-Box \GVI (\BBGVI)}
\label{sec:black_box_GVI}

%
Standard \VI is scalable using doubly stochastic, model-agnostic optimization techniques \citep[e.g.][]{VISS, SVI, BBVI2, ControlVariateVI, AntitheticBBVI} collectively known as  black box \VI \citep{BBVI}. 
We extend these methods to black box \GVI (\BBGVI),
an inference algorithm directly inheriting the modularity of the posteriors defined by $P(\ell, D, \mathcal{Q})$.
This makes it easy to build \BBGVI into existing software: 
%
For example, adapting the Deep Gaussian Process implementation of \citet{DeepGPsVI} required <100 lines of Python code. 
%

Suppose 
$\mathcal{Q} = \{q(\*\theta|\*\kappa): \*\kappa \in K \}$ 
and that for all $(\*\kappa, \*\theta) \in (K, \*\Theta)$, 
    one can sample $\*\theta \sim q(\*\theta|\*\kappa)$. Suppose also that
    the derivatives $\nabla_{\*\kappa} \log(q(\*\theta|\*\kappa))$ and $\nabla_{\*\kappa} D(q||\pi)$ exist ($q$-almost surely).
For many choices of $D$, $\mathcal{Q}$ and $\pi$, $\nabla_{\*\kappa} D(q||\pi)$ is available in closed form. In this case, \BBGVI is particularly attractive and \GVI posteriors can be computed through an unbiased gradient estimate given as
\begin{IEEEeqnarray}{lCr}
    \nabla_{\*\kappa} \hat{L}(q|\ell, D, \mathcal{Q}) & = & \frac{1}{S}\sum_{s=1}^S\left\{ 
            \sum_{i=1}^n\ell(\*\theta^{(s)}, x_i) \cdot     
        \nabla_{\*\kappa}\log(q(\*\theta^{(s)}|\*\kappa))  \right\}
        + 
        \nabla_{\*\kappa}D(q||\pi)
    \label{eq:GVI_L_estimate_1}
\end{IEEEeqnarray}
and relying on an independent sample $\*\theta^{(1:S)} \overset{i.i.d}{\sim} q(\*\theta|\*\kappa)$.
%
%
If a closed form for $\nabla_{\*\kappa} D(q||\pi)$ is not available but  $D(q||\pi) = \mathbb{E}_{q(\*\theta|\*\kappa)}\left[ \ell^D_{\*\kappa, \pi}(\*\theta)\right]$ for a function $\ell^D_{\*\kappa, \pi}:\*\Theta \to \mathbb{R}$, one can use the alternative unbiased gradient estimate
\begin{IEEEeqnarray}{lCr}
      \nabla_{\*\kappa}\hat{L}(q|\ell, D, \mathcal{Q}) & = & \frac{1}{S}\sum_{s=1}^S\left\{ 
        \left[
            \sum_{i=1}^n\ell(\*\theta^{(s)}, x_i) +     \ell^D_{\*\kappa, \pi}(\*\theta^{(s)})
        \right]\cdot
        \nabla_{\*\kappa}\log(q(\*\theta^{(s)}|\*\kappa)) 
        + 
        \nabla_{\*\kappa}\ell^D_{\*\kappa, \pi}(\*\theta^{(s)})
    \right\}. \quad \:\:
    \label{eq:GVI_L_estimate_2}
\end{IEEEeqnarray}
This can be deployed for most divergences of interest, including the family of $f$-divergences.
In some cases however, divergences will not be linear in the argument $\*\theta$ so that one has $D(q\|\pi) = \tau\left(\mathbb{E}_{q(\*\theta|\*\kappa)}\left[ \ell^D_{\*\kappa, \pi}(\*\theta)\right]\right)$ for some non-linear function $\tau:\mathbb{R}\to\mathbb{R}$. In this case, \BBGVI can be performed based on the biased gradient estimate
\begin{IEEEeqnarray}{lCr}
    \nabla_{\*\kappa} \hat{L}(q|\ell, D, \mathcal{Q}) & = & \frac{1}{S}\sum_{s=1}^S\left\{ 
            \sum_{i=1}^n\ell(\*\theta^{(s)}, x_i) \cdot     
        \nabla_{\*\kappa}\log(q(\*\theta^{(s)}|\*\kappa))  \right\}
        +  \nonumber \\
    &&
        \tau\left(
            \frac{1}{S}\sum_{s=1}^S \ell^D_{\*\kappa, \pi}(\*\theta^{(s)})
        \right) 
        \cdot
        \frac{1}{S}\sum_{s=1}^S \nabla{\*\kappa} \ell^D_{\*\kappa, \pi}(\*\theta^{(s)}).
    \label{eq:GVI_L_estimate_3}
\end{IEEEeqnarray}
%
%
Any gradient form admits black box {variance reduction} through some control variate $h$ \citep[][]{BBVI, AntitheticBBVI}, see Appendix \ref{Appendix:BBGVI} for details. 
%
%
Algorithm \ref{algorithm:BBVI} summarizes a generic \BBGVI procedure. 
%

\begin{algorithm}[ht!]
	\caption{\textbf{Black box \GVI (\BBGVI)}}
   \label{Algorithm_BOCPDMS}
\begin{algorithmic}[0]
   \State{}
   \State {\bfseries Input:} $x_{1:n}$, $\pi$, $D$, $\ell$, $\mathcal{Q}$, $h$, $\text{StoppingCriterion}$, $\*\kappa_0$, $K$, $S$, $t=0$, $\text{LearningRate}$
   \State{}
   \State $\text{done} \leftarrow \text{False}$
   \While{not $\text{done}$} 
   \State{} 
    \State // STEP 1: Get a subsample from $x_{1:n}$ of size $K$
    \State $\rho_{1:K}\leftarrow \text{SampleWithoutReplacement}(1:n, K)$ 
    \State $x(t)_{1:K} \leftarrow x_{\rho_{1:K}}$ 
    \State{}
    \State // STEP 2: Sample from $q(\*\theta|\*\kappa_t)$ and compute losses
    \State $\*\theta^{(1:S)} \overset{i.i.d.}{\sim} q(\*\theta|\*\kappa_t)$
    \State $\ell_{i, s} \leftarrow \ell(\*\theta^{(s)}, x(t)_i) \cdot \nabla_{\*\kappa_t}\log q(\*\theta^{(s)}|\*\kappa_t)$ for all $s = 1, 2, \dots S$ and $i=1, 2, \dots, K$
    \State $\ell_s \leftarrow \frac{n}{K}\sum_{i=1}^K \ell_{i,s}$ for all $s = 1, 2, \dots S$
    \State{}
    \State // STEP 3: Compute uncertainty quantifier
    \If{$D(q\|\pi)$ admits closed form}
        \State $\ell_{s}  \leftarrow \ell_{s} + \nabla_{\*\kappa}D(q\|\pi)$ for all $s = 1, 2, \dots S$
    \ElsIf{$D(q\|\pi) = \mathbb{E}_{q}[\ell^D_{\*\kappa, \pi}(\*\theta)]$}
        \State $\ell_{s}  \leftarrow \ell_{s} + \ell^D_{\*\kappa, \pi}(\*\theta^{(s)})\nabla_{\*\kappa_t}\log q(\*\theta^{(s)}|\*\kappa_t) + \nabla_{\*\kappa_t}\ell^D_{\*\kappa_t, \pi}(\*\theta^{(s)})$ for all $s = 1, 2, \dots S$
    \ElsIf{$D(q\|\pi) = \tau\left(\mathbb{E}_{q}[\ell^D_{\*\kappa, \pi}(\*\theta)]\right)$}
        \State $\ell_s \leftarrow \ell_s + 
        \tau\left(
            \frac{1}{S}\sum_{s=1}^S \ell^D_{\*\kappa_t, \pi}(\*\theta^{(s)})
        \right) 
        \cdot
        \nabla{\*\kappa_t} \ell^D_{\*\kappa_t, \pi}(\*\theta^{(s)})
        $ for all $s = 1, 2, \dots S$
    \EndIf
    \State{}
    \State // STEP 4: Apply variance reduction via $h$ if desired
    \If{$h \neq \text{None}$}
        \State $h_s \leftarrow h(\*\theta^{(s)}, \ell_{s})$
        \State $\ell_{s} \leftarrow \ell_{s} - h_s$ \text{ for all } for all $s = 1, 2, \dots S$ 
    \EndIf
    \State{}
    \State // STEP 5: Update $\*\kappa_t$ and stopping criterion
    \State $\rho_t \leftarrow \text{LearningRate}(t)$ 
    \State $L \leftarrow  \frac{1}{S}\sum_{s=1}^S\ell_s$ 
    \State $\*\kappa_{t+1} \leftarrow \*\kappa_t + \rho_t \cdot L$
    \State $\text{done} \leftarrow \text{StoppingCriterion}(\*\kappa_{t+1}, \*\kappa_t, t)$
    \State $t \leftarrow t + 1$
   \EndWhile
   \State{}
\end{algorithmic}
\label{algorithm:BBVI}
\end{algorithm}

As gradient estimates based on eq. \eqref{eq:GVI_L_estimate_1} will have strictly lower variance than estimates based on eq. \eqref{eq:GVI_L_estimate_2} or eq. \eqref{eq:GVI_L_estimate_3}, one may wonder under which conditions closed forms for $\nabla_{\*\kappa}D(q\|\pi)$ are available. Proposition \ref{proposition:closed_form_D} clarifies this for robust divergences.
\begin{proposition}[Closed form $D$]
    Let $q, \pi$ with natural parameters $\*\eta_q,\*\eta_{\pi}$ be in the exponential family 
    $\mathcal{Q} = \{ q(\*\theta|\*\eta) = h(\*\theta)\exp\left\lbrace\*\eta'T(\*\theta)-A(\*\eta)\right\rbrace: \eta \in \mathcal{N}\}$
    with natural parameter space $\mathcal{N}=\left\lbrace \*\eta:\exp\{A(\*\eta)\}<\infty\right\rbrace$. Then,  
    %
    \begin{itemize}
        \myitem{(1)}
        \ADqp and \RADqp have a closed form if $\alpha \in (0,1)$ or if $\alpha\*\eta_q + (1-\alpha)\*\eta_{\pi} \in \mathcal{N}$
        \label{item:RAD_AD_closed_forms}
        \myitem{(2)}
        \BDqp has a closed form if $h(\*\theta) = h$ does not depend on $\*\theta$ and additionally, $(\beta-1)\cdot \*\eta_1+\*\eta_2 \in \mathcal{N}$ for any $\*\eta_1,\*\eta_2 \in \mathcal{N}$ (amongst others, this holds for Beta, Gamma, Gaussian, exponential or Laplace distributions)
        \label{item:BD_closed_forms}
        \myitem{(3)}
        \GDqp has closed form if \BDqp does for $\beta = \gamma$.
        \label{item:GD_closed_forms}
    \end{itemize}
    \label{proposition:closed_form_D}
\end{proposition}

\section{Experiments}
\label{sec:experiments}

Having introduced an inference strategy that is generic enough to work on high-dimensional, black box Bayesian models, the remainder of the paper studies \GVI on two applications of interest in Bayesian Deep Learning. 
Before doing so, notice that as indicated in Table \ref{table:inference_methods}, previous work constitutes various interesting special cases of \GVI with other strong empirical results \citep[e.g.,][]{AISTATSBetaDiv, RBOCPD, bVAE, MMDBayes, MartinJankowiak}
We add to this body of evidence by deploying \GVI on Bayesian Neural Networks (\BNNs) and Deep Gaussian Processes (\DGP{}s) to address the particular ways in which these two models challenge the assumptions underlying the standard Bayesian paradigm.
All code used for generating the experiments is available from \url{https://github.com/JeremiasKnoblauch/GVIPublic}.

\subsection{Bayesian Neural Network Regression}
\label{sec:experiments_BNN}

As alluded to in Example \ref{example:violation_P} and Section \ref{sec:GVI_prior_misspecification}, \BNN models should be expected to suffer from prior misspecification. 
Focusing on the regression case, we wish to alleviate this problem using \GVI's modularity and thus focus on varying $D$.
Accordingly, we fix the loss function to the usual negative log likelihood $\ell(\*\theta, y_i, x_i, \sigma^2) = - \log p_{\mathcal{N}}(y_i|x_i, \sigma^2, F(\*\theta))$ for 
\begin{IEEEeqnarray}{rCl}
    p_{\mathcal{N}}(y_i|x_i, \sigma^2, F(\*\theta)) & = & \mathcal{N}(y_i|F(\*\theta), x_i, \sigma^2),
    \nonumber
\end{IEEEeqnarray}
and choose $\mathcal{Q} = \mathcal{Q}_{\text{MFN}}$ as the normal mean field variational family given in eq. \eqref{eq:Q_MFN}.  
With this in hand, we compare three different constructions of posterior beliefs:
\begin{itemize}
    \myitem{(1)} \VIColor{\textbf{Standard \VI}} as described in Section \ref{sec:VI_opt}; 
    \myitem{(2)} \DVIColor{\textbf{\DVI}} methods introduced as approximations to the standard Bayesian posterior $q^{\ast}_{\text{B}}(\*\theta)$ that find $q^{\ast}_{\text{A}}(\*\theta) = \argmin_{q\in\mathcal{Q}}D(q\|q^{\ast}_{\text{B}}(\*\theta))$ with $D$ being the $\alpha$-divergence \citep{AlphaDiv}\footnote{
		We align the parameterization  of the \AD with the current paper, meaning $1 - \alpha_{\text{current}} = \alpha_{\text{H.-L. et al. (2016)}}$
    } and R\'enyi's $\alpha$-divergence \citep{RenyiDiv};
    \myitem{(3)} \GVIColor{\textbf{\GVI}}  with $D=\RAD$.
\end{itemize}
To make comparisons as fair as possible, our implementation is built on top of that used for the results of \citet{RenyiDiv} and only changes the objective being optimized. 
Similarly, all settings and data sets for which the methods are compared are unchanged and taken directly from \citet{RenyiDiv} and \citet{AlphaDiv}:
We use a single-layer network with 50 ReLU nodes on all experiments. Inference is performed via probabilistic back-propagation \citep{PBP} and the ADAM optimizer \citep{ADAM} with its default settings, 500 epochs and a batch size of 32. Priors and variational posteriors are both fully factorized normal distributions. 
Further, the results are also evaluated on the same selection of UCI data sets \citep{UCI} and in the same way as they were in \citet{RenyiDiv} and \citet{AlphaDiv}: Using 50 random splits of the relevant data into training (90\%) and test (10\%) sets, the inferred models are evaluated predictively on the test sets using the average negative log likelihood (\NLL) as well as the average root mean square error (\RMSE). For each of the 50 splits, predictions are computed based on 100 samples from the variational posterior.

We summarize the two main results of our experiments as follows: First, Figure \ref{Fig:BNN_experiment_1} depicts what appears to be the most typical relationship between \VIColor{\textbf{\VI}}, \DVIColor{\textbf{\DVI}} and \GVIColor{\textbf{\GVI}} on \BNN{}s.
Second, Figure \ref{Fig:posterior_predictives} explores a surprising finding about the typical relationship further and connects it back to the modularity result in Theorem \ref{Thm:GVI_modularity}.
The Appendix contains a small number of further results.
%
%

\begin{figure}[h!]
\begin{center}
\includegraphics[trim= {0.25cm 0.0cm 0.25cm 0.25cm}, clip,  
 width=1.00\columnwidth]{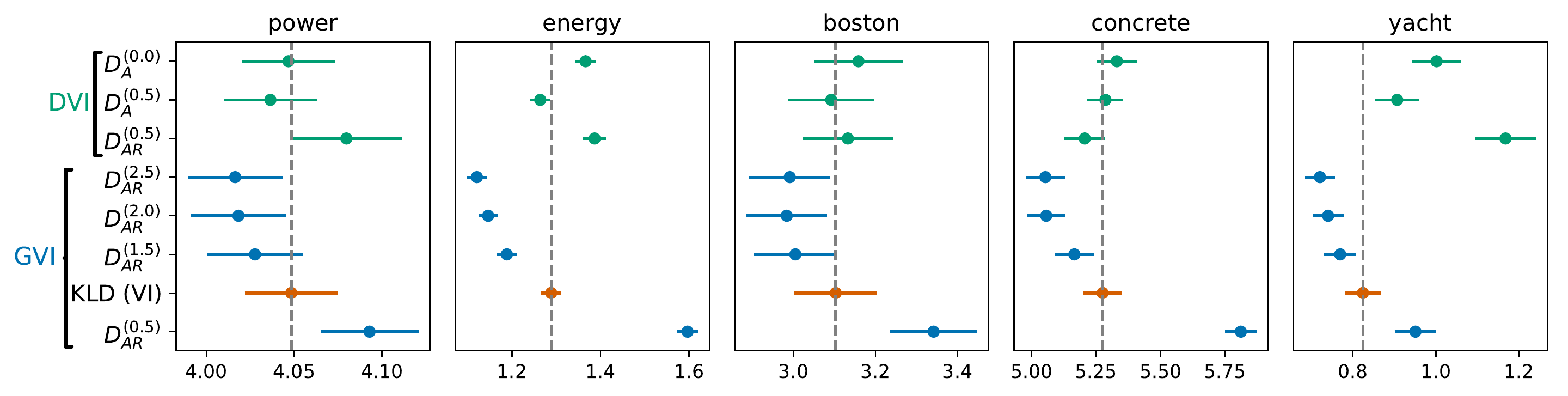}
 \includegraphics[trim={0.25cm 0.25cm 0.25cm 0.675cm}, clip,  
 width=1.00\columnwidth]{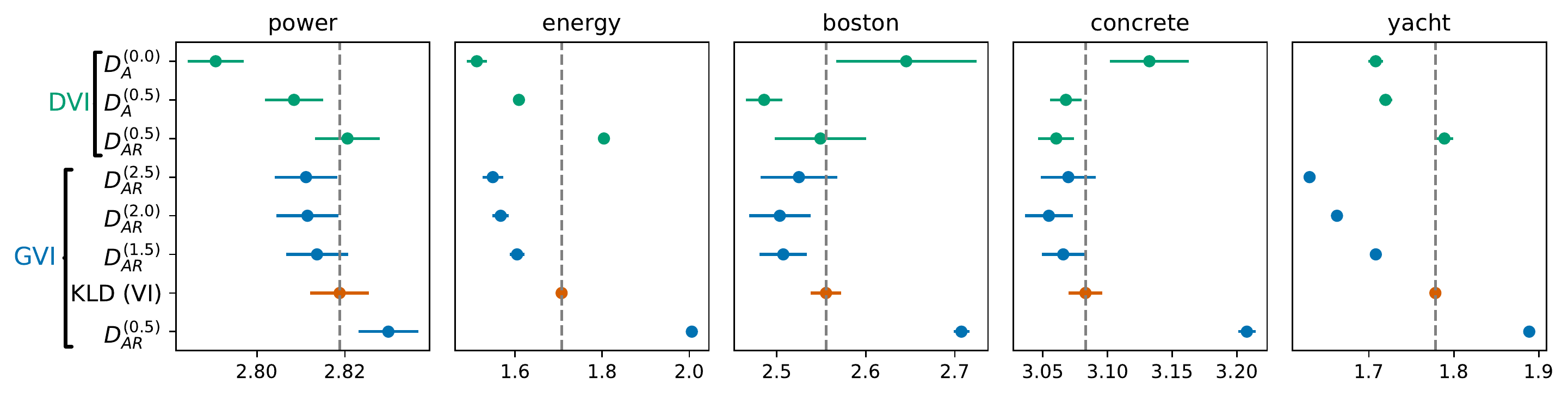}
\caption{
    \bv
	Top row depicts \RMSE, bottom row the \NLL across a range of data sets using \BNNs. Dots correspond to means, whiskers to standard errors. The further to the left, the better the predictive performance. 
	For the depicted selection of data sets, a clear common pattern exists for the performance differences between \VIColor{\textbf{standard \VI}}, \DVIColor{\textbf{\DVI}} and \GVIColor{\textbf{\GVI}}.
}
\label{Fig:BNN_experiment_1}
\end{center}
\end{figure}

\subsubsection{Typical patterns (Figure \ref{Fig:BNN_experiment_1})}

As Figure \ref{Fig:BNN_experiment_1} demonstrates, several findings form a consistent pattern across a range of data sets. Three findings are most poignant.
\begin{itemize}
	\myitem{(A)} \DVIColor{\textbf{\DVI}} can often achieve a performance gain for the \NLL relative to \VIColor{\textbf{standard \VI}}, but much less so for \RMSE. On both metrics, there is no clear pattern of improvement.
	\label{item:BNN_1_DVI}
	\myitem{(B)}  Relative to \VIColor{\textbf{standard \VI}}, \GVIColor{\textbf{\GVI}} significantly improves performance for both \NLL and \RMSE if $\alpha > 1$. Conversely,
	\GVIColor{\textbf{\GVI}} worsens performance if $\alpha \in (0,1)$. In other words, \textit{larger} posterior variances adversely affect predictive quality.
	\label{item:BNN_1_GVI}
	\myitem{(C)} \GVIColor{\textbf{\GVI}} performance is a clear banana-shaped function of $\alpha$ across all data sets: While predictive performance benefits as $\alpha$ gets larger than one, the improvement flattens out and bends back in a banana shape as $\alpha$ grows too large.
	\label{item:BNN_1_banana}
\end{itemize}
Finding \ref{item:BNN_1_GVI} has a straightforward interpretation: Since it 
holds that $\RAD \leq \KLD$ for $\alpha > 1$ (see \citet{RenyiKLD}\footnote{
    Note that their result holds for a different parameterization of the \RAD, but it is easy to show that our parameterization is strictly smaller than theirs for $\alpha > 1$.
} and Figure \ref{Fig:div_magnitudes}), the \GVI posteriors associated with \RAD for $\alpha > 1$ are \textit{more} concentrated than the standard \VI posteriors, a phenomenon also depicted on toy models in Figure \ref{Fig:marginal_variances}. 
In other words: Ignoring more of the poorly specified prior and consequently being closer to a point mass at the empirical risk minimizer is beneficial for predictive performance. 
As alluded to in Example \ref{example:violation_P}, this is perhaps to be expected: Not only is the likelihood function of a \BNN extremely flexible so that even a point estimate is likely to produce decent predictions, but it is also doubtful if a literal interpretation of the prior as in \ref{assumption:prior} is appropriate for \BNN{}s.
%
As finding \ref{item:BNN_1_banana} shows however, this does not mean that point estimates are preferable to posterior beliefs: Increasing the value of $\alpha$ shrinks the variances too much, eventually impeding predictive performance. 
%

\begin{figure}[ph!]
\vskip -1.0cm
\begin{center}
\includegraphics[trim= {1.5cm 1cm 2.25cm 0cm}, clip,  
width=1\columnwidth]{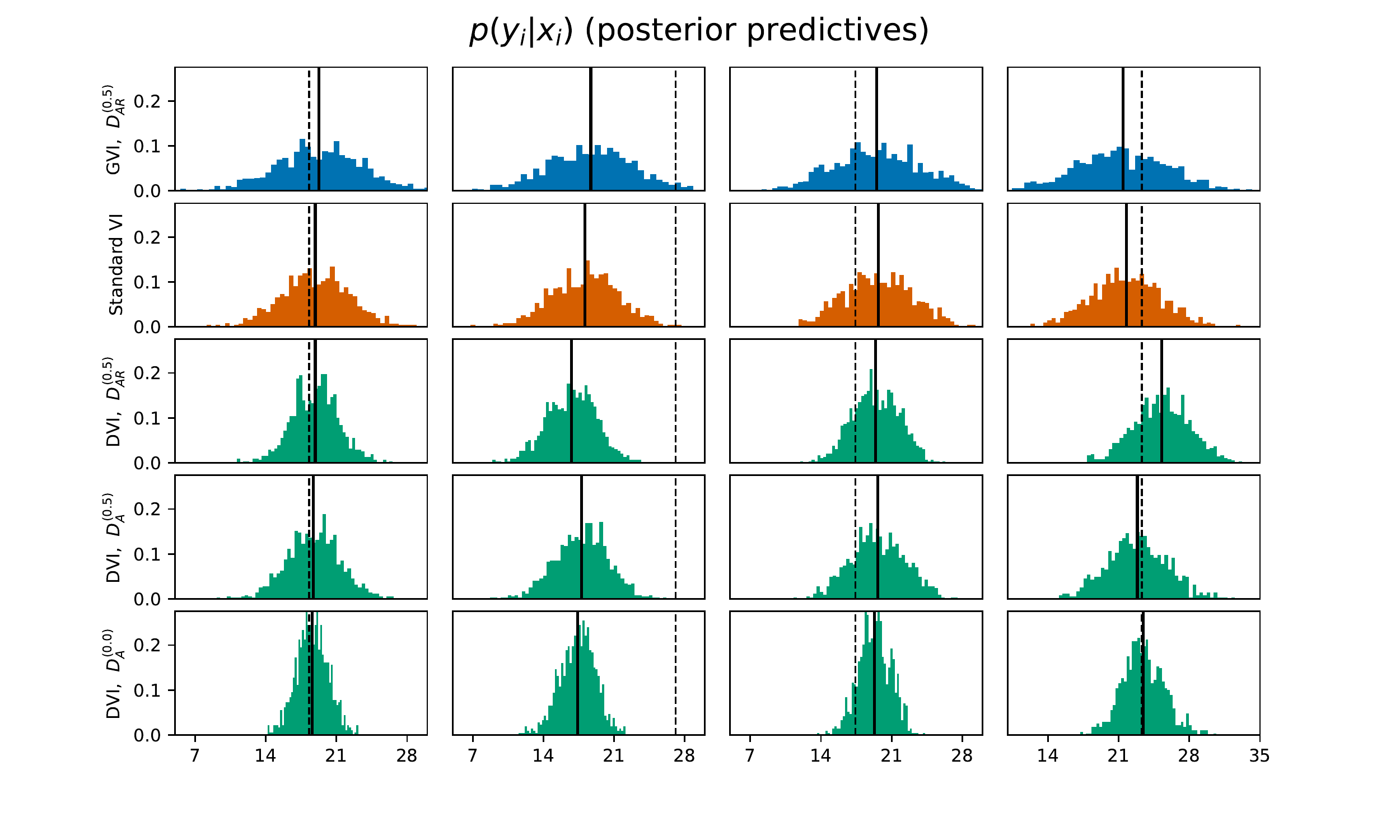}
\includegraphics[trim= {1.5cm 1.25cm 2.25cm 0cm}, clip,  
width=1\columnwidth]{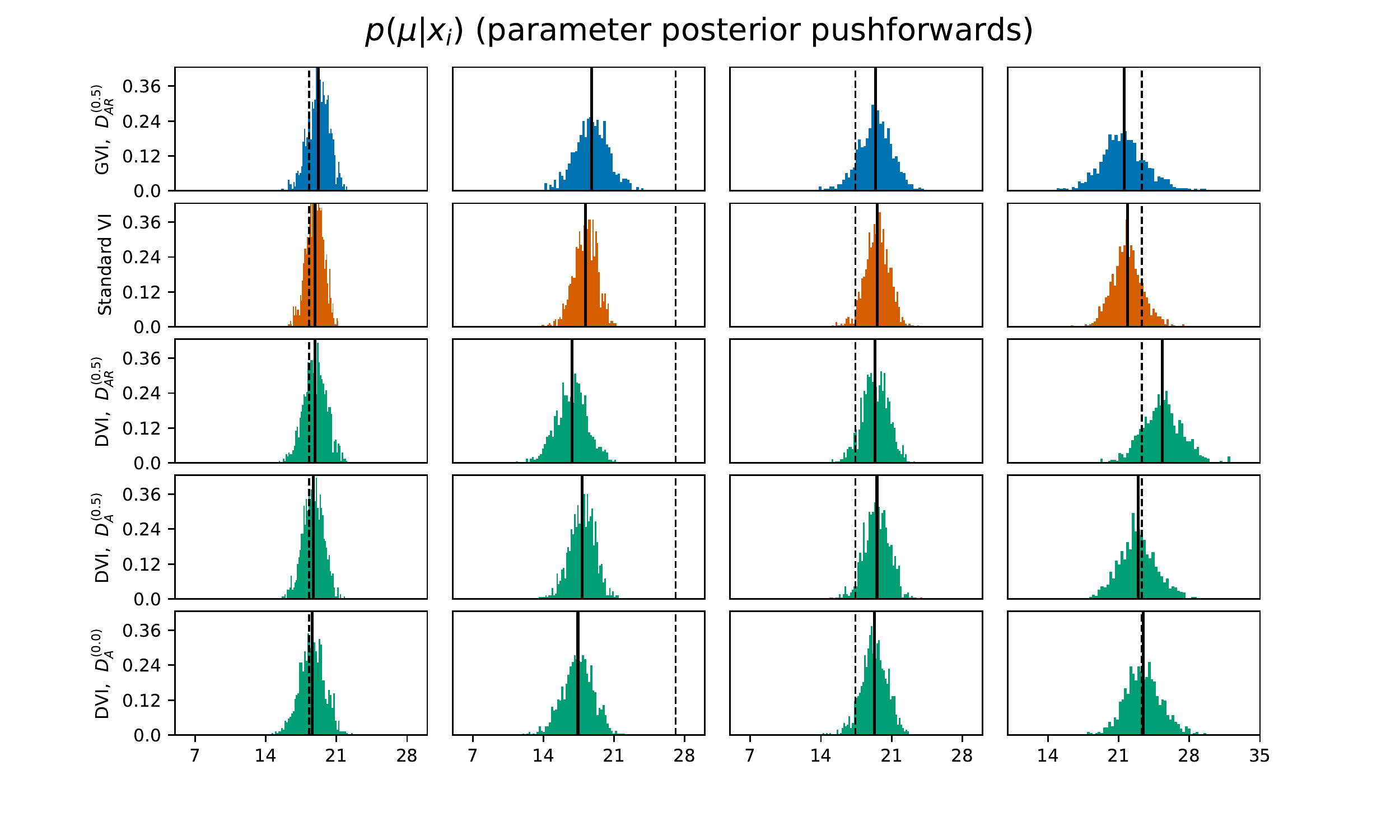}
\end{center}
\caption{
  \bv
  Depicted are test set predictions based on posterior predictives (\textbf{top panel}) and parameter posterior pushforwards (\textbf{bottom panel}) with four observations in the boston data set. Each column shows one observation (dashed line). 
  The predictive distributions (histogram) and their means (solid line) for each row correspond to \VIColor{\textbf{standard \VI}}, \DVIColor{\textbf{\DVI}} and \GVIColor{\textbf{\GVI}}. 
}
\label{Fig:posterior_predictives}
\vskip -0.25in
\end{figure}

\subsubsection{The surprising benefits of modularity (Figure \ref{Fig:posterior_predictives})}

While findings \ref{item:BNN_1_GVI} and \ref{item:BNN_1_banana} should not come as a surprise by themselves, they do raise an interesting question: In particular, \GVI for \RAD with $\alpha  = 0.5$ is the \textit{worst-performing} setting across the board. 
This is remarkable because this setting also constructs the only \GVI posteriors in our experiments with \textit{wider variances} than standard \VI. 
At the same time, producing wider variances and more conservative uncertainty quantification is one of the main motivations for Expectation Propagation (\EP) the presented \DVI methods, see for example Figure 1(a) in \citet{RenyiDiv} or Figure 8 in \citet{AlphaDiv}. 
%
This is puzzling: Are wider variances for $\*\theta$ somehow beneficial for \DVI posteriors' predictive performance while damaging that of \GVI posteriors?
As it turns out, this is not the case. 
Rather, while both \GVI with $\alpha = 0.5$ and all \DVI produce parameter posteriors with larger variances, in the case of \DVI this does not translate into predictive uncertainty. 
%

%

This phenomenon is depicted in Figure \ref{Fig:posterior_predictives}, which clearly shows that the additional uncertainty in the \DVI parameter posteriors $q^{\ast}_{\DVI}(\*\theta|\*\kappa^{\ast})$ is completely overshadowed by an extreme degree of variance shrinkage in the corresponding posterior predictives. 
In other words, the increased uncertainty in $\*\theta$ is outweighed by extremely small values for $\sigma^2$.
The plot demonstrates this by comparing the push-forward  $F \# q^{\ast}_{\DVI}(\cdot|\*\kappa^{\ast})$ with the posterior predictives. 
Formally, the push-forward is given by
\begin{IEEEeqnarray}{rCl}
    p(\mu|x_i) & = & \left( F \# q^{\ast}_{\DVI}(\cdot|\*\kappa^{\ast})\right)(\mu),
    \nonumber
\end{IEEEeqnarray}
where the operation $\#$ is simply a formalization of the following two operations: (i) sample $\*\theta \sim q^{\ast}_{\DVI}(\*\theta|\*\kappa^{\ast})$, (ii) compute $\mu = F(\*\theta)$.
The posterior predictive then integrates the push-forward measure $p(\mu|x_i)$ over the likelihood function as
\begin{IEEEeqnarray}{rCl}
    p(y_i|x_i) 
    & = & 
    \int_{\*\Theta}p_{\mathcal{N}}(y_i|x_i, \sigma^2, F(\*\theta)) q^{\ast}_{\DVI}(\*\theta|\*\kappa^{\ast}) d\*\theta
    = \int_{\mathbb{R}}p_{\mathcal{N}}(y_i|x_i, \sigma^2, \mu)p(\mu|x_i)d\mu.
    \nonumber
\end{IEEEeqnarray}
%
%
As Figure \ref{Fig:posterior_predictives} shows, the push-forward behaves as expected for both \GVI and \DVI. For \DVI, the same cannot be said for the posterior predictive: Interestingly, they generally have much \textit{less} variance than for standard \VI.

This surprising phenomenon is due to hyperparameter optimization for $\sigma^2$ and has an intimate link with the modularity result of Theorem \ref{Thm:GVI_modularity}. 
Since variational inference on $\sigma^2$ complicates the \DVI objectives,
both \citet{AlphaDiv} and \citet{RenyiDiv} do not infer $\sigma^2$ probabilistically. Instead, it is optimized over their objectives. This approach poses an optimization problem which for  $D = \AD$ and $D= \RAD$ is given by
%
%
\begin{IEEEeqnarray}{rCl}
    \widehat{\sigma}^2, q^{\ast}_{\DVI}(\*\theta|\*\kappa^{\ast}) &=& \argmin_{\sigma^2}\left\{ \argmin_{q \in \mathcal{Q}}D(q(\*\theta| \*\kappa)||q^{\ast}_{\text{B}}(\*\theta|\sigma^2,  x_{1:n}, y_{1:n})) \right\}.
    \label{eq:double_opt_DVI}
\end{IEEEeqnarray}
%
%
%
Crucially, the inner part of this objective conditions on the exact Bayesian posterior for a \textit{fixed} value of $\sigma^2$ and then seeks to approximate the posterior belief given by
\begin{IEEEeqnarray}{rCl}
    q^{\ast}_{\text{B}}(\*\theta|\sigma^2, x_{1:n}, y_{1:n})
    & \propto &
    \pi(\*\theta) \prod_{i=1}^n p_{\mathcal{N}}(y_i|x_i, \sigma^2, F(\*\theta)).
    \quad
    \nonumber
\end{IEEEeqnarray}
At the same time however, the outer part of the objective seeks to find a value for $\sigma^2$ which makes the posterior $q^{\ast}_{\text{B}}(\*\theta|\sigma^2, x_{1:n}, y_{1:n})$ as easily approximable as possible. 
 In other words, an objective which is explicitly motivated as a projection of $q^{\ast}_{\text{B}}(\*\theta|\sigma^2, x_{1:n}, y_{1:n})$ into $\mathcal{Q}$ also changes the very point from which to project into $\mathcal{Q}$. 
%

Though it would be computationally easy to perform probabilistic inference on $\sigma^2$ within \GVI, we also optimize $\sigma^2$ as a hyperparameter for comparability. Thus, we pose the alternative optimization problem
\begin{IEEEeqnarray}{rCl}
    \widehat{\sigma}^2, q^{\ast}_{\GVI}(\*\theta|\*\kappa^{\ast}) &=& \argmin_{\sigma^2}\left\{ \argmin_{q \in \mathcal{Q}}
    \left\{
    \mathbb{E}_{q}\left[ \sum_{i=1}^n -\log p_{\mathcal{N}}(y_i| x_i, \sigma^2, F(\*\theta))\right] + 
    \RAD(q||\pi)
    \right\}
    \right\}.
    \quad
    \label{eq:double_opt_GVI}
\end{IEEEeqnarray}
As Figure \ref{Fig:posterior_predictives} shows, the outcomes are drastically different: Unlike in the \DVI case, the predictive uncertainty for the \GVI posteriors move in the same direction as parameter uncertainty as $\alpha$ varies.
%
The modularity of \GVI makes it obvious what the optimization over $\sigma^2$ corresponds to in eq. \eqref{eq:double_opt_GVI}: Rather than choosing the best posterior  $q^{\ast}_{\text{B}}(\*\theta|\sigma^2, x_{1:n}, y_{1:n})$ from which to project into $\mathcal{Q}$, the optimization problem simply seeks to find the best possible loss $\ell_{\sigma^2}(y_i|x_i,F(\*\theta)) = -\log p(y_i|x_i, \sigma^2, F(\*\theta))$ over all $\sigma^2 \in \mathbb{R}_{+}$. 
%

\subsection{Deep Gaussian Processes}
\label{sec:experiments_GPs}

Deep Gaussian Processes (\DGPs) were introduced by \citet{DGPs} and extend the logic of deep learning to the nonparametric Bayesian setting. 
The principal idea is to construct a hierarchy of Gaussian Process (\GP) priors over latent spaces. 
Unlike with the \BNNs presented in the last section, the priors in \DGPs are usually refined at run-time by using various hyperparameter optimization schemes.
This is in fact crucial for \DGPs to provide good inferences: 
Indeed, it ensures that the inputs $\*X$ are mapped into latent spaces which are informative for the outputs $\*Y$.
As a consequence, unlike with \BNNs we expect there to be comparatively little merit in varying the uncertainty quantifier $D$ for \DGPs---a suspicion we experimentally confirm in Appendix \ref{Appendix:GP_additional_experiments}.
%
Accordingly, we instead focus on experiments that vary the loss $\ell$.
More specifically, we consider replacing the negative log score with a robust  scoring rule for the likelihood which is derived from the $\gamma$-divergence \citep{GammaDivSummable},
which drastically improves predictive performance. 
%
%
%

In the remainder, we first introduce \DGPs (Section \ref{sec:DGP_prelims}). Next, we provide a brief overview of the doubly stochastic inference procedure in \citet{DeepGPsVI} (Section \ref{sec:DGP_SVI}) and show how to adapt \DGPs to \GVI (Section \ref{sec:DGP_adaption}). Lastly, we present numerical experiments and their results (Section \ref{sec:DGP_experiment_results}).
These findings are also summarized with a higher level of detail in a separate technical report \citep{RDGP}.


%

\subsubsection{Preliminaries for \DGPs}
\label{sec:DGP_prelims}

Given observations $(\*X, \*Y)$ where $\*X \in \mathbb{R}^{n\times D_x}$ and $\*Y \in \mathbb{R}^{n \times p}$, a \DGP of $L$ layers introduces $L$ latent functions $\{\*F^l\}_{l=1}^L$. Here, $\*F^l$ is matrix-valued and of dimension $D^l\times D^{l+1}$.
%
%
Setting $\*F^0 = \*X$, $D^0 = D_x$ and $D^{l+1} = p$, one can write the \DGP construction as
\begin{IEEEeqnarray}{rClCl}
    \*Y|\*F^L &&& \sim & 
    p\left( \*Y \middle|\; \*F^L \right)
    \nonumber \\
    \*F^L|\*F^{L-1} & = & f^{L}(\*F^{L-1}) & \sim & \GP\left(\mu^{L}(\*F^{L-1}), \K^{L}(\*F^{L-1}, \*F^{L-1}) \right) \nonumber \\
         &&& \quad \dots \nonumber \\
    \*F^1|\*F^0 & = & f^{1}(\*F^{0})& \sim & \GP\left(\mu^1(\*F^0), \K^1(\*F^0, \*F^0) \right), \nonumber
\end{IEEEeqnarray}
where the mean and covariance functions are  $\mu^l:\mathbb{R}^{D^l}\to\mathbb{R}^{D^{l+1}}$ and $\K^l:\mathbb{R}^{D^l\times D^l}\to\mathbb{R}^{D^{l+1}\times D^{l+1}}$.
Scalable inference strategies for this model generally rely on \VI \citep{DGPs, VAEDGPs, DeepGPsVI, NestedVariationalDGPs}, Monte Carlo methods  \citep{SamplingDGPs, SMCDGPs} or more specialized approaches \citep{RandomFourierDGP}.
%
%
In the remainder, we discuss the implications of Generalized Variational Inference (\GVI) in relation to the arguably most promising \VI approach of \citet{DeepGPsVI}.
Unlike previous \VI methods, it encodes conditional dependence into the variational family $\mathcal{Q}$ and comprehensively outperformed Expectation Propagation (\EP) based alternatives \citep{DGPEP}.

\subsubsection{Doubly stochastic \VI in \DGPs}
\label{sec:DGP_SVI}

%
First, define $m$ inducing points $\*Z^l = (\*z^{l}_1, \*z^{l}_2, \dots, \*z^{l}_m)^T$ and their function values $\*U^l = (f^l(\*z^{l}_1), f^l(\*z^{l}_2), \dots, f^l(\*z^{l}_m))^T$ \citep[for details on inducing points, see ][]{GPApprox2, GPApprox4, InducingBonilla, InducingProcesses}. 
For improved readability, we drop $\*X$ and $\*Z^l$ from the conditioning sets and denote the $i$-th row of $\*F^l$ as $\*f^{L}_i$.
With this, the joint distribution of the \DGP is
\begin{IEEEeqnarray}{rCl}
    p\left(\*Y, 
    \{\*F^l\}_{l=1}^L, \{\*U^l\}_{l=1}^L
    \right)
    & = &
    \underbrace{\prod_{i=1}^n p(\*y_i|\*f^{L}_i)}_{\text{likelihood}} \times
    \underbrace{\prod_{l=1}^L 
        p\left(\*F^l \middle|\; \*U^l, \*F^{l-1}, \*Z^{l-1} \right)p\left( \*U^l \middle|\; \*Z^{l-1}  \right)}_{\text{ (\DGP) prior}}. \quad\quad \nonumber
\end{IEEEeqnarray}
The posteriors $p\left( \{\*F^l\}_{l=1}^L, \{\*U^l\}_{l=1}^L
\right)$ and $p\left( \{\*F^l\}_{l=1}^L \right)$ 
%
are intractable. 
%
%
The \VI method proposed in \citet{DeepGPsVI} overcomes this with the variational family given by
\begin{IEEEeqnarray}{rCl}
    q\left( \{\*F^l\}_{l=1}^L, \{\*U^l\}_{l=1}^L 
    \right) 
    & = & 
    \prod_{l=1}^L 
        p\left(\*F^l \middle|\; \*U^l, \*F^{l-1}, \*Z^{l-1} \right)
        q\left( \*U^l \right); 
        \quad
    q\left( \*U^l \right)
     = 
        \mathcal{N}\left(\*U^l\middle|\; \*m^l, \*S_l \right).
    \nonumber
\end{IEEEeqnarray}
%
%
This allows for exact integration over the inducing points $\{\*U^l\}_{l=1}^L$, yielding 
%
\begin{IEEEeqnarray}{rCl}
    q\left(\{\*F^l\}_{l=1}^L
    \right)
    & = &
    \prod_{l=1}^L
    \mathcal{N}\left(\*F^l \middle|\; \*\mu^l, \*\Sigma_l \right).
    \nonumber
\end{IEEEeqnarray}
As shown in \citet{DeepGPsVI}, this enables a doubly stochastic minimization of the negative Evidence Lower Bound (\ELBO) given by 
\begin{IEEEeqnarray}{rCl}
    &&
    \mathbb{E}_{q(\*F^L)}\left[ \sum_{i=1}^n -\log p(\*y_i|\*F^L) \right] 
    +
    \KLD\left(  q(\{\*F^l\}_{l=1}^L, \{\*U^l\}_{l=1}^L) 
                \bigg\| 
                p(\{\*F^l\}_{l=1}^L, \{\*U^l\}_{l=1}^L) 
    \right) 
    \nonumber \\
    & = & -\sum_{i=1}^n\mathbb{E}_{q(\*f^{L}_i)}
    \left[ 
        \log p(\*y_i|\*f^{L}_i)
    \right]
    +
    \sum_{l=1}^L\KLD(q(\*U^l)||p(\*U^l)).
    \label{eq:DGP_elbo}
\end{IEEEeqnarray}
For optimization, the samples for $\*F^l$ are drawn using the variational posteriors from the previous layers so that approximating the expectations over $q(\*f_i^L)$ induces the first layer of stochasticity.
The second layer is due to drawing mini-batches from $\*X = \*F^0$ and $\*Y$.
Because of this large degree of stochasticity, it is appealing if $\mathbb{E}_{q(\*f^{L}_i)}\left[\log p(\*y_i|\*f^{L}_i)\right]$ is available in closed form, which is for instance the case if $p = p_{\mathcal{N}}$ is a normal likelihood.   

\subsubsection{Adaption to \GVI}
\label{sec:DGP_adaption}

The objective in eq. \eqref{eq:DGP_elbo} suggests itself naturally to a \GVI variant.
This raises two questions: 
\begin{itemize}
    \myitem{($\*D$)} Is it theoretically coherent with the meaning and function of $D$ in the axiomatic development of Section \ref{sec:axioms} to simply replace the \KLD-terms layer-wise?
    \label{item:DGP_D_mod}
    \myitem{($\*\ell$)} Can one derive closed forms for the expectations when the log scoring rule is replaced by robust alternatives $\mathcal{L}^{\beta}$ or $\mathcal{L}^{\gamma}$ derived from the $\beta$- and $\gamma$-divergence?
    \label{item:DGP_loss_mod}
\end{itemize}
As shown next, we can give a positive answer to both these questions.
%
%

\subparagraph{\ref{item:DGP_D_mod}}
Conveniently and as shown in  \citet{DeepGPsVI}, 
%
\begin{IEEEeqnarray}{rCl}
    \KLD\left(  q(\{\*F^l\}_{l=1}^L, \{\*U^l\}_{l=1}^L) 
                \bigg\| 
                p(\{\*F^l\}_{l=1}^L, \{\*U^l\}_{l=1}^L) 
    \right) 
    & = &
    \sum_{l=1}^L\KLD(q(\*U^l)||p(\*U^l)).
    \label{eq:DGP_regularizer}
\end{IEEEeqnarray}
%
%
A natural question is whether one can reverse-engineer this finding: If we simply pick a collection of other divergences $D^l(q(\*U^l)||p(\*U^l))$ for each layer $l$ and combine them additively, does the result define a valid divergence between $q(\{\*F^l\}_{l=1}^L, \{\*U^l\}_{l=1}^L)$ and $p(\{\*F^l\}_{l=1}^L, \{\*U^l\}_{l=1}^L)$?
As the next Corollary shows, one can prove that reverse-engineering prior regularizers inspired by eq. \eqref{eq:DGP_regularizer} is feasible so long as the layer-specific divergences $D^l$ are $f$-divergences or monotonic transformations of $f$-divergences.
The proof relies on a technical Lemma and is given in Appendix \ref{Appendix:GP_divergence_change}

\begin{corollary}
    In the \DGP construction of eq. \eqref{eq:DGP_elbo}, replacing the sum of \KLD-terms by 
    \begin{IEEEeqnarray}{rCl}
        \sum_{l=1}^LD^l(q(\*U^l)||p(\*U^l))
        \nonumber
    \end{IEEEeqnarray}
    defines a valid divergence between $q(\{\*F^l\}_{l=1}^L, \{\*U^l\}_{l=1}^L)$ and $p(\{\*F^l\}_{l=1}^L, \{\*U^l\}_{l=1}^L)$  so long as $D^l$ is an $f$-divergence or a divergence obtained as a monotonic transform $g$ of an $f$-divergence for all $l=1,2,\dots L$.
    \label{corollary:GP_divergence_change}
\end{corollary}

\subparagraph{\ref{item:DGP_loss_mod}}
Next, we turn attention to modifying the loss terms in eq. \eqref{eq:DGP_elbo}. 
First, note that
\begin{IEEEeqnarray}{rCl}
    \mathbb{E}_{q(\*F^L)}\left[ \sum_{i=1}^n -\log p(\*y_i|\*F^L) \right] 
    & = &
    -\sum_{i=1}^n\mathbb{E}_{q(\*f^{L}_i)}
    \left[ 
        \log p(\*y_i|\*f^{L}_i)
    \right].
    \nonumber
\end{IEEEeqnarray}
%
%
This identity still holds if one replaces the negative log with other scoring rules. 
As the next Proposition shows, we even retain closed forms for the regression case and the scoring rules
\begin{IEEEeqnarray}{rCl}
    \Lb(\*f_i^L, \*y_i) & = &
        -
			\frac{1}{\beta-1}
			p(\*y_i|\*f_i^L)^{\beta-1} 
			+ \frac{I_{p, \beta}(\*f_i^L)}{\beta}
		\nonumber    \\
	\Lg(\*f_i^L, \*y_i) & = &
        -
			\frac{1}{\gamma-1}p(\*y_i|\*f_i^L)^{\gamma-1} \cdot
			  \frac{\gamma}{I_{p, \gamma}(\*f_i^L)^{-\frac{\gamma-1}{\gamma}}}
			. \; 
			\nonumber
\end{IEEEeqnarray}
Crucially, the integral term $I_{p, c}(\*f_i^L) =  \int p(\*y|\*f_i^L)^{c}d\*y$ is generally available in closed form for exponential families. 
%
As the notation suggests, $\Lb$ is linked to the $\beta$-divergence in the same way we linked the $\log$ score to the \KLD in Section \ref{sec:GVI_model_misspecification}, see also \citet{BasuDPD}.  
Similarly, $\Lg$ is derived from the $\gamma$-divergence as explained in \citet[][]{GammaDivSummable}.
%
As also alluded to in Section \ref{sec:quasi-conjugate-inference}, $\Lg$ ($\Lb$) recovers the log score as $\gamma \to 1$  ($\beta \to 1$) and produces robust inferences for $\gamma > 1$ ($\beta > 1$).
%
Figure \ref{Fig:Influence_fct_pic} depicts this for $\Lb$, and the behaviour is very similar for $\Lg$. 
\begin{proposition}[Closed forms for robust \DGP regression]
    If it holds that $\*y_i \in \mathbb{R}^d$,
    \begin{IEEEeqnarray}{rClCrCl}
        p(\*y_i|\*f^{L}_i) & = & \mathcal{N}\left(\*y_i; \*f^{L}_i, \sigma^2I_d\right); & \quad &
        q(\*f^{L}_i) & = & \mathcal{N}(\*f^{L}_i; \*\mu, \*\Sigma), 
        \nonumber
    \end{IEEEeqnarray} 
    then for the quantities given by
    \begin{IEEEeqnarray}{rClCrClCrCl}
        \widetilde{\*\Sigma}^{-1} & = & \left(\frac{c}{\sigma^s}\*I_d + \*\Sigma^{-1}\right);
        & \quad &
        \widetilde{\*\mu} & = & \left( \frac{c}{\sigma^2}\*y_i + \*\Sigma^{-1}\*\mu \right);
        & \quad &
        I(c) & = & (2\pi\sigma^2)^{-0.5dc}c^{-0.5d}
        \nonumber
    \end{IEEEeqnarray}
    and for 
    \begin{IEEEeqnarray}{rCl}
        E(c) & = & \frac{1}{c}
    \left({2\pi}\sigma^2\right)^{-0.5dc}
    \frac{|\widetilde{\*\Sigma}|^{0.5}}{ |\*\Sigma|^{0.5} }
    \exp\left\{
 -\frac{1}{2}\left(
     \frac{c}{\sigma^2}\*y_i^T\*y_i + \*\mu^T\*\Sigma^{-1}\*\mu -
     \widetilde{\*\mu}^T\widetilde{\*\Sigma}\widetilde{\*\mu}
 \right)
    \right\} \nonumber
    \end{IEEEeqnarray}
    the following expectations are available in closed form:
    \begin{IEEEeqnarray}{rCl}
        \mathbb{E}_{q(\*f^{L}_i)}\left[\Lb(\*f^{L}_i, \*y_i)\right] & = &
        -E(\beta-1) + \frac{I(\beta)}{\beta}
        \nonumber\\
        \mathbb{E}_{q(\*f^{L}_i)}\left[\Lb(\*f^{L}_i, \*y_i)\right] & = &
        -E(\gamma-1)\cdot\frac{\gamma}{I(\gamma)^{\frac{\gamma}{\gamma-1}}}
        \nonumber
    \end{IEEEeqnarray}
    \label{Proposition:closed_form_robust_DGP_regression}
\end{proposition}
As shown in Appendix \ref{Appendix:GP_loss_change}, it is easy but tedious to derive this result.
While the results of using $\Lb$ and $\Lg$ are often virtually identical (see for instance Figure \ref{Fig:BNN_experiment_3} in Appendix \ref{Appendix:BNN}), our experiments on \DGPs will exclusively use the $\Lg$.
This is done because unlike for $\Lb$, computations with $\Lg$ can be performed in its numerically more stable $\log$ form.

\begin{figure}[h!]
\begin{center}
\includegraphics[trim= {0.25cm 0.0cm 0.3cm 0.0cm}, clip,  
 width=1.00\columnwidth]{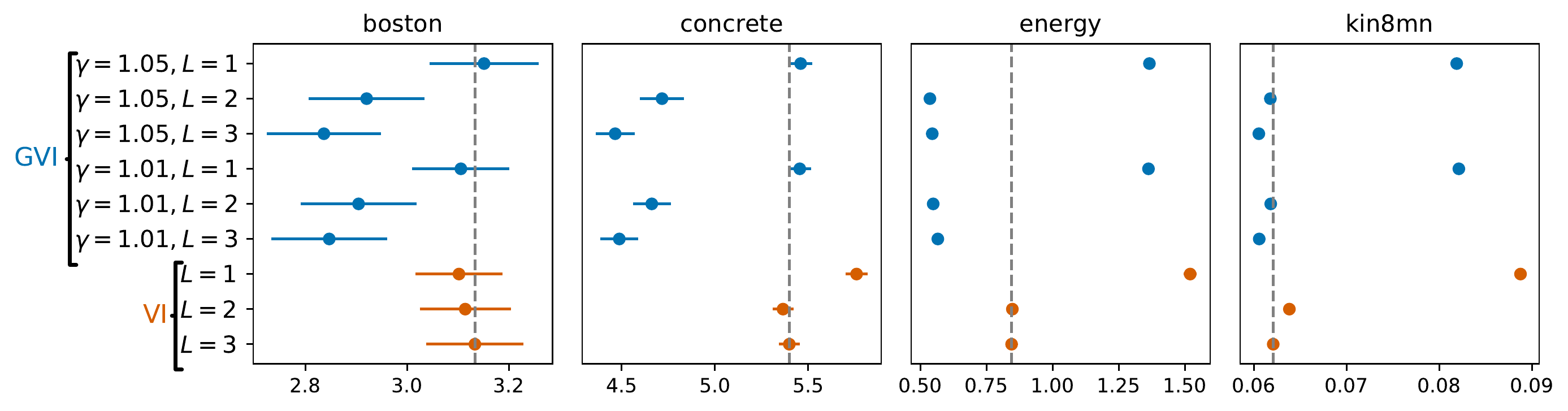}
 \includegraphics[trim={0.25cm 0.0cm 0.3cm 0.75cm}, clip,  
 width=1.00\columnwidth]{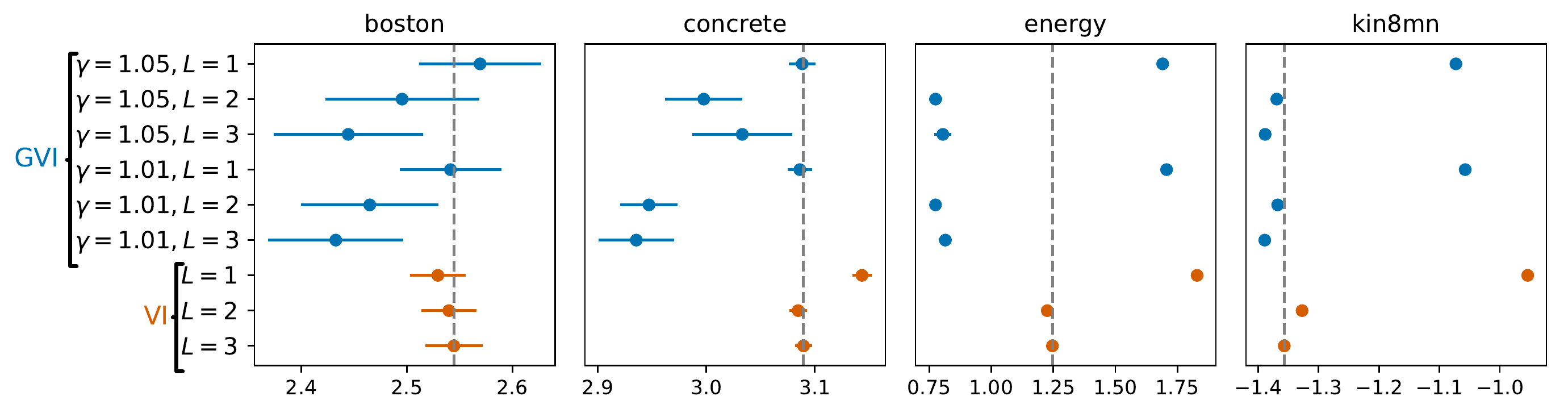}
\vspace*{0.25cm}
\includegraphics[trim= {2.3cm 0.3cm 3.25cm 0.0cm}, clip,  
 width=1.00\columnwidth]{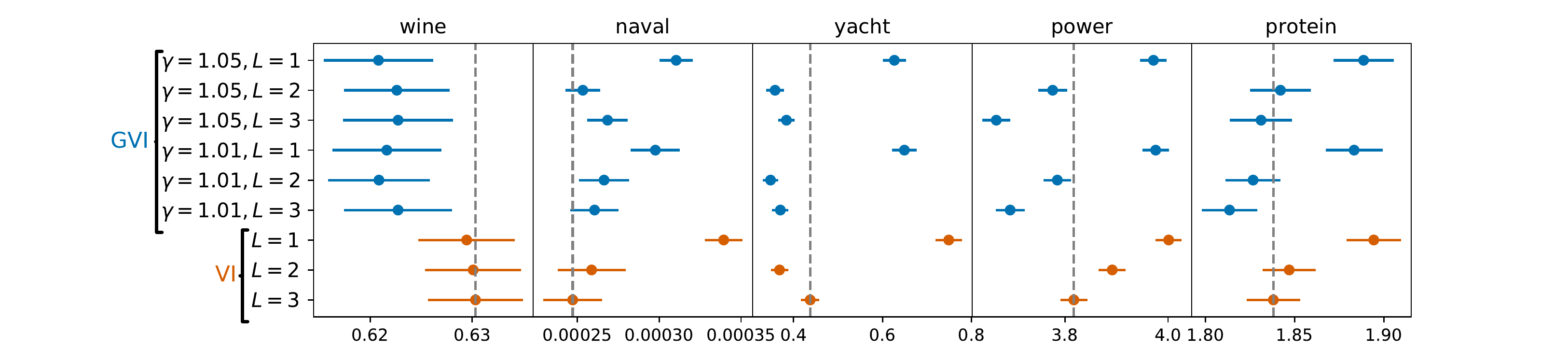}
 \includegraphics[trim={2.3cm 0.0cm 3.25cm 0.90cm}, clip,  
 width=1.00\columnwidth]{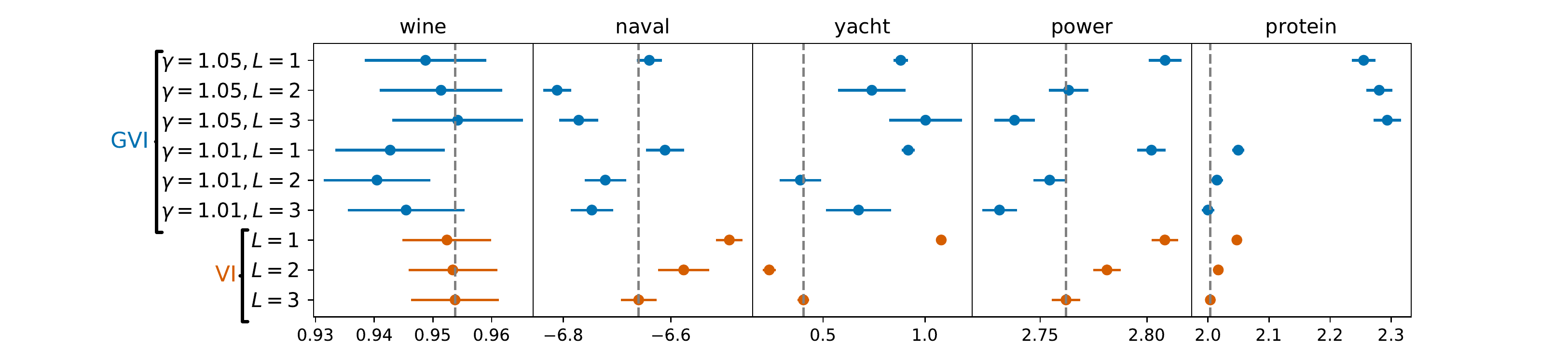}
\caption{
    \bv
	Top rows depict \RMSE, bottom rows the \NLL across a range of data sets using \DGPs. Dots correspond to means, whiskers to standard errors. The further to the left, the better the predictive performance. 
	For the depicted selection of data sets, \GVIColor{\textbf{\GVI}} comprehensively outperforms \VIColor{\textbf{standard \VI}}.
}
\label{Fig:DGP_experiment}
\end{center}
\end{figure}

\subsubsection{Results}
\label{sec:DGP_experiment_results}

As with the experiments on \BNNs in the previous section, we make comparisons as fair as possible by using the \texttt{gpflow} \citep{gpflow} implementation of \citet{DeepGPsVI}.
Further, we use the same settings, meaning that all experiments use 20,000 iterations of the ADAM optimizer \citep{ADAM} with a learning rate of 0.01 and default settings for all other hyperparameters.
We perform inference for each of the UCI data sets \citep{UCI} after normalization using the RBF kernel with dimension-wise lengthscales, 
100 inducing points, with batch sizes of $\min(1000, n)$ and $D^l = \min(D_x, 30)$.
As before, we use 50 random splits with 90\% training and 10\% test data to assess predictive performance in terms of negative log likelihood (\NLL) and root mean square error (\RMSE). 
With this, we compare two inference schemes:
\begin{itemize}
    \item[(1)] The state of the art \VIColor{\textbf{standard \VI}} techniques of \citet{DeepGPsVI};
    \item[(2)] A \GVIColor{\textbf{\GVI}} variant of the same inference method which replaces the log score with the robust $\gamma$-divergence based scoring rule $\Lg$.
\end{itemize}
For choosing $\gamma$, we note that inferences are robust for $\gamma > 1$ and that $\Lg$ recovers the log score as $\gamma \to 1$. 
At the same time, the scoring rule will grow increasingly happy to ignore virtually all of the data as $\gamma \to \infty$. 
Accordingly, one will typically want to pick
\begin{IEEEeqnarray}{rCl}
    \gamma = 1 + \varepsilon
    \nonumber
\end{IEEEeqnarray} 
for a small $\varepsilon > 0$. 
Choosing $\gamma$ in this way encodes the intuition that a good scoring rule will behave like the log score for all but the most extreme outliers.
We thus pick $\varepsilon \in \{0.01, 0.05\}$, a range of values also successfully used in \citep{MartinJankowiak}. 
We note that hyperparameter optimization might appear to be the natural choice for picking $\gamma$, but will not perform well in practice: Rather than producing robust inferences, this will select for a value of $\gamma$ generally producing the smallest \GVI objective values across $\mathcal{Q}$\footnote{
    In practice, this means that hyperparameter optimization pushes $\gamma \to 1$ or $\gamma \to \infty$, depending on the magnitudes of $\{p(\*y_i|\*f_i^L)\}_{i=1}^n$.
}.

The results are depicted in Figure \ref{Fig:DGP_experiment} and confirm our two main intuitions about robustness: 
Firstly, the robust scoring rule provides a significant performance improvement. 
Secondly, the smaller value of $\gamma$ (which will be closer to the log score) generally outperforms the larger value of $\gamma$, though both choices are equally good in many data sets\footnote{
    We expect this second finding about $\gamma$ to generalize to new settings so long as the inputs are normalized and the outputs are not high-dimensional (see also Figure \ref{Fig:BNN_experiment_3} for some empirical evidence of this on \BNNs), which would make $\gamma = 1.01$ a decent default choice in such scenarios.
    }.
We believe that the performance gains of the robust scoring rule is due to large parts of the latent spaces being non-informative. 
This implies that it is beneficial to  implicitly down-weight the influence of these non-informative parts of the latent space. 
It is clear that robust scoring rules do exactly that (see for instance Figure \ref{Fig:Influence_fct_pic}), which explains their superior performance in the \DGP experiments.
This intuition is further bolstered by the following observation: Generally, performance \textit{improves} with a larger number of layers $L$ under the robust score $\Lg$, but \textit{worsens} under the log score.
In other words: The more dispersed the  prior over the latent space (i.e., the \DGP) becomes, the more inferential outcomes benefit from implicitly ignoring its non-informative regions.
In Appendix \ref{Appendix:DGPs} we provide a small batch of additional results showing that as expected, modifying the uncertainty quantifier $D$ is less beneficial for \DGPs than it is for \BNNs. 
Most likely, this is due to hyperparameter optimization for the kernels of the \DGP. Together with the fact that Gaussian Processes are far more informative priors than fully factorized normals, careful selection of the hyperparameters ensures that unlike in the \BNN case, the prior is well-specified.


\section{Discussion \& Conclusion}

In this work, we re-examined the working assumptions that have proven powerful and useful in spreading Bayesian inference into virtually all domains of scientific endeavour.
Studying the challenges of contemporary inference, we concluded that the traditional assumptions underlying Bayesian statistics are misaligned with the realities of modern large-scale problems. 
At the same time, we adopted an optimization-centric view on Bayesian inference that allows us to prove a novel optimality result for standard Vatiational Inference (\VI). 
In spite of this theoretical result, we pointed out that belief distributions computed as alternative approximations to the Bayesian posterior often perform better in practice. 
We explained that this is because standard \VI is optimal \textit{only} relative to  a particular objective function---specifically, an objective function whose origins are the very assumptions that are misaligned with reality.
Inspired by this insight, we proceed to derive a new class of posterior belief distributions that do not rely on these assumptions.
%
To do so, we first set out a new axiomatic foundation for Bayesian inference culminating in the Rule of Three (\RoT). The \RoT is an optimization problem with three arguments, each of which addresses one of the shortcomings of the standard Bayesian assumptions. Yet, while it defines a much larger family of posteriors, the \RoT also recovers the standard Bayesian update rule as a special case.
Building on this novel generalized class of posteriors, we introduce Generalized Variational Inference (\GVI). 
In essence, \GVI restricts attention to the tractable subset of \RoT posterior beliefs contained within a variational family.
Next, we show that \GVI satisfies a number of desirable theoretical properties: It is modular (in the sense of Theorem \ref{Thm:GVI_modularity}) and consistent in the frequentist sense. Moreover, the objectives for a sub-class of \GVI posteriors  form an approximate evidence lower bound on a generalized Bayesian posterior.
On the practical side, we show three generic applications of \GVI in the broad context of customized uncertainty quantification and robustness. Specifically, we demonstrate how \GVI can be used to adjust posterior variances and produce inferences that are robust to model and prior misspecification. 
Lastly, we demonstrate \GVI's power, usefulness and versatility on two model classes that encapsulate the misalignment between the assumptions underlying the traditional Bayesian paradigm and the realities of modern large-scale Bayesian inference: Bayesian Neural Networks (\BNNs) and Deep Gaussian Processes (\DGPs). 
%

The current work makes two major contributions. The first of these is conceptual: We propose a generalization of Bayesian inference through the Rule of Three (\RoT). This aspect of our work stands in the tradition of previous generalizations of Bayesian inference such as the one in \citet{Bissiri} and \citet{Jewson}. Unlike previous work however, we take the first step in the development of Bayesian inference procedures that generalize beyond multiplicative belief updates. 
Indeed, this step is natural once one observes the intimate connection between Bayesian inference and infinite-dimensional optimization as set out in Observation \ref{observation:isomorphism}.
As explained in Sections \ref{sec:VI_and_Bayes_posteriors} and \ref{sec:RoT}, an immediate consequence of this generalization is a meaningful taxonomy of various variational inference procedures: Unlike most other variational approximations to the Bayesian posterior, standard Variational Inference (\VI) is a special case of the \RoT. This special standing of standard \VI also expresses itself in Theorem \ref{thm:VI_optimality}, which endows standard \VI with an (objective-dependent) quality guarantee that is absent from alternative approximation procedures.

The second contribution is methodological and consists in making the \RoT useful for real world inference problems via Generalized Variational Inference (\GVI). 
We show that \GVI modularly addresses the three shortcomings associated with traditional Bayesian inference. This is done by linking it to the literature on generalized Bayesian inference and robust scoring rules as well as to the literature on robust divergences.
As Section \ref{sec:experiments} shows with two applications on large-scale inference problems, \GVI posteriors of this form can yield significant predictive performance improvements in modern statistical machine learning models.

With the provision of a new optimization-centric generalization on Bayesian inference, the current paper is only the first step on a long road to designing posteriors that conform with the demands of contemporary models and inferential problems. 
In the wake of this, several important questions have been left unanswered. For example, it is unclear how to choose hyperparameters occurring in the loss or uncertainty quantifier beyond simplistic (albeit well-working) rules of thumb.
Further, we have not characterized the class of posteriors satisfying the axioms in Section \ref{sec:axioms} uniquely. Though the \RoT is unique when restricting attention to elementary functions and is arguable the most desirable form due to its relationship to the standard Bayesian and variational posteriors, we are convinced that the uniqueness result of Theorem \ref{Thm:additivityD} can be derived under much stronger conditions.
\GVI also has an obvious intimate connections with PAC-Bayesian approaches that we will be exploring in the near future.
Moreover, the flexibility in choosing different uncertainty quantifiers $D$ brings about another interesting question: Given that frequentist consistency holds, what impact does $D$ have on the contraction rate? And do certain special cases of $D$ endow \GVI with compelling geometric interpretations?

In summary, the current work is but the start of an investigation into the theoretical, methodological and applied consequences of the \RoT and \GVI.
It is clear that the ideas introduced in the current paper---while barely scratching the surface of the possible---have produced valuable insights and shown much promise in all three of these regards.
Consequently, it is with much excitement that we look forward to future contributions on questions of theory, methodology and practice surrounding the \RoT and \GVI. 

\acks{
    We would like to cordially thank Edwin Fong, Benjamin Guedj, Chris Holmes, David Dunson,  Mark van der Wilk, Giles Hooker and Alex Alemi for fruitful discussions, insights, comments and pointers that were invaluable for improving the paper.    
    JK and JJ are funded by \EPSRC grant EP/L016710/1 as part of the Oxford-Warwick  Statistics Programme (\OxWaSP). JK is additionally funded by the Facebook Fellowship Programme and the London Air Quality project at the Alan Turing Institute for Data Science and AI. TD acknowledges funding from \EPSRC grant EP/T004134/1, the Lloyd's Register Foundation programme on Data Centric Engineering, and the London Air Quality project at the Alan Turing Institute for Data Science and AI. This work was furthermore supported by The Alan Turing Institute for Data Science and AI under \EPSRC grant EP/N510129/1 in collaboration with the Greater London Authority.
}


\newpage

\appendix

\section{Definitions for robust divergences}
\label{Appendix:definitions}


The following is an overview of definitions for the most important divergences that are used throughout the paper.
\begin{definition}[The $\alpha\beta\gamma$-divergence \ABGD \citep{ABCdiv}]
The $\alpha\beta\gamma$-divergence \ABGD \cite{ABCdiv} takes the form 
\begin{IEEEeqnarray}{rCl}
\ABGD(q(\*\theta)||\pi(\*\theta)) 
&=&\frac{1}{\alpha(\beta-1)(\alpha+\beta-1)r}\left[\left(\tilde{D}_{G}^{(\alpha,\beta)}(q(\*\theta)||\pi(\*\theta))+1\right)^{r}-1\right]
\nonumber
\end{IEEEeqnarray}
where $r>0$, $\alpha\neq 0$, $\beta\neq 1$ and
\begin{IEEEeqnarray}{rCl}
\tilde{D}_{G}^{(\alpha,\beta)}(q(\*\theta)||\pi(\*\theta)) 
&=&\int \left(\alpha q(\*\theta)^{\alpha+\beta-1}+(\beta-1)\pi(\*\theta)^{\alpha+\beta-1}-(\alpha+\beta-1)q(\*\theta)^{\alpha}\pi(\*\theta)^{\beta-1}\right)d\*\theta
\nonumber
\end{IEEEeqnarray}
\label{Def:ABGdiv}
\end{definition}
Below we list some well-known special cases of the  family of divergences defined by \ABGD that we use in the main paper. This exposition is a summary of the review conducted in \cite{ABCdiv}. 
%
\begin{definition}[The $\alpha$-divergence (\AD) \citep{chernoff1952measure,AmariBook}]
The $\alpha$-divergence is defined as
\begin{IEEEeqnarray}{rCl}
\AD(q(\*\theta)||\pi(\*\theta))&=&\frac{1}{\alpha(1-\alpha)}\left\lbrace 1-\int q(\*\theta)^{\alpha}\pi(\*\theta)^{1-\alpha}d\*\theta \right\rbrace,
\nonumber
\end{IEEEeqnarray}
where $\alpha\in\mathbb{R}\setminus \left\lbrace0,1\right\rbrace$. Note that \AD is recovered from \ABGD when $r=1$ and $\beta=2-\alpha$. \AD is also a member of the $f$-divergence family.
\label{Def:alphaD}
\end{definition}

\begin{definition}[R\'enyi's $\alpha$-divergence (\RAD) \citep{renyi1961measures}]
R\'enyi's $\alpha$-divergence is defined as
\begin{IEEEeqnarray}{rCl}
\RAD(q(\*\theta)||\pi(\*\theta))&=&\frac{1}{\alpha(\alpha-1)}\log\left(\int q(\*\theta)^{\alpha}\pi(\*\theta)^{1-\alpha}d\*\theta\right),
\nonumber
\end{IEEEeqnarray}
where $\alpha\in\mathbb{R}\setminus \left\lbrace0,1\right\rbrace$.  \RAD is recovered from \ABGD in the limit as $r\rightarrow0$ and $\beta=2-\alpha$. 
Note that we use the rescaled version proposed by \citet{ConvexStatisticalDistances, ABCdiv} rather than the original parameterization of \citet{renyi1961measures} because it links the divergence more closely to other robust alternatives of the \KLD.
%
\label{Def:renyiAlphaD}
\end{definition}

\begin{definition}[The $\beta$-divergence (\BD)  \citep{BasuDPD,mihoko2002robust}]
The $\beta$-divergence \citep{mihoko2002robust} was originally introduced under the name ''density power divergence‘‘ and is defined as
\begin{IEEEeqnarray}{rCl}
\BD(q(\*\theta)||\pi(\*\theta))
&=&\frac{1}{\beta(\beta-1)}\int q(\*\theta)^{\beta}d\*\theta+\frac{1}{\beta}\int\pi(\*\theta)^{\beta}d\*\theta -\frac{1}{\beta-1}\int q(\*\theta)\pi(\*\theta)^{\beta-1}d\*\theta,\nonumber
\end{IEEEeqnarray}
where $\beta\in\mathbb{R}\setminus \left\lbrace0,1\right\rbrace$. \BD is recovered from \ABGD when $r=\alpha=1$. \BD is a member of the Bregman-divergence family. 
\end{definition}

\begin{definition}[The $\gamma$-divergence (\GD) \citep{GammaDivNotSummable}]
The $\gamma$-divergence \citep{GammaDivNotSummable} is defined as
\begin{IEEEeqnarray}{rCl}
\GD(q(\*\theta)||\pi(\*\theta))&=&\frac{1}{\gamma(\gamma-1)}\log\frac{\left(\int q(\*\theta)^{\gamma}d\*\theta\right)\left(\int\pi(\*\theta)^{\gamma}d\*\theta\right)^{\gamma-1}}{\left(\int q(\*\theta)\pi(\*\theta)^{\gamma}d\*\theta\right)^{\gamma}},
\nonumber
\end{IEEEeqnarray}
where $\gamma\in\mathbb{R}\setminus \left\lbrace0,1\right\rbrace$. \GD is recovered from \ABGD in the limit as  $r\rightarrow0$, $\alpha=1$ and $\beta=\gamma$.
The \GD can be shown to be generated from the \BD applying the following transformation
\begin{IEEEeqnarray}{rCl}
    c_0\int g(x)^{c_1}f(x)^{c_2} dx&\rightarrow& c_0\log \int g(x)^{c_1}f(x)^{c_2} dx \nonumber
\end{IEEEeqnarray}
to all three of the \BD terms. This is of interest because the \RAD is generated by the \AD by applying the same transformation of its two terms.
\label{Def:GammaD}
\end{definition}

\begin{remark}[Recovering the \KLD]
The \AD, \RAD, \BD and \GD all recover the \KLD in the limit as $\alpha=\beta=\gamma\rightarrow 1$. This can be shown using the \textit{replica trick}:
%
\begin{IEEEeqnarray}{rCl}
   \lim_{x\rightarrow0}\frac{Z^x-1}{x}=\log(Z).\nonumber 
\end{IEEEeqnarray}
\end{remark}

\section{Comparing robust divergences as uncertainty quantifiers}
\label{Appendix:uncertainty_quantification}

In order to understand the impact the choice of divergence used for regularization and its hyperparameter have on the inference, this section studies variations in the argument $D$. This investigation is conducted on a simple Bayesian linear regression example with two highly correlated predictors given by
\begin{IEEEeqnarray}{rCl}
    \sigma^2&\sim&\mathcal{I}\mathcal{G}(a_0,b_0)\nonumber\\
    \*\theta|\sigma^2&\sim&\mathcal{N}_2\left(\*\mu_0,\sigma^2V_0\right)\\
    y_i|\*\theta,\sigma^2&\sim&\mathcal{N}\left(X_i\*\theta,\sigma^2\right).
    \nonumber\label{Equ:BLRresponse}
\end{IEEEeqnarray}
We choose this example because it provides a closed form exact Bayesian posteriors and closed form objectives for the variational  objectives of \VI and \GVI. 
%
%
Consequently, no sampling is required---neither for calculating the exact posterior nor for the optimization of the \GVI and \VI posteriors---so that numerical errors and uncertainties are kept to a minimum.

Studying the exact closed form Bayesian (normal) posterior for $\*\theta = (\theta_1, \theta_2)^T$, one observes that if the two predictors are correlated, then the posterior covariance of $\*\theta$ will inherit this correlation. 
As we wish to investigate the underestimation of marginal variances for standard \VI as well as the way in which \GVI can address this, our numerical studies leverage this finding. 
In particular, we simulate the highly correlated predictors 
 $$(x_1, x_2)^T \sim\mathcal{N}_2\left(\begin{pmatrix} 0\\ 0\\ \end{pmatrix},\begin{pmatrix} 1 & 0.9 \\ 0.9 & 1\\ \end{pmatrix}\right)$$ 
 and compare the performance of the different \GVIColor{\textbf{\GVI}} and \VIColor{\textbf{\VI}} posteriors on the resulting Bayesian linear regression. 
 All posteriors are based on the the mean field normal variational family
\begin{IEEEeqnarray}{rCl}
\mathcal{Q} & = & \{
    q(\theta_1|\sigma^2,\*\kappa_n)q(\theta_2|\sigma^2,\*\kappa_n)q(\sigma^2|\*\kappa_n)\} \text{ so that } \*\kappa_n = (a_n, b_n, \mu_{1,n}, \mu_{2,n}, v_{1,n}, v_{2,n})^T
    \nonumber \\
    && \quad  \text{ with } a_n, b_n, v_{1,n}, v_{2,n} > 0 \text{ and } \mu_{1,n}, \mu_{2,n} \in \mathbb{R} \nonumber \\
q(\sigma^2|\*\kappa_n)&=&\mathcal{I}\mathcal{G}(\sigma^2|a_n,b_n)
\nonumber\\
q(\theta_1|\sigma^2,\*\kappa_n)&=&\mathcal{N}\left(\theta_1|\mu_{1,n},\sigma^2v_{1,n}\right)
\nonumber \\
q(\theta_2|\sigma^2,\*\kappa_n)&=&\mathcal{N}\left(\theta_2|\mu_{2,n},\sigma^2v_{2,n}\right).
\nonumber
\end{IEEEeqnarray}
For all experiments, $n=25$ observations are simulated from eq. \eqref{Equ:BLRresponse} with $\*\theta=(2,3)$ and $\sigma^2=4$.
We use the negative log-likelihood of the correctly specified model as given in eq. \eqref{Equ:BLRresponse} as loss function. 
To investigate \GVI's behaviour across different uncertainty quantifiers, we vary its choice as $D\in\left\lbrace \AD,\BD,\RAD,\GD\right\rbrace$. 
The results are depicted in Figs. \ref{Fig:GVI_fail_AD} and \ref{Fig:GVI_uncertainty_comp_RADBDGD}-\ref{Fig:GD_robust_prior}.
We summarize the most interesting results from these plots in the following three subsections.



\begin{figure}[t!]
\begin{center}
\includegraphics[trim= {1.8cm 0.5cm 2.3cm 0.8cm}, clip,  
 width=1\columnwidth]{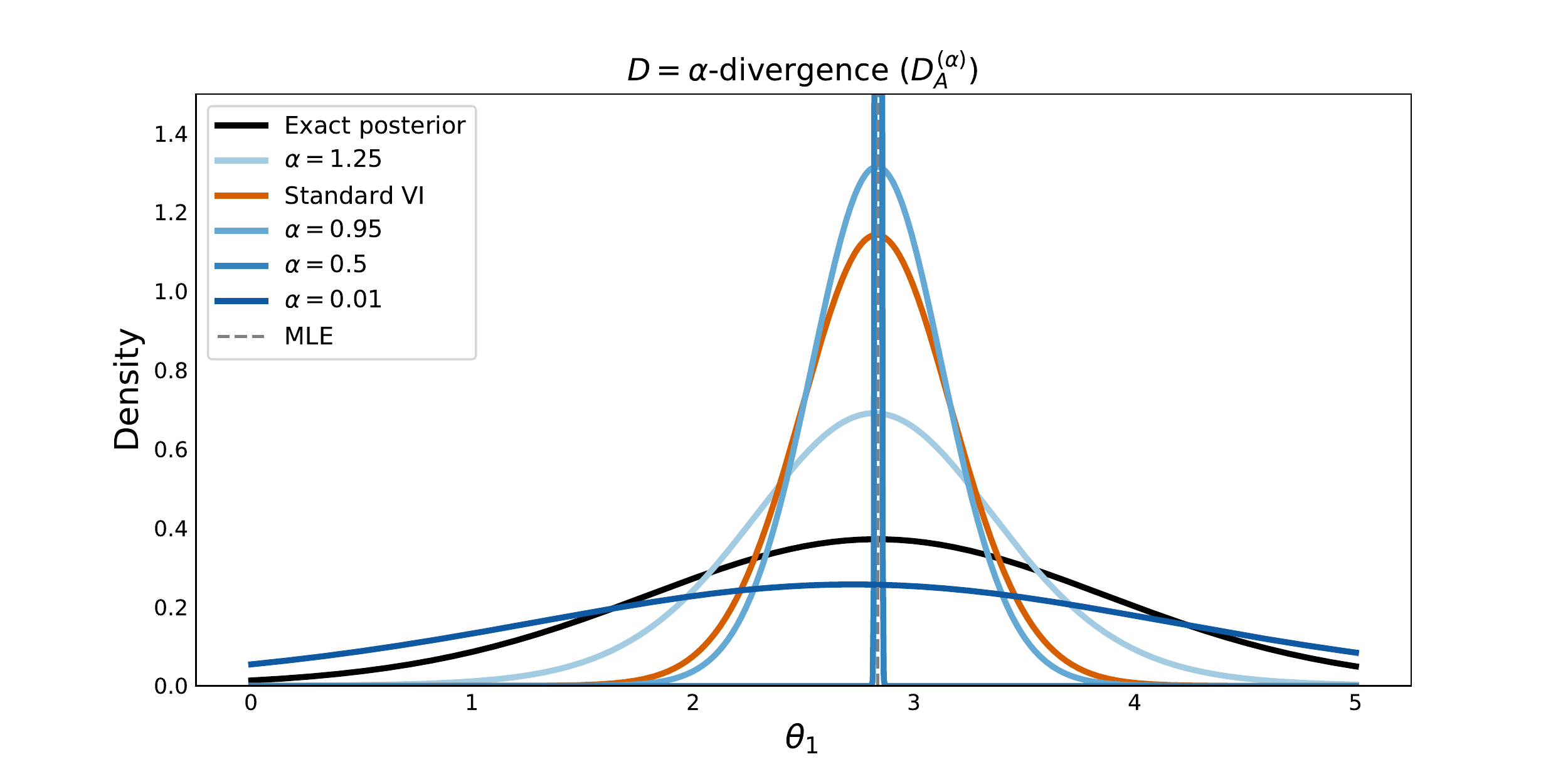}
\caption{
\bv
Marginal {\color{VIColor}{\textbf{\VI}}} and {\color{GVIColor}{\textbf{\GVI}}} posterior for the $\*\theta_1$ coefficient of a Bayesian linear model under the \AD prior regularizer for different values of $\alpha$. 
The boundedness of the \AD causes {\color{GVIColor}{\textbf{\GVI}}} posteriors to severely over-concentrate if $\alpha$ is not carefully specified. Prior Specification: $\sigma^2\sim\mathcal{IG}(20,50)$, $\theta_1|\sigma^2\sim\mathcal{N}(0,25\sigma^2)$ and $\theta_2|\sigma^2\sim\mathcal{N}(0,25\sigma^2)$.}
\label{Fig:GVI_fail_AD}
\end{center}
\end{figure}

\subsection{A cautionary tale: The boundedness of the $\alpha$-divergence (\AD)}

Of the alternative divergences to the \KLD contained within the \ABGD family defined in Appendix \ref{Appendix:definitions}, \AD is arguably the most well known. 
Our results in Figure \ref{Fig:GVI_fail_AD} show that in spite of its popularity in other contexts, the \AD is not a reliable uncertainty quantifier within \GVI, at least for $\alpha \in (0,1)$.
In particular, the plot shows that the solutions to $P(\ell,\AD, \mathcal{Q})$ can produce essentially degenerate posteriors if $\alpha \in (0,1)$. Note also that this happens in spite of the relatively small sample size of $n=25$. 
For example, when $\alpha=0.5$, $P(\ell,\AD, \mathcal{Q})$ is visually indistinguishable from a point mass at the maximum likelihood estimate.
This is a consequence of the boundedness of \AD for $\alpha\in(0,1)$: Specifically, it holds that $\AD \leq \left(\alpha(1-\alpha)\right)^{-1}$ for $\alpha\in(0,1)$. 
As $\alpha$ decreases from 1, this upper-bound initially also decreases until reaching its minimum for $\alpha = 0.5$. As a result, decreasing $\alpha$ from unity to $0.5$ significantly decreases the maximal penalty for posterior beliefs far from the prior. 
In turn, this forces the posterior to focus mostly on minimising the in-sample loss. 

\begin{figure}[t!]
\begin{center}
\includegraphics[trim= {1.75cm 0.0cm 2.2cm 0.5cm}, clip,  
 width=1\columnwidth]{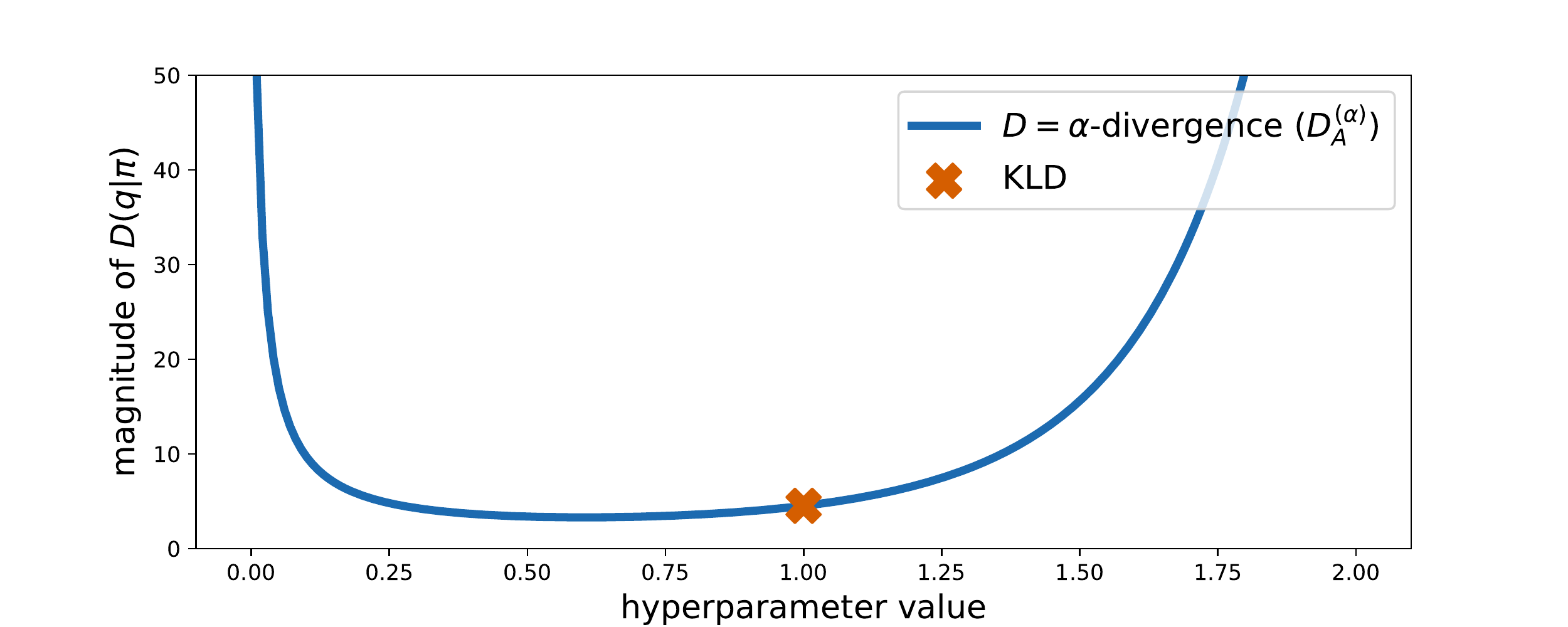}
\caption{A comparison of the size of \textbf{\AD} for various values of $\alpha$ between two bivariate Normal Inverse Gamma distributions with $a_n=512$, $b_n=543$, $\boldsymbol{\mu}_n=(2.5,2.5)$, $\mathbf{V}_n=\text{diag}(0.3,2)$ and $a_0=500$, $b_0=500$, $\mu_0=(0,0)$, $V_0=\text{diag}(25,2)$. 
}
\label{Fig:AlphaDivMagnitude}
\end{center}
\end{figure}

This phenomenon is depicted in Figure \ref{Fig:AlphaDivMagnitude}, which also shows that the divergence magnitude increases again as $\alpha$ approaches zero or if $\alpha>1$. %
Comparing the plot with that in Figure \ref{Fig:div_magnitudes}, it is clear why hyperparameter selection for the other members of the \ABGD family of divergences is a less complicated endeavour than for the $\alpha$-divergence.
This does not mean that the \AD cannot be used for producing \GVI posteriors: For example, in Figure \ref{Fig:GVI_fail_AD}, the \AD is able to achieve marginal variances that more closely correspond to the exact posterior for $\alpha=1.25$ and $\alpha=0.01$.  
Generally speaking, for values of $\alpha$ close to zero or above unity, it is possible to achieve more conservative uncertainty quantification. 
Yet, the \AD also functions primarily as a cautionary tale: Without understanding the properties of the uncertainty quantifier $D$ sufficiently well, \GVI may well yield unsatisfactory posteriors.

\begin{figure}[b!]
\begin{center}
\includegraphics[trim= {1.6cm 1.75cm 2.2cm 1.75cm}, clip,  
  width=1\columnwidth]{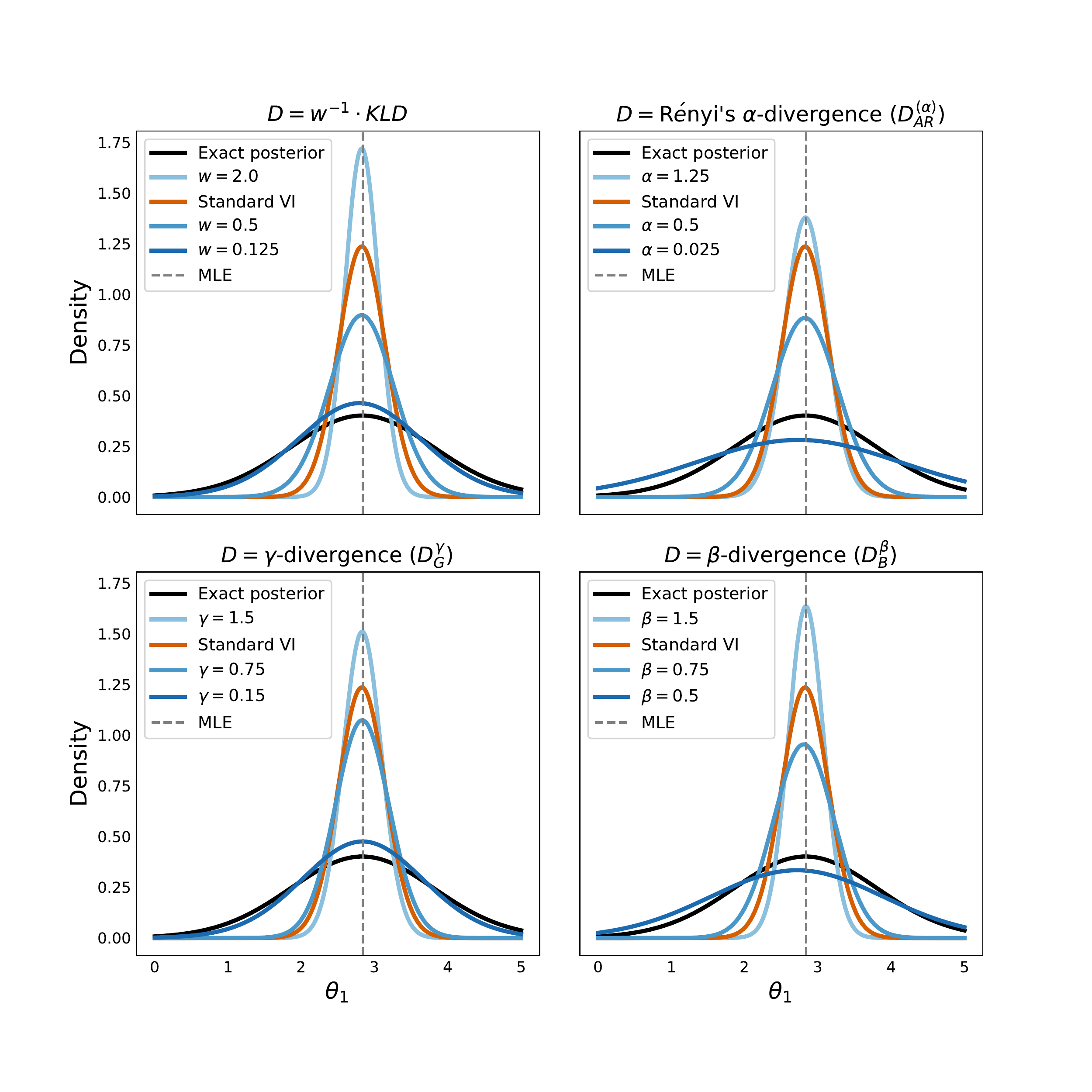}
\caption{
\bv
Marginal \VIColor{\textbf{\VI}} and \GVIColor{\textbf{\GVI}}  posterior for the first coefficient of a Bayesian linear model under the \RAD, \BD, \GD and \wKLD prior regularizers for different uncertainty quantifiers. 
Correlated covariates cause dependency in the \textbf{exact Bayesian posterior} of the coefficients and as a result \VIColor{\textbf{\VI}} underestimates marginal variances. 
\GVIColor{\textbf{\GVI}} has the flexibility to produce wider marginal variances. Prior Specification: $\sigma^2\sim\mathcal{IG}(20,50)$, $\theta_1\sim\mathcal{N}(0,5^2)$ and $\theta_2\sim\mathcal{N}(0,5^2)$.
}
\label{Fig:GVI_uncertainty_comp_RADBDGD}
\end{center}
\end{figure}

\subsection{Larger divergences produce larger marginal variances}

In this section, we summarize the impact that a selection of robust divergences have on the marginal variances of the solution to $P(\ell,D,\mathcal{Q})$, again using the Bayesian Linear regression model from before.
%
For a range of robust divergences, Figure \ref{Fig:GVI_uncertainty_comp_RADBDGD} illustrates the impact that changes in $D$ have on the marginal variances of the resulting posteriors. 
As one should expect from re-examining Figure \ref{Fig:div_magnitudes}, the plot shows that \BD, \RAD and \GD are able to produce larger posterior variances for $\beta,\alpha,\gamma<1$ and smaller posterior variances for $\beta,\alpha,\gamma>1$. 
This is a manifestation of the posterior being penalized more heavily ($\beta,\alpha,\gamma<1$) or less heavily ($\beta,\alpha,\gamma>1$) for deviating from the prior than under the traditional \VI. 
It follows that by choosing the divergence appropriately, \GVI can allow greater control over the uncertainty quantification characteristics of the resulting posterior than what is possible under standard \VI. 
Note that Figure \ref{Fig:GVI_uncertainty_comp_RADBDGD} also compares the robust divergences against the re-weighted \KLD. 
While the re-weighted \KLD can prove a successful alternative for producing desirable variational posteriors with larger variances robust divergences if the prior is well-specified, this is no longer the case if the information contained in the prior cannot be relied upon.
We study this and related findings surrounding robustness to the prior in the next section.

\begin{figure}[hp!]
\begin{center}
\includegraphics[trim= {1.7cm 1.85cm 2.3cm 1.20cm}, clip,  
  width=1\columnwidth]{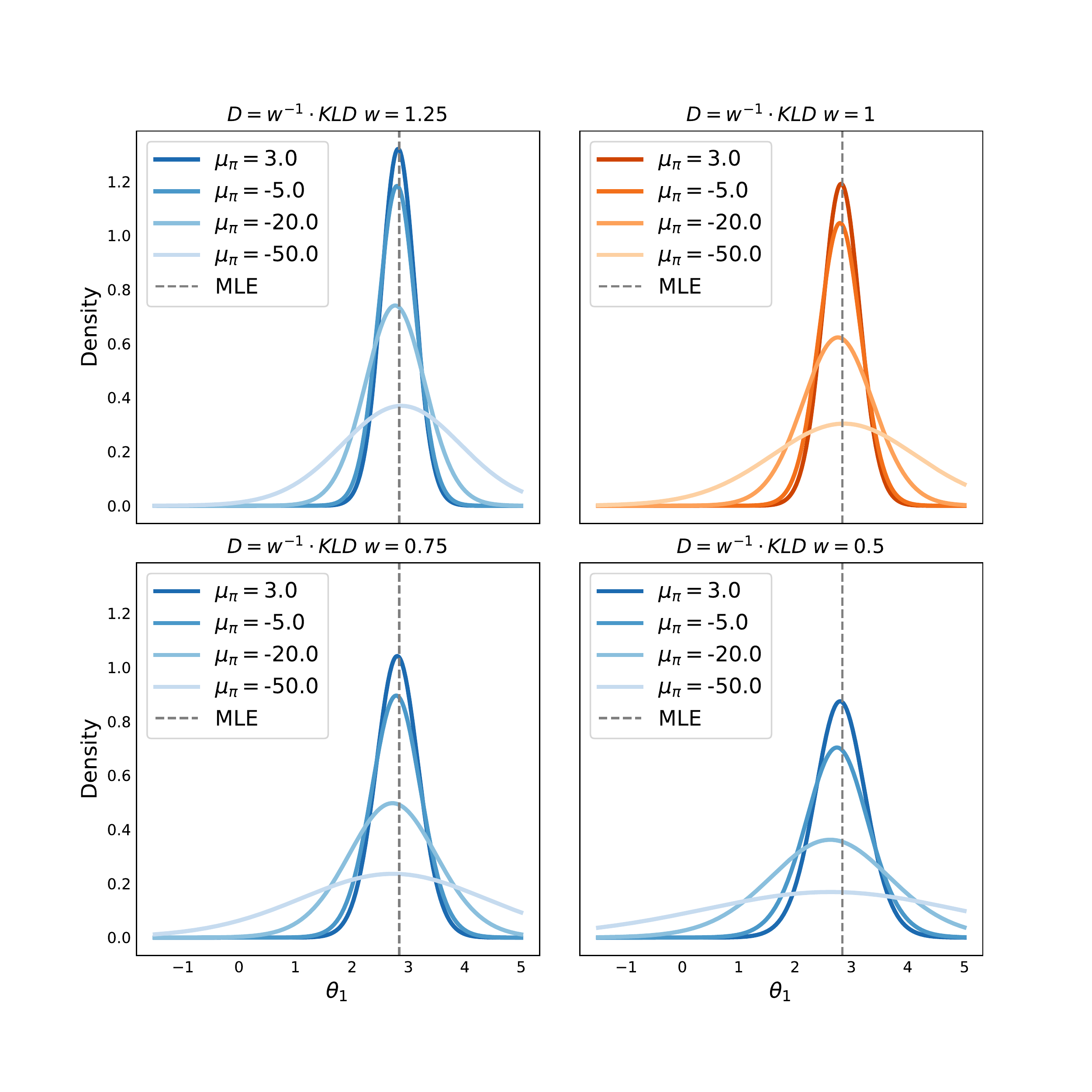}
\caption{
\bv
Marginal \VIColor{\textbf{\VI}} and \GVIColor{\textbf{\GVI}} posterior for the coefficient of a Bayesian linear model under different priors using  $D=\wKLD$ as uncertainty quantifier (\wKLD recovers \KLD for $w = 1$). 
The prior specification is given by $\theta_1|\sigma^2 \sim \mathcal{N}(\mu_{\pi}, \sigma^2)$ with $\sigma^2\sim\mathcal{IG}(3,5)$.}
\label{Fig:wKLD_robust_prior}
\end{center}
\end{figure}

\subsection{Robustness to the prior}

Next, we compare the impact of changing the uncertainty quantifier on the posterior's sensitivity to appropriate specification of the prior. Specifically, we consider and compare \BD, \RAD, \GD and \wKLD. 
When comparing \wKLD with \RAD and \GD, we fixed $\alpha=\gamma=w$. 
Setting the values of these various hyperparameters to be the same is intuitively appealing for comparison due to \GVI's interpretation as approximate Evidence Lower Bound (\ELBO), see Theorems \ref{Thm:LowerBoundMarginalLossLikelihoodRAD} and \ref{Thm:LowerBoundMarginalLossLikelihoodGD}.
For the \BD, different values of $\beta$ had to be selected to ensure its availability in a closed form.

\subsubsection{weighted \KLD (\wKLD)}

Figure \ref{Fig:wKLD_robust_prior} examines how changing the weight $w$ affects the posteriors $P(\ell,\wKLD,\mathcal{Q})$. Notice  that this is equivalent to changing the negative log likelihood to a power likelihood with power $w$.
Further, it should be clear that choosing $w<1$ leads to posteriors that encourage larger variances, making them amenable to conservative uncertainty quantification. 
Unfortunately and again unsurprisingly, this comes at the price of making posteriors \textit{more} sensitive to the prior: After all, one up-weights the term penalizing deviations from the prior.
Conversely, $w>1$ will result in posteriors that are less sensitive to the prior than standard \VI. At the same time, they will  also be more concentrated around the Maximum Likelihood Estimator.
This makes  $D=\wKLD$ less attractive than it could be: Setting $w$ to smaller values will yield larger posteriors variances (at the expense of not being robust to the prior), while setting $w$ to larger values will make the posterior more robust to misspecified priors (but at the expense of far more concentrated posteriors).
As we shall see, this undesirable trade-off is \textit{not} shared by the other (robust) divergences considered in this section. Unlike the \wKLD, they often provide a way to have your cake and eat it, too.

\begin{figure}[hp!]
\begin{center}
\includegraphics[trim= {1.7cm 1.85cm 2.3cm 1.20cm}, clip,  
  width=1\columnwidth]{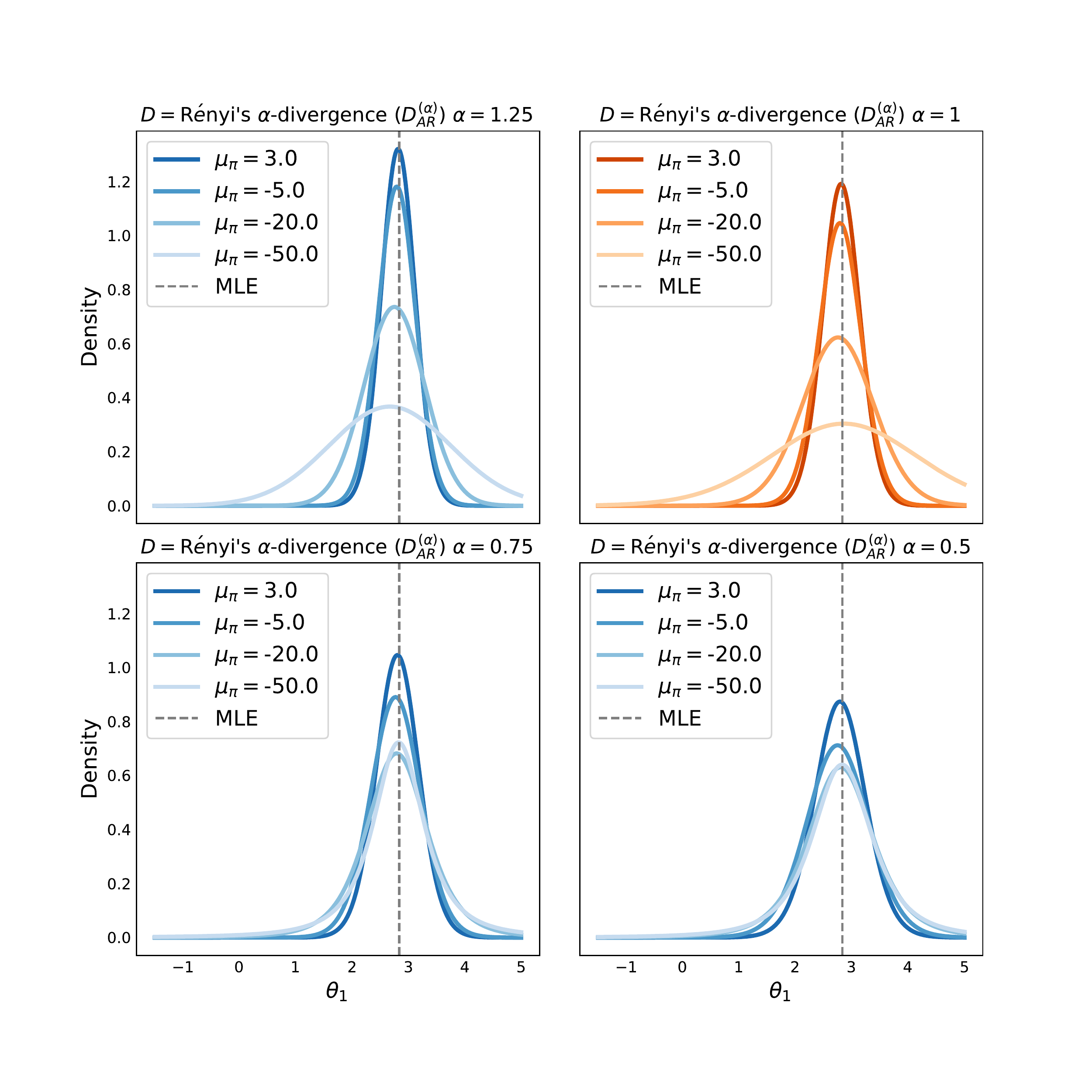}
\caption{
    \bv
    Marginal \VIColor{\textbf{\VI}} and \GVIColor{\textbf{\GVI}} posterior for the coefficient of a Bayesian linear model under different priors using  $D=\RAD$ as uncertainty quantifier (\RAD recovers \KLD as $\alpha \to 1$). 
    The prior specification is given by $\theta_1|\sigma^2 \sim \mathcal{N}(\mu_{\pi}, \sigma^2)$ with $\sigma^2\sim\mathcal{IG}(3,5)$.}
\label{Fig:RAD_robust_prior}
\end{center}
\end{figure}

\subsubsection{R\'enyi's $\alpha$-divergence (\RAD)} 
Figure \ref{Fig:RAD_robust_prior} demonstrates the sensitivity of $P(\ell,\RAD,\mathcal{Q})$ to  prior specification. For $0<\alpha<1$, the posterior exhibits the kind of behaviour that is difficult to attain with standard \VI: It both produces larger marginal variances \textit{and} is robust to badly specified priors.
This is no longer true if $\alpha>1$: For $\alpha > 1$, $\RAD\leq \KLD$, so that it is more sensitive to the prior than the \KLD. 
This flip in robustness as $\alpha$ crosses from $(0,1)$ into values larger than one may seem strange, but can be understood by investigating the form of the \RAD:
\begin{IEEEeqnarray}{rCl}
\RAD(q(\*\theta)||\pi(\*\theta))=\frac{1}{\alpha(\alpha-1)}\log\int q(\*\theta)^{\alpha}\pi(\*\theta)^{1-\alpha}d\*\theta 
&=&\frac{1}{\alpha(\alpha-1)}\log\int \frac{q(\*\theta)^{\alpha}}{\pi(\*\theta)^{\alpha-1}}d\*\theta.
\nonumber 
\end{IEEEeqnarray}
It is clear that the magnitude of the divergence is determined by a ratio of two densities. Glancing closer, for $\alpha>1$ this means that if $q(\*\theta)$ is large in an area where $\pi(\*\theta)$ is not, then a severe penalty is incurred. This limits how far the $q(\*\theta)$ can move from the prior and thus results in lack of prior robustness. 
Conversely, if $\alpha \in (0,1)$, then $\pi(\*\theta)^{\alpha -1} > \pi(\*\theta)$ for regions where $\pi(\*\theta)<1$, which allows the posterior to spread its mass in a less concentrated way than for $\alpha > 1$.
In fact, this very finding is also implicitly stated in  Theorem \ref{Thm:LowerBoundMarginalLossLikelihoodRAD}.

\begin{figure}[hp!]
\begin{center}
\includegraphics[trim= {1.7cm 1.85cm 2.3cm 1.20cm}, clip,  
  width=1\columnwidth]{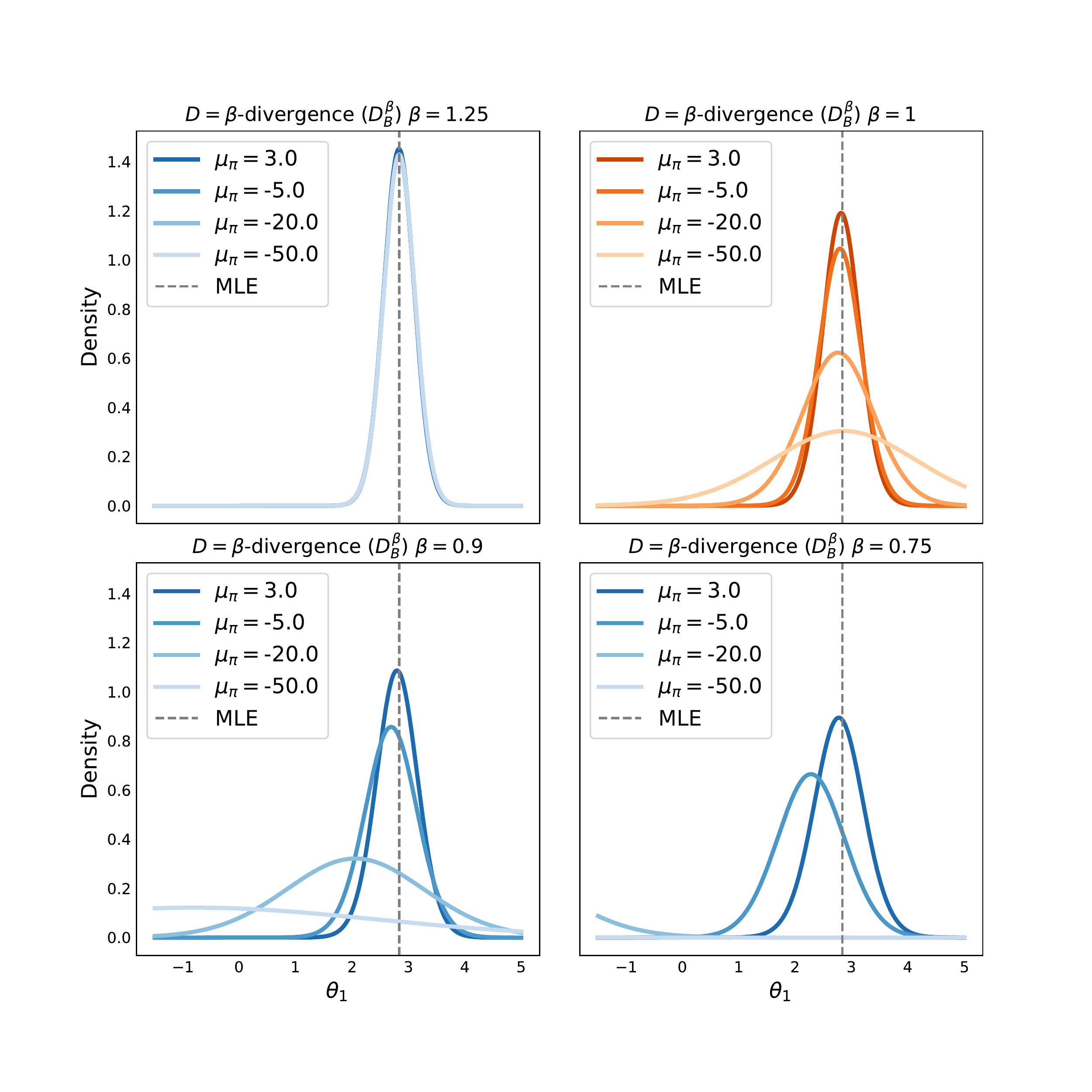}
\caption{
    \bv
    Marginal \VIColor{\textbf{\VI}} and \GVIColor{\textbf{\GVI}} posterior for the coefficient of a Bayesian linear model under different priors using  $D=\BD$ as uncertainty quantifier (\BD recovers \KLD as $\beta \to 1$). 
    The prior specification is given by $\theta_1|\sigma^2 \sim \mathcal{N}(\mu_{\pi}, \sigma^2)$ with $\sigma^2\sim\mathcal{IG}(3,5)$.}
\label{Fig:BD_robust_prior}
\end{center}
\end{figure}

\subsubsection{$\beta$-divergence (\BD)}
Figure \ref{Fig:BD_robust_prior} demonstrates the sensitivity of $P(\ell,\BD,\mathcal{Q})$ to prior specification. The plot shows that $\beta>1$ is able to achieve extreme robustness to the prior, while $\beta<1$ causes extreme sensitivity to the prior. This phenomenon is a result of the fact that the \BD decomposes into three integrals, one containing just the prior, one containing just $q(\*\theta)$ and one containing an interaction between them. 
\begin{IEEEeqnarray}{rCl}
\BD(q(\*\theta)||\pi(\*\theta))&=&\frac{1}{\beta}\int \pi(\*\theta)^{\beta}d\*\theta-\frac{1}{\beta-1}\int \pi(\*\theta)^{\beta-1}q(\*\theta)d\*\theta+\frac{1}{\beta(\beta-1)}\int q(\*\theta)^{\beta}d\*\theta
\end{IEEEeqnarray}
The integral depending only on the prior does not depend $q(\*\theta)$, so we can ignore it (since the prior is fixed across the different values of $\beta$). 
If $\beta$ increases substantially above 1, the second term which expresses an interaction between $\pi(\*\theta)$ and $q(\*\theta)$ will have a relatively smaller weight in the optimisation than the term only involving $q(\*\theta)$. 
As a result, the optimisation will focus on decreasing $\int q^{\beta}(\*\theta)d\*\theta$ rather than increasing $\int \pi^{\beta-1}(\*\theta)q(\*\theta)d\*\theta$. We note that this is closely linked to the so-called \textit{ignorance to the data} phenomenon as discussed in \cite{Jewson}.
The uncertainty quantification for large values of $\beta$ is therefore largely controlled by the third term, which only depends on $q(\*\theta)$. 
This third integral will become very large if the variance of $q(\*\theta)$ gets very small, which prevents it from quickly converging to a point mass at the maximuml likelihood estimate.
%
As a consequence, the \BD is able to provide virtually prior-invariant uncertainty quantification for $beta > 1$. 
For $\beta\in (0,1)$, the opposite effect is observed: Here, the third integral term depending only on $q(\*\theta)$  has smaller weight relative to the interaction between $\pi(\*\theta)$ and $q(\*\theta)$ given by $\int \pi^{\beta-1}(\*\theta)q(\*\theta)d\*\theta$. As a result, the corresponding posterior will be very close to the prior. (In fact, notice that that two of the four posteriors for $\beta = 0.75$ favour the prior so much that the density around the maximum likelihood estimate is virtually zero.)

\begin{figure}[hp!]
\begin{center}
\includegraphics[trim= {1.7cm 1.85cm 2.3cm 1.20cm}, clip,  
  width=1\columnwidth]{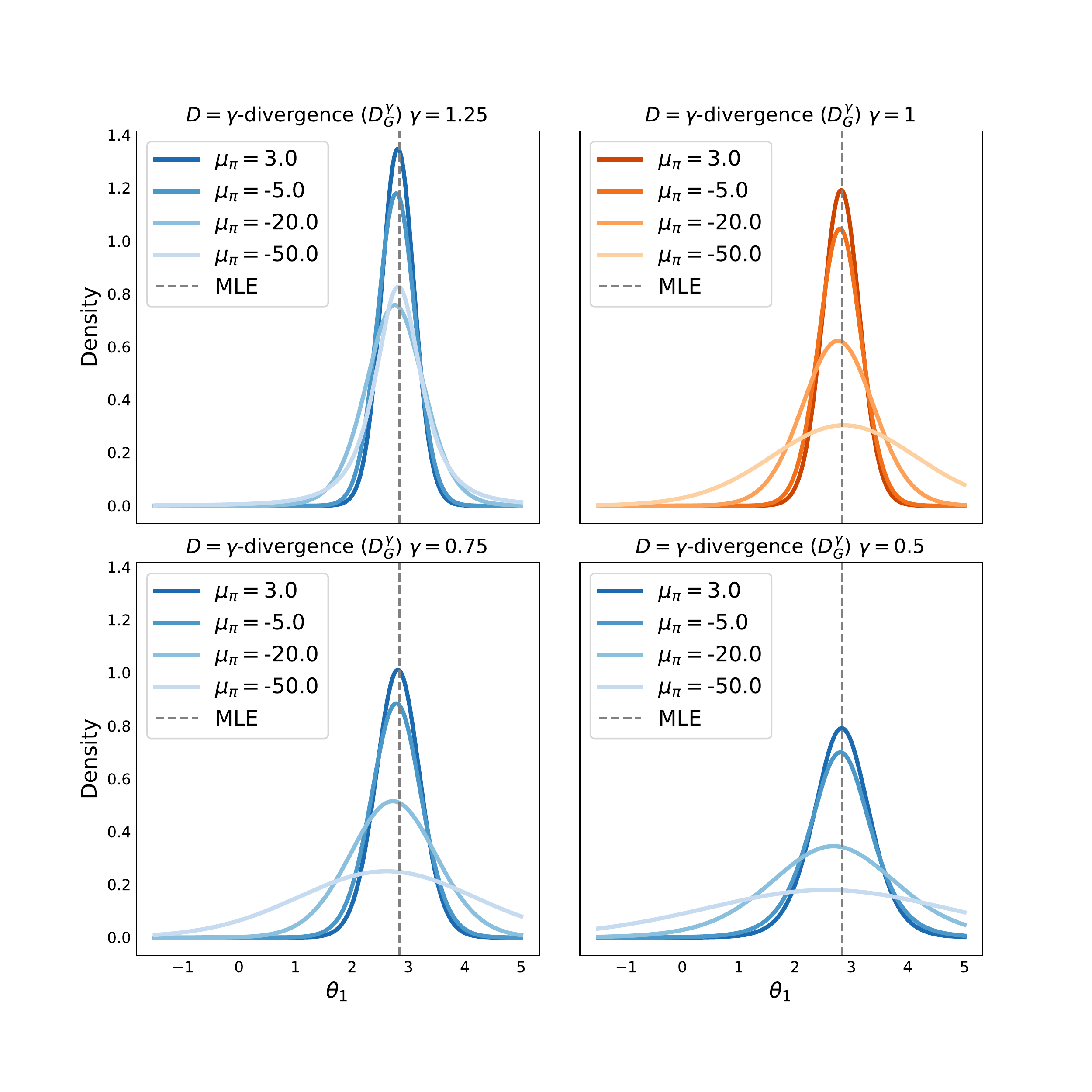}
\caption{
    \bv
    Marginal \VIColor{\textbf{\VI}} and \GVIColor{\textbf{\GVI}} posterior for the coefficient of a Bayesian linear model under different priors using  $D=\GD$ as uncertainty quantifier (\GD recovers \KLD as $\gamma \to 1$). 
    The prior specification is given by $\theta_1|\sigma^2 \sim \mathcal{N}(\mu_{\pi}, \sigma^2)$ with $\sigma^2\sim\mathcal{IG}(3,5)$.}
\label{Fig:GD_robust_prior}
\end{center}
\end{figure}

\subsubsection{$\gamma$-divergence (\GD)}

Lastly, Fig. \ref{Fig:GD_robust_prior} demonstrates the sensitivity of $P(\ell,\GD,\mathcal{Q})$ to  prior specification. 
For $\gamma<1$ it appears as though the \GD reacts similarly to the \wKLD for $w<1$.
The \GD with $\gamma>1$ produces greater robustness to the prior than the $\wKLD$ uncertainty quantifier with $w>1$, but this robustness is less extreme as it was for $D =\BD$. 
The reason for this is that although the \GD consists of the same three integral terms as the \BD, these terms are now transformed into the logarithmic scale. 
This means that the three integrals are combined multiplicatively (in the \GD) rather than additively (in the \BD), which makes the variation across $\gamma$ much smoother than across $\beta$: Unlike for the \BD, minimising the \GD no longer disregards any one term in order to minimise the others.

\section{Proof of Theorem \ref{Thm:GVI_modularity}}
\label{Appendix:modularity_proof}

Before we can prove Theorem \ref{Thm:GVI_modularity}, we first formally define the notion of robustness to model misspecification.
%
%
Our understanding of robustness to model misspecification is aligned with \citet{RobustStatsBook} and \citet{Tukeey1960}. In the words of the latter, robustness stands for
\begin{itemize}
    \item[] \textit{a tacit hope in ignoring deviations from
ideal models was that they would not matter; that statistical procedures which are optimal
under the strict model would still be approximately optimal under the approximate
model. Unfortunately, it turned out that this hope was often drastically wrong; even mild
deviations often have much larger effects than were anticipated by most statisticians.}
\end{itemize}
%
%
Formalizing this, we arrive at the following definition.
%
\begin{definition}[Robustness]
    Let $M_j=P(D_j, \ell_{j}, \Pi)$ with
    $\*\theta^{\ast}_{j} = \argmin_{\*\theta}\left\{ \mathbb{E}_{\*X}\left[ \ell_{j}(\*\theta, \*X) \right] \right\}$ for $j=1,2$. 
    Then, $M_1$ is more robust for $\*\theta$ than $M_2$ relative to the (implicit) assumptions $A$ on the data generating mechanism of $\*X$ if
    (i) $\*\theta^{\ast}_{1}$  is a better result than  $\*\theta^{\ast}_{2}$ if $A$ is untrue and (ii)
    $\*\theta^{\ast}_{1} = \*\theta^{\ast}_{2}$ if $A$ is true.
    \label{Def:robustness}
\end{definition}
\begin{remark}
It is hard to say what \textit{a better result} means, but we note that regardless of its precise meaning, this definition requires that robust inference directly affects $\*\theta^{\ast}$, i.e. that $\*\theta^{\ast}_{1} \neq \*\theta^{\ast}_{2}$ unless $A$ is true. 
While one could substantially strengthen this definition by formalizing what exactly \textit{a better result} means, this would necessarily be context-dependent, complicate matters substantially and obfuscate the point of robustness.
\end{remark}

\begin{proof}
First, we prove claim (i) about \textbf{robustness} to model misspecification:
    %
    By Definition \ref{Def:robustness}, robustness implies a change in $\*\theta^{\ast} = \argmin_{\*\theta}\left\{\mathbb{E}_{\*X}\left[\ell(\*\theta, \*X)\right]\right\}$ if distributional assumptions about $\*X$ are incorrect. 
    Notice that $\*\theta^{\ast}$ is not affected by $D$ or $\Pi$, but is affected by $\ell$. 
    %
Next, we turn to the claims (ii) and (iii) about \textbf{uncertainty quantification} and \textbf{prior robustness}:
    %
    %
    First, note that $\Pi$ and $\pi$ are not allowed to change by assumption and 
    so cannot affect uncertainty quantification.
    Next, while $\ell$ is allowed to change, the parameter of interest it not allowed to change. In other words, $\ell$ may only be changed in a ways that leave $\hat{\*\theta}_n$ and $\*\theta^{\ast}$ unaffected.
    Notice that changing $\ell$ to $\ell'$ will affect $\hat{\*\theta}_n = \argmin_{\*\theta}\left\{\frac{1}{n}\sum_{i=1}^n\ell(\*\theta, x_i)\right\}$ and $\*\theta^{\ast} = \mathbb{E}_{\*X}\left[ \ell(\*\theta, \*x) \right]$ unless $\ell' = C + w\cdot\ell$ for some constants $C$ and $w > 0$.
    Since $P(\ell, D, \Pi) = P(\ell + C, D, \Pi)$ for any $C$ by Axiom \ref{Axiom:translationInvariance}, we can disregard $C$ and turn to $w$.
    Indeed, the uncertainty quantification of $P(\ell, D, \Pi)$ will be different from that of $P(w\cdot\ell, D, \Pi)$ for any constant $w \neq 1$. 
    However, dividing by $w$ in eq. \eqref{eq:RoT_thm} shows that $P(w\cdot\ell, D, \Pi) = P(\ell, \frac{1}{w}D, \Pi)$. 
    Hence, any change in the loss that does not affect $\hat{\*\theta}_n$ and $\*\theta^{\ast}$ can be rewritten as a change in $D$.
    It follows that changing the uncertainty quantification or making the \RoT robust to the prior belief must amount to changing $D$.
\end{proof}

\section{Proof of Theorem \ref{Thm:LowerBoundMarginalLossLikelihoodRAD} and additional lower bounds}
\label{Appendix:lower_bound_proofs}



This section of the Appendix provides proofs for the lower bound interpretation of certain \GVI objectives.
First, we prove the result stated in the main paper. Next, we show equivalent results for the case of the $\beta$-divergence (\BD) and $\gamma$-divergence (\GD). While the following results and corresponding proofs are somewhat tedious to read, they are conceptually simple: In fact, all that is needed to derive the results is some basic algebra, Jensen's inequality and a further inequality involving the logarithm, see Lemma \ref{Lem:LogBound}.

\subsection{Proof for \RAD (Theorem \ref{Thm:LowerBoundMarginalLossLikelihoodRAD})}

Firstly, we provide explicit forms for the function quoted in Theorem \ref{Thm:LowerBoundMarginalLossLikelihoodRAD}
\begin{IEEEeqnarray}{rCl}
S_1(\alpha,q,\pi) &=& \begin{cases}
\RAD(q(\*\theta)||\pi(\*\theta) - \KLD(q(\*\theta)||\pi(\*\theta)) &\textrm { if } 0 < \alpha < 1\\
0 &\textrm{ if } \alpha >1.\\
\end{cases}\label{Equ:LowerBoundMarginalLossLikelihoodRAD_terms}
\end{IEEEeqnarray}
Next we provide a proof of the Theorem

\begin{proof}
For this proof we have to consider two cases for $\alpha$ as the positivity and negativity of $\alpha-1$ affect the results that can be used.

\noindent \textbf{Case 1)} $\alpha>1$:
Jensen's inequality and the concavity of the natural logarithm give us that 
%

\begin{IEEEeqnarray}{rCl}
&&\mathbb{E}_{q(\*\theta)}\left[\sum_{i=1}^n\ell(\*\theta,\*x_i)\right]+\frac{1}{\alpha(\alpha-1)}\log \mathbb{E}_{q(\*\theta)}\left[\left(\frac{q(\*\theta)}{\pi(\*\theta)}\right)^{\alpha-1}\right]\nonumber\\
&\geq&\mathbb{E}_{q(\*\theta)}\left[\sum_{i=1}^n\ell(\*\theta,\*x_i)\right]+\frac{1}{\alpha}\mathbb{E}_{q(\*\theta)}\left[\log\left(\frac{q(\*\theta)}{\pi(\*\theta)}\right)\right]\nonumber\\
&=&\frac{1}{\alpha}\KLD(q(\*\theta)||\pi^{\alpha\ell}(\*\theta|\mathbf{\*x}))-\frac{1}{\alpha}\log \int\pi(\*\theta)\exp\left(-\alpha\sum_{i=1}^n\ell(\*\theta,\*x_i)\right)d\*\theta.\nonumber
\end{IEEEeqnarray}



\noindent \textbf{Case 2)} $0<\alpha<1$: 
Here the negativity of $\frac{1}{\alpha(\alpha-1)}$ means we cannot apply Jensen's inequality in the above way. Instead, we can write 
\begin{IEEEeqnarray}{rCl}
&&\mathbb{E}_{q(\*\theta)}\left[\sum_{i=1}^n\ell(\*\theta,\*x_i)\right]+\RAD(q(\*\theta)||\pi(\*\theta))\nonumber \\
&=&\mathbb{E}_{q(\*\theta)}\left[\log(\pi(\*\theta))\right]-\mathbb{E}_{q(\*\theta)}\left[\log\left(\frac{\pi(\*\theta)\exp(-\sum_{i=1}^n\ell(\*\theta,\*x_i))}{\int \pi(\*\theta)\exp(-\sum_{i=1}^n\ell(\*\theta,\*x_i))d\*\theta}\right)\right]\nonumber\\
&&-\log\int \pi(\*\theta)\exp(-\sum_{i=1}^n\ell(\*\theta,\*x_i))d\*\theta+\RAD(q(\*\theta)||\pi(\*\theta)) \nonumber\\
&=&\KLD(q(\*\theta)||\pi^{\ell}(\*\theta|x))
-\log\int \pi(\*\theta)\exp(-\sum_{i=1}^n\ell(\*\theta,\*x_i))d\*\theta\nonumber\\
&&+\RAD(q(\*\theta)||\pi(\*\theta))-\KLD(q(\*\theta)||\pi(\*\theta)).\nonumber
\end{IEEEeqnarray}
%
%
Combined these two cases provides the term in Eq. \eqref{Equ:LowerBoundMarginalLossLikelihoodRAD} and \eqref{Equ:LowerBoundMarginalLossLikelihoodRAD_terms}
\end{proof}
Next we state, prove and interpret equivalent results for the \BD and \GD prior regularisers. But before we do so we need the following lemma
\begin{lemma}[A Taylor series bound for the natural logarithm]
The natural logarithm of a positive real number $Z$ can be bounded as follows
\begin{IEEEeqnarray}{rCl}
&&\begin{cases}
\log(Z)\leq \frac{Z^{x}-1}{x} &\textrm{ if } x>0\\
\log(Z)\geq \frac{Z^{x}-1}{x} &\textrm{ if } x<0.\\
\end{cases}\nonumber
\end{IEEEeqnarray}
\label{Lem:LogBound}
\end{lemma}
\begin{proof}
Using the series expansion of $\exp(x)$ and the Lagrange form of the remainder we see that
\begin{IEEEeqnarray}{rCl}
    \frac{Z^{x}-1}{x}&=&\frac{\exp\left(x\log Z\right)-1}{x} 
    =\frac{\left(x\log Z\right)+\frac{1}{2!}\left(x\log Z\right)^2+\frac{1}{3!}\left(x\log Z\right)^3+\ldots}{x}\nonumber\\
&=&\frac{\left(x\log Z\right)+\frac{1}{2}\exp(c)\left(x\log Z\right)^2}{x} 
=\log Z+\frac{\frac{1}{2!}\exp(c)\left(x\log Z\right)^2}{x}\nonumber
\end{IEEEeqnarray}
where $c\in[0,x\log(Z)]$. Now the numerator of the remainder term $\frac{\frac{1}{2!}\exp(c)\left(x\log Z\right)^2}{x}$ is always positive and therefore the sign of $x$ determines whether this remainder term forms an upper or lower bound for $\log(Z)$.
\end{proof}

\subsection{The \BD prior regulariser}

\begin{theorem}[\GVI as approximate Evidence Lower bound with $D=\BD$]
The objective of a \GVI posterior based on $P(\ell, \BD, \mathcal{Q})$ has an interpretation as lower bound on the $c(\beta)$-scaled (generalized) evidence lower bound of $P(w(\beta)\cdot \ell, \KLD, \mathcal{P}(\*\Theta))$:
\begin{IEEEeqnarray}{rCl}
    \mathbb{E}_{q(\*\theta)}\left[\sum_{i=1}^n\ell(\*\theta,\*x_i)\right]+\BD(q(\*\theta)||\pi(\*\theta))&\geq
    &-c(\beta)\ELBO^{w(\beta)\ell}(q)+S_1(\beta, q, \pi) \label{Equ:LowerBoundMarginalLossLikelihoodBD}
\end{IEEEeqnarray}
where $\ELBO^{w(\beta)\ell}$ denotes the Evidence Lower Bound associated with standard \VI relative to the generalized Bayesian posterior given by
\begin{IEEEeqnarray}{rCl}
    q^{w(\beta)\ell}_B(\*\theta)& \propto &
    \pi(\*\theta)\exp\left(-w(\beta)\sum_{i=1}^n\ell(\*\theta,\*x_i)\right),
    \nonumber
\end{IEEEeqnarray}
where $c(\beta) = \min\{1, \beta^{-1}\}$, $w(\beta) = \max\{1, \beta\}$ and where $S_1(\beta, q, \pi)$ is a closed form slack term with
\begin{IEEEeqnarray}{rCl}
S_1(\beta,q,\pi) &=& \begin{cases}
\frac{1}{\beta(\beta-1)}\mathbb{E}_{q(\*\theta)}\left[q(\*\theta)^{\beta-1}\right]-\mathbb{E}_{q(\*\theta)}\left[\log q(\*\theta)\right]-\frac{1}{\beta-1} &\textrm { if } 0 < \beta < 1\\
\frac{1}{\beta}\mathbb{E}_{q(\*\theta)}\left[\log \pi(\*\theta)\right]-\frac{1}{\beta-1}\mathbb{E}_{q(\*\theta)}\left[\pi(\*\theta)^{\beta-1}\right]-\frac{1}{\beta(\beta-1)} &\textrm{ if } \beta >1.\\
\end{cases}\label{Equ:LowerBoundMarginalLossLikelihoodBD_terms}
\end{IEEEeqnarray}
\label{Thm:LowerBoundMarginalLossLikelihoodBD}
\end{theorem}

\begin{proof}
Firstly we note that the objective function associated with the \RoT $P(\BD, \ell_n, Q)$ can be simplified by removing the terms in the $\BD$ that don't depend on $q(\*\theta)$
\begin{IEEEeqnarray}{rCl}
&&\arg\min_{q\in\mathcal{Q}} \left\lbrace\mathbb{E}_{q(\*\theta)}\left[\sum_{i=1}^n\ell(\*\theta,\*x_i)\right]+\BD(q(\*\theta)||\pi(\*\theta))\right\rbrace\nonumber \\
=&&\arg\min_{q\in\mathcal{Q}} \left\lbrace\mathbb{E}_{q(\*\theta)}\left[\sum_{i=1}^n\ell(\*\theta,\*x_i)\right]+\frac{1}{\beta(\beta-1)} \mathbb{E}_{q(\*\theta)}\left[q(\*\theta)^{\beta-1}\right]-\frac{1}{(\beta-1)}\mathbb{E}_{q(\*\theta)}\left[\pi(\*\theta)^{\beta-1}\right]\right\rbrace.\nonumber
\end{IEEEeqnarray}
We have to consider two cases for $\beta$ as the positivity and negativity of $\beta-1$ affect which part of Lemma \ref{Lem:LogBound} we use.

\noindent\textbf{Case 1) $0<\beta<1$:}
Lemma \ref{Lem:LogBound} gives us that for $\beta-1<0$, $\frac{Z^{\beta-1}}{\beta-1}\leq \log(Z)+\frac{1}{\beta-1}$ therefore
\begin{IEEEeqnarray}{rCl}
   &=&\mathbb{E}_{q(\*\theta)}\left[\sum_{i=1}^n\ell(\*\theta,\*x_i)\right]+\frac{1}{\beta(\beta-1)} \mathbb{E}_{q(\*\theta)}\left[q(\*\theta)^{\beta-1}\right]-\frac{1}{(\beta-1)}\mathbb{E}_{q(\*\theta)}\left[\pi(\*\theta)^{\beta-1}\right]\nonumber\\
    &\geq &\mathbb{E}_{q(\*\theta)}\left[\sum_{i=1}^n\ell(\*\theta,\*x_i)\right]+\frac{1}{\beta(\beta-1)} \mathbb{E}_{q(\*\theta)}\left[q(\*\theta)^{\beta-1}\right]-\mathbb{E}_{q(\*\theta)}\left[ \log(\pi(\*\theta))\right]-\frac{1}{\beta-1}\nonumber\\
    &=&\KLD(q(\*\theta)||\pi^{\ell}(\*\theta|x))-\log\int \exp(-\sum_{i=1}^n\ell(\*\theta,\*x_i))\pi(\*\theta) d\*\theta\nonumber\\
    &&+\frac{1}{\beta(\beta-1)} \mathbb{E}_{q(\*\theta)}\left[q(\*\theta)^{\beta-1}\right]-\mathbb{E}_{q(\*\theta)}\left[ \log(q(\*\theta))\right]-\frac{1}{\beta-1}.\nonumber
\end{IEEEeqnarray}
\textbf{Case 2) $\beta>1$:}
Lemma \ref{Lem:LogBound} gives us that for $\beta-1>0$, $\frac{Z^{\beta-1}}{\beta-1}\geq \log(Z)+\frac{1}{\beta-1}$ therefore
\begin{IEEEeqnarray}{rCl}
&=&\mathbb{E}_{q(\*\theta)}\left[\sum_{i=1}^n\ell(\*\theta,\*x_i)\right]+\frac{1}{\beta(\beta-1)} \mathbb{E}_{q(\*\theta)}\left[q(\*\theta)^{\beta-1}\right]-\frac{1}{(\beta-1)}\mathbb{E}_{q(\*\theta)}\left[\pi(\*\theta)^{\beta-1}\right]\nonumber\\
&\geq&\mathbb{E}_{q(\*\theta)}\left[\sum_{i=1}^n\ell(\*\theta,\*x_i)\right]+\frac{1}{\beta}\left(\mathbb{E}_{q(\*\theta)}\left[\log\left(q(\*\theta)\frac{\pi(\*\theta)}{\pi(\*\theta)}\right)\right]+\frac{1}{\beta-1}\right)-\frac{1}{(\beta-1)}\mathbb{E}_{q(\*\theta)}\left[\pi(\*\theta)^{\beta-1}\right]\nonumber\\
&=&\frac{1}{\beta}\KLD(q(\*\theta)||\pi^{\beta\ell}(\*\theta|\mathbf{x})) -\frac{1}{\beta}\log\int\pi(\*\theta)\exp(-\beta\sum_{i=1}^n\ell(\*\theta,\*x_i))d\*\theta\nonumber\\
&+&\frac{1}{\beta}\mathbb{E}_{q(\*\theta)}\left[\log(\pi(\*\theta))\right]-\frac{1}{(\beta-1)}\mathbb{E}_{q(\*\theta)}\left[\pi(\*\theta)^{\beta-1}\right]+\frac{1}{\beta(\beta-1)}.\nonumber
\end{IEEEeqnarray}
Combined these two cases provides the term in Eq. \eqref{Equ:LowerBoundMarginalLossLikelihoodBD} and \eqref{Equ:LowerBoundMarginalLossLikelihoodBD_terms}
\end{proof}


\subsection{The \GD prior regulariser}

\begin{theorem}[\GVI as approximate Evidence Lower bound with $D=\GD$]
The objective of a \GVI posterior based on $P(\ell, \GD, \mathcal{Q})$ has an interpretation as lower bound on the $c(\gamma)$-scaled (generalized) evidence lower bound of $P(w(\gamma)\cdot \ell, \KLD, \mathcal{P}(\*\Theta))$:
\begin{IEEEeqnarray}{rCl}
    \mathbb{E}_{q(\*\theta)}\left[\sum_{i=1}^n\ell(\*\theta,\*x_i)\right]+\GD(q(\*\theta)||\pi(\*\theta))&=
    &-c(\gamma)\ELBO^{w(\gamma)\ell}(q)+S(\gamma, q, \pi) \label{Equ:LowerBoundMarginalLossLikelihoodGD}
\end{IEEEeqnarray}
where $\ELBO^{w(\gamma)\ell}$ denotes the Evidence Lower Bound associated with standard \VI relative to the generalized Bayesian posterior given by
\begin{IEEEeqnarray}{rCl}
    q^{w(\gamma)\ell}_B(\*\theta)& \propto &
    \pi(\*\theta)\exp\left(-w(\gamma)\sum_{i=1}^n\ell(\*\theta,\*x_i)\right),
    \nonumber
\end{IEEEeqnarray}
where $c(\gamma) = \min\{1, \gamma^{-1}\}$, $w(\gamma) = \max\{1, \gamma\}$ and where $S_1(\gamma, q, \pi)$ is a closed form slack term with
\begin{IEEEeqnarray}{rCl}
S_1(\gamma,q,\pi) &=& \begin{cases}
\frac{1}{\gamma(\gamma-1)}\log\mathbb{E}_{q(\*\theta)}\left[q(\*\theta)^{\gamma-1}\right]-\mathbb{E}_{q(\*\theta)}\left[\log q(\*\theta)\right] &\textrm { if } 0 < \gamma < 1\\
\frac{1}{\gamma}\mathbb{E}_{q(\*\theta)}\left[\log \pi(\*\theta)\right]-\frac{1}{\gamma-1}\log\mathbb{E}_{q(\*\theta)}\left[\pi(\*\theta)^{\gamma-1}\right] &\textrm{ if } \gamma >1.\\
\end{cases}\label{Equ:LowerBoundMarginalLossLikelihoodGD_terms}
\end{IEEEeqnarray}
\label{Thm:LowerBoundMarginalLossLikelihoodGD}
\end{theorem}

\begin{proof}
Firstly we note that the objective function associated with $P(\GD, \ell_n, Q)$ can be simplified by removing the terms in the $\GD$ that don't depend on $q(\*\theta)$
\begin{IEEEeqnarray}{rCl}
&&\arg\min_{q\in\mathcal{Q}} \left\lbrace\mathbb{E}_{q(\*\theta)}\left[\sum_{i=1}^n\ell(\*\theta,\*x_i)\right]+\GD(q(\*\theta)||\pi(\*\theta))\right\rbrace= \nonumber\\
&&\arg\min_{q\in\mathcal{Q}} \left\lbrace\mathbb{E}_{q(\*\theta)}\left[\sum_{i=1}^n\ell(\*\theta,\*x_i)\right]+\frac{1}{\gamma(\gamma-1)}\log \mathbb{E}_{q(\*\theta)}\left[q(\*\theta)^{\gamma-1}\right]-\frac{1}{(\gamma-1)}\log\mathbb{E}_{q(\*\theta)}\left[\pi(\*\theta)^{\gamma-1}\right]\right\rbrace.\nonumber
\end{IEEEeqnarray}
We have to consider two cases for $\gamma$ as the positivity and negativity of $\gamma-1$ affect the results we can use.

\noindent \textbf{Case 1)} $0<\gamma<1$:
Jensen's inequality and the concavity of the natural logarithm applied to $\mathbb{E}_{q(\*\theta)}\left[\pi(\*\theta)^{\gamma-1}\right]$ provides
\begin{IEEEeqnarray}{rCl}
&=&\mathbb{E}_{q(\*\theta)}\left[\sum_{i=1}^n\ell(\*\theta,\*x_i)\right]+\frac{1}{\gamma(\gamma-1)}\log \mathbb{E}_{q(\*\theta)}\left[q(\*\theta)^{\gamma-1}\right]-\frac{1}{(\gamma-1)}\log\mathbb{E}_{q(\*\theta)}\left[\pi(\*\theta)^{\gamma-1}\right] \nonumber\\
&\geq&\mathbb{E}_{q(\*\theta)}\left[\sum_{i=1}^n\ell(\*\theta,\*x_i)\right]+ \frac{1}{\gamma(\gamma-1)}\log \mathbb{E}_{q(\*\theta)}\left[q(\*\theta)^{\gamma-1}\right]-\mathbb{E}_{q(\*\theta)}\left[\log\pi(\*\theta)\right]\nonumber  \\
&=&  \KLD(q(\*\theta)||\pi^{\ell}(\*\theta|x))
-\log \int\pi(\*\theta)\exp\left(-\sum_{i=1}^n\ell(\*\theta,\*x_i)\right)d\*\theta\nonumber\\
&&+\frac{1}{\gamma(\gamma-1)}\log \mathbb{E}_{q(\*\theta)}\left[q(\*\theta)^{\gamma-1}\right]+-\mathbb{E}_{q(\*\theta)}\left[\log q(\*\theta)\right].\nonumber
\end{IEEEeqnarray}
%
\textbf{Case 2)} $\gamma>1$:
Jensen's inequality and the concavity of the  natural logarithm applied to $\mathbb{E}_{q(\*\theta)}\left[q(\*\theta)^{\gamma-1}\frac{\pi(\*\theta)^{\gamma-1}}{\pi(\*\theta)^{\gamma-1}}\right]$ provides

\begin{IEEEeqnarray}{rCl}
&=& \mathbb{E}_{q(\*\theta)}\left[\sum_{i=1}^n\ell(\*\theta,\*x_i)\right]+\frac{1}{\gamma(\gamma-1)}\log \mathbb{E}_{q(\*\theta)}\left[q(\*\theta)^{\gamma-1}\right]-\frac{1}{(\gamma-1)}\log\mathbb{E}_{q(\*\theta)}\left[\pi(\*\theta)^{\gamma-1}\right]\nonumber \\
&\geq& \mathbb{E}_{q(\*\theta)}\left[\sum_{i=1}^n\ell(\*\theta,\*x_i)\right]+\frac{1}{\gamma} \mathbb{E}_{q(\*\theta)}\left[\log\frac{q(\*\theta)\pi(\*\theta)}{\pi(\*\theta)}\right]-\frac{1}{(\gamma-1)}\log\mathbb{E}_{q(\*\theta)}\left[\pi(\*\theta)^{\gamma-1}\right]\nonumber \\
&=& \frac{1}{\gamma} \KLD(q(\*\theta)||\pi^{\gamma\ell}(\*\theta|x))-\frac{1}{\gamma}\int \pi(\*\theta)\exp\left(-\gamma\sum_{i=1}^n\ell(\*\theta,\*x_i)\right)d\*\theta\nonumber\\
&&+\frac{1}{\gamma} \mathbb{E}_{q(\*\theta)}\left[\log\pi(\*\theta)\right]-\frac{1}{(\gamma-1)}\log\mathbb{E}_{q(\*\theta)}\left[\pi(\*\theta)^{\gamma-1}\right].\nonumber
\end{IEEEeqnarray}
Combined these two cases provides the term in Eq. \eqref{Equ:LowerBoundMarginalLossLikelihoodGD} and \eqref{Equ:LowerBoundMarginalLossLikelihoodGD_terms}
\end{proof}

\subsubsection{Interpretation} Theorems \ref{Thm:LowerBoundMarginalLossLikelihoodBD} and \ref{Thm:LowerBoundMarginalLossLikelihoodGD} provide a lower bound on an objective function that is to be minimised so that interpreting this lower bound provides some insight into the behaviour of the \GVI posterior. 
First, we investigate the case where the hyperparameters $\beta$ and $\gamma$ are in $(0,1)$. As expected, the form of the \GVI objective leads us to conclude that the variance will be larger than that for standard \VI within this range of values. Next, we investigate the case where the hyperparameters $\beta$ and $\gamma$ are $>1$. Again unsurprisingly, this leads to a shrinkage of the posterior variance relative to standard \VI.

\subparagraph{Case 1: $0<\beta=\gamma<1$.}
For $0<\beta=\gamma<1$ the terms $c(\beta)=c(\gamma)$ and $w(\beta)=w(\gamma)$ produce an objective equivalent to standard \VI. This suggests that \GVI continues to minimise the \KLD between the variational and standard Bayesian posterior. 
Unlike standard \VI however, \GVI with $D=\BD$ or $D=\GD$ additionally minimises the slack terms $S_1(\beta,q,\pi)$ or $S_1(\gamma,q,\pi)$. It is easy to show that the these adjustment terms encourage the solution to $P(\BD, \ell, Q)$ with $0<\beta<1$ and $P(\GD, \ell, Q)$ with $0<\gamma<1$ to have greater variance than the standard \VI posterior given by $P(\KLD, \ell_n, Q)$. 
For the \BD, we can see this by rewriting
\begin{IEEEeqnarray}{rCl}
  S^{(0,1)}(\beta,q,\pi)&=& -\frac{1}{\beta}h_{T}^{(\beta)}(q(\*\theta))+h_{\KLD}(q(\*\theta))+\frac{1-\beta}{\beta}.  \nonumber
\end{IEEEeqnarray}
Here, $h_{\KLD}(q(\*\theta))$ is the Shannon entropy of $q(\*\theta)$ and $h_{T}^{(\beta)}(q(\*\theta))$ is the Tsallis entropy of $q(\*\theta)$ with parameter $\beta$. 
Again applying Lemma \ref{Lem:LogBound}, we find that for $0<\beta<1$, $h_{T}^{(\beta)}(q(\*\theta))>h_{\KLD}(q(\*\theta))$. 
It immediately follows that minimising $-\frac{1}{\beta}h_{T}^{(\beta)}(q(\*\theta))+h_{\KLD}(q(\*\theta))$ for $0<\beta<1$ will make $h_{T}^{(\beta)}(q(\*\theta))$ large---an effect that is achieved by increasing the variance of $q(\*\theta)$. 

Applying the same type of logic to the \GD, one can rewrite
\begin{IEEEeqnarray}{rCl}
  S^{(0,1)}(\gamma,q,\pi)&=& -\frac{1}{\gamma}h_{R}^{(\gamma)}(q(\*\theta))+h_{\KLD}(q(\*\theta)). \nonumber 
\end{IEEEeqnarray}
As before, $h_{\KLD}(q(\*\theta))$ is the Shannon entropy of $q(\*\theta)$, but unlike before $h_{R}^{(\gamma)}(q(\*\theta))$ is now the \renyi entropy of $q(\*\theta)$ with parameter $\gamma$. 
With this, one can extend Theorem 3 in \citet{RenyiKLD} to show that $h_R^{(\gamma)}(q(\*\theta))$ is decreasing in $\gamma$. 
Since it is also well-known that $\lim_{\gamma\rightarrow 1}h_R^{(\gamma)}(q(\*\theta))=h_{\KLD}(q(\*\theta))$, it follows that minimising $-\frac{1}{\gamma}h_{R}^{(\gamma)}(q(\*\theta))+h_{\KLD}(q(\*\theta))$ for $0<\gamma<1$ will make $h_{R}^{(\gamma)}(q(\*\theta))$ large---an effect that is again achieved by increasing the variance of $q(\*\theta)$.

\subparagraph{Case 2: $\beta=\gamma=k>1$.}
For $k=\gamma=\beta>1$,  $c(k)=\frac{1}{k}$ and $w(k)= k$. Minimising $\KLD(q||q^{\ast}_{k})$ for $k>1$ will encourage $P(\BD, \ell, Q)$ or $P(\GD, \ell, Q)$ to be more concentrated around the empirical risk minimizer $\hat{\*\theta}_n$ of $\ell$ than the standard \VI posterior given by $P(\KLD, \ell, Q)$. 
Additionally, one can show that minimising the adjustment term also favours shrinking the variance of $q(\*\theta)$. To see this for the case of \BD, rewrite
\begin{IEEEeqnarray}{rCl}
    S^{(1,\infty)}(\beta,q,\pi)&=&\frac{1}{\beta}\mathbb{E}_{q(\*\theta)}\left[\log(\pi(\*\theta))\right]-\frac{1}{\beta-1}\mathbb{E}_{q(\*\theta)}\left[\pi(\*\theta)^{\beta-1}-1\right]-\frac{1}{\beta}.\label{Equ:UpperBoundKLDBD_posB_AddTerm}
\end{IEEEeqnarray}
Applying Lemma \ref{Lem:LogBound} then shows  that for $\beta>1$,
\begin{IEEEeqnarray}{rCl}
    \frac{1}{\beta-1}\mathbb{E}_{q(\*\theta)}\left[\pi(\*\theta)^{\beta-1}-1\right]\geq\mathbb{E}_{q(\*\theta)}\left[\log(\pi(\*\theta))\right]\geq \frac{1}{\beta}\mathbb{E}_{q(\*\theta)}\left[\log(\pi(\*\theta))\right].\nonumber
\end{IEEEeqnarray}
From this, it follows that minimising Eq. \eqref{Equ:UpperBoundKLDBD_posB_AddTerm} will  make $\frac{1}{\beta-1}\mathbb{E}_{q(\*\theta)}\left[\pi(\*\theta)^{\beta-1}\right]$ large. 
Fixing $\pi(\*\theta)$, maximising $\frac{1}{\beta-1}\mathbb{E}_{q(\*\theta)}\left[\pi(\*\theta)^{\beta-1}\right]$ plus $\frac{1}{\beta} \times$ the Tsallis entropy of $q(\*\theta)$ is equivalent to minimising $\BD(q(\*\theta)||\pi(\*\theta))$.
Because \BD is a divergence, this maximization would naturally seek to choose $q(\*\theta)$ close to $\pi(\*\theta)$.
The Tsallis entropy term in this formulation would have acted to increase the variance of $q(\*\theta)$. But since we maximize only $\frac{1}{\beta-1}\mathbb{E}_{q(\*\theta)}\left[\pi(\*\theta)^{\beta-1}\right]$---i.e. without adding the Tsallis entropy of $q(\*\theta)$---choices of $\beta > 1$ will lead to shrinking the variance of $q(\*\theta)$ relative to standard \VI.

For the \GD, Jensen's inequality shows that for $\gamma>1$,
\begin{IEEEeqnarray}{rCl}
    \frac{1}{\gamma-1}\log\mathbb{E}_{q(\*\theta)}\left[\pi(\*\theta)^{\gamma-1}\right]\geq\mathbb{E}_{q(\*\theta)}\left[\log(\pi(\*\theta))\right]\geq \frac{1}{\gamma}\mathbb{E}_{q(\*\theta)}\left[\log(\pi(\*\theta))\right].\nonumber
\end{IEEEeqnarray}
As a result, minimising $S^{(1,\infty)}(\gamma,q,\pi)$ will seek to make $\frac{1}{\gamma-1}\log\mathbb{E}_{q(\*\theta)}\left[\pi(\*\theta)^{\gamma-1}\right]$ large. Fixing again $\pi(\*\theta)$, maximising $\frac{1}{\gamma-1}\log\mathbb{E}_{q(\*\theta)}\left[\pi(\*\theta)^{\beta-1}\right]$ plus $\frac{1}{\gamma} \times$ the \renyi entropy of $q(\*\theta)$ is equivalent to minimising $\GD(q(\*\theta)||\pi(\*\theta))$, and thus seeks $q(\*\theta)$ close to $\pi(\*\theta)$. The \renyi entropy term would have acted to increase the variance of $q(\*\theta)$. Therefore and similarly to the case of \BD, maximising $\frac{1}{\gamma-1}\log\mathbb{E}_{q(\*\theta)}\left[\pi(\*\theta)^{\gamma-1}\right]$ without adding the \renyi entropy will lead to shrinkage of the variance of $q(\*\theta)$.


\section{Proof of Proposition \ref{Thm:QuasiConjugacyGD}}
\label{Appendix:quasi-conjugacy-GD}

\begin{proof}
Proposition \ref{Thm:QuasiConjugacyGD} considers the following forms of the prior and likelihood 
\begin{align}
\pi(\*\theta|\*\kappa_0)&= h(\*\theta)\exp\left\lbrace\eta(\*\kappa_0)^TT(\*\theta)-A(\eta(\*\kappa_0))\right\rbrace\nonumber\\
q(\*\theta|\*\kappa)&= h(\*\theta)\exp\left\lbrace\eta(\*\kappa)^TT(\*\theta)-A(\eta(\*\kappa))\right\rbrace\nonumber\\
p(\*x|\*\theta) &= h(\*\theta) \exp(g(\*x)^T T(\*\theta)- B(\*x)),\nonumber
\end{align}
where $A(\eta(\*\kappa))=\log \int h(\*\theta)\exp\left\lbrace\eta(\*\kappa)^TT(\*\theta))\right\rbrace d\*\theta$ and $h(\*\theta) = \frac{1}{\int \exp(g(\*x)^T T(\*\theta)- B(\*x))d\*x}$. 

The \GVI objective function in this scenario, which we term an \ELBO as we use the \KLD prior regulariser is 
\begin{align}
\ELBO(\*\kappa) &= \mathbb{E}_{q(\*\theta|\*\kappa)}\left[\sum_{i=1}^n\ell^{(\gamma)}_G(\*\theta,\*x_i)\right] + \KLD(q(\*\theta|\*\kappa)||q(\*\theta|\*\kappa_0))\nonumber\\
&= \sum_{i=1}^n\underbrace{\int \underbrace{\ell^{(\gamma)}_G(\*\theta,\*x_i)}_{C_1(\*\kappa,\*\theta, x_i)}q(\*\theta|\*\kappa)d\theta}_{C_2(\*\kappa,\*x_i)} + \underbrace{\KLD(q(\*\theta|\*\kappa)||\pi(\*\theta|\*\kappa_0))}_{C_3(\*\kappa,\*\kappa_0)}.\nonumber
\end{align}
We have decomposed this into three terms that we need to check are closed forms of $\*\kappa$. Firstly 
\begin{equation}
C_1(\*\kappa,\*\theta, \*x_i) = \ell^{(\gamma)}_G(\*x_i,\*\theta) = -\frac{1}{\gamma-1}p(\*x_i;\*\theta)^{\gamma-1} \frac{\gamma}{\left[\int p(\*z;\*\theta)^{\gamma}d\*z\right]^{\frac{\gamma-1}{\gamma}}},\nonumber
\end{equation}
and in order for this to be a closed form function of $\*\kappa$, $\*\theta$, and $\*x_i$ requires that 
\begin{align}
I^{(\gamma)}(\*\theta) = \int p(\*z|\*\theta)^{\gamma}dz &= \int h(\*\theta)^{\gamma} \exp(\gamma g(\*z)^T T(\*\theta)- \gamma B(\*z))dz, \nonumber
\end{align}
where the theorem statement ensures that $I^{(\gamma)}(\*\theta)$ is a closed form function of $\*\theta$. 
%
%
Next
\begin{align}
&C_2(\*\kappa,\*x_i)\nonumber\\
=&-\frac{\gamma}{\gamma-1}\int  h(\*\theta)^{\gamma-1} \exp((\gamma-1) g(\*x_i)^T T(\*\theta)- (\gamma-1) B(\*x_i)) \frac{1}{\left[h(\*\theta)^{\gamma} I^{(\gamma)}(\*\theta)\right]^{\frac{\gamma-1}{\gamma}}}q(\*\theta|\*\kappa)d\*\theta\nonumber\\
=&-\frac{\gamma}{\gamma-1}\frac{\exp\left((1-\gamma)B(\*x_i)+A\left(\eta(\*\kappa)+(\gamma -1)g(\*x_i)\right)\right)}{\exp\left(A(\eta(\*\kappa))\right)}\mathbb{E}_{q(\*\theta|\left(\eta(\*\kappa)+(\gamma -1)g(\*x_i)\right))}\left[ I^{(\gamma)}(\*\theta)^{\frac{1-\gamma}{\gamma}}\right],\nonumber
\end{align}
where the theorem statement  ensures that $\left(\eta(\*\kappa_n)+(\gamma -1)g(\*x_i)\right)\in\mathcal{N}$ for all $x_i$ and that $F_2(\*\kappa^{\ast}) = \mathbb{E}_{q(\*\theta|\*\kappa^{\ast})}\left[ I^{(\gamma)}(\*\theta)^{\frac{1-\gamma}{\gamma}}\right]$ is closed form function of $\*\kappa^{\ast}$ for all $\*\kappa^{\ast}\in\mathcal{N}$. Lastly
\begin{align}
C_3(\*\kappa, \kappa_0)
=& \int h(\*\theta)\exp\left\lbrace\eta(\*\kappa)^TT(\*\theta)-A(\eta(\*\kappa))\right\rbrace \log \frac{h(\*\theta)\exp\left\lbrace\eta(\*\kappa)^TT(\*\theta)-A(\eta(\*\kappa))\right\rbrace}{h(\*\theta)\exp\left\lbrace\eta(\*\kappa_0)^TT(\*\theta)-A(\eta(\*\kappa_0))\right\rbrace} d\*\theta\nonumber\\
=&A(\eta(\*\kappa_0))-A(\eta(\*\kappa)) + \left(\eta(\*\kappa)-\eta(\*\kappa_0)\right)^T\mathbb{E}_{q(\*\theta|\*\kappa)}\left[T(\*\theta)\right],\nonumber
\end{align}
where the theorem statement  ensures that $F_1(\kappa^{\ast}) = \mathbb{E}_{q(\*\theta|\*\kappa^{\ast})}\left[T(\*\theta)\right]$ is a closed form function of $\*\kappa^{\ast}$ for all $\kappa^{\ast}\in\mathcal{N}$ .
\end{proof}

\section{Black Box \GVI (\BBGVI)}
\label{Appendix:BBGVI}

The following sections first recall the (implicit and explicit) assumptions one typically makes for black box \VI.   They are then compared to assumptions that are reasonable for black box \GVI (\BBGVI). The corresponding methods, their special cases and the relevant black box variance reduction techniques are then derived and elaborated upon. 
While there are many black box \VI strategies, we center attention on the framework provided for by \citet{BBVI}.
Throughout, we denote $q(\*\theta) = q(\*\theta|\*\kappa)$ as a posterior distribution in a set of variational families $\mathcal{Q}$ and parameterized by some parameter $\*\kappa \in \*K$.

\subsection{Preliminaries and assumptions}

The variance reduction techniques of \citet{BBVI} crucially rely on three implicit assumptions that are reasonable for many applications of standard \VI. 
\begin{itemize}
    \myitem{(A1)} Structured mean-field variational inference is used, which means that we can factorize the variational family as $\mathcal{Q} = \{q(\*\theta|\*\kappa) = \prod_{j=1}^kq_j(\*\theta_j|\*\kappa_j): \*\kappa_j \in K_j \text{ for all } j \}$.
    \label{BBGVI:AS:MeanFieldInference}
    \myitem{(A2)} For all factors $\*\theta_j$, we have a Markov blanket $\*\theta_{(j)}$ for which we can additively decompose $\ell(\*\theta, x_i) = \ell^{(j)}(\*\theta_j, \*\theta_{(j)}, x_i) + \ell^{(-j)}(\*\theta_{-j}, x_i)$. Here, $\ell^{(j)}$ is an additive component of the loss $\ell$ that only depends on the $j$-th factor and its Markov blanket, while $\ell^{(-j)}$ is an additive component of the loss that may depend on all of $\*\theta$ except for its $j$-th factor. 
    %
    Note that such additivity holds for standard \VI for which the likelihood and the prior
    are such that the components $\*\theta_j$ are conditionally independent. In this case, the conditioning set is the Markov blanket.
    \label{BBGVI:AS:MarkovBlanket}
    \myitem{(A3)} $D=\frac{1}{w}\cdot\KLD$ (with $w = 1$ for standard \VI).
    \label{BBGVI:AS:wKLD}
\end{itemize}
Note that \ref{BBGVI:AS:MeanFieldInference} is always satisfied for both standard \VI and \GVI, because any variational family factorizes into at least a single factor.
In contrast, note that \ref{BBGVI:AS:MarkovBlanket} does not even necessarily hold for standard \VI unless one imposes some conditional independence structure on the $\*\theta_j$.
For \GVI, both \ref{BBGVI:AS:MarkovBlanket}  and \ref{BBGVI:AS:wKLD} do not necessarily hold. 
If they do however, they can greatly simplify \BBGVI or improve its numerical performance.
In the remainder of this section, we discuss different constellations of assumptions and their consequences for \BBGVI.

\subsection{Standard black box \VI with \ref{BBGVI:AS:MarkovBlanket} and \ref{BBGVI:AS:wKLD}}

If the regularizer used is still a rescaled version of the \KLD, one recovers an internally rescaled version of the objective in \citep{BBVI}. Namely, the gradient is given by
\begin{IEEEeqnarray}{rCl}
    \mathbb{E}_{q(\*\theta|\*\kappa)}\bigg[
 \nabla_{\*\kappa}
     \log(q(\*\theta|\*\kappa))
     \bigg(
  -\sum_{i=1}^n\ell(\*\theta, x_i)
   - w\pi(\*\theta)
     - w\log(q(\*\theta|\*\kappa))
     \bigg)
    \bigg]. 
    \nonumber
\end{IEEEeqnarray}
and can be approximated in a smart way by sampling from $q(\*\theta|\*\kappa)$, see for instance \citet{BBVI} for details and the viable strategies for variance reduction.
Next, we turn attention to the cases that are more interesting: If \ref{BBGVI:AS:wKLD} does not hold (so that $D\neq \KLD$) and when the losses are not necessarily negative log likelihoods, meaning that \ref{BBGVI:AS:MarkovBlanket} requires more careful consideration.

\subsection{\BBGVI under \ref{BBGVI:AS:MarkovBlanket}}

If the losses are decomposable along the factors, two cases need to be distinguished:
\begin{itemize}
    \myitem{(D1)} $\nabla_{\*\kappa}D(q\|\pi)$ has closed form for all $q \in \mathcal{Q}$;
    \label{BBGVI:D:closedForm}
    \myitem{(D2)} $D(q\|\pi) = \mathbb{E}_{q(\*\theta|\*\kappa)}\left[\ell^D_{\*\kappa, \pi}(\*\theta)\right]$ for some function $\ell^D_{\*\kappa, \pi}:\*\Theta \to \mathbb{R}$.
    \label{BBGVI:D:expectationForm}
\end{itemize}
Under each condition, we find a different solution using as much of the available information as possible to improve inference outcomes.
For simplicity, we first explain how the derivation works without using the additional information that \ref{BBGVI:AS:MarkovBlanket}. In a second step, we shall see how this additional information can be used for variance reductions in the Rao-Blackwellization spirit also used by \citet{BBVI}.

\subsubsection{Gradients if \ref{BBGVI:D:closedForm} holds, not using \ref{BBGVI:AS:MarkovBlanket}}

In this case, we can obtain the objective given in the main paper. 
Define $L(q)$ to be the \GVI objective function of $q(\*\theta|\*\kappa)$. It holds that
\begin{IEEEeqnarray}{rCl}
 \nabla_{\*\kappa}L(q) 
    & = &
 \nabla_{\*\kappa}\left[\int_{\*\theta}
    \sum_{i=1}^n\ell(\*\theta, x_i)
     q(\*\theta|\*\kappa)d\*\theta + D(q||\pi) \right]
     \nonumber \\
    & = &
 \int_{\*\theta}
     \ell_n(\*\theta, \*x)
 \nabla_{\*\kappa}q(\*\theta|\*\kappa)d\*\theta + \nabla_{\*\kappa}D(q||\pi)
 \nonumber \\
    & = & \mathbb{E}_{q(\*\theta|\*\kappa)}\left[
  \sum_{i=1}^n\ell(\*\theta, x_i)
     \nabla_{\*\kappa}\log(q(\*\theta|\*\kappa))
    \right] + \nabla_{\*\kappa}D(q||\pi).
    \nonumber
\end{IEEEeqnarray}
Correspondingly, the gradient can then be estimated without bias and computing the corresponding sample average $\frac{1}{S}\sum_{s=1}^SG(\*\theta^{(s)})$, where the individual terms are given by
\begin{IEEEeqnarray}{rCl}
     G(\*\theta^{(s)})& = & 
     \sum_{i=1}^n\ell(\*\theta^{(s)}, x_i)
 \nabla_{\*\kappa}\log(q(\*\theta^{(s)}))
     + \nabla_{\*\kappa}D(q||\pi)
     \nonumber
\end{IEEEeqnarray}

\subsubsection{Gradients if \ref{BBGVI:D:expectationForm} holds, not using \ref{BBGVI:AS:MarkovBlanket}}

If the uncertainty quantifier is not available in closed form, one instead can rely on
\begin{IEEEeqnarray}{rCl}
 \nabla_{\*\kappa}L(q) 
    & = &
 \nabla_{\*\kappa}\left[
 \int_{\*\theta}\left[ 
     \sum_{i=1}^n\ell(\*\theta, x_i) + 
     \ell^D_{\*\kappa, \pi}(\*\theta)\right]q(\*\theta|\*\kappa)d\*\theta\right]
     \nonumber \\
    & = &
 \int_{\*\theta}\left[ 
     \sum_{i=1}^n\ell(\*\theta, x_i) + 
     \ell^D_{\*\kappa, \pi}(\*\theta)
 \right]\nabla_{\*\kappa}q(\*\theta|\*\kappa)d\*\theta + 
 \int_{\*\theta}\left[
     \nabla_{\*\kappa}\ell^D_{\*\kappa, \pi}(\*\theta)
 \right]q(\*\theta|\*\kappa)d\*\theta
 \nonumber \\
    & = & \mathbb{E}_{q(\*\theta|\*\kappa)}\left[
  \left(\sum_{i=1}^n\ell(\*\theta, x_i) + 
     \ell^D_{\*\kappa, \pi}(\*\theta)\right)
     \nabla_{\*\kappa}\log(q(\*\theta|\*\kappa))
    \right] +  \mathbb{E}_{q(\*\theta|\*\kappa)}\left[
 \nabla_{\*\kappa}\ell^D_{\*\kappa, \pi}(\*\theta)\right]. \quad \quad 
 \nonumber
\end{IEEEeqnarray}
This derivation is a more general case of the one given in \citet{BBVI}, but further simplifies to the one therein if $D = \KLD$. The gradient is estimated without bias by sampling $\*\theta^{(1:S)}$ from $q(\*\theta|\*\kappa)$ and again computing $\frac{1}{S}\sum_{s=1}^SG(\*\theta^{(s)})$ for the slightly different
\begin{IEEEeqnarray}{rCl}
     G(\*\theta^{(s)})& = & \left[
     \sum_{i=1}^n\ell(\*\theta^{(s)}, x_i) +     \ell^D_{\*\kappa, \pi}(\*\theta^{(s)})
 \right]
 \nabla_{\*\kappa}\log(q(\*\theta^{(s)}|\*\kappa)) + 
 \nabla_{\*\kappa}\ell^D_{\*\kappa, \pi}(\*\theta^{(s)}).  
     \nonumber
\end{IEEEeqnarray}

\subsubsection{Rao-Blackwellization for variance reduction, using \ref{BBGVI:AS:MarkovBlanket}}

If the losses define a markov blanket over the factors $\*\theta_j$, one can employ  Rao-Blackwellization for variance reduction. This is done by rewriting for $q_{-j}(\*\theta_{-j}|\*\kappa_{-j}) = \prod_{l=1, l \neq j}^kq_l(\*\theta_l|\*\kappa_l)$ the partial derivatives as
\begin{IEEEeqnarray}{rCl}
    \nabla_{\*\kappa_j}L(q) & = & \nabla_{\*\kappa_j}\mathbb{E}_{q_j(\*\theta_j|\*\kappa_j)}\left[
 \mathbb{E}_{q_{-j}(\*\theta_{-j}|\*\kappa_{-j})}\left[
     L(q)|\*\theta_{j}
  \right]
    \right]. \nonumber
\end{IEEEeqnarray}
The hope is then to get around computing as many of the inner expectations over $q_{-j}(\*\theta_{-j}|\*\kappa_{-j})$ as possible.
Assume for the moment that at least \ref{BBGVI:D:expectationForm} holds. 
Further, denote 
$q_{-j}(\*\theta_{-j}|\*\kappa_{-j}) = q_{-j}$,
$q_{j}(\*\theta_{j}|\*\kappa_{j}) = q_{j}$, 
and in similar fashion the distributions
$q_{(j)}$, $q_{-(j)}$, $q$. Moreover, denote 
$\ell_i = \ell(\*\theta, x_i)$, $\ell^D = \ell^D_{\*\kappa, \pi}(\*\theta)$ and in a similar fashion $\ell_i^{(j)}$, $\ell_i^{-(j)}$. 
Now, assuming that \ref{BBGVI:AS:MarkovBlanket} holds relative to the factors $\*\theta_j$ of the variational family $\mathcal{Q}$, one finds
\begin{IEEEeqnarray}{rCl}
    \nabla_{\*\kappa_j}L(q) 
    & = &
    \mathbb{E}_{q_j}\left[
 \nabla_{\*\kappa_j}\log(q_j)\left(
     \mathbb{E}_{q_{-j}}\left[\sum_{i=1}^n\ell_i^{(j)}\right]  + \mathbb{E}_{q_{-j}}[\ell_n^{-j}] +
     \mathbb{E}_{q_{-j}}[\ell^D]
 \right)\right] + \mathbb{E}_{q_{-j}}[\nabla_{\*\kappa_j}\ell^D].
    \nonumber
\end{IEEEeqnarray}
Observing that $\mathbb{E}_{q_j}[\nabla_{\*\kappa_j}\log(q_j)] = 0$ and that $\mathbb{E}_{q_{-j}}[\ell^{-(j)}]$ is constant in $\*\theta_j$ by \ref{BBGVI:AS:MarkovBlanket}, this drastically simplifies to
\begin{IEEEeqnarray}{rCl}
    \nabla_{\*\kappa_j}L(q) 
    & = &
    \mathbb{E}_{q_j}\left[
 \nabla_{\*\kappa_j}\log(q_j)
     \mathbb{E}_{q_{-j}}\left[\sum_{i=1}^n\ell_i^{(j)}\right]  + 
     \mathbb{E}_{q_{-j}}\left[\ell^D + \nabla_{\*\kappa_j}\ell^D\right]
    \right]. \nonumber 
\end{IEEEeqnarray}    
Next, observe that by virtue of how $\ell^{(j)}$ was constructed, it holds that we can also simplify
\begin{IEEEeqnarray}{rCl}
    \mathbb{E}_{q_j}\left[
 \nabla_{\*\kappa_j}\log(q_j)
     \mathbb{E}_{q_{-j}}\left[\sum_{i=1}^n\ell_i^{(j)}\right]
     \right]
     & = &
     \mathbb{E}_{q_{(j)}}\left[\sum_{i=1}^n\ell_i^{(j)}\right].
     \nonumber
\end{IEEEeqnarray}
Putting the above together, we finally arrive at
\begin{IEEEeqnarray}{rCl}
    \nabla_{\*\kappa_j}L(q) 
    & = &
    \mathbb{E}_{q_j}\left[
 \nabla_{\*\kappa_j}\log(q_j)
        \left(
     \mathbb{E}_{q_{-j}}\left[\sum_{i=1}^n\ell_i^{(j)}\right]  + 
     \mathbb{E}_{q_{-j}}[\ell^D]
     \right) 
     + 
     \mathbb{E}_{q_{-j}}[\nabla_{\*\kappa_j}\ell^D]
    \right]
    \nonumber \\
    & = &
    \mathbb{E}_{q_{(j)}}\left[
 \nabla_{\*\kappa_j}\log(q_j)
     \sum_{i=1}^n\ell_i^{(j)}
     \right]  
     + 
         \mathbb{E}_{q}
         \left[
            \nabla_{\*\kappa_j}\log(q_j)\ell^D + \nabla_{\*\kappa_j}\ell^D
         \right]. \nonumber 
\end{IEEEeqnarray}    
which is the final form under \ref{BBGVI:D:closedForm}. Should \ref{BBGVI:D:closedForm} to hold, one can instead use the  lower variance estimate
\begin{IEEEeqnarray}{rCl}
    \nabla_{\*\kappa_j}L(q) 
    & = &
    \mathbb{E}_{q_{(j)}}\left[
 \nabla_{\*\kappa_j}\log(q_j)
     \sum_{i=1}^n\ell_i^{(j)}
     \right]   
         + 
         \nabla_{\*\kappa_j}D(q\|\pi). \nonumber 
\end{IEEEeqnarray}    
These derivations are very similar to the ones in the supplement of \citet{BBVI}, but importantly the former are restricted to negative log likelihood losses. The more general version presented here holds for arbitrary decomposable losses.
The $J$ terms $ \nabla_{\*\kappa_j}L(q)$ can be combined into a global gradient estimate simply by setting
\begin{IEEEeqnarray}{rCl}
     \nabla_{\*\kappa}L(q) & = & \left( \nabla_{\*\kappa_1}L(q) ,  \nabla_{\*\kappa_2}L(q) , \dots  \nabla_{\*\kappa_J}L(q)  \right)^T
     \nonumber.
\end{IEEEeqnarray}
To make the meaning of \ref{BBGVI:AS:MarkovBlanket} more tangible for the case of general losses, we next provide a short example in the context of multivariate regression.

\begin{example}[Markov blankets without conditional independence]
    Suppose each $x_i = (x_{i,1}, x_{i,2}, x_{i,3})'$ consists of three measurements that we wish to relate to some other observables $y_{i}$  through
    \begin{IEEEeqnarray}{rCl}
        x_{i,1} & = & a + y_{i}b + \xi_1 \nonumber \\
        x_{i,2} & = & b + y_{i}c + \xi_2
        \nonumber\\
        x_{i,3} & = & d + \xi_3\nonumber
    \end{IEEEeqnarray}
    where $\xi_j$ are unknown slack variables (or errors), the parameters of interest are $\*\theta = (a,b,c,d,e)$ and we wish to produce a belief distribution over $\*\theta$ that is informative about good values of $\*\theta$ relative to some prediction loss
    \begin{IEEEeqnarray}{rCl}
        \ell(\*\theta, x_i) = 
        \|f^1_1(\*\theta_1, \*\theta_{(1)}, y_i) - x_{i,1}\|_p^p + 
        \|f^2_2(\*\theta_1, \*\theta_{(1)}, y_i) - x_{i,2}\|_p^p + 
        \|f^3_2(\*\theta_2, \*\theta_{(2)}, y_i) - x_{i,3}\|_p^p 
        ,
        \nonumber
    \end{IEEEeqnarray}
    where $\|\cdot\|_p^p$ denotes some $p$-norm for $p\geq 1$ and $f_l^j$ seeks to predict only the $l$-th dimension of $x_i$ by means of the $l$-th factor of $\*\theta$ and its blanket. Suppose that $f_l^j$ will correspond to the $l$-th row  written down in the above model for $x_i$ (excluding of course the error term), which means that
    \begin{IEEEeqnarray}{rCl}
        f^1_1(\*\theta_1, \*\theta_{(1)}) 
        & = & a + y_{i}b \nonumber \\
        f^2_2(\*\theta_1, \*\theta_{(1)}, y_i) 
        & = & b + y_{i}c \nonumber \\
        f^3_2(\*\theta_2, \*\theta_{(2)}, y_i) 
        & = & d \nonumber 
    \end{IEEEeqnarray}
    In this case, the two factors of $\*\theta$ will clearly be given by 
    \begin{IEEEeqnarray}{rCl}
        \*\theta_1 & = & (a,b,c)^T, \quad \quad \*\theta_2  = (d). \nonumber
    \end{IEEEeqnarray}
\end{example}
As before, one will in practice need to approximate the gradients with a sample $\*\theta^{(1:S)}$ drawn from $q(\*\theta|\*\kappa)$. For one of the fixed samples $\*\theta^{(s)}$, the relevant terms are computed as
\begin{IEEEeqnarray}{rCl}
    G_j(\*\theta^{(s)}) & = &
 \nabla_{\*\kappa_j}\log(q_j(\*\theta^{(s)}_j|\*\kappa_j))
         \sum_{i=1}^n\ell^{(j)}(\*\theta^{(s)}_j, \*\theta^{(s)}_{(j)}, x_i)
     + 
     \widetilde{D}(s, j) \nonumber 
\end{IEEEeqnarray}    
for some function $\widetilde{D}(s,j)$.
If \ref{BBGVI:D:expectationForm} holds and there is no closed form for the uncertainty quantifier, this function is given by
\begin{IEEEeqnarray}{rCl}
    \widetilde{D}(s,j) & = &
    \nabla_{\*\kappa_j}\log(q_j(\*\theta^{(s)}_j|\*\kappa_j))
            \ell^D_{\pi, \*\kappa}(\*\theta^{(s)}) + 
            \nabla_{\*\kappa_j}\ell^D_{\pi, \*\kappa}(\*\theta^{(s)})
    \nonumber
\end{IEEEeqnarray}
and in case the stricter requirement \ref{BBGVI:D:closedForm} holds, it is simply given by the closed form
\begin{IEEEeqnarray}{rCl}
   \widetilde{D}(s,j) & = & 
         \nabla_{\*\kappa_j}D(q\|\pi). \nonumber 
\end{IEEEeqnarray}  

\subsection{\BBGVI if neither \ref{BBGVI:AS:MarkovBlanket} nor \ref{BBGVI:AS:wKLD} hold}

It is of course possible that neither \ref{BBGVI:AS:MarkovBlanket} nor \ref{BBGVI:AS:wKLD} hold. Alternatively, it may simply be convenient to build an implementation that can work reliably without imposing any assumptions.
In this case, one will have to use the naive version of \BBGVI that is given in the main paper and only depends on the distinction between \ref{BBGVI:D:expectationForm} and \ref{BBGVI:D:closedForm}.
However---even though we do not do so in our experiments--there still are valid black box variance reduction techniques for this case. 
The next section presents these techniques, again by adapting notation and logic from \citet{BBVI}.

\subsection{Generically applicable variance reduction}

While the Rao-Blackwellization variance reduction will generally be more effective, some variance reduction techniques can work in circumstances where Rao-Blackwellization does not.
Conversely, this means that if the Rao-Blackwellization is applicable,  one can actually deploy two variance reduction schemes at once to substantially speed up convergence.
The control variate we use is simply
\begin{IEEEeqnarray}{rCl}
    h(\*\theta) & = & \nabla_{\*\kappa} \log q(\*\theta|\*\kappa) \nonumber
\end{IEEEeqnarray}
with an optimal scaling parameter that can be estimated as
\begin{IEEEeqnarray}{rCl}
    \hat{a}^{\ast} & = & \dfrac{
 \sum_{s=1}^S\widehat{\text{Cov}}(L(\*\theta^{(s)}), h(\*\theta^{(s)})) 
    }{
 \sum_{s=1}^S \widehat{\text{Var}}(h(\*\theta^{(s)}))
    }. \nonumber
\end{IEEEeqnarray}
Based on this, one may now compute the variance reduced term $G_{\text{VR}}(\*\theta^{(s)})$ from $G(\*\theta^{(s)})$ as
\begin{IEEEeqnarray}{rCl}
    G_{\text{VR}}(\*\theta^{(s)}) & = &
    G(\*\theta^{(s)}) - \hat{a}^{\ast} \cdot h(\*\theta^{(s)}).
    \nonumber
\end{IEEEeqnarray}
Of course, the exact same logic can be applied to the Rao-Blackwellized terms $G_j(\*\theta^{(s)})$ to reduce the variance a second time.

\section{Closed forms for divergences \& proof of Proposition \ref{proposition:closed_form_D}}

This section proves various closed forms for the uncertainty quantifiers in the \GVI problem.
We do so by proving conditions for closed forms of the $\alpha\beta\gamma$-divergence (\ABGD) introduced in Appendix \ref{Appendix:definitions}.
Note that the special case of these results for the \RAD has been derived before \citep[see][]{closedFormsPaper, closedFormsThesis, ConvexStatisticalDistances}. Unlike previous work, our results apply to a range of other divergences, too.
This is convenient because all other robust divergences we discuss throughout the paper are special cases of \ABGD. 

\subsection{High-level overview of results and preliminaries}

%
%
Summarizing some of the most important findings of this section, we find that if both $q(\*\theta)$ and $\pi(\*\theta)$ are in the same exponential variational family $\mathcal{Q}$,
\begin{itemize}
    \item $\RAD(q||\pi)$ and $\AD(q|\pi)$ are always available in closed form if $\alpha \in (0,1)$ (see Corollary \ref{Thm:ClosedFormAD})
    \item $\RAD(q||\pi)$ and $\AD(q|\pi)$ are  available in closed form if $\alpha>1$ for most exponential families (see again Corollary \ref{Thm:ClosedFormAD})
    \item $\BD(q||\pi)$ and $\GD(q||\pi)$ are available in closed form for $\beta > 1$ and $\gamma > 1$ for most exponential families (See Corollary \ref{Thm:ClosedFormBD2}).
\end{itemize}
We note that these findings are interesting because closed forms for the divergence term drastically reduce the variance of black box \GVI, see also Appendix \ref{Appendix:BBGVI}.
The remainder of this section is devoted to tedious but rigorous derivations of these findings.
Before stating any results, it is useful to state the definition of an exponential family and its natural parameter space upon which the proofs rely.
\begin{definition}[Exponential families]
Object $\theta\in\Theta\subset\mathbb{R}^d$, $d\geq1$ has an exponential family distribution with parameters $\kappa\in\*K\subset\mathbb{R}^{p^{\prime}}$, $p^{\prime}\geq 1$  if there exist functions $\eta:\*K\rightarrow\mathcal{N}\subset \mathbb{R}^p$, $p\geq 1$, $T:\Theta\rightarrow\mathcal{T}\subset\mathbb{R}^p$, $h:\Theta\rightarrow\mathbb{R}_{\geq 0}$ and $A:\mathcal{N}\rightarrow\mathbb{R}$ such that 
\begin{IEEEeqnarray}{rCl}
p(\*\theta|\eta(\*\kappa)) &=& h(\*\theta)\exp\left\lbrace\eta(\*\kappa)^TT(\*\theta)-A(\eta(\*\kappa))\right\rbrace,\nonumber
\end{IEEEeqnarray}
where $A(\eta(\*\kappa)) = -\log \left(\int h(\*\theta)\exp\left\lbrace\eta(\*\kappa)^TT(\*\theta)\right\rbrace d\theta\right)$. 
The set $\mathcal{N}$ is called the natural parameter space and is defined to ensure $p(\*\theta|\eta(\*\kappa))$ is a normalised probability density, $\mathcal{N}=\left\lbrace \eta(\*\kappa):A(\eta(\*\kappa))<\infty\right\rbrace$.
\label{Def:ExponentialFamily}
\end{definition}
Throughout the rest of this section, we assume that the following condition holds for both the prior and the variational family $\mathcal{Q}$.
\begin{condition}[The prior and variational families] It holds that 
\begin{enumerate}[label=\roman*), ref=\roman*]
\item the variational family $\mathcal{Q}=\left\lbrace q(\*\theta|\eta(\*\kappa))\right\rbrace$ is an exponential family of the form given by Definition \ref{Def:ExponentialFamily}
\item the prior $\pi(\*\theta|\eta(\*\kappa_0))$ is a member of that variational family.
\end{enumerate}
Amongst other things, this implies that the log-normalising constant is a closed form function of the natural parameters and that we can derive generic conditions for closed forms by using the canonical representation of exponential families. 
\end{condition}
To showcase the implications of the derived results, we use the Mulitvariate Gaussian (\MVN) to provide examples along the way.
\begin{definition}[The \MVN exponential family]
The density of the \MVN exponential family for vector $\*\theta$ of dimension $d$ is $p(\*\theta|\eta(\*\kappa))= h(\*\theta)\exp\left\lbrace\eta(\*\kappa)^TT(\*\theta)-A(\eta(\*\kappa))\right\rbrace$ where
\begin{IEEEeqnarray}{RCl}
\eta(\*\kappa) &= \begin{pmatrix} \*V^{-1}\*\mu\\
-\frac{1}{2}\*V^{-1}
\end{pmatrix}\qquad
T(\*\theta) &= \begin{pmatrix} \*\theta\\
\*\theta\*\theta^T
\end{pmatrix}\nonumber\\
h(\*\theta) &= \left(2\pi\right)^{-d/2}\qquad
A(\eta(\*\kappa)) &= \left[\frac{1}{2}\log \left|\*V\right|+\frac{1}{2}\*\mu \*V^{-1}\*\mu\right]\nonumber
\end{IEEEeqnarray}
and the natural parameter space requires that $\*\mu$ is a real valued vector of the same dimension as $\*\theta$ and $\*V$ is a $d\times d$ symmetric semi-positive definite matrix.
\label{Def:MVN}
\end{definition}

\subsection{Results, proofs \& examples}

The remainder of this section is structued as follows: First, we give the main result for the $\alpha\beta\gamma$-divergence (\ABGD)  in Proposition \ref{Proposition:ClosedFormABGDMaster}.
This ``master result'' is then applied to various special cases for \ABGD that are of practical interest, namely the $\alpha$-divergence (\AD), R\'enyi's $\alpha$-divergence (\RAD), the $\beta$-divergence (\BD) as well the $\gamma$-divergence (\GD).

\subsubsection{Master result for \ABGD}

While the following result and corresponding proof are somewhat tedious to read, they are conceptually simple: In fact, all that is needed to derive the results is some basic algebra and the canonical form of the exponential family.
%
%
%
\begin{proposition}[Closed form \ABGD between exponential families] The \ABGD between a variational posterior $q(\*\theta|\*\kappa_n)$ and prior $\pi(\*\theta|\*\kappa_0)$ is available in closed form under the following conditions
\begin{enumerate}[label=\roman*), ref=\roman*]
\item \label{item:ABGC3} $\eta(\*\kappa_0),\eta(\*\kappa_n)\in\mathcal{N}\Rightarrow \left(\alpha \eta(\*\kappa_0)+(\beta-1)\eta(\*\kappa_n)\right)\in\mathcal{N}$;
\item \label{item:ABGC4} $\mathbb{E}_{p(\*\theta|\eta(\*\kappa))}\left[ h(\*\theta)^{\alpha+\beta-2}\right]$ is  a closed form function of $\eta(\*\kappa)\in\mathcal{N}$.
\end{enumerate}
If these conditions hold the \ABGD can be written as
%
%
%
\begin{IEEEeqnarray}{rCl}
&&\tilde{D}_{G}^{(\alpha,\beta)}(q(\*\theta|\*\kappa_n)||\pi(\*\theta|\*\kappa_0))\nonumber\\
&=& \alpha B(\*\kappa_n, (\alpha+\beta-1))E(\*\kappa_n,(\alpha+\beta-1)) {}+{} (\beta-1)B(\*\kappa_0, (\alpha+\beta-1))E(\*\kappa_0,(\alpha+\beta-1))\nonumber\\
&&-\> (\alpha+\beta-1)C(\*\kappa_n,\*\kappa_0,\alpha,(\beta-1))\tilde{E}(\*\kappa_n,\*\kappa_0,\alpha,(\beta-1))\nonumber
\label{Equ:SubGeneralDivergenceClosedForm}
\end{IEEEeqnarray}
where 
%
\begin{IEEEeqnarray}{rClrCl}
B(\*\kappa,\delta) &=& \frac{\exp\left\lbrace A(\delta\eta(\*\kappa))\right\rbrace}{\exp\left\lbrace A(\eta(\*\kappa))\right\rbrace^{\delta}}, \qquad& 
 C(\*\kappa_1,\*\kappa_2,\delta_1,\delta_2) &=& \frac{\exp\left\lbrace A\left(\delta_1 \eta(\*\kappa_1)+\delta_2)\eta(\*\kappa_2)\right)\right\rbrace}{\exp\left\lbrace A(\eta(\*\kappa_1))\right\rbrace ^{\delta_1}\exp\left\lbrace A(\eta(\*\kappa_2))\right\rbrace^{\delta_2}}\nonumber\\
E(\*\kappa,\delta) &=& \mathbb{E}_{p(\theta|\delta\eta(\*\kappa))}\left[ h(\*\theta)^{\delta-1}\right],\qquad&
\tilde{E}(\*\kappa_1,\*\kappa_2,\delta_1,\delta_2) &=& \mathbb{E}_{p(\theta|\delta_1\eta(\*\kappa_1)+\delta_2\eta(\*\kappa_2))}\left[ h(\*\theta)^{\delta_1+\delta_2-1}\right]\nonumber
\end{IEEEeqnarray}
we suppress the dependence of these functions on $A(\cdot)$ and $h(\cdot)$ as these derive form the definition of the exponential family (Definition \ref{Def:ExponentialFamily}).
\label{Proposition:ClosedFormABGDMaster}
\end{proposition}

\begin{proof}
The \ABGD is a closed form function of \ABGDReparam given in Definition \ref{Def:ABGdiv}. Hence if \ABGDReparam is available in closed form, then so is \ABGD. In order to ensure that $\tilde{D}_{G}^{(\alpha,\beta)}(q(\*\theta|\*\kappa_n)||\pi(\*\theta|\*\kappa_0))$ has closed form, we need to make sure the three integrals below are available in closed form for the exponential family. 
\begin{IEEEeqnarray}{rCl}
G_1:=\int q(\*\theta|\*\kappa_n)^{\alpha+\beta-1}d\*\theta,&& \quad G_2:=\int \pi(\*\theta|\*\kappa_0)^{\alpha+\beta-1}d\*\theta,\nonumber\\
G_3:=\int q(\*\theta|\*\kappa_n)^{\alpha}&&\pi(\*\theta|\*\kappa_0)^{\beta-1}d\*\theta.\nonumber
\end{IEEEeqnarray}
%
%
First we tackle $G_1$.
\begin{IEEEeqnarray}{rCl}
G_1
&=&\int h(\*\theta)^{\alpha+\beta-1}\exp\left\lbrace(\alpha+\beta-1)\eta(\*\kappa_n)^TT(\*\theta)-(\alpha+\beta-1)A(\eta(\*\kappa_n))\right\rbrace d\*\theta\nonumber\\
&=&\exp\left\lbrace A((\alpha+\beta-1)\eta(\*\kappa_n)) -(\alpha+\beta-1)A(\eta(\*\kappa_n))\right\rbrace \mathbb{E}_{p(\*\theta|(\alpha+\beta-1)\eta(\*\kappa_n)}\left[h(\*\theta)^{\alpha+\beta-2}\right],\nonumber
\end{IEEEeqnarray}
where condition (\ref{item:ABGC3}) with $\eta(\*\kappa_0)=\eta(\*\kappa_n)$ ensures that 
\begin{IEEEeqnarray}{rCl}
 A((\alpha+\beta-1)\eta(\*\kappa_n))&=&\int h(\*\theta)\exp\left\lbrace(\alpha+\beta-1)\eta(\*\kappa_n)^TT(\*\theta)\right\rbrace d\*\theta<\infty,\nonumber
\end{IEEEeqnarray}
%
%
which in turn ensures that $p(\*\theta|(\alpha+\beta-1)\eta(\*\kappa_n)$ is a normalised probability density and that \newline 
$\mathbb{E}_{p(\*\theta|(\alpha+\beta-1)\eta(\*\kappa_n)}\left[h(\*\theta)^{\alpha+\beta-2}\right]$ is a valid expectation. Now, condition (\ref{item:ABGC4}) guarantees this is a closed form function of $\eta(\*\kappa_n)$. Similarly for $G_2$,
\begin{IEEEeqnarray}{rCl}
G_2
&=&\int h(\*\theta)^{\alpha+\beta-1}\exp\left\lbrace(\alpha+\beta-1)\eta(\*\kappa_0)^TT(\*\theta)-(\alpha+\beta-1)A(\eta(\*\kappa_0))\right\rbrace d\*\theta\nonumber\\
&=&\exp\left\lbrace A((\alpha+\beta-1)\eta(\*\kappa_0)) -(\alpha+\beta-1)A(\eta(\*\kappa_0))\right\rbrace \mathbb{E}_{p(\*\theta|(\alpha+\beta-1)\eta(\*\kappa_0)}\left[h(\*\theta)^{\alpha+\beta-2}\right],\nonumber
\end{IEEEeqnarray}
where in analogy to $G_1$, conditions (\ref{item:ABGC3}) and (\ref{item:ABGC4}) with $\eta(\*\kappa_k)=\eta(\*\kappa_0)$ ensure this has a closed form.  Lastly for $G_3$,
\begin{IEEEeqnarray}{rCl}
G_3
&=&\int h(\*\theta)^{\alpha}\exp\left\lbrace \alpha \eta(\*\kappa_n)^TT(\*\theta)-\alpha A(\eta(\*\kappa_n))\right\rbrace\nonumber\\
&& \cdot h(\*\theta)^{\beta-1}\exp\left\lbrace (\beta-1) \eta(\*\kappa_0)^TT(\*\theta)-(\beta-1) A(\eta(\*\kappa_0))\right\rbrace d\*\theta\nonumber\\
&=&\exp\left\lbrace A\left(\alpha \eta(\*\kappa_n)+(\beta-1)\eta(\*\kappa_0)\right) -\alpha A(\eta(\*\kappa_n))-(\beta-1) A(\eta(\*\kappa_0))\right\rbrace\nonumber \\
&&\cdot\mathbb{E}_{p(\*\theta|\left(\alpha \eta(\*\kappa_n)+(\beta-1)\eta(\*\kappa_0)\right)}\left[h(\*\theta)^{\alpha+\beta-2}\right],\nonumber
\end{IEEEeqnarray}
where once again in analogy to $G_1$ and $G_2$, conditions (\ref{item:ABGC3}) and (\ref{item:ABGC4})  ensure this is a closed form function of $\eta(\*\kappa_n)$ and $\eta(\*\kappa_0)$.
%
%
%

Therefore, provided conditions (\ref{item:ABGC3}) and(\ref{item:ABGC4}) hold, the integrals $G_1, G_2$ and $G_3$ are available in closed form, implying that the same holds for $\ABGD(q(\*\theta|\*\kappa_n)||\pi(\*\theta|\*\kappa_0))$.

\end{proof}

\begin{remark}[Conditions of Proposition \ref{Proposition:ClosedFormABGDMaster} for the \MVN exponential family]\label{Rem:GeneralConditionsMVN}
In order to illuminate the meaning and generality of the conditions of Theorem \ref{Proposition:ClosedFormABGDMaster}, we apply them to the \MVN exponential family described in Definition \ref{Def:MVN}. In this case the two conditions become:
\begin{enumerate}[label=\roman*), ref=\roman*]
\item  For $\*\mu^{\ast}:=\left\lbrace\*\mu_1+\*\mu_2-\left(\left(\frac{1}{\alpha}\*V_1\right)^{-1}+\left(\frac{1}{\beta-1}\*V_2\right)^{-1}\right)^{-1}\left(\left(\frac{1}{\alpha}\*V_1\right)^{-1}\*\mu_2+\left(\frac{1}{\beta-1}\*V_2\right)^{-1}\*\mu_1\right)\right\rbrace$ we require that
\begin{IEEEeqnarray}{rCl}
\alpha\begin{pmatrix} \*V_1^{-1}\*\mu_1\\
-\frac{1}{2}\*V_1^{-1}\end{pmatrix} + (\beta-1)\begin{pmatrix} \*V_2^{-1}\*\mu_2\\
-\frac{1}{2}\*V_2^{-1}\end{pmatrix}&=&\begin{pmatrix}\left(\frac{1}{\alpha}\*V_1\right)^{-1}\*\mu_1+\left(\frac{1}{\beta-1}\*V_2\right)^{-1}\*\mu_2\\
-\frac{1}{2}\left\lbrace\left(\frac{1}{\alpha}\*V_1\right)^{-1}+\left(\frac{1}{\beta-1}\*V_2\right)^{-1}\right\rbrace\end{pmatrix}\nonumber\\
&=&\begin{pmatrix} \left\lbrace\left(\frac{1}{\alpha}\*V_1\right)^{-1}+\left(\frac{1}{\beta-1}\*V_2\right)^{-1}\right\rbrace\*\mu^{\ast}\\
-\frac{1}{2}\left\lbrace\left(\frac{1}{\alpha}\*V_1\right)^{-1}+\left(\frac{1}{\beta-1}\*V_2\right)^{-1}\right\rbrace\end{pmatrix}\in\mathcal{N}\nonumber
\end{IEEEeqnarray}
\item $\mathbb{E}_{p(\*\theta|\eta(\*\kappa))}\left[\left(2\pi\right)^{-d/2(\alpha+\beta+2)}\right]=\left(2\pi\right)^{-d/2(\alpha+\beta+2)}=f(\eta(\*\kappa))$ where $f$ is a closed form function. 
\end{enumerate}
Part ii) shows that the second condition is trivially satisfied for the \MVN exponential family. Part i) shows that for the \MVN exponential family, the first condition is satisfied provided $\left(V^{\ast}\right)^{-1}=\left\lbrace\left(\frac{1}{\alpha}\*V_1\right)^{-1}+\left(\frac{1}{\beta-1}\*V_2\right)^{-1}\right\rbrace$ is a positive definite matrix. This condition is enough to ensure that $V^{\ast}$ is invertible and thus that $\mu^{\ast}$ is well-defined. We elaborate further on what this means for certain parametrisations below.
\end{remark}

\subsubsection{Corollary: The special cases of \AD, \RAD}


Next, we consider the \AD and \RAD special cases of the \ABGD family. Definitions \ref{Def:alphaD} and \ref{Def:renyiAlphaD} can be used to show that the \RAD is available as the following closed form function of the \AD. In particular, it holds that
\begin{IEEEeqnarray}{rCl}
\RAD(q(\*\theta)||\pi(\*\theta))&=&\frac{1}{\alpha(\alpha-1)}\log\left\lbrace 1+\alpha(1-\alpha)\AD(q(\*\theta)||\pi(\*\theta))\right\rbrace.
\label{Equ:RADfromAD}
\end{IEEEeqnarray} 
Thus, as demonstrated in Corollary \ref{Thm:ClosedFormRAD} below, the \AD being available in closed form immediately provides the \RAD in closed form.
Before stating these results, we note that \citet{closedFormsPaper, closedFormsThesis, ConvexStatisticalDistances} have shown our closed form results for the \RAD (and thus implicitly the \AD) before. 
We nevertheless think there is merit in stating them, since our results refer to the \ABGD and thus are more general, recovering both the \AD and \RAD only as a special case.

\begin{corollary}[Closed form \AD for exponential families]
The \AD between a variational posterior $q(\*\theta|\*\kappa_n)$ and prior $\pi(\*\theta|\*\kappa_0)$ is available in closed form under the following conditions
\begin{enumerate}[label=\roman*), ref=\roman*]
\item \label{item:ADC1} $\left(\alpha \eta(\*\kappa_n)+(1-\alpha)\eta(\*\kappa_0)\right)\in\mathcal{N}$
  \end{enumerate}
and in this case the \AD can be written as
\begin{IEEEeqnarray}{rCl}
\AD(q(\*\theta|\*\kappa_n)||\pi(\*\theta|\*\kappa_0)&=&\frac{1}{\alpha(1-\alpha)}\left[1-C(\*\kappa_n,\*\kappa_0,\alpha,(1-\alpha))\right],\nonumber
 \end{IEEEeqnarray}
 where $C(\*\kappa_1,\*\kappa_2,\delta_1,\delta_2)$ was defined in Proposition \ref{Proposition:ClosedFormABGDMaster} .
\label{Thm:ClosedFormAD}
\end{corollary}
\begin{proof}
Following \cite{ABCdiv} the single-parameter \AD is recovered as a member of the \ABGD family when $r=1$ and $\beta=2-\alpha$. In this situation, Condition 
 (\ref{item:ABGC4}) of Theorem \ref{Proposition:ClosedFormABGDMaster} holds automatically and we are left with Condition (\ref{item:ABGC3}). Substituting $\beta=2-\alpha$ provides Condition (\ref{item:ADC1}) of the Theorem above. 
 
 If $\alpha\in (0,1)$ then the convexity of the natural parameter space ensures that providing $\eta(\*\kappa_n)\in\mathcal{N}$ and $\eta(\*\kappa_0)\in\mathcal{N}$ then $\alpha \eta(\*\kappa_n)+(1-\alpha)\eta(\*\kappa_0)\in\mathcal{N}$. If $\alpha<0$ or $\alpha>1$, then this can no longer be guaranteed.
\end{proof}
Corollary \ref{Thm:ClosedFormRAD} is then an immediate consequence of Corollary \ref{Thm:ClosedFormAD}.
\begin{corollary}[Closed form \RAD for exponential families]
The \RAD between a variational posterior $q(\*\theta|\*\kappa_n)$ and prior $\pi(\*\theta|\*\kappa_0)$ will have closed form providing the \AD  between the same two densities for the same value of $\alpha$ has closed form.
\label{Thm:ClosedFormRAD}
\end{corollary}
\begin{proof}
The proof of this follows immediately from the fact that the \RAD can be recovered using the closed form function of the \AD shown in eq. \eqref{Equ:RADfromAD}
\end{proof}

 

\begin{remark}[Conditions for Corollary \ref{Thm:ClosedFormAD} for the \MVN exponential family]\label{Rem:GeneralConditionsMVNAD} The condition that $\alpha \eta(\*\kappa_n)+(1-\alpha)\eta(\*\kappa_0)\in\mathcal{N}$ can only be guaranteed for $\alpha\in(0,1)$. However we can see from Remark \ref{Rem:GeneralConditionsMVN} that provided $\*V^{\ast}=\left(\left(\frac{1}{\alpha}\*V_1\right)^{-1}+\left(\frac{1}{\beta-1}\*V_2\right)^{-1}\right)^{-1}$ is a symmetric semi-positive definite (\SPD) matrix for $\beta=2-\alpha$ then this condition will be satisfied. For $\alpha>1$ or $\alpha<0$ we cannot guarantee that $\*V^{\ast}$ is \SPD. However, we implement the \RAD to quantify uncertainty for $\alpha>1$ in the main paper. Corollary \ref{Thm:ClosedFormAD} demonstrates that these parameters will still produce a closed form divergence provided the prior has sufficiently large variance, which can always be guaranteed to hold in practice. 
\end{remark}

\subsubsection{Corollary: The special cases of \BD, \GD}

Next, we turn attention to the $\beta$- and $\gamma$-divergence families. 
Definition \ref{Def:GammaD} shows that the \GD can be recovered as  a closed form function of the terms of the \BD and thus, as demonstrated in Corollary \ref{Thm:ClosedFormGD} below, the \BD being available in closed form immediately provides that the \GD is available in closed form
While the conditions for these are slightly more restrictive than they were for the \AD and \RAD, one can still obtain closed form uncertainty quantifiers for a large range of settings.

 \begin{corollary}[Closed form \BD for exponential families]
 The \BD between a variational posterior $q(\*\theta|\*\kappa_n)$ and prior $\pi(\*\theta|\*\kappa_0)$ is available in closed form under the following conditions 
\begin{enumerate}[label=\roman*), ref=\roman*]
\item \label{item:BC3} $\eta(\*\kappa_1),\eta(\*\kappa_2)\in\mathcal{N}\Rightarrow \left((\beta-1)\eta(\*\kappa_1)+\eta(\*\kappa_2)\right)\in\mathcal{N}$
\item \label{item:BC4} $\mathbb{E}_{p(\*\theta|\eta(\*\kappa))}\left[ h(\*\theta)^{\beta-1}\right]$ is a closed form function of $\eta(\*\kappa)\in\mathcal{N}$.
 \end{enumerate}
 and in this case the \BD can be written as
 \begin{IEEEeqnarray}{rCl}
 &&\BD(q(\*\theta|\*\kappa_n)||\pi(\*\theta|\*\kappa_0))=\frac{1}{\beta(\beta-1)}B(\*\kappa_n,\beta)E(\*\kappa_n,\beta)+\frac{1}{\beta}B(\*\kappa_0,\beta)E(\*\kappa_0,\beta)\nonumber\\
 &&-\frac{1}{(\beta-1)}C(\*\kappa_n,\*\kappa_0, 1, (\beta-1))\tilde{E}(\*\kappa_n,\*\kappa_0,1,(\beta-1)),\nonumber
\end{IEEEeqnarray}
where the functions $B(\*\kappa,\delta)$, $C(\*\kappa_1,\*\kappa_2,\delta_1,\delta_2)$, $E(\*\kappa,\delta)$ and $\tilde{E}(\*\kappa_1,\*\kappa_2,\delta_1,\delta_2)$ are defined in Proposition \ref{Proposition:ClosedFormABGDMaster}.
\label{Thm:ClosedFormBD}
\end{corollary}

\begin{proof}
Following \cite{ABCdiv}, the single-parameter \BD is recovered as a member of the \ABGD family when $r=1$ and $\alpha=1$. In this situation, Condition (\ref{item:ABGC3})-(\ref{item:ABGC4}) of Theorem \ref{Proposition:ClosedFormABGDMaster} become (\ref{item:BC3})-(\ref{item:BC4}) above. 
\end{proof}
Corollary \ref{Thm:ClosedFormGD} is then an immediate consequence of Corollary \ref{Thm:ClosedFormBD}.
\begin{corollary}[Closed form \GD for exponential families]
The \GD between a variational posterior $q(\*\theta|\*\kappa_n)$ and prior $\pi(\*\theta|\*\kappa_0)$ will have closed form providing the \BD  between the same two densities with $\beta=\gamma$ has closed form.
\label{Thm:ClosedFormGD}
\end{corollary}
\begin{proof}
The proof of this follows immediately from the fact that the \GD can be recovered from the \BD using closed form function as outlined in Definition \ref{Def:GammaD}.
\end{proof}

\begin{remark}[Conditions for Corollary \ref{Thm:ClosedFormBD} under the \MVN exponential family]\hspace{0cm}
\hspace{0cm}
Following Remark \ref{Rem:GeneralConditionsMVN},
    Corollary \ref{Thm:ClosedFormBD} is satisfied providing $\*V^{\ast}=\left(\left(\*V_n\right)^{-1}+\left(\frac{1}{\beta-1}\*V_0\right)^{-1}\right)^{-1}$ is a symmetric \SPD matrix. The sum of two symmetric \SPD matrices is symmetric \SPD and additionally the inverse of a symmetric \SPD matrix is also \SPD. Therefore provided $\beta>1$ we can be sure that Condition iii) will be satisfied. Similarly to Remark \ref{Rem:GeneralConditionsMVNAD}, when $\beta<1$ closed forms will require that the prior has a sufficiently large variance. 
\label{Rem:GeneralConditionsMVNBD}
\end{remark}
In fact Remark \ref{Rem:GeneralConditionsMVNBD} can be extended to many other exponential families if we constrain $\beta=\gamma>1$, this is formalised in Corollary \ref{Thm:ClosedFormBD2}.
\begin{corollary}[Closed form \BD and \GD for exponential families when $\beta=\gamma>1$]
When $\beta=\gamma>1$, the conditions for Corollary \ref{Thm:ClosedFormBD} are satisfied by any exponential family whose $h(\*\theta)$ is a constant function of $\*\theta$ and whose natural parameter space is closed under addition and scalar multiplication. This includes the Beta, Gamma, Gaussian, exponential and Laplace families.
\label{Thm:ClosedFormBD2}
\end{corollary}

\begin{proof}
The proof of Corollary \ref{Thm:ClosedFormBD2} follows straight from that of Corollary \ref{Thm:ClosedFormBD}.
\end{proof}

\section{Experiments}

While the most interesting findings of our numerical studies can be found in the main paper, here we give a brief overview over additional results. 
More importantly, we state the proofs for the theoretical groundwork necessary to deploy \GVI on \DGPs.

\subsection{Bayesian Neural Networks (\BNNs)}
\label{Appendix:BNN}

We provide two more sets of experiments for further insights into \BNNs. 
The first set consists in three more data sets with the same settings as used in the main paper. While these findings do not change the overall picture, they do require more careful analysis and dissemination. 
The second set of results investigates the interaction between robustifying inference relative to the loss with robustifying it relative to the prior. The results suggest a clear relationship for predictive performance as measured by the root mean square error: If robust losses are used, the \KLD generally performs better. Moreover, the combination of robust loss and $D=\KLD$ outperforms \VI and the investigated \DVI methods on all data sets studied. 
The relationship is less clear for the predictive negative log likelihood, both between loss and uncertainty quantifier as well as between the performance to be expected under \GVI, \VI and \DVI.

\begin{figure}[t!]
\begin{center}
\includegraphics[trim= {0.75cm 0.25cm 3.25cm 0.25cm}, clip,  
 width=1.00\columnwidth]{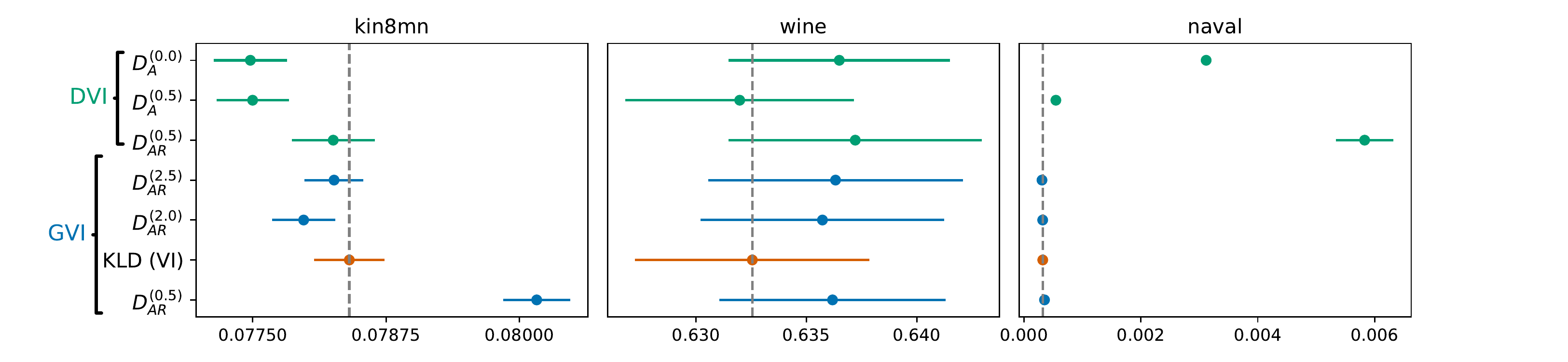}
 \includegraphics[trim= {0.75cm 0.25cm 3.25cm 0.75cm}, clip,  
 width=1.00\columnwidth]{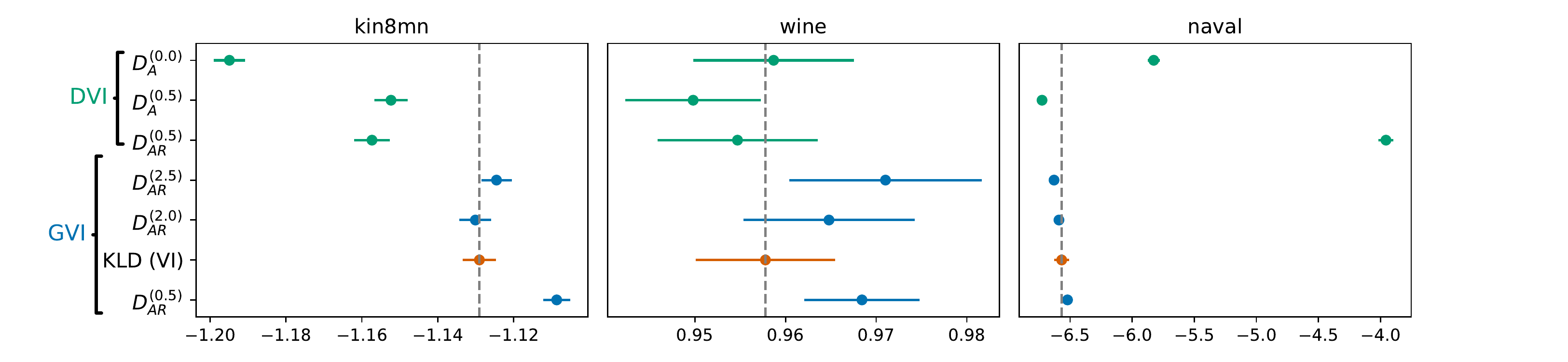}
\caption{
	Top row depicts \RMSE, bottom row the \NLL across a range of data sets using \BNNs. Dots correspond to means, whiskers the standard errors. The further to the left, the better the predictive performance. 
	For the depicted selection of data sets, no common pattern exists for the performance differences between \VIColor{\textbf{standard \VI}}, \DVIColor{\textbf{\DVI}} and \GVIColor{\textbf{\GVI}}.
}
\label{Fig:BNN_experiment_2}
\end{center}
\end{figure}

\subsubsection{First set of additional experiments (Figure \ref{Fig:BNN_experiment_2})}

Figure \ref{Fig:BNN_experiment_2} provides the predictive outcomes on three more data sets using the exact same settings and experimental setup as described in the main paper.
The findings generally reinforce the findings of the main paper.
First, while the \GVI methods with $\alpha > 1$ still perform as good as or better than standard \VI on the \texttt{kin8mn} data set,  \DVI methods show a clear performance gain relative to either of the two.
Crucially, it is not clear what leads to this improvement gain, though the fact that the best-performing \DVI method is the one recovering \EP (\AD for $\alpha = 0$) suggests that there is tangible merit in producing mass-covering approximations to the posterior of $\*\theta$ on this data set.
While the deployment of \DVI methods looks tempting on the \texttt{kin8mn} data set, the results on the \texttt{naval} data set are a reminder that the behaviour of these methods is in many ways unpredictable. 
Moreover, it shows that the risks we identified in Example \ref{example:label_switching} readily translate into real world applications: By using \DVI methods, we may accidentally conflate the role of the loss and the role of uncertainty quantification. If the loss is well-suited for the data at hand---as the \RMSE panel suggests it is in the \texttt{naval} case---the mass-covering behaviour of \DVI methods can be extremely detrimental.
Lastly, the \texttt{wine} data set provides a very similar picture to the results in Figure \ref{Fig:BNN_experiment_1}: Varying $\alpha$ introduces a banana-shaped curve for the \GVI methods. As it so happens, the ideal choice of $\alpha$ on the \texttt{wine} data set appears to be around $\alpha = 1$ (i.e., standard \VI). 
Taking into account the predictive uncertainty in form of the whiskers, it is doubtful if any of the methods is dominating another one on \texttt{wine}.
Presumably, the reason for this is that the true posterior is relatively well approximated with the mean field normal family, yielding very similar results across all settings.

\begin{figure}[t!]
\begin{center}
\includegraphics[trim=  {2.25cm 0.3cm 3.15cm 0.25cm}, clip,  
 width=1.00\columnwidth]{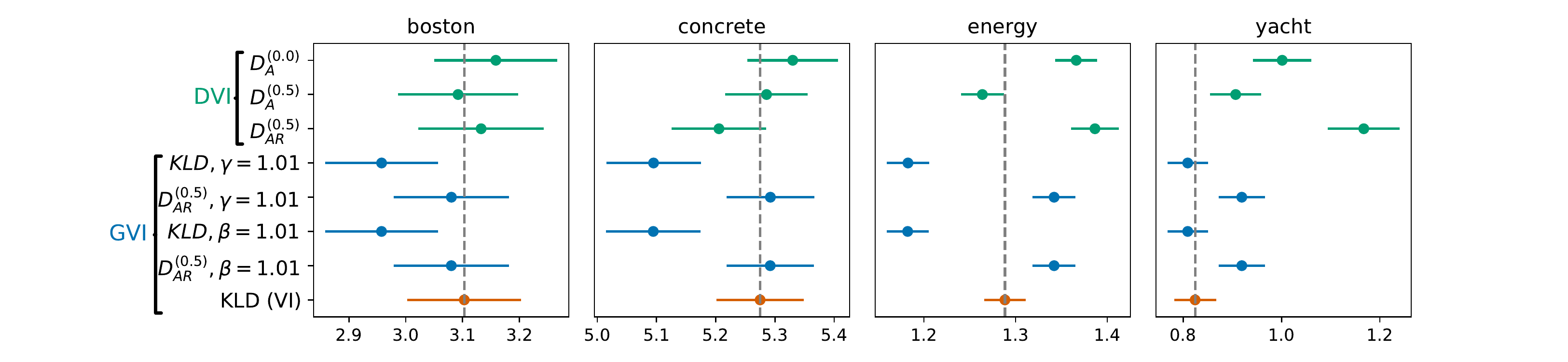}
 \includegraphics[trim=  {2.25cm 0.3cm 3.15cm 0.80cm}, clip,  
 width=1.00\columnwidth]{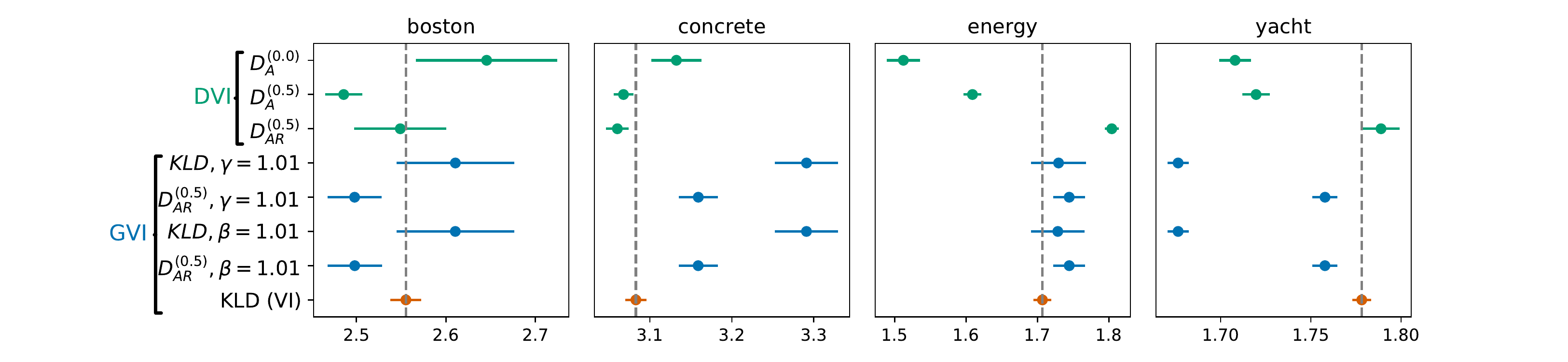}
\caption{
	Top row depicts \RMSE, bottom row the \NLL across a range of data sets using \BNNs. Dots correspond to means, whiskers the standard errors. The further to the left, the better the predictive performance. 
	For the depicted selection of data sets, patterns exists for the interplay between the loss and uncertainty quantifier for \GVIColor{\textbf{\GVI}}.
}
\label{Fig:BNN_experiment_3}
\end{center}
\end{figure}

\subsubsection{Second set of additional experiments (Figure \ref{Fig:BNN_experiment_3})}

In a second set of additional experiments, we varied the loss function to be a robust scoring rule. Specifically, we used scoring rules based on the $\beta$-divergence and the $\gamma$-divergence. 
See Section \ref{sec:DGP_adaption} for the definition and more detail on these robust scoring rules.
As for the \DGP examples, we choose values of the scoring rule that are close to the $\log$ score, but sufficiently far to induce robust behaviour.
All settings for optimization, initialization as well as the code are the same as for the results provided in the main paper.
Figure \ref{Fig:BNN_experiment_3} shows the results: For the \RMSE, the results are unambiguous: Combining a robust scoring rule with the standard uncertainty quantifier $D=\KLD$ appears to be the winning combination across all four data sets.
The picture is less clear for the \NLL: Relative to both \VI and \DVI, the performance gains depend on the data set. Even within the class of \GVI posteriors, it is data-set dependent which uncertainty quantifier should be chosen: For example, it is clearly beneficial to choose the \RAD as uncertainty quantifier in the \texttt{boston} and \texttt{concrete} data sets, but the opposite is true on the \texttt{yacht} data set. 
Above all other things, this highlights the need for a good selection strategy of \GVI hyperparameters: Oftentimes, intuitions about the correct uncertainty quantifier or the appropriate loss may be incorrect.

\subsection{Deep Gaussian Processes (\DGPs)}
\label{Appendix:DGPs}

Unlike \BNNs, \DGPs require some theoretical groundwork before they are amenable to changes in the loss and uncertainty quantifier.
Specifically, we need to show that it is valid to define new divergences layer-wise. 
Moreover, while not required it is beneficial if one can obtain closed forms for the robustified likelihood terms.
The following sections proceed to do both. Thereafter, we also show an additional short example to illustrate the effect of changing the uncertainty quantifier in \DGPs.

\subsubsection{Proof of Corollary \ref{corollary:GP_divergence_change}}
\label{Appendix:GP_divergence_change}

We first prove a Lemma that plays a key role in the proof of Corollary \ref{corollary:GP_divergence_change}.

\begin{lemma}[Divergence recombination]
    Let $D_l$ be divergences and $c_l>0$ scalars for $l=1,2,\dots, L$. Further, denote $\*\theta_{-l} = \*\theta_{1:l-1, l+t:L}$ and let $q_l(\*\theta_l|\*\theta'_{-l})$ and $\pi_l(\*\theta_l|\*\theta'_{-l})$ be the conditional distributions of $\*\theta_l$ for $q(\*\theta)$ and $\pi(\*\theta)$ conditioned on $\*\theta_{-l} = \*\theta'_{-l}$. 
    Then, 
    ${D}^{\*\theta'}(q||\pi) = \sum_{l=1}^Lc_lD_l\left(q_l(\*\theta_l|\*\theta'_{-l})||\pi_l(\*\theta_l|\*\theta'_{-l})\right)$ 
    is a divergence between $q(\*\theta)$ and $\pi(\*\theta)$ if 
        (i) ${D}^{\*\theta^{\circ}}(q||\pi) = {D}^{\*\theta'}(q||\pi)$ for all conditioning sets $\*\theta^{\circ}$, $\*\theta'$ and
        (ii) a Hammersley-Clifford Theorem holds for the collection of conditionals $\pi_l(\*\theta_l|\*\theta'_{-l})$ and $q_l(\*\theta_l|\*\theta'_{-l})$.
    \label{lemma:DivRecomb}
\end{lemma}
\begin{proof}
First, observe by definition of a divergence,  $D_l(q_l(\*\theta_l|{\*\theta'}_{-l})||\pi_l(\*\theta_l|{\*\theta'}_{-l})) = 0$ for all $l$ and over all potential conditioning sets $\*\theta'$ holds if and only if $q_l(\*\theta_l|{\*\theta'}_{-l}) = \pi_l(\*\theta_l|{\*\theta'}_{-l})$.
Next, note that we have assumed that ${D}^{{\*\theta'}}(q||\pi) = {D}^{{\*\theta}^{\circ}}(q||\pi)$ for all conditioning sets $\*\theta'$, $\*\theta^{\circ}$. 
In other words, if ${D}^{\*\theta'}(q||\pi) = 0$ for some $\*\theta'$, then it will also be $0$ for \textit{any} conditioning set $\*\theta^{\circ}$. 
This immediately entails that for arbitrary $\*\theta'$, ${D}^{\*\theta'}(q||\pi) = 0$ if and only if $q_l(\*\theta_l|\*\theta'_{-l}) = \pi_l(\*\theta_l|\*\theta'_{-l})$ for \textit{all} $l$ and for \textit{any} choice of ${\*\theta'}_{-l}$. 
In other words, the conditionals are the same. Since the positivity condition holds, we can then apply the Hammersley-Clifford Theorem to conclude that the conditionals fully specify the joint. This finally yields the desired result: ${D}^{\*\theta'}(q||\pi) = 0$ if and only if $q(\*\theta) = \pi(\*\theta)$.
\end{proof}
With this technical result in hand, one can now prove Corollary \ref{corollary:GP_divergence_change}, which shows that reverse-engineering prior regularizers inspired by eq. \eqref{eq:DGP_regularizer} is feasible so long as the layer-specific divergences $D^l$ are $f$-divergences or monotonic transformations of $f$-divergences.

\begin{proof}
    Suppressing again $\*Z^l$ and $\*X$ for readability, first recall that 
    \begin{IEEEeqnarray}{rCl}
        q(\{\*U^l\}_{l=1}^L, \{\*F^l\}_{l=1}^L) & = &
        \prod_{l=1}^L
        p(\*F^l|\*U^l, \*F^{l-1})q(\*U^l)
        \nonumber  \\
        p(\{\*U^l\}_{l=1}^L, \{\*F^l\}_{l=1}^L) & = &
        \prod_{l=1}^L
        p(\*F^l|\*U^l, \*F^{l-1})p(\*U^l)
        \nonumber
    \end{IEEEeqnarray}
    and write for a \textit{fixed} conditioning set $\{\*F_{\circ}^l\}_{l=1}^L$ the new divergence 
    \begin{IEEEeqnarray}{rCl}
        &&{D}^{
        \{\*F_{\circ}^l\}_{l=1}^L
        }\left(
        q(\{\*U_l\}_{l=1}^L, \{\*F_l\}_{l=1}^L) \|
        p(\{\*U_l\}_{l=1}^L, 
        \{\*F_l\}_{l=1}^L)\right) 
        \nonumber \\
        &=& \sum_{l=1}^LD^l\left(
        p(\*F^l|\*U^l, \*F_{\circ}^{l-1})q(\*U^l)
        \|
        p(\*F^l|\*U^l, \*F_{\circ}^{l-1})p(\*U^l)
        \right) 
        =
        \sum_{l=1}^LD^l\left(
        q(\*U^l)
        \|
        p(\*U^l)
        \right) 
        \nonumber
    \end{IEEEeqnarray}
    The first equality is simply the definition of the new divergence. The second equality follows by virtue of $D^l$ being a monotonic function of an $f$-divergences or an $f$-divergence for all $l$, which ensures that the $l$-th term is given by
    \begin{IEEEeqnarray}{rCl}
        && D^l\left(p(\*F^l|\*U^l, \*F_{\circ}^{l-1})q(\*U^l)\|p(\*F^l|\*U^l, \*F_{\circ}^{l-1})p(\*U^l)\right) 
        \\ \nonumber
        & = & 
        g\left( \mathbb{E}_{p(\*F^l|\*U^l, \*F_{\circ}^{l-1})p(\*U^l)}\left[f\left( \frac{p(\*F^l|\*U^l, \*F_{\circ}^{l-1})q(\*U^l)}{p(\*F^l|\*U^l, \*F_{\circ}^{l-1})p(\*U^l)} \right)\right] \right).
        \nonumber \\
        & = &
        g\left( \mathbb{E}_{p(\*U^l)}\left[f\left( \frac{q(\*U^l)}{p(\*U^l)} \right)\right] \right) = D^l\left(
        q(\*U^l)
        \|
        p(\*U^l)
        \right). 
        \nonumber
    \end{IEEEeqnarray}
    Now note that we can invoke Lemma \ref{lemma:DivRecomb}: The first condition is satisfied because the derivation was independent of the chosen $\{\*F_{\circ}^l\}_{l=1}^L$. The second condition is satisfied as both conditionals satisfy the positivity condition required for the Hammersley-Clifford Theorem to hold. 
\end{proof}

\subsubsection{Proof of Proposition \ref{Proposition:closed_form_robust_DGP_regression}}
\label{Appendix:GP_loss_change}

\begin{proof}
The likelihood is Gaussian with a fixed variance parameter $\sigma^2$, i.e. for $\*y_i \in \mathbb{R}^d$ with $i=1,2,\dots, n$
\begin{IEEEeqnarray}{rCl}
    p(\*y_i|\*f_i^L) & = & (2\pi\sigma^2)^{-0.5d}\exp
 \left\{ 
     -\frac{1}{2\sigma^2}(\*y_i - \*f_i^L)^T(\*y_i - \*f_i^L)
 \right\} \nonumber
\end{IEEEeqnarray}
With this, note that integrating out the normal density yields
\begin{IEEEeqnarray}{rCl}
    I_{p, c}(\*f_i^L) & = & (2\pi\sigma^2)^{-0.5dc}c^{-0.5d}. 
    \label{eq:I_integral}
\end{IEEEeqnarray}
Note in particular that this is a constant and does not depend on $\*f$, which makes computing the expectation over $q(\*f_i^L)$ depend only on the power likelihood. 
Next, we show that the power likelihood is also available in closed form. 
This is laborious but not difficult and relies on the same algebraic tricks in the Appendix of \citet{RBOCPD}. To simplify notation, we write $\*f = \*f_i^L$. Note also that the variational parameters $\*\mu$ and $\*\Sigma$ are (stochastic) functions of the draws of $\*f_{i}^{1:L-1}$ from the previous layers, but we suppress this dependency, again for readability.
\begin{IEEEeqnarray}{rCl}
    \mathbb{E}_{q(\*f|\*\mu, \*\Sigma)}
    \left[
 \frac{1}{c}
			p(\*y_i|\*f)^{c} 
    \right] 
    & = &
    \frac{1}{c}(2\pi\sigma^2)^{-0.5dc} \cdot 
    \mathbb{E}_{q(\*f|\*\mu, \*\Sigma)}
    \left[
 \exp
 \left\{ 
     -\frac{c}{2\sigma^2}
     (\*y_i^T\*y_i + \*f^T\*f - 2\*f^T\*y_i)
 \right\} 
    \right] 
    \nonumber
    \\
    & = &
    \frac{1}{c}(2\pi\sigma^2)^{-0.5dc} 
    \exp\left\{-\frac{c}{2\sigma^2}\*y_i^T\*y_i\right\}  \cdot 
    \mathbb{E}_{q(\*f|\*\mu, \*\Sigma)}
    \left[
 \exp
 \left\{ 
     -\frac{c}{2\sigma^2}
     (\*f^T\*f - 2\*f^T\*y_i)
 \right\} 
    \right] 
     \nonumber
     \\
    & = &
    \frac{1}{c}(2\pi\sigma^2)^{-0.5dc} 
    (2\pi\sigma^2)^{-0.5d} |\*\Sigma|^{-0.5} 
    \exp\left\{-\frac{c}{2\sigma^2}\*y_i^T\*y_i\right\}  \times \nonumber \\
    &&
    \bigintssss
 \exp
 \left\{ 
     -\frac{1}{2}
     \left(
     \frac{c}{\sigma^2}\*f^T\*f - \frac{2c}{\sigma^2}\*f^T\*y_i + 
     (\*f - \*\mu)^T\*\Sigma^{-1}(\*f - \*\mu)
     \right)
 \right\}d\*f 
  \nonumber
  \\
    & = &
    \frac{1}{c}
    (2\pi\sigma^2)^{-0.5dc} 
    (2\pi)^{-0.5d} |\*\Sigma|^{-0.5} 
    \exp\left\{-\frac{1}{2}\left(\frac{c}{\sigma^2}\*y_i^T\*y_i + \*\mu^T\*\Sigma^{-1}\*\mu \right)\right\}  
    \times \nonumber \\
    &&
    \bigintssss
 \exp
 \left\{ 
     -\frac{1}{2}
     \left(
     \frac{c}{\sigma^2}\*f^T\*f - \frac{2c}{\sigma^2}\*f^T\*y_i + 
     \*f^T\*\Sigma^{-1}\*f -
     2\*f^T\*\Sigma^{-1}\*\mu
     \right)
 \right\}d\*f 
    \nonumber
\end{IEEEeqnarray}
The integral suggests one can obtain a closed form through the Gaussian integral by completing the squares. Defining $\widetilde{\*\Sigma}^{-1} = \left(\frac{c}{\sigma^s}\*I_d + \*\Sigma^{-1}\right)$, $\widetilde{\*\mu} = \left( \frac{c}{\sigma^2}\*y_i + \*\Sigma^{-1}\*\mu \right)$,
$\widehat{\*\mu} = \widetilde{\*\Sigma}\widetilde{\*\mu}$, one indeed has
\begin{IEEEeqnarray}{rCl}
    \frac{c}{\sigma^2}\*f^T\*f - \frac{2c}{\sigma^2}\*f^T\*y_i + 
     \*f^T\*\Sigma^{-1}\*f -
     2\*f^T\*\Sigma^{-1}\*\mu 
    &=&
    \*f^T\left(\*I_d\frac{c}{\sigma^2} + \*\Sigma^{-1}\right)\*f
    -
    2\*f^T\left( \frac{c}{\sigma^2}\*y_i + \*\Sigma^{-1}\*\mu \right)  \nonumber\\
    & = &
    \left(\*f - \widehat{\*\mu}\right)^T \widetilde{\*\Sigma}^{-1}
    \left(\*f - \widehat{\*\mu}\right)
    -
    \widetilde{\*\mu}^T\widetilde{\*\Sigma}\widetilde{\*\mu},
    \nonumber
\end{IEEEeqnarray}
%
%
which allows us to finally rewrite the integral as
\begin{IEEEeqnarray}{rCl}
    && \bigintssss
 \exp
 \left\{ 
     -\frac{1}{2}
     \left(
     \frac{c}{\sigma^2}\*f^T\*f - \frac{2c}{\sigma^2}\*f^T\*y_i + 
     \*f^T\*\Sigma^{-1}\*f -
     2\*f^T\*\Sigma^{-1}\*\mu
     \right)
 \right\}d\*f
  \nonumber
  \\
    & = &
    \exp\left\{ 
 -\frac{1}{2}\widetilde{\*\mu}^T\widetilde{\*\Sigma}\widetilde{\*\mu}
    \right\}
    \bigintssss
 \exp
 \left\{ 
     -\frac{1}{2}
     \left(\*f - \widehat{\*\mu}\right)^T \widetilde{\*\Sigma}^{-1}
     \left(\*f - \widehat{\*\mu}\right)
 \right\}d\*f  
    =
    \exp\left\{ 
 \frac{1}{2}\widetilde{\*\mu}^T\widetilde{\*\Sigma}\widetilde{\*\mu}
    \right\}
    (2\pi)^{0.5d}
    |\widetilde{\*\Sigma}|^{0.5}.
    \nonumber
\end{IEEEeqnarray}
Putting everthing together and simplifying expressions, this means that
%
\begin{IEEEeqnarray}{rCl}
\mathbb{E}_{q(\*f|\*\mu, \*\Sigma)}
    \left[
 \frac{1}{c}
			p(\*y_i|\*f)^{c} 
    \right] 
    & = &
    \frac{1}{c}
    \left({2\pi}\sigma^2\right)^{-0.5dc}
    \frac{|\widetilde{\*\Sigma}|^{0.5}}{ |\*\Sigma|^{0.5} }
    \exp\left\{
 -\frac{1}{2}\left(
     \frac{c}{\sigma^2}\*y_i^T\*y_i + \*\mu^T\*\Sigma^{-1}\*\mu -
     \widetilde{\*\mu}^T\widetilde{\*\Sigma}\widetilde{\*\mu}
 \right)
    \right\} 
    \nonumber
    \quad\quad \quad
\end{IEEEeqnarray}
Depending on whether one uses the $\beta$- or $\gamma$-divergence for robustifying the loss, one thus obtains the closed form expressions
\begin{IEEEeqnarray}{rCl}
    \mathbb{E}_{q(\*f|\*\mu, \*\Sigma)}
    \left[ 
 -\frac{1}{\beta-1}p(\*y_i|\*f)^{\beta-1} +
			\frac{ I_{p, \beta}(\*f) }{\beta}
    \right] & = &
    \mathbb{E}_{q(\*f|\*\mu, \*\Sigma)}
    \left[
 -\frac{1}{\beta-1}
			p(\*y_i|\*f)^{\beta-1} 
    \right] +
    \frac{ I_{p, \beta}(\*f) }{\beta}
    \quad\quad \quad\quad
    \nonumber \\
    \mathbb{E}_{q(\*f|\*\mu, \*\Sigma)}
    \left[ 
 -\frac{1}{\gamma-1}p(\*y_i|\*f)^{\gamma-1} \cdot
    \frac{\gamma}{I_{p,\gamma}(\*f)^{\frac{\gamma-1}{\gamma}}}
    \right] & = &
    \mathbb{E}_{q(\*f|\*\mu, \*\Sigma)}
    \left[
 -\frac{1}{\gamma-1}
			p(\*y_i|\*f)^{\gamma-1} 
    \right] \cdot
    \nonumber
    \frac{\gamma}{I_{p,\gamma}(\*f)^{\frac{\gamma-1}{\gamma}}},
\end{IEEEeqnarray}
with the expectation over $q(\*f|\*\mu, \*\Sigma)$ as in and the integrals $I_{p,\beta}(\*f)$, $I_{p,\gamma}(\*f)$ as defined above.
Note that we have derived the general case for $\*y_i \in \mathbb{R}^d$, where $\*\Sigma$, $\*f$ and $\*\mu$ are matrix- and vector-valued. 
\end{proof}
In fact, we can simplify everything even further in the univariate case. We summarize this in the next part. 
\begin{remark}
Since the derivation of \citet{DeepGPsVI} shows that one in fact only needs to integrate over the marginals $\*f_i^L$, if $d=1$ (as in all experiments in both this paper and \citep{DeepGPsVI}), the computation corresponding to the expression above  simplifies considerably as no matrix inverses and determinants are needed.
In particular, denoting the uni-variate mean and variance parameters as $\mu, \Sigma$ and defining $\widetilde{\Sigma} = \frac{1}{\frac{c}{\sigma^s} + \frac{1}{\Sigma}}$ and $\widetilde{\mu} = \left( \frac{cy_i}{\sigma^2} + \frac{\mu}{\Sigma} \right)$, the expectation term over the posterior $q$ simplifies to
\begin{IEEEeqnarray}{rCl}
\mathbb{E}_{q(f|\mu, \Sigma)}
    \left[
 \frac{1}{c}
			p(y_i|f)^{c} 
    \right] 
    & = &
    \frac{1}{c}s    
    \left({2\pi}\sigma^2\right)^{-0.5c}
    \sqrt{\frac{\widetilde{\Sigma}}{ \Sigma }}\cdot
    \exp\left\{
 -\frac{1}{2}\left(
     \frac{cy_i^2}{\sigma^2} + \frac{\mu^2}{\Sigma} -
     \widetilde{\mu}^2\widetilde{\Sigma}
 \right)
    \right\}.
    \nonumber
    \quad\quad \quad
\end{IEEEeqnarray}
\end{remark}

\begin{figure}[t!]
\begin{center}
\includegraphics[trim=  {0.0cm 0.3cm 2.8cm 0.75cm}, clip,  
 width=1.00\columnwidth]{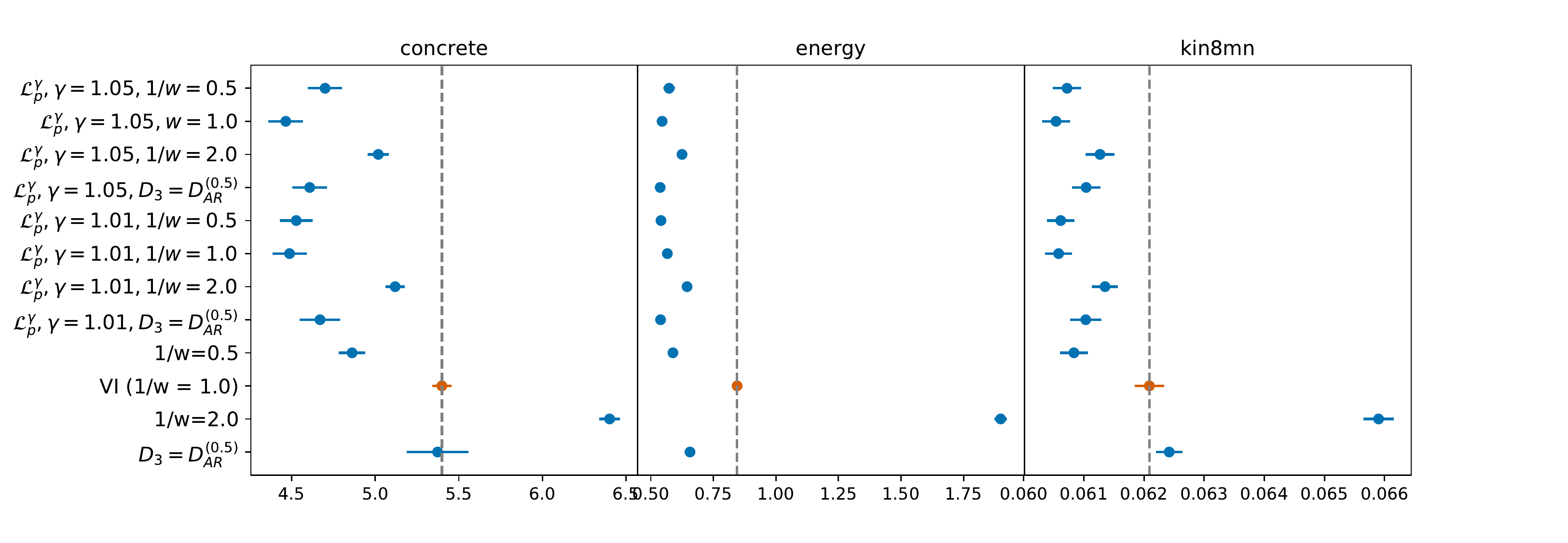}
 \includegraphics[trim=  {0.0cm 0.3cm 2.8cm 0.75cm}, clip,  
 width=1.00\columnwidth]{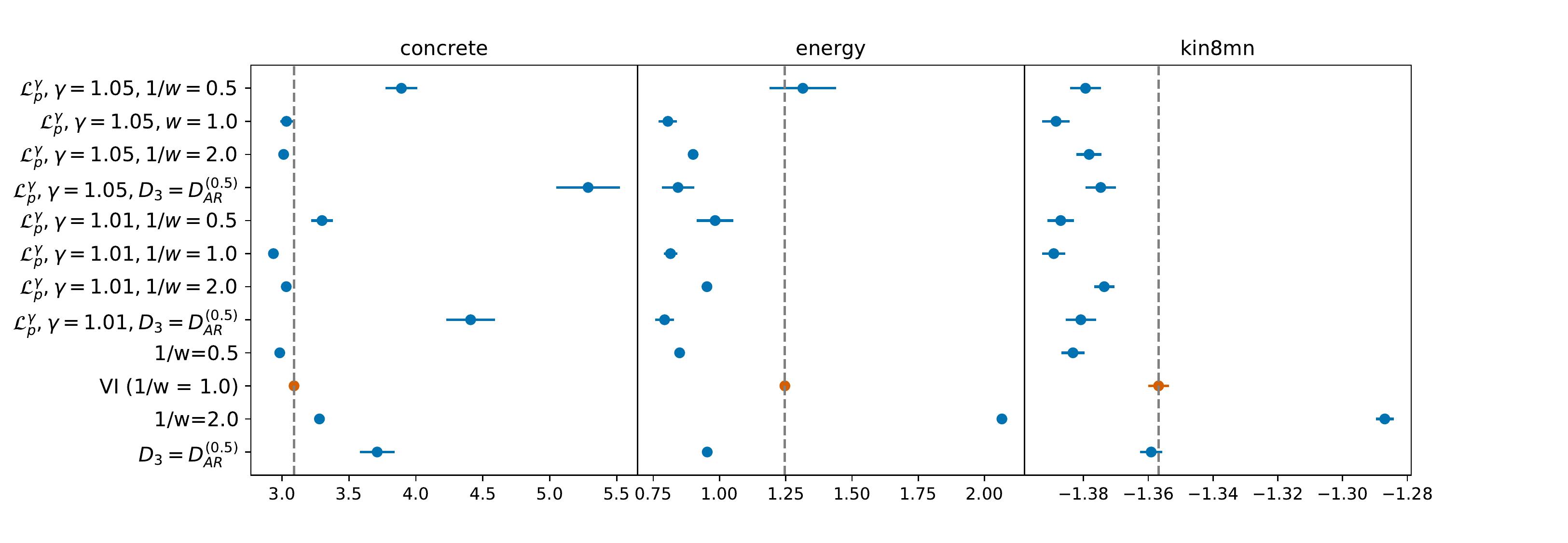}
\caption{
    \bv
	Top row depicts \RMSE, bottom row the \NLL across a range of data sets using \DGPs with $L=3$ layers. Dots correspond to means, whiskers the standard errors. The further to the left, the better the predictive performance. 
}
\label{Fig:DGPResults3}
\end{center}
\end{figure}

\subsubsection{Additional experiments varying $D$ (Figure \ref{Fig:DGPResults3})}
\label{Appendix:GP_additional_experiments}

While we showed that \DGPs allow for the variation of both losses and uncertainty quantifiers, the main paper did not use the flexibility afforded by varying $D$. The main reason for this is that much like for the \BNNs, the results when jointly varying loss and uncertainty quantifier are less intuitively interpretable.
We showcase this in Figure \ref{Fig:DGPResults3}, which compares a number of different \GVI posteriors for \DGPs with $L=3$ layers. The loss is either the robust loss $\Lg$ for $\gamma \in \{1.01, 1.05\}$ (top 8 entries in each row) or the standard $\log$ score  (bottom 4 entries in each row). 
We also compare $D = \frac{1}{w} \KLD$ for $w=2.0, 1.0, 0.5$ as well as the composite layer-wise divergence
\begin{IEEEeqnarray}{rCl}
    D(q\|\pi) & = & \sum_{l=1}^3 D_l(q_l\|\pi_l), \quad D_1=D_2=\KLD, D_3 = \RAD \text{ for } \alpha = 0.5. \nonumber
\end{IEEEeqnarray}
Aligned with the intuition that the priors in \DGPs are rather informative due to various hyperparameter optimization schemes, changing the uncertainty quantifier from the \KLD to the \RAD generally typically has either fairly little or even adverse impact.
Similarly, up- or down-weighting the \KLD seems not to be beneficial across the board and will depend on the loss function. For the case of the $\log$ score however, we find a consistent improvement for down-weighting the \KLD: Predictively, it improves the predictions on both metrics and across all data sets relative to standard \VI. Similarly, up-weighting the \KLD term is counterproductive under the $\log$ score and yields a performance deterioration across all data sets.
This indicates that despite best efforts to the contrary, \DGPs are probably violating \ref{assumption:prior} so that their predictive performance can be enhanced by ignoring more prior information, ensuring posteriors that are concentrated around the empirical risk minimizer.

\vskip 0.2in
\bibliography{library}

\end{document}